\documentclass[twoside,11pt]{article}

\usepackage{blindtext}

\usepackage[preprint]{jmlr2e}

\usepackage{lastpage}
\jmlrheading{23}{2025}{1-\pageref{LastPage}}{4/3; Revised 5/22}{9/22}{21-0000}{Miroslav {\v{S}}trupl, Oleg Szehr, Francesco Faccio, Dylan R.\ Ashley, Rupesh Kumar Srivastava, and J{\"{u}}rgen Schmidhuber}

\ShortHeadings{Convergence and Stability of UDRL, GCSL, and ODTs}{{\v{S}}trupl, Szehr, Faccio, Ashley, Srivastava, and Schmidhuber}
\firstpageno{1}

\usepackage{graphics}
\usepackage{amsmath}
\usepackage{amssymb}
\usepackage{thmtools}
\usepackage{thm-restate}
\usepackage{hyperref}
\usepackage{xcolor}
\usepackage{subcaption}

\newcommand{\eUDRL}{eUDRL}
\newcommand{\subnote}[1]{{\color{gray}{#1\:}}}
\newcommand{\prob}{\mathbb{P}}
\newcommand{\stag}[1]{#1^{\Sigma}}
\newcommand{\argmax}{\mathrm{argmax}}
\DeclareMathOperator*{\supp}{\mathrm{supp}}

\newcommand{\indicator}{\mathbf{1}}
\DeclareMathOperator*{\ev}{\mathbb{E}}

\newcommand{\Traj}{\mathrm{Traj}}

\newcommand{\Seg}{\mathrm{Seg}}
\newcommand{\trail}{\mathrm{trail}}
\newcommand{\diag}{\mathrm{diag}}
\newcommand{\SegDiag}{\Seg^{\diag}}
\newcommand{\SegTrail}{\Seg^{\trail}}
\newcommand{\theinterestingstates}{\bar{\mathcal{S}}_{\tkernel_0}}
\newcommand{\num}{\mathrm{num}}
\newcommand{\den}{\mathrm{den}}
\newcommand{\oactions}{\mathcal{O}}

\newcommand{\tkernel}{\mathit{\lambda}}

\newcommand\eref[1]{(\ref{#1})} 

\numberwithin{equation}{section}

\newlength{\typeoverlen}

\makeatletter
\def\section{\clearpage\@startsiction{section}{1}{\z@}{-0.24in}{0.10in}
             {\Large\bf\raggedright}}
\makeatother

\begin{document}

\title{On the Convergence and Stability of\\
Upside-Down Reinforcement Learning,\\
Goal-Conditioned Supervised Learning,\\
and Online Decision Transformers}

\author{%
    \name Miroslav {\v{S}}trupl $^{1}$ \email miroslav.strupl@idsia.ch
    \AND
    \name Oleg Szehr $^{1}$ \email oleg.szehr@idsia.ch
    \AND
    \name Francesco Faccio $^{1,2}$ \email francesco.faccio@idsia.ch
    \AND
    \name Dylan R.\ Ashley $^{1,2}$ \email dylan.ashley@idsia.ch
    \AND
    \name Rupesh Kumar Srivastava \email rupspace@gmail.com
    \AND
    \name J{\"{u}}rgen Schmidhuber $^{1,2,3}$ \email juergen.schmidhuber@kaust.edu.sa
    \hfill\\
    \vspace{-0.25cm}\\
    \addr $^1$ Dalle Molle Institute for Artificial Intelligence (IDSIA) - USI/SUPSI, Lugano, %
    Switzerland
    \\
    \addr $^2$ Center of Excellence for Generative AI, King Abdullah University of Science and Technology,
    \\
    \addr $^{\text{\enspace}}$ Thuwal, %
    Saudi Arabia
    \\
    \addr $^3$ NNAISENSE, %
    Lugano, Switzerland
}

\editor{My editor}

\maketitle

\begin{abstract}
This article presents a rigorous analysis of the convergence and stability of Episodic Upside-Down Reinforcement Learning, Goal-Conditioned Supervised Learning, and Online Decision Transformers. While these algorithms perform well across various benchmarks, their theoretical foundations is limited to heuristics and special cases. We develop a theoretical framework for algorithms that build on the paradigm of approaching reinforcement learning through supervised learning. We identify conditions under which the mentioned algorithms identify optimal solutions and evaluate their stability in noisy environments. Specifically, we study the continuity and asymptotic convergence of command-conditioned policies, values, and the goal-reaching objective depending on the transition kernel of the underlying Markov Decision Process. We show that near-optimal behavior occurs when the transition kernel is sufficiently close to a deterministic kernel, and that these quantities are continuous at deterministic kernels, both asymptotically and after a finite number of learning cycles. Our methods provide the first estimates for the convergence and stability of policies and values in terms of the transition kernel. To achieve this, we introduce novel concepts to reinforcement learning, such as segment spaces, quotient topologies, and the application of fixed-point theory from dynamical systems. The theoretical study is supported by investigations of example environments and numerical experiments.
\end{abstract}


\tableofcontents

\section{Introduction}
Reinforcement Learning (RL) algorithms are designed to learn policies that choose optimal actions while interacting with an environment. The environment does not reveal optimal actions, but only provides higher rewards for taking better actions. This is in direct contrast with Supervised Learning (SL), where the correct output for each input is available for learning. Nevertheless, a series of algorithms have been proposed that attempt to solve reinforcement learning tasks using purely supervised learning techniques. Upside-Down Reinforcement Learning~\citep{schmidhuber2019reinforcement, srivastava2019training} (UDRL) inverts the traditional RL process by mapping desired returns/goals to actions, treating action prediction as a supervised learning problem. Goal-conditioned Supervised Learning~\citep{ghosh2021learning} (GCSL) utilizes goal information to guide the model’s learning process, and Online Decision Transformers~\citep{zheng2022online} (ODT) leverage transformer architectures to model entire trajectories, treating past states, actions, and rewards as sequences to predict optimal actions. Experiments have shown that in addition to being strikingly simple and scalable due to their dependence on SL, these algorithms can produce good results on several RL benchmarks (such as Vizdoom \citep{kempka2016vizdoom}, robotic manipulation \citep{ahn2020robel} and locomotion \citep{fu2020d4rl}). Their theoretical understanding, however, is limited to heuristics and the study of restrictive special cases. Through a rigorous analysis of convergence and stability, this work initiates the development of a theoretical foundation for algorithms that build on the broad idea of approaching RL via SL or sequence modeling. Two questions guide our investigation: \textit{1)} What can be said about the convergence of UDRL, GCSL and ODT assuming that an explicit model (transition kernel) for the underlying Markovian environment is given? What behavior can be expected of typical objects of interest, such as policies, state and action-values in the limit of infinite resources. \textit{2)} How stable are these quantities under perturbation or in the presence of errors in the environment model? Guarantees that ensure that algorithms reliably identify optimal solutions and remain stable under varying conditions are foundational for their practical deployment in real-world systems~\citep{neuroDynProg1996,numericalOptimization2006}.

Stepping back to establish some basic background, it is notable that UDRL, GCSL, and ODT are very similar algorithms. Although architectural details vary, at their core stand common ideas about the acquisition of information by the learning agent. They all focus on directly predicting actions based on reward signals from trajectories, rather than learning value functions. The key shared ingredient is the interpretation of rewards, observations and planning horizons as task-defining inputs from which a command is computed for the learning agent. The agent's rule of action (the \emph{policy}) is then updated through SL, to map previous trajectory observations and commands into actions, completing the learning process. More formally, suppose a number of trajectory samples have been collected by a learning agent that follows a certain rule of action $\pi_{old}$. Given a trajectory segment that starts with a state-action pair $(s,a)$, has a length of $h$, and where the goal $g$ is a quantity that is computed from features of the segment (such as the sequences of states and rewards), one could reason that the action $a$ is useful for achieving $g$ from $s$ in $h$ steps. It then appears natural to interpret $(h,g)$ as a command for the agent, fitting the new rule of action $\pi_{new}$ to the distribution $a|s,h,g$ using SL,
$$\pi_{new}=\argmax_{\pi}\mathbb{E}[\text{loss}(\pi(a|s,h,g))],$$
where $\text{loss}$ is an appropriate loss function and the expectation is computed over all segments in the trajectory sample. Learning proceeds iteratively by replacing \(\pi_{old}\) with \(\pi_{new}\), sampling new trajectories from \(\pi_{old}\), computing the achieved horizons and goals \((h, g)\), and finally using this information to update \(\pi_{new}\). 

Our analysis of the UDRL, GCSL, and ODT algorithms is conducted within the overarching framework of episodic UDRL (eUDRL), which is characterized by a specific structure of the goal. Specifically \eUDRL{} assumes that a goal map $\rho$ is applied to a segment's terminal state to evaluate whether the goal has been reached $g=\rho(s)$. GCSL can be viewed as a slightly restricted version of \eUDRL{}, as it focuses solely on state-reaching tasks and operates with a fixed horizon. Decision Transformer~\citep{chen2021decision} (DT) essentially corresponds to one iteration of \eUDRL{} aimed at offline RL, and ODT can be seen as a form of \eUDRL{} with entropy regularization; see the background section~\ref{se:background} for details. Unlike standard RL policies which define an agent's action probabilities based on a given state, the \lq\lq{}policies\rq\rq{} \(\pi_{old}\) and \(\pi_{new}\) also condition these probabilities on the command \((h, g)\). This leads to the formalism of Command Extensions (cf.~definition~\ref{de:CE}), a special class of Markov Decision Processes, where the command is included as part of the state, \(\bar{s} = (s, h, g)\). On the technical side this article develops the mathematics of Command Extensions, which provides a solid foundation for exploring RL through SL. 

\eUDRL{}'s learning process only requires a single SL step, with no need for value function or policy proxies. While simplicity and high experimental performance coincide remarkably in the \eUDRL{} algorithm, it is evident only in deterministic environments that \eUDRL{} identifies the optimal policy. This limitation was acknowledged already in the original work \citep{schmidhuber2019reinforcement}, where an alternative approach was proposed for non-deterministic environments (e.g., by operating with expected returns rather than just actual returns). In practice, \eUDRL{}'s simplicity has led to its adoption in non-deterministic environments as well. In part this was based on the observation that many non-deterministic environments exhibit only minor non-determinism, i.e.~they can be viewed as perturbed deterministic environments. This approach was explored in the case of {\eUDRL{}} in the article \citep{srivastava2019training}, as well as concurrently {in the case of GCSL} in \citep{ghosh2021learning}, demonstrating the algorithm's practical utility in various domains, including several MuJoCo tasks. However, these articles did not provide solid convergence guarantees. While \cite{ghosh2021learning} showed that GCSL optimizes a lower bound on the goal-reaching objective, the guarantees regarding its tightness are limited. In fact, \eUDRL{}'s goal-reaching objective can be sensitive to perturbations of the transition kernel even in near-deterministic environments. We will discuss this behavior in several instances throughout the article illustrating it with specific examples and computations. Consider figure~\ref{fig:discontIntro} for an illustration of the mentioned discontinuity of the UDRL-generated goal-reaching objective in a specific environment (see section~\ref{se:contfinite} for details). The figure shows the values of the goal-reaching objective (denoted by $J_\alpha$) along two continuous one-parameter families of environment transition kernels, where the families are parametrized by $\alpha$ and intersect at a parameter value of $\alpha=0$. In the case of figure~\ref{fig:discontIntro:a} the one-parameter families intersect at a deterministic environment. In the case of figure~\ref{fig:discontIntro:b} the families intersect a specific non-deterministic environment. While figure~\ref{fig:discontIntro:a} suggests a continuous behavior, the goal-reaching objective appears to be discontinuous in figure~\ref{fig:discontIntro:b}. Despite \eUDRL{}'s remarkable characteristics, the algorithm's stability is clearly a concern.

\begin{figure}[!htb]
	\centering
 \begin{subfigure}{0.48\linewidth}
            \includegraphics[width=\linewidth]{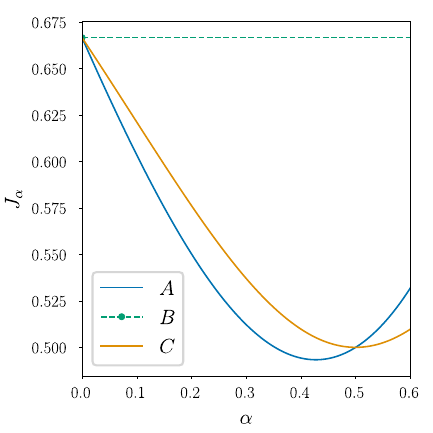}
		\caption{Determinist.~Kernel - Goal-R.~Objective}
         \label{fig:discontIntro:a}
	\end{subfigure} 
	\begin{subfigure}{0.48\linewidth}
            \includegraphics[width=\linewidth]{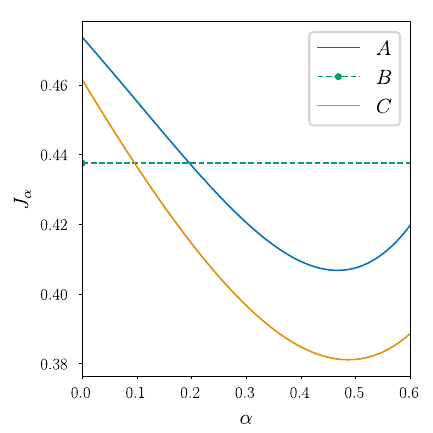}
		\caption{Non-Deter.~Kernel - Goal-R. Objective}
   \label{fig:discontIntro:b}
	\end{subfigure}
	\hfil
	\caption{Illustration of discontinuity of \eUDRL{}-generated goal-reaching objective along two continuous one-parameter rays (items $A$ and $C$) of environments (transition kernels). Horizontal axes show the value of the ray-parameter $\alpha$; the point of intersection is $\alpha=0$. The exact value of the respective quantities at $\alpha=0$ is represented by a horizontal line (item $B$).}
 \label{fig:discontIntro}
\end{figure}

Non-optimal behavior of \eUDRL{} in non-deterministic environments has been highlighted by \cite{strupl2022upsidedown} and \cite{paster2022cant}. The former analyzes the principles behind this non-optimality, while the latter addresses issues with DT and aims for a fully functional variation of the algorithm. Since ODT can be seen as a variation of DT with online fine-tuning, the observed non-optimality is also present in ODT. Subsequent studies, such as \citep{yang2022dichotomy}, have proposed various improvements to UDRL, GCSL, and DT. Although~\cite{brandfonbrener2023does} provide an in-depth analysis of the first \eUDRL{} iteration in the context of offline RL, a comprehensive treatment of the behavior of UDRL, GCSL and ODT at a finite number of iterations and in the asymptotic limit remains an open problem. 

This article does not propose another approach to ensure that \eUDRL{} works in non-deterministic environments, but seeks to understand its original design and justify its empirical success, as seen in \citep{srivastava2019training} and \citep{ghosh2021learning}. We aim to determine if iterating \eUDRL{} beyond the first step leads to convergence to a desired asymptotic behavior, focusing on rigorous convergence guarantees and asymptotic error bounds. To extend convergence guarantees for deterministic transition kernels to nearly-optimal behavior for kernels close to deterministic ones we employ a topological approach. This approach includes mathematical tools that are new to the area of RL, introducing notions like continuity in quotient topologies that employ the weak convergence of measures relative to compact sets (the so-called \emph{relative continuity}). We prove the relative continuity of policies and the continuity of the goal-reaching objectives at deterministic kernels after any finite iteration. Despite its significance, this topic is largely unexplored, partly due to its complexity. As we will show, \eUDRL{} policies become discontinuous from the second iteration onward, particularly at the boundary of the set of transition kernels (including deterministic kernels). This suggests that the first iteration is as effective as any finite iteration but it does not clarify whether repeated iterations reduce the error to optimal behavior. By establishing bounds for the \eUDRL{} recursion {under specific conditions}, we demonstrate the relative continuity of accumulation point sets of \eUDRL{}-generated policies and prove the continuity of related quantities like the goal-reaching objective, along with useful bounds and $q$-linear convergence rates. These conditions outline two important special cases: The first involves a condition on the support of the initial distribution and the second is characterized by the uniqueness of the optimal policy in a deterministic environment, see section~\ref{se:continfty} for details. While the general discussion of continuity and stability of \eUDRL{} remains an open problem, we have found that the theory developed in this article is sufficient to address the regularized \eUDRL{} recursion. The regularized \eUDRL{} recursion closely approximates the ODT, which utilizes entropy regularization. Our analysis of the regularized \eUDRL{} recursion, formulated as a convex combination between the updated policy with the uniform distribution, akin to $\epsilon$-greedy policies, is presented in full generality, without restricting it to any special cases. We restrict
all discussions to finite (discrete) environments, ignoring issues of function approximation and limited sample size.

\subsection{Outline}
Section~\ref{se:background} provides the necessary background for understanding the article.

Section~\ref{se:recrewrites} describes the \eUDRL{} recursion in specific segment subspaces and discusses the connection to the algorithm Reward-weighted Regression.

Section~\ref{se:suppstab} describes sets of states, stable under environment perturbations across all iterations, which proves instrumental for the discussion of continuity of \eUDRL{}-generated quantities in the subsequent sections. 

Section~\ref{se:optdet} proves that \eUDRL{} converges to optimal policies in case of deterministic environments.

Section~\ref{se:contfinite} investigates the continuity of \eUDRL{}-generated quantities at a finite number of iteration for deterministic environments.

Section~\ref{se:continfty} investigates the continuity of sets of accumulation points of \eUDRL{}-generated quantities for deterministic environments.

Section~\ref{se:regrec} investigates the  continuity of sets of accumulation points of quantities generated by regularized \eUDRL{} in full generality.

Section~\ref{se:relwork} discusses related work and section~\ref{se:conclusion} concludes the presentation. 

To make this article accessible to a wider audience several appendices are included.

Appendix~\ref{ap:SegmentDist} contains details about the construction of segment distribution.

Appendix~\ref{ap:examples} contains examples from the main text, worked out in full detail. 

Appendix~\ref{ap:interiorcont} investigates the continuity of \eUDRL{}-generated quantities at interior points (of the set of all transition kernels) for a finite number of iterations.

Appendix~\ref{ap:regrec} contains all details on lemmas and proofs about the regularized \eUDRL{} recursion of section~\ref{se:regrec}.

\section{Background}\label{se:background}
This section provides the necessary background for understanding the article. It offers an orientation on the placement of our work within the \lq\lq{}theoretical landscape\rq\rq{}, reintroducing fundamental theoretical concepts, particularly focusing on Markov Decision Processes of Command Extension type and distributions over segment space. Furthermore, this section outlines how our work fits within the existing literature, describing the integration of ODTs within the \eUDRL{} framework. The innovative developments presented in this article will be introduced in the following sections.

We will be dealing exclusively with finite, real random variables. They are defined as maps $X:\Omega\rightarrow \mathbb{R}$ with a measure space $(\Omega,\mathcal{F},\prob)$, where $\Omega$ is a finite set, $\mathcal{F}$ is a $\sigma$-algebra over $\Omega$ and $\prob$ denotes a probability measure.

\subsection{Markov Decision Process} 
Markov decision processes (MDPs) are a mathematical framework for sequential decision problems in uncertain dynamic environments~\citep{puterman2014markov}. Formally an MDP is a five-tuple $\mathcal{M}=(\mathcal{S},\mathcal{A},\tkernel,\mu,R)$, with a set $\mathcal{S}$ of admissible states, a set of possible actions $\mathcal{A}$, a transition probability kernel $\tkernel$ that defines the probability of entering a new state given that an action is taken in a given state, a distribution $\mu$ over initial states, and a (random) reward $R$ that is granted
through this transition. In MDPs an agent interacts iteratively with an environment over a sequence
of time steps $t$ (where we add the subscript $t$ to random variables to emphasise that they belong to a specific point of time). Let the random variables $S_t:\Omega\rightarrow \mathcal{S}$ describe the state of the MDP and the random variables $A_t:\Omega\rightarrow \mathcal{A}$ describe the actions chosen by the agent. Beginning with an initial distribution over states $\mu(s) = \prob( S_0=s )$ at each step of the MDP the agent observes the current state $s\in\mathcal{S}$ and takes a
respective action $a\in\mathcal{A}$ according a policy $\pi(a |s ) = \mathbb{P}(A_t=a|S_t=s)$, subsequently the environment transitions to $s'$ according to the probability $\tkernel(s'|s,a)= \prob( S_{t+1} = s_{}' \mid S_{t} = s,A_{t} = a)$. A transition kernel $\tkernel$ is \emph{deterministic} if for each $s,a \in \mathcal{S}\times\mathcal{A}$ the distribution $\tkernel(\cdot|s,a)$ is deterministic, i.e. if there exists a state $s'_{s,a}$ such that $\tkernel(s'_{s,a}|s,a) = 1$. Similarly a policy $\pi$ is deterministic if for each $s\in \mathcal{S}$ the distribution $\pi(\cdot|s)$ is deterministic. Consecutive sequences of state-action transitions are commonly called trajectories of the MDP. In this article we will always assume deterministic rewards, i.e., the reward is given by
$R_{t+1}(S_{t+1},S_t,A_t) = r(S_{t+1},S_t,A_t)$ with a deterministic function $r$.
The return $G_t = \sum_{k \in \mathbb{N}_0} R_{t+k+1}$ is the accumulated reward beginning with time $t$ over the entire MDP episode, where we do not discount future rewards.
The performance of an agent following a policy $\pi$ can be measured by means of the
state value function $V^{\pi}(s) = \ev [G_t \mid S_t = s;\pi]$ and the 
action value function $Q^{\pi}(s,a) = \ev [G_t \mid S_t = s,A_t = a;\pi]$. There exist a unique optimal state value function $V^* = \max_{\pi} V^{\pi}$ and a unique optimal action value function $Q^* = \max_{\pi} Q^{\pi}$, where the maximization is taken over the set of all policies. A policy $\pi^*$ for which $V^{\pi^*} = V^*$ and consequently also $Q^{\pi^*} = Q^*$ is called optimal. In what follows we will work with special types of MDPs called \textit{Command Extensions}. In this article we will reserve the symbols $\pi$, $V^{\pi}$, $Q^{\pi}$ for the respective quantities for MDPs of Command Extension type.

\subsection{Markov Decision Processes of Command Extension Type}
The aim of the training process in \eUDRL{} is to achieve that the agent becomes better at executing commands. In \eUDRL{} the command is provided to the agent in the form of a \lq\lq{}goal\rq\rq{} and a \lq\lq{}horizon\rq\rq{}. The goal says which state the agent is supposed to achieve and the horizon says when this state should be achieved. A goal map $\rho:\mathcal{S}\rightarrow\mathcal{G}$ will be employed to evaluate if the command's goal has been reached at the command's horizon, i.e.~if $\rho(s)$ equals to the command's specified goal for the final state $s$. The codomain $\mathcal{G}$ of the goal map is the set of all possible goals, weather they are reached in one specific trajectory or not. Notice that since $\rho$ is defined on the entire state space $\mathcal{S}$, every state $s$ at the horizon corresponds to a valid goal $\rho(s)$. However, this goal might not be the specified target of the chosen command. The introduction of the goal map allows
one to study prototypical \eUDRL{} tasks within a unified formalism: The state reaching task
is covered by $\mathcal{G}=S, \rho = \mathrm{id}_{\mathcal{S}}$ (the identity map on $\mathcal{S}$), while goals related to return of the original MDP $\mathcal{M}$ can be covered by extending the state by a component that accumulates rewards and defining $\rho$ to be the projection on this component.

Apart from the MDP state the \eUDRL{} agent takes an additional command input. Let the random variable $G_t:\Omega\rightarrow \mathcal{G}$ describe the goal of the \eUDRL{} agent and let $H_t:\Omega\rightarrow \bar{\mathbb{N}}_0=\{0,1,...,N\}$, $N\geq1$, be a random variable that describes the remaining horizon. The \eUDRL{} agent can be viewed as an ordinary agent operating in an MDP whose state space is extended by the command. In this context one has to provide an initial distribution over goals and horizons 
$\prob(H_0=h,G_0=g \mid S_0=s)$ and with each transition of the extended MDP the remaining horizon decreases by $1$ until a horizon of $0$ is reached. At this stage the extended MDP enters an absorbing state (see below for a definition), from which no further evolution occurs. A reward is granted if the specified goal is reached when the horizon turns to $0$. In summary we have the following definition of command extensions (cf. Definition 1 in \citep{strupl2022upsidedown}):

\begin{definition}[Command Extension] \label{de:CE}
A Command Extension (CE)
of an MDP of the form $\mathcal{M}=(\mathcal{S}, \mathcal{A}, \tkernel, \mu, r)$
is an MDP of the form $\bar{\mathcal{M}} = (\bar{\mathcal{S}}, \mathcal{A}, \bar{\tkernel},  \bar{\mu}, \bar{r},\rho)$, 
where:
\begin{itemize}
\item A command is a pair $(g,h)\in\mathcal{G}\times \bar{\mathbb{N}}_0$, where  $\mathcal{G}\subset\mathbb{Z}^{n_G}$ and $\bar{\mathbb{N}}_0 = \{0,1,...,N\}$. $N\geq 1$ stands for the maximum horizon and $\mathcal{G}$ for the goal set of the CE. The goal map 
$\rho:S \rightarrow \mathcal{G}$ is used to evaluate if a goal $g\in\mathcal{G}$ has been reached.

\item
The extended state space is $\bar{\mathcal{S}} = \mathcal{S}\times\bar{\mathbb{N}}_0\times{\mathcal{G}}$ and the extended state is a triple $\bar{s} = (s,h,g) \in \bar{\mathcal{S}}$ composed of an original MDP state $s$ and the command $(g,h)$.

\item
The initial distribution of $\bar{\mathcal{M}}$ is given by a product of a distribution of commands and the initial distribution of $\mathcal{M}$
$$\bar{\mu}(\bar{s}) = \prob(H_0=h ,G_0=g \mid S_0=s)\mu(s).$$

\item
The transition kernel $\bar{\tkernel}$ is defined for all $(s,h,g) \in \bar{\mathcal{S}}$, all $s'\in \mathcal{S}$ and $a\in\mathcal{A}$ by 
\begin{align*}
\bar{\tkernel} \Big( (s',h-1,g) \mid (s,h,g), a \Big)
&= \tkernel(s' \mid s,a) &\textnormal{if}\quad h>0,\\
\bar{\tkernel} \Big( (s',h,g) \mid (s,h,g), a \Big)\quad
&= \delta_{ss'}&\textnormal{if}\quad h=0.
\end{align*}

\item The reward function $\bar{r}$ is defined for all $\bar{s}'=(s',h',g')\in\bar{\mathcal{S}},\bar{s}=(s,h,g)\in\bar{\mathcal{S}}$ and all $a\in\mathcal{A}$ by
\begin{align*}
\bar{r}\Big( (s',h',g'), (s,h,g),a\Big) = \begin{cases}
\indicator_{\{\rho(s') = g\}}\quad&\textnormal{if}\quad h=1, h'=0, g'=g,\\
0\quad &\textnormal{otherwise}.
\end{cases}
\end{align*}
\end{itemize}
We will commonly refer to $\mathcal{M}$ as \emph{the original MDP} and to $\bar{\mathcal{M}}$ as \emph{the extended MDP}. Similarly, $s$ will be the \emph{original state} and $\bar s$ will be the \emph{extended state}.
\end{definition}

Our subsequent investigation almost exclusively revolves around the policies of extended MDPs rather than those of the underlying original MDPs. For brevity we reserve the notation $\pi$ (rather than $\bar\pi$) for the policies of CE agents. The reward is defined in a way as to guide the agent towards achieving the intended command. The choice of a binary reward function allows one to interpret the expected reward values as probabilities. In summary the computation of state and action-value functions yields that for all $(s,h,g)\in \bar{\mathcal{S}}$, $h>1$, and all actions $a\in\mathcal{A}$
\begin{equation}
\begin{aligned}
    Q^{\pi}((s,h,g),a) &=
\prob\Big( \rho(S_{t+h})=g\Big|S_{t}=s,H_{t}=h,G_{t}=g,A_t=a;\pi\Big),
\\
V^{\pi}(s,h,g) &= 
\prob\Big(\rho(S_{t+h})=g\Big|S_t=s,H_t=h,G_t=g;\pi \Big).
\end{aligned}
\label{eq:valueexpr}
\end{equation}
With these state and action-value functions, CEs could be interpreted as ordinary finite-horizon RL problems. One could then choose from a variety of algorithms that are guaranteed to converge to {an} optimal policy, see e.g.~the monograph by \cite{sutton2018reinforcement}. However, this work is dedicated to the study of the \eUDRL{} algorithm, which is not just an ordinary RL algorithm applied to an MDP of CE type. Instead \eUDRL{} is defined through a specific iterative procedure on CEs (that involves trajectory sampling, supervised learning and policy updates), see section~\ref{subsection:eUDRL} for a detailed description, which (despite of certain convenient properties) is not always guaranteed to converge to the optimum. In the study of \eUDRL{} it will be convenient to employ the \textit{goal-reaching objective} of \cite{ghosh2021learning}, which is derived from the CE value function and can be written as
\begin{equation}
J^{\pi} = \sum_{\bar{s}\in\bar{\mathcal{S}}} \mu(\bar{s}) V^{\pi}(\bar{s}).
\label{eq:goalreachingobj}
\end{equation}
There exist a unique optimal goal-reaching objective $J^* = \max_{\pi} J^{\pi}$, 
 where maximization is taken over the set of all policies. It is easy to see that for any optimal policy $\pi^*$ it holds that $J^{\pi^*} = J^*$.
It will be convenient to introduce a separate notation of absorbing states of the CE. Recall that a state $s\in\mathcal{{S}}$ is called \emph{absorbing} for $\mathcal{M}$ if once the absorbing state is reached the underlying process always remains in this state, i.e.~for all states $s$, $s'$ and actions $a$ the kernel equals the Kronecker delta $\tkernel(s'|s a)=\delta_{ss'}$.

\begin{remark}[Absorbing State of CE] Let $\bar{\mathcal{M}} = (\bar{\mathcal{S}}, \mathcal{A}, \bar{\tkernel},  \bar{\mu}, \bar{r},\rho)$ be a CE. A state of the form ${(s,g,h)\in\bar{\mathcal{S}}}$ is an absorbing state if {and only if} $h=0$. A state that is not absorbing is called transient. We denote by $\bar{\mathcal{S}}_A\subset\bar{\mathcal{S}}$ and $\bar{\mathcal{S}}_T\subset\bar{\mathcal{S}}$, respectively, the sets of absorbing and transient states.
\end{remark}

Initializing a CE into an absorbing state results in no learnable information because the CE will remain in this state, irrespective of the chosen actions. For this reason we will focus our investigation on such CEs that are not initialized into an absorbing state. In other words we will assume that the initial distribution of the studied CE is such that it assigns $0$ probability to absorbing states,
\begin{equation}
\supp \bar{\mu} \subset \bar{\mathcal{S}}_T.
\label{eq:mucond}
\end{equation}
Such CEs are called \emph{non-degenerate}. Hereafter we will only study non-degenerate CEs, loosely referring to them as simply as CEs.

\subsection{Trajectory and Segment Distributions of CEs}

A key role in the \eUDRL{} algorithm is played by segments of sampled CE trajectories. A CE trajectory contains a sequence of consecutive transitions comprised of extended states and actions.
In full generality CE trajectories could be represented by infinite sequences of the form
$$\tau = ((s_0,h_0,g_0),a_0,(s_1,h_0-1,g_0),a_1,\ldots, (s_{h_0-1},1,g_0),a_{h_0-1}, (s_{h_0},0,g_0),\ldots).$$
However, once the maximum horizon of the CE is reached, the CE enters an absorbing state, from which no further evolution occurs. For this reason it is sufficient to represent trajectories by finite sequences of a maximum length $N$. Let $\Traj \subset (\bar{\mathcal{S}}\times\mathcal{A})^N\times\bar{\mathcal{S}}$ denote the subset of finite trajectories of length $N$ whose transitions satisfy the requirements of definition~\ref{de:CE} on the dynamics of horizons and goals. In what follows we will represent trajectories by
\begin{align*}\tau =((s_0,h_0,g_0),a_0,&(s_1,h_0-1,g_0),a_1,\ldots\\&\ldots,(s_{l(\tau)-1},1,g_0), a_{l(\tau)-1},(s_{l(\tau)},0,g_0),\ldots (s_{N},0,g_0)) \in \Traj.
\end{align*}
The quantity $l(\tau)$ stands for the number of transitions until an absorbing state is entered for the first time; it is equal to the initial horizon of $\tau$.
Assuming finite state and
action spaces $\Traj$ is measurable and we denote by $\mathcal{T}:\Omega\rightarrow\Traj$ a random variable with
components $\mathcal{T} = ((S_0,H_0,G_0),A_0,\ldots,(S_N,H_N,G_N))$ whose outcomes are CE trajectories. The probability of $\tau \in \Traj$ is given by
$$
\prob( \mathcal{T} = \tau;\pi)
=
\left( \prod_{t=1}^{l(\tau)}
\tkernel(s_t | a_{t-1}, s_{t-1})
\right)
\cdot
\left( \prod_{t=0}^{l(\tau)-1}
\pi(a_t|\bar{s}_t)
\right)\cdot\bar{\mu}(\bar{s}_0)
.$$

Each segment contains a chunk of consecutive state-action transitions from a given trajectory and also the segment's initial horizon, goal and length. We will represent segments by tuples of the form
$$\sigma=(l(\sigma),s_0^{\sigma},h_0^{\sigma},g_0^{\sigma},a_0^{\sigma},
s_1^{\sigma},a_1^{\sigma},\ldots,s_{l(\sigma)}^{}),$$
where $a_0^{\sigma},\ldots,a_{l(\sigma)-1}^{\sigma}$ are the chosen actions, $s_0^{\sigma},\ldots,s_{l(\sigma)}^{\sigma}$ are the reached states of the original MDP and the subscript $t$ indicates the time of the respective quantity \textit{within the segment} (i.e.~irrespective of the actual time when the state or action occurred in the original MDP's trajectory). The quantity $l(\sigma)$ stands for the length of the segment (i.e.~the number of actions executed during the course of the segment) and $h_0^{\sigma}$ and $g_0^{\sigma}$ denote the segment's initial horizon and goal. In our notation we leave out the horizons and goals for all extended states except from the first one because they are already determined by the CE's horizon and goal dynamics. We write 
$\Seg$ for the {set} of all tuples of this form. Assuming finite state and action spaces $\Seg$ is measurable and 
we denote by $\Sigma: \Omega \rightarrow \Seg$ a random variable with components $\Sigma = (l(\Sigma),\stag{S}_0,\stag{H}_0,\stag{G}_0,\stag{A}_0,
\stag{S}_1,\stag{A}_1,\ldots,\stag{S}_{l(\Sigma)})$ whose outcomes are segments of CE trajectories. We provide all formal details about the construction of the distribution of $\Sigma$ in appendix~\ref{ap:SegmentDist}. Here we only summarize that for all $\sigma \in \Seg$ holds
$$
\prob(\Sigma=\sigma;\pi)
= c^{-1}\sum_{t \leq N -l(\sigma)}
\prob(S_t=s_0^{\sigma},H_t=h_0^{\sigma},G_t=g_0^{\sigma},A_t=a_0^{\sigma}, \ldots, S_{t+l(\sigma)}=s_{l(\sigma)}^{\sigma} ;\pi)
,
$$
where $c$ is the normalization constant
$$
c = \sum_{\sigma\in \Seg}
\sum_{t \leq N -l(\sigma)}
\prob(S_t=s_0^{\sigma},H_t=h_0^{\sigma},G_t=g_0^{\sigma},A_t=a_0^{\sigma}, \ldots, S_{t+l(\sigma)}=s_{l(\sigma)}^{\sigma} ;\pi).
$$
It is shown in the appendix that the following bounds hold for the normalization constant:
\begin{equation}
0 < c \leq \frac{N(N+1)}{2}.\label{eq:cupperbound}
\end{equation}
We will also demonstrate the following \lq\lq Markovianity properties\rq\rq{} for the distribution of $\Sigma$ in the appendix.
For every segment of length $k$ and any $i \leq k$ we have that
\begin{align}%
&\prob\left(\stag{A}_i=a_i\Big|l(\Sigma)=k,\stag{S}_0=s_0,\stag{H}_0=h_0,\stag{G}_0=g_0,\stag{A}_0=a_0,\ldots,\stag{S}_i=s_i;\pi_n\right)\label{eq:segactions}\\
&=\pi_n(a_i|s_i,h-i,g),\nonumber\\
&\prob\left(\stag{S}_i=s_i\Big|l(\Sigma)=k,\stag{S}_0=s_0,\stag{H}_0=h_0,\stag{G}_0=g_0,\ldots,\stag{S}_{i-1}=s_{i-1},\stag{A}_{i-1}=a_{i-1};\pi_n\right)\nonumber\\
&
=
\prob(\stag{S}_i=s_i|\stag{S}_{i-1}=s_{i-1},\stag{A}_{i-1}=a_{i-1})
=
\tkernel(s_i|s_{i-1},a_{i-1})
\label{eq:segtransitions}
.
\end{align}

\paragraph{Restriction to Trailing Segments:}
Sometimes it is useful restrict the analysis to specific subsets of segments in $\Seg$. For instance it will useful to consider so called \lq\lq{}trailing segments\rq\rq{}.
A segment
$\sigma
=
(l(\sigma),s_0^{\sigma},h_0^{\sigma},g_0^{\sigma},a_0^{\sigma},\ldots,s_{l(\sigma)}^{\sigma})$
is trailing if it is aligned with the end of a trajectory, i.e.,
if $l(\sigma)=h_0^{\sigma}$. The subspace of $\Seg$ consisting of
all trailing segments will be denoted by 
$\SegTrail:=
\{\sigma \in \Seg \mid l(\sigma)=h_0^{\sigma} \}$.
ODT and also sometimes \eUDRL{} operates on trailing segments. In the case of \eUDRL{} this restriction
was motivated by speeding up the learning process reflecting the episodic nature of the problem.

\paragraph{Restriction to Diagonal Segments:}
A segment that satisfies
$g_0^{\sigma}=\rho(s_{l(\sigma)}^{\sigma})$ is called
a goal-diagonal segment. Such segments are characterized by the fact that the initial goal of the segment is actually achieved at the end of the segment. A segment which is trailing and
goal-diagonal will be called diagonal.
The subspace of $\Seg$ which consists of all diagonal segments
will be denoted
$\SegDiag:=
\{\sigma \in \Seg | l(\sigma)=h_0^{\sigma}, g_0^{\sigma}=\rho(s_{l(\sigma)}^{\sigma}) \}$.
We will consider these segments when discussing the relation
to the Reward-Weighted Regression~\citep{peters2007reinforcement}, see section~\ref{se:RWR}.

\subsection{The eUDRL Algorithm} \label{subsection:eUDRL}
Suppose that a CE $\bar{\mathcal{M}}$ is given. Starting with an initial policy $\pi_0$ \eUDRL{} generates a sequence of policies $(\pi_n)_{n\geq0}$ in an iterative process. Each \eUDRL{} iteration comprises the following steps. First, a batch of trajectories is generated from $\bar{\mathcal{M}}$ following the current policy $\pi_n$. Second, {segments} $\sigma$ of trajectories are sampled from the batch according to the distribution $d_{\Sigma}^{\pi_n}$ of $\Sigma$. Finally, a subsequent policy $\pi_{n+1}$ is fitted using supervised learning to the sampled trajectory segments. In practice, a class of parameterized policies will usually be assumed and $\pi_{n+1}$ will be one of the policies that minimize log-likelihood loss within this class. Here, we investigate what \eUDRL{} can achieve in principle. We assume no specific representation of policies and we suppose that $\pi_{n+1}$ is computed using cross-entropy, %
\begin{equation}
\pi_{n+1} = \argmax_{\pi} \ev_{\sigma\sim d_{\Sigma}^{\pi_n}} \log\left(\pi\left(a_0^{\sigma}\:\Big|\: s_0^{\sigma},l(\sigma),\rho(s_{l(\sigma)}^{\sigma})\right)\right),
\label{eq:objective}
\end{equation}
which reflects infinite sample size or complete knowledge about the distribution $d_\Sigma^{\pi_n}$. In other words the new policy $\pi_{n+1}$ is fitted to the conditional probability\footnote{Notice that $\pi_{n+1}$ is defined by \eref{eq:obj2} if and only if
$\prob( \stag{S}_0=s,l(\Sigma)=h,\rho(\stag{S}_{l(\sigma)})=g;\pi_n) > 0$,
otherwise we leave it undefined for now. Later, it will be convenient to choose wlog. $\pi_{n+1}=1/{|\mathcal{A}|}$ if $\prob( \stag{S}_0=s,l(\Sigma)=h,\rho(\stag{S}_{l(\sigma)})=g;\pi_n) = 0$.}
\begin{align}\pi_{n+1}(a \mid (s,h,g)) =
\prob\left(\stag{A}_0=a\:\Big|\: \stag{S}_0=s,l(\Sigma)=h,\rho(\stag{S}_{l(\sigma)})=g;\pi_n\right).
\label{eq:obj2}
\end{align} 

The choice of the conditional in equation~\eqref{eq:objective} can be motivated by a comparison with the Hindsight Experience Replay (HER) algorithm of~\cite{andrychowicz2017hindsight}. The way \eUDRL{} uses segment samples is similar to that of HER in that not only trajectories that achieve the intended goals are used for learning but also trajectories that do not achieve the intended goals, together with their effective outcomes. Following a similar line of reasoning, if a sample $\sigma \sim d_{\Sigma}^{\pi_n}$ is observed, one might assume that the first action $a_0^{\sigma}$ is a good choice for executing the realized command $(l(\sigma),\rho(s_{l(\sigma)}^{\sigma}))$, irrespective of the actual command $(h_0^{\sigma},g_0^{\sigma})$ chosen for this segment. 

Many algorithms that follow the paradigm of \lq\lq{}reinforcement learning via supervised learning\rq\rq{} including the practical \eUDRL{} implementations of \cite{srivastava2019training}, Goal-Conditioned Supervised Learning (GCSL) of \cite{ghosh2021learning} and Online Decision Transformer (ODT) of~\cite{zheng2022online} fit into the framework of this article. It should be said that our assumption of infinite sample size implies a simplification as compared to the \eUDRL{} variations discussed by \cite{srivastava2019training} or \cite{ghosh2021learning}, where we omit the replay buffer (since it is not required).
The algorithm in \cite{srivastava2019training} uses a trajectory organized replay buffer prioritized by return, where goals are related to the trajectory returns. Further initial commands are sampled according to a best quantile of the distribution of returns in the replay buffer. This means that the initial distribution is changing over time, while here we assume that it is fixed. The GCSL algorithm of~\cite{ghosh2021learning} can be viewed as a simplified version of~\cite{srivastava2019training} in that GCSL executes pure state-reaching tasks. GCSL omits the horizon component from commands assuming a fixed horizon for its tasks. Similarly to \cite{srivastava2019training}, GCSL employs a replay buffer, but without any prioritization.
The relation between \eUDRL{} and ODT is described in a separate section, section~\ref{subseciton:ODT}, below.

A convenient feature of the \eUDRL{} iteration is that it is based solely on supervised learning as opposed to algorithms that employ Bellman-type updates like, e.g., in value iteration or actor-critic algorithms.
It was hypothesized that if \eUDRL{} is close to SL it could be easier used for training large scale networks which may require many stabilization tricks with traditional RL algorithms.
While the motivation for the \eUDRL{} iteration is to improve successive policies through updates, it was shown by \cite{strupl2022upsidedown} that the described updates neither guarantee monotonic policy improvements nor the convergence to an optimal policy in stochastic episodic environments (in sense of CE value functions). It is interesting that when introducing \eUDRL{}, \cite{schmidhuber2019reinforcement}
reserved the described iteration for deterministic environments, where, as we will discuss in the main body of this article, the algorithm converges to an optimum. On the other hand, there are also successful reports about \eUDRL{}'s practical implementations, see work by \cite{srivastava2019training} and \cite{ghosh2021learning}, on certain stochastic environments.

\subsection{The Relation of \eUDRL{} to the ODT Recursion}\label{subseciton:ODT}
The idea of combining the successful transformer architecture (see \citep{schmidhuber1992learningFWP}, \citep{vaswani2017attention}, \citep{schlag2021linear}) with UDRL led to the concept of Online Decision Transformers (ODTs) that interpret RL as a sequence modeling problem. The sequences in question consist of tuples $(s_t',a_t',g_t')$ of states, actions and so-called \lq\lq{}returns-to-go\rq\rq{} of the MDP to be solved by ODT. The return-to-go $g_t'$ is computed once a trajectory is completed. It is given by the rewards accumulated from the time $t$ until the end of the trajectory. 

In ODT the transformer model receives sequence of states and returns-to-go to predict the subsequent action. To approximate the distribution of $a_t'$ at each time step $t$ the transformer (with context length $K$) receives sequences of length $\min\{t,K\}$ of preceding states $s_{-K,t}' := s_{\max\{0,t-K+1\}:t}'$ and returns-to-go $g_{-K,t}' := g_{\max\{0,t-K+1\}:t}'$. Similar to UDRL this distribution is employed as a policy that is conditioned on returns-to-go. In contrast to the original transformer architecture, which uses positional encoding, ODT uses time embeddings. Furthermore, ODT operates exclusively on trailing segments. As compared to its predecessor, the Decision Transformer (DT), ODT introduces a series of improvements that allow for online fine-tuning. First, ODT uses stochastic policies (while DT solely allowed for deterministic policies) and, second, these policies are trained using the maximum likelihood criterion (while DT relied on mean squares). These features are present in similar form also in the \eUDRL{} algorithm.
Furthermore, ODT introduced entropy regularization, which motivates the investigation of regularization techniques also for \eUDRL{}. In section~\ref{se:regrec} we will consider \eUDRL{} with policy regularization using a convex combination with a uniform distribution, which brings our investigation close to the ODT algorithm. %

Like the \eUDRL{} implementations of \cite{srivastava2019training} and \cite{ghosh2021learning}, ODT, too, needs some simplifications in order to conform to the CE framework. As before we omit replay buffers and we assume a fixed horizon, which entails that CE initializes the horizon to $N$. This is to ensure a one-to-one correspondence between the remaining horizon $h$ and current time $t$ via $t = N-h$. Consequently, after ODT performs the time embedding this difference becomes obsolete. With these simplifications ODTs match goal-reaching scenarios in the CE formalism in situations, where the reward is granted in the end of the episode. The correspondence is as follows: First the states $s_t'$ of the MDP to be solved by ODT are grouped together into sequences of past states $s_t= s_{-K,t}'$. The original MDP underlying the CE then operates on the grouped states $s_t$.
The codomain of the goal map $\rho$ is chosen as the set $\mathcal{G} = \{\textnormal{\lq\lq{}all possible rewards\rq\rq{}}\}$, where the rewards depend only on the terminal states $s_N$. Policy updates are executed using maximum likelihood matching to the conditional
$a_t| s_t, l(\sigma), \rho(s_N)$. Using the ODT assumptions of trailing segments, fixed horizon $l(\sigma) = h_t$ and the fact that $\rho(s_N)$ computes rewards corresponding to $s_N'$ this is equivalent to matching $a_t| s_{-K,t}', h_t, \rho(s_N)$. The general scenario can be phrased in the language of CEs only partially by the described accumulation of rewards as part of the state. We assume the state is given by $s_t = (s_{-K,t}', z_{-K,t})$ and $z_t$ corresponds to the accumulated reward of the underlying MDP. The goal map is chosen to be $\rho(s_t) = z_t$.
In this case the policy updates are executed by maximum likelihood matching to the conditional
$a_t| (s_{-K,t}',z_{-K,t}), h_t, z_N$, where 
returns-to-go can be recovered from $z_N - z_{-K,t}$. This does not
fit perfectly the CE framework since apart from the desired returns-to-go $z_N - z_{-K,t}$ we also condition on $(z_N,z_{-K,t})$. In principle this might introduce some noise into the learning process. However, this poses no problem as we assume no limitation on the number of available samples. The formulation of ODTs in the CE framework has no drawback as the effect of noise becomes negligible in this way.

\section{The \eUDRL{} Recursion, Reward-Weighted Regression and Practical Implementations}
\label{se:recrewrites}

The standard \eUDRL{} recursion is carried out on the entire segment space. In this section, we will show recursion formulas corresponding to the subspaces of trailing and diagonal segments. This will allow us to connect \eUDRL{} to the Reward-Weighted Regression (RWR) algorithm, see~\citep{dayan1997using} and~\citep{peters2007reinforcement}. The known convergence of RWR~\citep{strupl2021reward} will be an intuitive guidance for our proofs in later sections.

\subsection{The \eUDRL{} Recursion in Specific Segment Spaces} 

The following lemma specifies how the \eUDRL{} recursion can be written when only segments from a specific subset of $\Seg$ are used for fitting in the recursion \eref{eq:objective}, see the discussion of trailing and diagonal segments above.
Recall that $\sigma = (l(\sigma),s_0^{\sigma},h_0^{\sigma},g_0^{\sigma},a_0^{\sigma},\ldots , s_{l(\sigma)}^{\sigma})\in\Seg$ is trailing if $l(\sigma)=h_{0}^{\sigma}$ and it is diagonal if it is trailing and moreover $\rho(s_{l(\sigma)}^{\sigma})=g_0^{\sigma}$.
We will often use the proportional $\propto$ sign instead of equality in the lemma (and also in its proof) in order to save some space by leaving out bulky normalizing factors in policy distributions (which can be easily recovered). As a consequence here the \lq\lq{}proportionality\rq\rq{} refers to the action dimension only.

\begin{lemma}
\label{le:recrewrites}
Consider the recursive policy updates in eUDRL described by equation~\eqref{eq:objective}.
\begin{enumerate}
\item
Suppose the recursion is carried out on the entire set $\Seg$. Then for all states $(s,h,g) \in \bar{S}_T$ and all actions $a \in \mathcal{A}$ it holds that 
\begin{equation}
\begin{aligned}
\pi_{n+1}(a \mid s,h,g)
&=
\prob(\stag{A}_0=a \mid \stag{S}_0=s,l(\Sigma)=h, \rho(\stag{S}_h)=g ;\pi_n)
\\
&\propto
\sum_{h'\geq h,g'\in \mathcal{G}}
\prob(\rho(S_h)=g\mid A_0=a, H_0=h', G_0=g', S_0=s ;\pi_n)
\\
&\quad\cdot
\pi_n(a \mid s, h',g')\:
\prob( \stag{H}_0=h', \stag{G}_0=g'  \mid \stag{S}_0=s, l(\Sigma)=h ;\pi_n).
\end{aligned}
\label{eq:seg}
\end{equation}
\item
Suppose the recursion is carried out on the set $\SegTrail$. Then for all states $(s,h,g) \in \bar{S}_T$ and all actions $a \in \mathcal{A}$ it holds that
\begin{equation}
\begin{aligned}
\pi_{n+1}^{\trail}(a\mid s,h,g) &=
\prob(\stag{A}_0=a \mid \stag{S}_0=s,l(\Sigma)=h, \rho(\stag{S}_h)=g, l(\Sigma) = \stag{H}_0 ;\pi_n^{\trail})
\\
&\propto
\sum_{g'\in \mathcal{G}}
\prob(\rho(S_h)=g\mid A_0=a, H_0=h, G_0=g', S_0=s ;\pi_n^{\trail})
\\
&\quad\cdot
\pi_n^{\trail}(a \mid s, h,g')\:
\prob( \stag{H}_0=h, \stag{G}_0=g' \mid \stag{S}_0=s, l(\Sigma)=h ;\pi_n^{\trail}).
\end{aligned}
\label{eq:segtrailing}
\end{equation}
\item
Suppose the recursion is carried out on the set $\SegDiag$. Then for all states $(s,h,g) \in \bar{S}_T$ and all actions $a \in \mathcal{A}$ it holds that
\begin{equation}
\begin{aligned}
\pi_{n+1}^{\diag}(a \mid s,&h,g)\\
&= 
\prob(\stag{A}_0=a \mid \stag{S}_0=s,l(\Sigma)=h, \rho(\stag{S}_h)=g, 
\rho(\stag{S}_h)=\stag{G}_0, l(\Sigma)=\stag{H}_0 ; \pi_n^{\diag} )
\\
&\propto
\prob(\rho(S_h)=g \mid A_0=a, H_0=h, G_0=g,  S_0=s ;\pi_n^{\diag})
\\
&\quad\cdot\pi_n^{\diag}(a \mid s,h,g)\prob(\stag{H}_0=h, \stag{G}_0=g \mid \stag{S}_0=s,l(\Sigma)=h ;\pi_n^{\diag}).
\end{aligned}
\label{eq:segdiag}
\end{equation}
Moreover, the policy can be written in terms of the $Q$-function as
\begin{align*}
\pi_{n+1}^{\diag}(a\mid s,h,g)&\propto
Q^{\pi_n^{\diag}}((s,h,g),a)
\pi_n^{\diag}(a \mid s,h,g)\\
&\quad\cdot
\prob(\stag{H}_0=h, \stag{G}_0=g \mid \stag{S}_0=s,l(\Sigma)=h;\pi_n^{\diag})
,
\\
\pi_{n+1}^{\diag}(a \mid s,h,g) &\propto
Q^{\pi_n^{\diag}}((s,h,g),a)
\pi_n^{\diag}(a \mid s,h,g).
\end{align*}

\end{enumerate}
\end{lemma}
Notice that the recursion formulas in $\Seg$, $\SegTrail$ and $\SegDiag$ bear a basic similarity in that while $\pi_{n+1}^{\diag/\trail}(a\mid s,h,g)$ depends on the conditional probabilities $\prob(\stag{H}_0=h, \stag{G}_0=g \mid \stag{S}_0=s,l(\Sigma)=h;\pi_n)$ it does not depend on the conditions $l(\Sigma) = \stag{H}_0$ and $\rho(\stag{S}_{l(\Sigma)})=\stag{G}_0$ that define $\SegDiag$ and $\SegTrail$. The recursion formulas are different only through the sets over which the summations are taken. Later this property will allow us to derive common bounds to all three recursions (i.e.~bounding $\pi_{n+1}$ in terms of $\pi_n$) resulting in simpler proofs.\\

\begin{proof}
Lemma~\ref{le:recrewrites}, Point 1.
The equality is just \eref{eq:obj2}. We start by applying Bayes' rule, where we assume $\prob(\stag{S}_0=s,l(\Sigma)=h, \rho(\stag{S}_h)=g ;\pi_n)>0$. Subsequently we marginalize over $\stag{H}_0$ and 
$\stag{G}_0$ provided that $\stag{H}_0 \geq l(\Sigma) = h$ (by definition a segment is always contained in a trajectory) and we apply the product rule. We obtain
\begin{align*}
\pi_{n+1}(a \mid s,h,g)
&=
\prob(\stag{A}_0=a \mid \stag{S}_0=s,l(\Sigma)=h, \rho(\stag{S}_h)=g ;\pi_n)
\\
&\propto
\prob(\rho(\stag{S}_h)=g, \stag{A}_0=a \mid \stag{S}_0=s,l(\Sigma)=h ;\pi_n)
\\
&=
\sum_{h'\geq h,g'\in \mathcal{G}}
\prob(\rho(\stag{S}_h)=g, \stag{A}_0=a, \stag{H}_0=h', \stag{G}_0=g' \mid \stag{S}_0=s,l(\Sigma)=h ;\pi_n)
\\
&=
\sum_{h'\geq h,g'\in \mathcal{G}}
\prob(\rho(\stag{S}_h)=g \mid \stag{A}_0=a, \stag{H}_0=h', \stag{G}_0=g', \stag{S}_0=s,l(\Sigma)=h ;\pi_n)
\\
&\qquad\qquad\;\cdot
\prob(\stag{A}_0=a \mid \stag{H}_0=h', \stag{G}_0=g', \stag{S}_0=s, l(\Sigma)=h ;\pi_n)
\\
&\qquad\qquad\;\cdot
\prob( \stag{H}_0=h', \stag{G}_0=g'  \mid \stag{S}_0=s, l(\Sigma)=h ;\pi_n)
\\
&=
\sum_{h'\geq h,g'\in \mathcal{G}}
\prob(\rho(S_h)=g \mid A_0=a, H_0=h', G_0=g', S_0=s ;\pi_n)
\\
&\qquad\qquad\;\cdot
\pi_n(a \mid s, h',g')
\prob( \stag{H}_0=h', \stag{G}_0=g' \mid \stag{S}_0=s, l(\Sigma)=h ;\pi_n),
\end{align*}
where in the last equality we applied the segment distribution properties \eref{eq:segactions}
and \eref{eq:segtransitions}.

Lemma~\ref{le:recrewrites}, Point 2.
We have that
\begin{align*}
\pi_{n+1}^{\trail}(a\mid s,h,g)
&=
\prob(\stag{A}_0=a \mid \stag{S}_0=s,l(\Sigma)=h, \rho(\stag{S}_h)=g, l(\Sigma) = \stag{H}_0 ;\pi_n^{\trail})
\\
&=
\prob(\stag{A}_0=a, \stag{H}_0=h \mid \stag{S}_0=s,l(\Sigma)=h, \rho(\stag{S}_h)=g, l(\Sigma) = \stag{H}_0 ;\pi_n^{\trail}).
\end{align*}
The first equality is a consequence of \eref{eq:objective}, where the condition $l(\Sigma) = \stag{H}_0$ is added to restrict the equation to $\SegTrail$. For the second equality we made use of the implication $(l(\Sigma) = h)\;\wedge\;(l(\Sigma) = \stag{H}_0)\implies \stag{H}_0 = h$. Reasoning as above we apply Bayes’ rule,
where we assume $\prob(\stag{S}_0=s,l(\Sigma)=h, \rho(\stag{S}_h)=g, l(\Sigma) = \stag{H}_0 ;\pi_n^{\trail})>0$. Subsequently we marginalize over $\stag{H}_0$ and 
$\stag{G}_0$, apply the product rule and make use of the implication $(l(\Sigma) = h)\;\wedge\;(\stag{H}_0 = h)\implies l(\Sigma) = \stag{H}_0$ to eliminate the event $l(\Sigma) = \stag{H}_0$. We obtain
\begin{align*}
\pi_{n+1}^{\trail}(a \mid s,h,g)&\propto
\prob(\rho(\stag{S}_h)=g, \stag{A}_0=a, \stag{H}_0=h \mid \stag{S}_0=s,l(\Sigma)=h, l(\Sigma) = \stag{H}_0 ;\pi_n^{\trail})
\\
&\hspace*{-13mm}=
\sum_{g'\in \mathcal{G}}
\prob(\rho(\stag{S}_h)=g, \stag{A}_0=a, \stag{H}_0=h, \stag{G}_0=g' \mid \stag{S}_0=s,l(\Sigma)=h, l(\Sigma) = \stag{H}_0 ;\pi_n^{\trail})
\\
\pi_{n+1}^{\trail}(a \mid s,h,g)&
\\
&\hspace*{-13mm}\propto
\sum_{g'\in \mathcal{G}}
\prob(\rho(\stag{S}_h)=g \mid \stag{A}_0=a, \stag{H}_0=h, \stag{G}_0=g', \stag{S}_0=s,l(\Sigma)=h, {l(\Sigma) = \stag{H}_0};\pi_n^{\trail})
\\
&\hspace*{-13mm}\qquad\;\cdot
\prob(\stag{A}_0=a \mid \stag{H}_0=h, \stag{G}_0=g', \stag{S}_0=s, l(\Sigma)=h, {l(\Sigma) = \stag{H}_0} ;\pi_n^{\trail})
\\
&\hspace*{-13mm}\qquad\;\cdot
\prob( \stag{H}_0=h, \stag{G}_0=g', {l(\Sigma) = \stag{H}_0}  \mid \stag{S}_0=s, l(\Sigma)=h ; \pi_n^{\trail})
\\
&\hspace*{-13mm}=
\sum_{g'\in \mathcal{G}}
\prob(\rho(S_h)=g \mid A_0=a, H_0=h, G_0=g', S_0=s ;\pi_n^{\trail})
\\
&\hspace*{-13mm}\qquad\;\cdot
\pi_n^{\trail}(a \mid s, h,g')
\prob( \stag{H}_0=h, \stag{G}_0=g'  \mid \stag{S}_0=s, l(\Sigma)=h ;\pi_n^{\trail}).
\end{align*}

Lemma~\ref{le:recrewrites}, Point 3. We have that
\begin{align*}
&\pi_{n+1}^{\diag}(a \mid s,h,g)\\
&\hspace*{-1mm}= 
\prob(\stag{A}_0=a \mid \stag{S}_0=s,l(\Sigma)=h, \rho(\stag{S}_h)=g, 
\rho(\stag{S}_h)=\stag{G}_0, l(\Sigma)=\stag{H}_0 ; \pi_n^{\diag} )
\\
&\hspace*{-1mm}=
\prob(\stag{A}_0=a, \stag{H}_0=h, \stag{G}_0=g \mid \stag{S}_0=s,l(\Sigma)=h, \rho(\stag{S}_h)=g, \rho(\stag{S}_h)=\stag{G}_0, l(\Sigma)=\stag{H}_0 ; \pi_n^{\diag} )
\end{align*}

The first equality is a consequence of~\eref{eq:objective}, where the conditions $l(\Sigma) = \stag{H}_0$ and $\rho(\stag{S}_h)=\stag{G}_0$
are added to restrict the equation to $\SegDiag$.  For the second equality we made use of the implications $(l(\Sigma)=h) \wedge (l(\Sigma)=\stag{H}_0) \implies \stag{H}_0=h$ and
$(\rho(\stag{S}_h)=g) \wedge (\rho(\stag{S}_h)=\stag{G}_0) \implies \stag{G}_0=g$. Reasoning as above we apply Bayes’ rule,
where we assume $\prob(\stag{S}_0=s,l(\Sigma)=h, \rho(\stag{S}_h)=g, 
\rho(\stag{S}_h)=\stag{G}_0, l(\Sigma)=\stag{H}_0 ; \pi_n^{\diag} )>0$. Subsequently we marginalize over $\stag{H}_0$ and 
$\stag{G}_0$, apply the product rule and make use of the implications $(\rho(\stag{S}_h)=g) \wedge (\stag{G}_0=g) \implies \rho(\stag{S}_h)=\stag{G}_0$ and $(l(\Sigma)=h) \wedge (\stag{H}_0=h) \implies l(\Sigma)=\stag{H}_0$ to eliminate the events $\rho(\stag{S}_h)=\stag{G}_0$ and $l(\Sigma) = \stag{H}_0$. We obtain
$$
\begin{aligned}
&\pi_{n+1}^{\diag}(a \mid s,h,g) \\
&\hspace*{-1mm}\propto
\prob(\rho(\stag{S}_h)=g,{\rho(\stag{S}_h)=\stag{G}_0}, \stag{A}_0=a, \stag{H}_0=h, \stag{G}_0=g \mid \stag{S}_0=s,l(\Sigma)=h,  l(\Sigma)=\stag{H}_0 ;\pi_n^{\diag})
\\
&\pi_{n+1}^{\diag}(a \mid s,h,g) \\
&\hspace*{-1mm}\propto
\prob(\rho(\stag{S}_h)=g \mid \stag{A}_0=a, \stag{H}_0=h, \stag{G}_0=g,  \stag{S}_0=s, l(\Sigma)=h,  {l(\Sigma)=\stag{H}_0} ;\pi_n^{\diag})
\\
&\hspace*{-1mm}\;\;\cdot
\prob( \stag{A}_0=a \mid \stag{H}_0=h, \stag{G}_0=g, \stag{S}_0=s,l(\Sigma)=h,  {l(\Sigma)=\stag{H}_0} ;\pi_n^{\diag})
\\
&\hspace*{-1mm}\;\;\cdot
\prob(\stag{H}_0=h, \stag{G}_0=g, {l(\Sigma)=\stag{H}_0} \mid \stag{S}_0=s,l(\Sigma)=h ;\pi_n^{\diag}).
\end{aligned}
$$
Finally, making use of the expression \eref{eq:valueexpr} for the action-value functions we obtain 
\begin{align*}
\pi_{n+1}^{\diag}(a \mid &s,h,g)\\
&\propto
Q^{\pi_n^{\diag}}((s,h,g),a)
\pi_n^{\diag}(a \mid s,h,g)
\prob(\stag{H}_0=h, \stag{G}_0=g | \stag{S}_0=s,l(\Sigma)=h;\pi_n^{\diag})\\
&\propto Q^{\pi_n^{\diag}}((s,h,g),a)
\pi_n^{\diag}(a \mid s,h,g).
\end{align*}%
\end{proof}

\subsection{The Relation of eUDRL to Reward-Weighted Regression}\label{se:RWR}
Reward-Weighted Regression (RWR, see \citep{dayan1997using}, \citep{peters2007reinforcement}, \citep{peng2019advantage}) is an RL algorithm that, similarly to \eUDRL{}, relies on a recursive sequence of policy updates. As in \eUDRL{} a batch of trajectories is generated following the current policy $\pi_n$ and the subsequent policy $\pi_{n+1}$ is fitted using supervised learning on the sampled trajectories but the contribution of actions is weighted by trajectory returns. In other words RWR goes beyond vanilla imitation learning on sampled trajectories in that it weights the relevance of trajectories. RWR is characterized by update formulas of the form $\pi_{n+1}(a\mid s) \propto Q(s,a)\pi_n(a \mid s)$, see~\cite{strupl2021reward} for details. In the context of CE this translates to the following update rule
$$
\pi_{n+1}(a \mid s,h,g) \propto Q^{\pi_n}((s,h,g),a) \pi_{n}(a \mid s,h,g).
$$
In view of Lemma~\ref{le:recrewrites}, point \textit{3.}~this corresponds to the update rule of \eUDRL{} in the set $\SegDiag$. In other words in this special situation \eUDRL{} becomes RWR. 
Notice, however, that contrary to general \eUDRL{} the RWR recursion always converges\footnote{Although~\cite{strupl2021reward} assumes  positive rewards, while rewards are non-negative in the article at hand, a generalization of~\cite{strupl2021reward} is straightforward.} to optimal policies~\cite{strupl2021reward}. On the other hand, operating on $\SegDiag$ instead of $\Seg$, RWR uses a far smaller set of samples, which leads to a slower overall learning process. We will not rely on the convergence of RWR in our continuity proofs, but the intuition that \lq\lq{}focusing solely on diagonal segments is sufficient\rq\rq{} will guide our discussion.

\subsection{The Role of Prioritized Replay}

As an application of the described correspondence of \eUDRL{} and RWR, we study the role of prioritized replay in \eUDRL{}. The \eUDRL{} implementation of \cite{srivastava2019training}
employs a replay buffer, where initial goals or returns-to-go are sampled from such trajectories in the buffer that correspond to a high quantile of returns. The following examples illustrate that prioritization by return can both increase and decrease the performance of the learning system. 
\paragraph{Performance deterioration through Prioritized replay:}
Consider an environment that is characterized by the following features:\\
\textit{i)} High returns are hard to reach: Even for optimal sequences of action (in the sense of the CE) the probability of obtaining a high return is small.\\
\textit{ii)} Moderate returns are easy to reach.\\
\textit{iii)} Attempting to reach a high return entails sequences of actions that can lead to 'dead ends', i.e.~low returns.

In this setup it might occur that the preference of the buffer for choosing high returns might lead to small expected returns. Choosing trajectories with high returns from the replay buffer might thus obfuscate the exploration of moderate returns that might have a higher expected value.
\paragraph{Performance improvement through Prioritized replay:}
Consider an environment that is characterized by the following features:\\
\textit{i)} There are two possible returns $\{0,1\}$.\\
\textit{ii)} There are only trailing segments, i.e.~segments are aligned with the end of the trajectory.

In this setup the return-prioritized buffer is filled over time with trajectories of return $1$. As a consequence sampling results in the same initial goals or returns-to-go of $1$ and the buffer will be filled with trajectories
with initial goal or return-to-go of $1$ and realized return of $1$.
Together with the assumption that segments are trailing, it follows that the algorithm operates on diagonal segments. This effectively transforms the algorithm to RWR, which is proven to converge.

\section{The Continuity of \eUDRL{} in the Transition Kernel and Stability Properties}
\label{se:suppstab}

We investigate the continuity of various \eUDRL{}-generated quantities like policies, associated values and the goal-reaching objective under small changes of the transition kernel. The established continuity results will be used to generalize the known optimality of \eUDRL{} when kernels are deterministic (see section~\ref{se:optdet} below) to the case of near optimality of \eUDRL{} (i.e.~optimality up to a fixed error) when kernels are nearly deterministic (i.e.~the kernels are located in a neighborhood of a deterministic kernel).

\subsection{Compatible Families and the Stability of Supports}
\label{subsection:compatible_families}

Let $\mathcal{M}$ be a given MDP and let $\bar{\mathcal{M}}$ denote its CE. Given a state and an action, one can view the transition kernel $\tkernel(\cdot\:|\:s,a)$ as a vector in a \emph{probability simplex} $\Delta S$, i.e., one can write $\tkernel \in (\Delta \mathcal{S})^{\mathcal{S}\times\mathcal{A}}\subset \mathbb{R}^{\mathcal{S}\times\mathcal{S}\times\mathcal{A}}$. We study how properties of $\mathcal{M}$ and $\bar{\mathcal{M}}$, such as \eUDRL{}-generated policies and values depend on $\tkernel$. For our purpose it is sufficient to study the sensitivity with respect to changes of $\tkernel$ (and $\bar{\tkernel}$, respectively), where other components of the MDP and its CE remain fixed. To this end we define the notion of \textit{compatible families of MDPs}.

\begin{definition}
(Compatible families of MDPs)
Let $\lambda_0$ be a given transition kernel, and let $\mathcal{M}=(\mathcal{S},\mathcal{A},\tkernel_0,\mu,r)$ be a respective MDP with CE $\bar{\mathcal{M}} = (\bar{\mathcal{S}},\mathcal{A},\bar{\tkernel_0},\bar{\mu},\bar{r},\rho)$. Write $\{\mathcal{M}_{\tkernel} = (\mathcal{S},\mathcal{A},\tkernel,\mu,r)\:|\:\lambda\in(\Delta \mathcal{S})^{\mathcal{S}\times\mathcal{A}} \}$ for the family of MDPs, whose kernel belongs to the same product of simplexes $(\Delta \mathcal{S})^{\mathcal{S}\times\mathcal{A}}$ as $\tkernel_0$. Similarly, write $\{\bar{\mathcal{M}}_{\tkernel} = (\bar{\mathcal{S}},\mathcal{A},\bar{\tkernel},\bar{\mu},\bar{r},\rho)\:|\:\lambda\in(\Delta \mathcal{S})^{\mathcal{S}\times\mathcal{A}} \}$ for the family of the respective CEs. We will refer to any pair of MDP families resulting from the above construction as \emph{compatible families}.
\end{definition}

Hereafter we will reserve the subscript $\tkernel$ to refer to quantities that stem from the MDPs $\mathcal{M}_\tkernel$ or $\bar{\mathcal{M}}_\tkernel$ of given compatible families. We will occasionally add the subscript $\lambda$ to emphasize that a certain quantity (such as a family of policies) is associated with a family of MDPs.

The following definition summarizes some notions that will play an important role in our proofs.
\begin{definition}
\label{de:theinterestingstates}
Let $\{\mathcal{M}_{\tkernel} | \tkernel\in(\Delta \mathcal{S})^{\mathcal{S}\times\mathcal{A}} \}$ and $\{\bar{\mathcal{M}}_{\tkernel} |\tkernel\in (\Delta \mathcal{S})^{\mathcal{S}\times\mathcal{A}} \}$ be compatible families.
\begin{enumerate}
\item
(Numerator and denominator in \eUDRL{} recursion)
For all $a \in \mathcal{A}$ and all $(s,h,g) \in \bar{S}_T$ let $\num_{\tkernel,\pi}(a,s,h,g)$ (numerator) and $\den_{\tkernel,\pi}(s,h,g)$ (denominator) be defined by
\begin{align*}
\num_{\tkernel,\pi}(a,s,h,g) &= \prob_\tkernel (\stag{A}_0=a, \stag{S}_0=s, l(\Sigma)=h, \rho(\stag{S}_h)=g; \pi ),\\
\den_{\tkernel,\pi}(s,h,g) &= \prob_\tkernel (\stag{S}_0=s, l(\Sigma)=h, \rho(\stag{S}_h)=g; \pi )
.
\end{align*}
\item
(The CE state visitation distribution)
The CE state visitation distribution for transition kernel $\tkernel$ and policy $\pi$ is defined
for all $\bar{s} \in \bar{S}$ by\footnote{Here we can refer to \cite{sutton2018reinforcement} equations (9.2) and (9.3) with addition that one has to restrict the computation only to transient states.}
$$
\nu_{\tkernel,\pi}(\bar{s})
= 
\frac{
\sum_{t<N} \prob_{\tkernel}(\bar{S}_t = \bar{s}, \bar{S}_t \in \bar{S}_T )
}{
\sum_{\bar{s} \in \bar{S}}
\sum_{t<N} \prob_{\tkernel}(\bar{S}_t = \bar{s}, \bar{S}_t \in \bar{S}_T )
}.
$$
A state is called \emph{visited} under $\tkernel$ and $\pi$ if and only if $\nu_{\tkernel,\pi}(\bar{s}) > 0$.

\item (The set of critical states)
Let $\tkernel_0 \in (\Delta \mathcal{S})^{\mathcal{S}\times \mathcal{A}}$ be a deterministic kernel and $\pi_0 >0$ a policy. The set of critical states (for \eUDRL{} learning) is
$$
\theinterestingstates = \supp \den_{\tkernel_0,\pi_0} \cap
\supp \nu_{\tkernel_0,\pi_0}.
$$
Notice that the particular choice of $\pi_0 >0$ does not matter.
\end{enumerate}
\end{definition}
Using the numerator and denominator defined above the \eUDRL{} recursion \eref{eq:obj2} can be written
(for any $\pi_0\in (\Delta\mathcal{A})^\mathcal{S}$) as
\begin{align}
\pi_{n+1}(a|s,h,g) &= \frac{\num_{\tkernel,\pi_n}(a,s,h,g)}{\den_{\tkernel,\pi_n}(s,h,g)} \quad \text{for}\:  (s,h,g) \in \supp \den_{\tkernel,\pi_n},\label{eq:recursionUsingNumeratorDenominator}
\end{align}
where we set\footnote{It is worth noting that our continuity proofs will not depend on the concrete way of defining
$\pi_{n+1}$ for $(s,h,g) \notin \supp \den_{\tkernel,\pi_n}$. The discontinuities that we will present cannot be removed by an alternative definition. See the discussion below example~\ref{ex:boundarypoint} in the appendix~\ref{ap:examples} for details.} $\pi_{n+1}(a|s,h,g)=1/{|\mathcal{A}|}$ outside the support $\supp \den_{\tkernel,\pi_n}$.
The state visitation distribution is restricted to the transient states because this is the set of states where the learned policy affects the evolution of the CE. The set of all visited states equals to the support of $\nu_{\tkernel,\pi}$. The support can be written shortly as
$$
\supp \nu_{\tkernel,\pi} = \{ \bar{s} \in \bar{S}_T ;
(\exists t < N) : \prob_{\tkernel}(\bar{S}_t = \bar{s} ;\pi) > 0
\}.
$$
While it is evident from the recursion equation that $\pi_{n}$ depends on $\lambda$ for $n>1$, we assume that the initial $\pi_0$ is constant through the family of MDPs. Notice that the definition of \lq\lq{}critical states\rq\rq{} $\theinterestingstates$ depends on the specified family through the set $(\Delta \mathcal{S})^{\mathcal{S}\times\mathcal{A}}$. The definition of $\theinterestingstates$ is motivated by the fact that the non-trivial behavior of our continuity discussion occurs on these sets of states. In fact the chosen intersection of states ensures the following criteria:\\
\textit{i)} the command can actually be satisfied (through $\prob_{\tkernel_0} (\stag{S}_0=s, l(\Sigma)=h, \rho(\stag{S}_h)=g; \pi_0 )>0$ on $\supp \den_{\lambda_0,\pi_0})$ and\\
\textit{ii)} critical states have a non-zero visitation probability (on $\supp \nu_{\lambda_0,\pi_0})$.\\
In view of these points, inspecting formula~\eqref{eq:goalreachingobj} yields that given a fixed kernel $\tkernel_0$, the states that are outside $\theinterestingstates$ do not contribute to the goal-reaching objective. %
We will discuss rigorously in section~\ref{se:contfinite} that policies and values in other states are not relevant for establishing the continuity of the goal-reaching objective.

The following lemma reveals some nice stability properties of $\theinterestingstates$ that will be useful later. In particular, we demonstrate that $\theinterestingstates$ is stable under small perturbations of the transitions kernel $\tkernel$ within a fixed compatible family and also through the course of the \eUDRL{} iteration process. For a $\delta>0$ we write
$$U_{\delta}(\tkernel_0)= \left\{\tkernel\:\Big|\:\max_{(s,a)\in\mathcal{S}\times\mathcal{A}}
\|\tkernel(\cdot|s,a)-\tkernel_0(\cdot|s,a) \|_1<\delta\right\}$$
for the $\delta$-neighborhood in composite $\max$-$1$ norm.

\begin{lemma}\label{le:suppstab} (Stability of $\theinterestingstates$)
Let $\{\mathcal{M}_{\tkernel} | \tkernel\in(\Delta \mathcal{S})^{\mathcal{S}\times\mathcal{A}} \}$ and $\{\bar{\mathcal{M}}_{\tkernel} |\tkernel\in (\Delta \mathcal{S})^{\mathcal{S}\times\mathcal{A}} \}$ be compatible families. Let $(\pi_{n,\tkernel})_{n\geq0}$ be a sequence of 
policies generated by the \eUDRL{} iteration given a transition kernel $\tkernel\in(\Delta \mathcal{S})^{\mathcal{S}\times\mathcal{A}}$ and an initial policy $\pi_0$ (that does not depend on $\tkernel$). Fix a deterministic transition kernel $\tkernel_0 \in (\Delta \mathcal{S})^{\mathcal{S}\times\mathcal{A}}$ 
then 
for all initial conditions $\pi_0 > 0$ it
holds:
\begin{enumerate}
    \item For all $n\geq 0$ and all $\tkernel \in U_{2}(\tkernel_0)$ we have that $\supp \num_{\tkernel_0,\pi_0} \cap ( \mathcal{A} \times \supp \nu_{\tkernel_0,\pi_0} ) \subset \supp \num_{\tkernel,\pi_n} \cap ( \mathcal{A} \times \supp \nu_{\tkernel,\pi_n} )$,
    where the inclusion becomes equality for $\tkernel = \tkernel_0$.
    \item For all $n\geq 0$ and all $\tkernel \in U_{2}(\tkernel_0)$ we have that
    $\theinterestingstates \subset \supp \den_{\tkernel,\pi_n} \cap \supp \nu_{\tkernel,\pi_n}$,
    where the inclusion becomes an equality for $\tkernel = \tkernel_0$.
    \item For all $n\geq 0$ and all $\tkernel\in U_{2}(\tkernel_0)$ we have that $$\prob_\tkernel (\stag{S}_0=s, l(\Sigma)=h, \rho(\stag{S}_h)=g, \stag{H}_0=h, \stag{G}_0=g; \pi_{n} ) > 0$$ for all $(s,h,g)\in \theinterestingstates$. 
    \item For all $(a,s,h,g) \in \supp \num_{\tkernel_0,\pi_0} \cap ( \mathcal{A} \times \supp \nu_{\tkernel_0,\pi_0} )$, $h>1$
    there exists an $s' \in \mathcal{S}$ so that 
    $\tkernel_0(s'|s,a) > 0$ and $(s',h-1,g) \in  \theinterestingstates$.
    \item
    For all $(a,s,h,g) \notin \supp \num_{\tkernel_0,\pi_0}$ with $(s,h,g) \in \supp \den_{\tkernel_0,\pi_0}$ we have for any policy $\pi$ that $Q_{\lambda_0}^{\pi}((s,h,g),a) = 0$ .
\end{enumerate}
\end{lemma}
It is straightforward to extend the lemma for arbitrary
$
\tkernel_0 \in (\Delta \mathcal{S})^{\mathcal{S}\times \mathcal{A}}
$, but we will need the result only for deterministic $\tkernel_0$.
\begin{proof}
Assuming that $\tkernel_0$ is deterministic, the following inclusion holds:  
\begin{align}\label{eq:kernelInclusion}
(\forall \tkernel \in U_{2}(\tkernel_0)):\; \supp \tkernel_0 \subset \supp \tkernel.
\end{align}
This means that any state that is reachable through an environment transition by the kernel $\tkernel_0$ is also reachable by $\tkernel$.
Since $\tkernel_0$
is deterministic there exists a unique $s'$ such that $1=\tkernel_0(s'|s,a)$.
It suffices to show that $\tkernel(s'|s,a)$ is non-zero. We have
\begin{align*}
2 > \| \tkernel(\cdot|s,a) - \tkernel_0(\cdot|s,a) \|_1 &=
| \tkernel(s'|s,a) - \tkernel_0(s'|s,a) | + \sum_{s''\neq s'} | \tkernel(s''|s,a) - \tkernel_0(s''|s,a) |
\\
&= 1-\tkernel(s'|s,a) + \sum_{s''\neq s'} \tkernel(s''|s,a)
= 2(1-\tkernel(s'|s,a))
\end{align*}
which is equivalent to $\tkernel(s'|s,a) > 0$.

1. %
Assuming $(a,s,h,g) \in \supp \num_{\tkernel_0,\pi_0} \cap ( \mathcal{A} \times \supp \nu_{\tkernel_0,\pi_0} )$, it follows that there exists a trajectory $\tau' = ((s_0',H',g'),a_0',(s_1',H'-1,g'),a_1',\ldots)$ with positive probability
$\prob_{\tkernel_0}(\mathcal{T} = \tau'; \pi_0) > 0$ with the following property: for a certain $t' \leq N-h$ we have
$s_{t'}' = s$, $a_{t'}' = a$, $H'-t' \geq h$ and $\rho(s_{t'+h}) = g$. We claim that there exists a trajectory
$\tau = ((s_0,H,g),a_0,(s_1,H-1,g),\ldots)$ and $t\leq N$ with positive probability $\prob_{\tkernel_0}(\mathcal{T} = \tau; \pi_0) > 0$ with the property that
$s_{t} = s$, $a_{t} = a$, $h=H-t$ and $\rho(s_{t+h}) = g$.
The existence of a beginning segment of a trajectory (so-called prefix) $((s_0,H,g),a_0,\ldots,(s_t,H-t,g))$ that fits $\tau$ at $t$ and has positive
probability (under $\tkernel_0,\pi_0$) follows from $(s_t,H-t,g) = (s,h,g) \in \theinterestingstates \subset \supp \nu_{\tkernel_0,\pi_0}$. For the terminal segment of the trajectory (the so-called suffix) we take the suffix of $\tau'$ (starting at $t'$) and adjust the horizon and goal components so that they coincide with those of $\tau$ at $t$. The entire trajectory will still have positive probability 
(under $\tkernel_0,\pi_0$), i.e.,
$\prob_{\tkernel_0}(A_t=a_t,\bar{S}_{t+1} = (s_{t+1},H-t+1,g),A_{t+1}=a_{t+1},\ldots  |\bar{S}_t = (s_t,H-t,g);\pi_0) > 0$, since we kept original MDP state components and we can use transitions from 
$\tkernel_0(s_{t'+i+1}'|s_{t'+i}',a_{t'+i}')$, $0\leq i\leq h-1$ and
$\pi_0 > 0$. This proves the existence of $\tau$.

Making use of the inclusion~\eqref{eq:kernelInclusion} we have $\prob_{\tkernel}(\mathcal{T} = \tau; \pi_0) > 0$ for all $\tkernel \in U_{\delta}(\tkernel_0)$. Assuming $\tkernel \in U_{\delta}(\tkernel_0)$ fixed, we prove by induction that
\begin{equation}
\prob_{\tkernel}(\mathcal{T} = \tau; \pi_n) > 0.
\label{eq:taupositive}
\end{equation}
The statement for $n=0$ is established already. Assume that it holds for $n\geq0$. Because the probability of $\tau$ is positive (under $\lambda,\pi_n$) we get
$\forall i : \prob_\tkernel (\stag{A}_0= a_i,\stag{S}_0= s_i ,l(\sigma)= H-i, \rho(\stag{S}_{H-i})=g;\pi_n) > 0$. It follows that
$\pi_{n+1}(a_i|s_i,H-i,g) = \prob_\tkernel(\stag{A}_0=a_i|\stag{S}_0=s_i,l(\sigma)= H -i, \rho(\stag{S}_{H-i})=g ; \pi_n ) > 0$ for all $i < H$.
Since the transition kernel $\tkernel$ and the initial distribution
$\mu$ remain unchanged $\prob_\tkernel(\mathcal{T} = \tau; \pi_{n+1}) > 0$ follows. 

From \eref{eq:taupositive} follows that $\num_{\tkernel,\pi_{n}}(a,s,h,g)>0$ and $\nu_{\tkernel,\pi_n}(s,h,g) > 0$,
also we see that $\prob_\tkernel(\stag{A}_0= a,\stag{S}_0= s ,l(\Sigma)=  h, \rho(\stag{S}_{h})=g, \stag{H}_0=h,\stag{G}_0=g; \pi_{n}) > 0$. This finishes the proof of the inclusion. To establish equality in case of $\tkernel = \tkernel_0$ it suffices to show $\supp \num_{\tkernel_0,\pi_n} \subset \supp \num_{\tkernel_0,\pi_0}$ and $\supp \nu_{\tkernel_0,\pi_n} \subset \supp \nu_{\tkernel_0,\pi_0}$.  But this follows since it is assumed that $\pi_0$ has maximal support.

2. This follows immediately from point 1. and the fact that for any $\tkernel$ we have (for all $n\geq0$, $(s,h,g) \in \bar{S}_T$)
$$
(s,h,g) \in \supp \den_{\tkernel,\pi_n} \iff
(\exists a \in \mathcal{A} : (a,s,h,g) \in \supp \num_{\tkernel,\pi_n} ).
$$

3. The fact that $\prob_\tkernel (\stag{S}_0=s, l(\Sigma)=h, \rho(\stag{S}_h)=g; \stag{H}_0=h, \stag{G}_0=g, \pi_{n} ) > 0 $ follows from inequality~\eqref{eq:taupositive}.

4. This follows immediately from existence of the trajectory $\tau$ constructed in the proof of point 1.

5.
For the sake of contradiction assume $Q_{\tkernel_0}^{\pi}((s,h,g),a) > 0$
for some $\pi$. By the maximality of the support of $\pi_0$ it follows that $Q_{\tkernel_0}^{\pi_0}((s,h,g),a) > 0$.
From $(s,h,g) \in \supp \den_{\tkernel_0,\pi_0}$ and the equivalence used in the proof of point 2.~above we deduce the existence of a trajectory $\tau'$ with positive probability (under $\tkernel_0$, $\pi_0$) (as in point 1. above). We will need just its prefix
$\prob_{\tkernel_0}(S_0 = s_0', H_0 = H', G_0 = g', \ldots S_{t'} = s,
H_{t'} = h', G_{t'} = g'; \pi_0) > 0$, $h' \geq h$.
Since $0< Q_{\tkernel_0}^{\pi_0}((s,h,g),a) =
\prob_{\tkernel_0} (\rho(S_h) = g | S_0 = s, H_0 = h, G_0 = g, A_0 = a ;\pi_0)$ it follows that 
$\prob_{\tkernel_0} (\rho(S_{t'+h}) = g | S_{t'} = s, H_{t'} = h', G_{t'} = g', A_{t'} = a ;\pi_0) > 0$, where we used that support of $\pi_0$ neglects horizon and goal components. Similarly it holds
$\pi_0(a|s,h',g') > 0$. Putting this together we conclude that there exists a trajectory $((s_0',H',g'), \ldots ,(s,h',g'),a, \ldots (s_{t'+h},h'-h,g'),\ldots)$ with $\rho(s_{t'+h})= g$
with positive probability (under $\tkernel_0$, $\pi_0$)
which demonstrates that $\num_{\tkernel_0,\pi_0}(a,s,h,g) > 0$ and contradicts our assumption.
\end{proof}

\subsection{On the Stability of Supports in Segment Sub-Spaces}
\label{sse:suppstabDiagTrail}

The discussion of stability of sets of critical states in compatible families presented in section~\ref{subsection:compatible_families} translate mutatis mutandis to a discussion of the \eUDRL{} recursion on segment sub-spaces $\SegTrail$ and $\SegDiag$. Motivated by the theoretical description of \eUDRL{}-type iterations in algorithms like ODT or RWR we briefly outline the respective stability properties in this section. Let  $\num_{\tkernel,\pi}^{\trail/\diag}$ and $\den_{\tkernel,\pi}^{\trail/\diag}$ denote the numerator and denominator of the \eUDRL{} recursion (compare definition~\ref{eq:recursionUsingNumeratorDenominator}). The quantities are modified by introducing conditions that define the respective segment sub-spaces. For example, for all $a\in \mathcal{A}$ and all $(s,h,g) \in \bar{S}_T$ we set
$$
\begin{aligned}
\num^{\mathrm{diag}}_{\tkernel,\pi}(a,s,h,g) &:= \prob_{\tkernel}(\stag{A}_0 = a, \stag{S}_0 = s, l(\Sigma) = h, \stag{S}_h = g | l(\Sigma) = \stag{H}_0, \rho(\stag{S}_{l(\Sigma)}) = \stag{G}_0 ; \pi),\\
\den^{\mathrm{diag}}_{\tkernel,\pi}(s,h,g) &:= \prob_{\tkernel}(\stag{S}_0 = s, l(\Sigma) = h, \stag{S}_h = g | l(\Sigma) = \stag{H}_0, \rho(\stag{S}_{l(\Sigma)}) = \stag{G}_0 ; \pi).
\end{aligned}
$$
The \eUDRL{} recursion on $\SegTrail$ and $\SegDiag$ has the form as equation~\eqref{eq:recursionUsingNumeratorDenominator} with the renaming $\pi_{n}\rightarrow\pi_{n}^{\diag/\trail}$, $\num\rightarrow \num^{\diag/\trail}$ and $\den\rightarrow \den^{\diag/\trail}$.
\begin{lemma}\label{le:suppstabTrailDiag} (Stability of $\theinterestingstates$ in trailing and diagonal sub-spaces) In the setting of Lemma~\ref{le:suppstab} consider restrictions of the \eUDRL{} recursion to the sets $\SegDiag$ and $\SegTrail$. Then the set of critical states $\theinterestingstates$ remains unchanged and the conclusion of Lemma~\ref{le:suppstab} remain valid under the renaming $\pi_n\rightarrow\pi_n^{\diag/\trail}$, $\num\rightarrow \num^{\diag/\trail}$ and $\den\rightarrow \den^{\diag/\trail}$.
\end{lemma}
\begin{proof}
We follow the proof of Lemma~\ref{le:suppstab} and only highlight the differences. Consider any $(a,s,h,g) \in \supp \num_{\tkernel_0,\pi_0}^{\diag/\trail} \cap (\mathcal{A}\times \supp \nu_{\tkernel_0,\pi_0})$ that is contained in a trajectory $\tau'$. We construct a trajectory $\tau$ with positive probability that contains $(a,s,h,g)$
in a diagonal segment. As the diagonal/trailing segment is also contained in $\Seg$, we can take the same approach as before.
As a byproduct of the construction of $\tau$ we conclude that
\begin{equation}
\supp \num_{\tkernel_0,\pi_0} \cap (\mathcal{A}\times \supp \nu_{\tkernel_0,\pi_0}) = \supp \num_{\tkernel_0,\pi_0}^{\diag/\trail} \cap (\mathcal{A}\times \supp \nu_{\tkernel_0,\pi_0}),
\label{eq:saisthesame}
\end{equation}
and consequently
\begin{equation}
\theinterestingstates = \supp \den_{\tkernel_0,\pi_0} \cap \supp \nu_{\tkernel_0,\pi_0} = \supp \den_{\tkernel_0,\pi_0}^{\diag/\trail} \cap  \supp \nu_{\tkernel_0,\pi_0}.
\label{eq:isthesame}
\end{equation}
In other words the set $\theinterestingstates$ and its state-action variant $\num_{\tkernel_0,\pi_0} \cap (\mathcal{A}\times \supp \nu_{\tkernel_0,\pi_0})$ remain unchanged. To prove that
$
\prob_{\tkernel}(\mathcal{T}=\tau;\pi_n^{\diag/\trail})>0
$
it suffice to realize that the segments used to prove $\pi_{n+1}(a_i|s_i,H-i,g) >0$ were diagonal, i.e.~$\pi_{n+1}^{\diag/\trail}(a_i|s_i,H-i,g) >0$.
The rest of point 1. and points 2.,3.,4. follow similarly. In the proof of point 5. the fact that we began with a diagonal/trailing segment is employed to demonstrate that $\num_{\tkernel_0,\pi_0}^{\diag/\trail}(a,s,h,g) > 0$ and to obtain a contradiction.
\end{proof}

\section{Convergence to Optimal Policies for Deterministic Kernels}
\label{se:optdet}

While the \eUDRL{} recursion converges to optimal policies for deterministic environments, it can result in sub-optimal behavior
in the presence of stochasticity. This fact is apparent already in the early paper by \cite{schmidhuber2019reinforcement},
where the \eUDRL{} recursion in stochastic environments is replaced by an algorithm that operates on expected rather than actual returns. The fact that GCSL converges to an optimal policy is mentioned by \cite{ghosh2021learning}. A related fact is also discussed by \cite{brandfonbrener2023does}, focusing on the first UDRL iteration (offline RL) in return-reaching tasks, where it is proved that intended returns coincide with expected returns in deterministic environments. The fact that GCSL converges to optimal policies is rather straightforward such that a full proof is usually omitted. We provide a complete proof in this article as we will later build on this result and it is convenient to introduce our notation. We begin by characterizing the set of optimal actions for CEs in the case of a deterministic environment in section~\ref{subsec:optimalActions}.
Subsequently, in sections \ref{subsec:detconvSeg}
and \ref{subsec:detconvOthers} we show that this set corresponds to the support of the policies emerging in the course of the \eUDRL{} recursion.
As a consequence the support of $\pi_n$ is constant for $n\geq1$, which demonstrates that \eUDRL{}-generated policies are optimal in deterministic environments for $n\geq1$.

\subsection{Optimal Actions for Deterministic Transition Kernels}\label{subsec:optimalActions}
The goal of this section is to identify the sets of optimal actions for CEs arising form MDPs with a given deterministic transition kernel. We begin by defining the notion of an optimal action.
\begin{definition}\label{def:optimalAction}
(Optimal actions)
Let $\mathcal{M} = (\mathcal{S},\mathcal{A},\tkernel,\mu,r)$ be an MDP. An action $a\in\mathcal{A}$ is called optimal in the state $s\in\mathcal{S}$ if and only if there exists
a policy $\pi$ with $\pi(a|s)>0$ and $\pi$ reaches maximal (optimal) value in $s$. We write $\oactions(s)$ for the set of optimal actions at $s$.
\end{definition}

Notice that, since CE is a particular MDP, this definition includes optimal actions in the context of CEs.
Accordingly, in the context of CEs, any action that maximizes the probability of achieving the intended command is optimal.

\begin{lemma}\label{le:optactions} (The set of optimal actions of CEs for deterministic kernels) Let $\lambda_0$ be a given deterministic transition kernel and let $\mathcal{M}=(\mathcal{S},\mathcal{A},\tkernel_0,\mu,r)$ be a respective MDP with CE $\bar{\mathcal{M}} = (\bar{\mathcal{S}},\mathcal{A},\bar{\tkernel_0},\bar{\mu},\bar{r},\rho)$ and consider a policy $\pi_0 > 0$. For any state $\bar{s}\in\supp \den_{\tkernel_0,\pi_0} (\cdot,\bar{s})$ the set of optimal actions is $\mathcal{O}(\bar{s}) = \supp \num_{\tkernel_0,\pi_0} (\cdot,\bar{s})$.
\end{lemma}

Lemma~\ref{le:optactions} identifies the set of optimal actions for deterministic kernels for general CE-based algorithm but, it has no reference to \eUDRL{} specifically. Notice that the occurring quantities are related to the segment space $\Seg$, which can be constructed without resorting to \eUDRL{}. In our subsequent discussion of continuity we will investigate the impact of changing $\tkernel$ around some fixed deterministic kernel $\lambda_0$. We will drop the index $\oactions_{\lambda_0}(\bar{s})$ in our notation, writing $\oactions(\bar{s})$ to denote optimal actions at fixed $\tkernel_0$ (we will not investigate $\oactions_{\lambda}(\bar{s})$ in situations that the kernel varies).

\begin{proof} We begin with the reverse inclusion. Fix a state $\bar{s}=(s,h,g) \in \supp \den_{\tkernel_0,\pi_0}$ and an action $ a \in \supp \num_{\tkernel_0,\pi_0}(\cdot,\bar{s})$. Since
$\num_{\tkernel_0,\pi_0}(a,\bar{s})>0$ there exists a trajectory $\tau' = ((s_0,H',g'),a_0,(s_1,H'-1,g'),a_1,\ldots)$ with positive probability
under $\pi_0$ with $s_t = s,a_t = a,\rho(s_{t+h})=g, t+h \leq H'$.
Moreover $\tkernel_0(s_{i+1}|s_{i},a_{i}) = 1$ for $i<H'$, since $\tkernel_0$ is deterministic.
Since $Q((s_{t+h-1},1,g),a_{t+h-1}) = \sum_{s' \in \rho^{-1}(\{g\})} \tkernel_0(s'|s_{t+h-1},a_{t+h-1}) = \tkernel_0(s_{t+h}|s_{t+h-1},a_{t+h-1})=1$ is the maximal value, $a_{t+h-1}$ is optimal in $(s_{t+h-1},1,g)$. Therefore, $\pi^*(a_{t+h-1}| s_{t+h-1},1,g) = 1$ is an optimal policy in this state. Furthermore,
\begin{align*}
V^{\pi^*}(s_{t+h-1},1,g) &= \sum_{a\in \mathcal{A}} Q((s_{t+h-1},1,g),a)\pi^*(a| s_{t+h-1},1,g)
\\
&= Q((s_{t+h-1},1,g),a_{t+h-1})\pi^*(a_{t+h-1}| s_{t+h-1},1,g) = 1,
\end{align*}
and
\begin{align*}
Q^{\pi^*}(s_{t+h-2},2,g,a_{t+h-2}) &= \sum_{s'\in \mathcal{S}} V^{\pi^*}(s',1,g) \tkernel_0(s'|s_{t+h-2},a_{t+h-2})
\\
&= V^{\pi^*}(s_{t+h-1},1,g) \tkernel_0(s_{t+h-1}|s_{t+h-2},a_{t+h-2}) = 1,
\end{align*}
which achieves the maximal possible value. Therefore, $a_{t+h-2}$ is optimal in $(s_{t+h-2},2,g)$. Repeating this procedure $h$ times yields that $Q^{\pi^*}(s_{t},h,g,a_t) = 1$, $a_t$ is an optimal action in $(s_{t},h,g)$, and the value is $V^{\pi^*}(s_{t},h,g) = 1$.

For the forward inclusion, let $(s,h,g) \in \supp \den_{\tkernel_0,\pi_0}$ and suppose
$a$ is optimal in $(s,h,g)$. Following the same reasoning as in the proof of the reverse inclusion presented above, we can observe that $(s,h,g)$ achieves the maximal value of $1$. Since $a$ is optimal, there exists a policy $\pi^*$ that reaches the goal $g$
in $h$ steps from the state $s$ following the action $a$. Due to the fact that $\pi_0>0$, it follows that $\pi_0$ also reaches $g$ in $h$ steps from $s$ following $a$, i.e.,
$\prob_{\tkernel_0}( \rho(S_h)=g, A_0=a| S_0=s,H_0=h,G_0=g; \pi_0) > 0$. Since $(s,h,g) \in \supp \den_{\tkernel_0,\pi_0}$, we can similarly as
in proof of point 1.~of lemma~\ref{le:suppstab} find 
$\tau' = (s_0',H',g'),\ldots,(s,h',g'),\ldots)$ with $h'\geq h$
with positive probability under $\tkernel_0$ and $\pi_0$. It follows that
$\prob_{\tkernel_0}( \rho(S_h)=g, A_0=a| S_0=s,H_0=h',G_0=g'; \pi_0) > 0$. Combining this with the suitable prefix of $\tau'$ we conclude
the existence of a trajectory that demonstrates that $\num_{\tkernel_0,\pi_0}(a,s,h,g) >0$.
\end{proof}

\subsection{Convergence of \eUDRL{} in the Space of all Segments}
\label{subsec:detconvSeg}
In the following lemma, we utilize the notion of optimal actions to demonstrate that \eUDRL{}-generated policies $(\pi_n)_{n\geq1}$ are optimal if the transition kernel is deterministic.

\begin{lemma}\label{le:detopt}(Optimality of \eUDRL{} policies for deterministic transition kernels) Let $\lambda_0$ be a given deterministic transition kernel and let $\mathcal{M}=(\mathcal{S},\mathcal{A},\tkernel_0,\mu,r)$ be a respective MDP with CE $\bar{\mathcal{M}} = (\bar{\mathcal{S}},\mathcal{A},\bar{\tkernel_0},\bar{\mu},\bar{r},\rho)$. Assume $\pi_0 > 0$ and let $(\pi_n)_{n\geq0}$ be the policy sequence generated by the \eUDRL{} iteration given $\pi_0$ and $\tkernel_0$. Then it holds:
\begin{enumerate}
\item For all $n \geq 1$ and all $\bar{s}\in \theinterestingstates$ the support of $\pi_{n+1} (\cdot|\bar{s})$ is identical to the set of optimal actions $\oactions(\bar{s})$. In particular, the policy  $\pi_n$ is an optimal policy on $\theinterestingstates$.
\item For all $n\geq 1$ and all $\bar{s}\in \theinterestingstates$ the value is maximal, i.e.,
$$
V^{\pi_n}(\bar{s}) = 1.
$$
\item For all $n\geq 1$, all $\bar{s}\in \theinterestingstates$, and all $a \in \mathcal{A}$, the action-value function is given by
\begin{align*}
Q^{\pi_n}(\bar{s},a) = \begin{cases}&1 \quad \text{if} \:  a \in \oactions(\bar{s}),
\\
&0 \quad \text{otherwise}. 
\end{cases}
\end{align*}
\end{enumerate}
\end{lemma}

\begin{proof}
1. Let $n\geq 0$, $(s,h,g) \in \theinterestingstates (= \supp \den_{\tkernel_0,\pi_n} \cap \supp \nu_{\tkernel_0,\pi_n} )$.
We have $a\in \oactions(s,h,g) \iff \num_{\tkernel_0,\pi_0}(a,s,h,g) >0$. By Lemma~\ref{le:suppstab}, point 1., this is equivalent to
$\num_{\tkernel_0,\pi_n}(a,s,h,g) >0$ and it follows
$$
\pi_{n+1}(a|s,h,g)  = \frac{\num_{\tkernel_0,\pi_n}(a,s,h,g)}{\den_{\tkernel_0,\pi_n}(s,h,g)}>0
\iff
a \in \oactions(s,h,g).
$$
Thus the support of $\pi_{n+1} (\cdot|s,h,g)$ is identical to the set of optimal actions.

2. Let $n\geq 1$, $(s,h,g) \in \theinterestingstates \subset \supp \den_{\tkernel_0,\pi_0}$. As in the proof of lemma~\ref{le:optactions}, $V^*(s,h,g) = 1$ and from point 1. above,
we know that $\pi_n$ is optimal on $\theinterestingstates$. Therefore
$V^{\pi_n}(s,h,g) = 1$.

3. Let $n\geq 1$, $(s,h,g) \in \theinterestingstates$.
First assume $a \in \oactions(s,h,g)$. By lemma~\ref{le:suppstab} point 4.~there exists $s' \in \mathcal{S}$
so that $\tkernel_0(s'|s,a)=1$ and $(s',h-1,g) \in \theinterestingstates$. By point 2. above, $V^{\pi_n}(s',h-1,g)=1$ and so $Q^{\pi_n}((s,h,g),a)= \sum_{s''\in \mathcal{S}} V^{\pi_n}(s'',h-1,g)\tkernel(s''|s,a) = 1$. If $a \notin \oactions(s,h,g)$ the assertion follows from lemma~\ref{le:suppstab} point 5.
\end{proof}

\subsection{Convergence of \eUDRL{} in Segment Sub-Spaces}
\label{subsec:detconvOthers}

In the context of CEs, the set of optimal actions $\oactions(\bar{s})$, $\bar{s} \in \supp \den_{\tkernel_0, \pi_0}$ is linked to the distribution on the segment space $\Seg$ through the quantities $\num_{\tkernel_0, \pi_0}$ and $\den_{\tkernel_0, \pi_0}$, cf.~lemma~\ref{le:optactions}. However, from~\eref{eq:saisthesame} we conclude that for all $\bar{s}\in \theinterestingstates$ it holds that
\begin{equation}
\oactions(\bar{s})=\supp \num_{\tkernel_0,\pi_0}^{\diag/\trail}(\cdot,\bar{s}),
\label{eq:optathesame}
\end{equation}
such that $\supp \num_{\tkernel_0,\pi_0}^{\diag/\trail}(\cdot,\bar{s})$ remains the set of optimal actions even when the algorithm operates on $\Seg^{\diag/\trail}$ as, e.g., in the case of RWR.
The variants of lemma~\ref{le:detopt} for the mentioned segment sub-spaces remain valid mutatis mutandis.

\begin{lemma}\label{le:detoptDiagTrail}(Optimality of \eUDRL{} policies for deterministic kernels in $\Seg^{\diag/\trail}$) In the setting of Lemma~\ref{le:detopt} consider restrictions of the \eUDRL{} recursion to the sets $\Seg^{\diag/\trail}$. The conclusions of the lemma~ remain valid under the renaming $\pi_n\rightarrow\pi_n^{\diag/\trail}$ (cf.~\eref{eq:segdiag}, \eref{eq:segtrailing}).
\end{lemma}

\begin{proof} The proof follows the same lines as the proof of lemma~\ref{le:detopt}, except that we apply lemma~\ref{le:suppstabTrailDiag} (instead of lemma~\ref{le:suppstab}) and equations~\eqref{eq:isthesame},~\eqref{eq:saisthesame}.
\end{proof}

\section{The Continuity of \eUDRL{} at Deterministic Transition Kernels for a Finite Number of Iterations}
\label{se:contfinite}

\begin{figure}[!htb]
	\centering
 \begin{subfigure}{0.48\linewidth}
		\includegraphics[width=\linewidth]{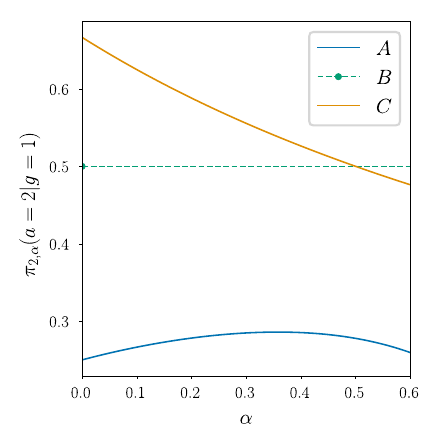}
		\caption{Determinist.~Kernel - Policy}
         \label{fig:discont:a}
	\end{subfigure}
	\hfil
	\begin{subfigure}{0.48\linewidth}
		\includegraphics[width=\linewidth]{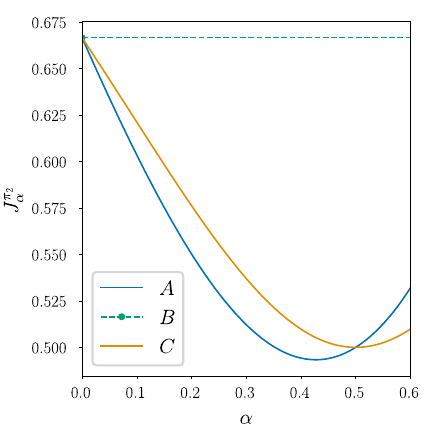}
		\caption{Determinist.~Kernel - Goal-R.~Objective}
  \label{fig:discont:b}
	\end{subfigure}
 
	\begin{subfigure}{0.48\linewidth}
		\includegraphics[width=\linewidth]{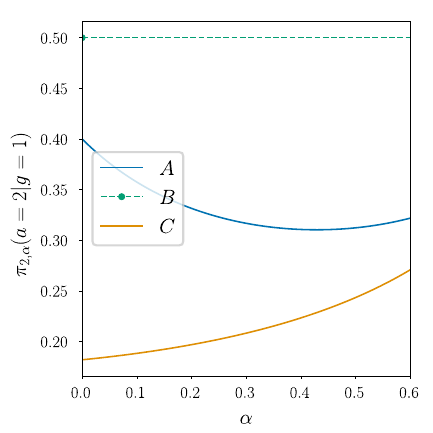}
		\caption{Non-Deter.~Kernel - Policy}
   \label{fig:discont:c}
	\end{subfigure}
	\hfil
	\begin{subfigure}{0.48\linewidth}
		\includegraphics[width=\linewidth]{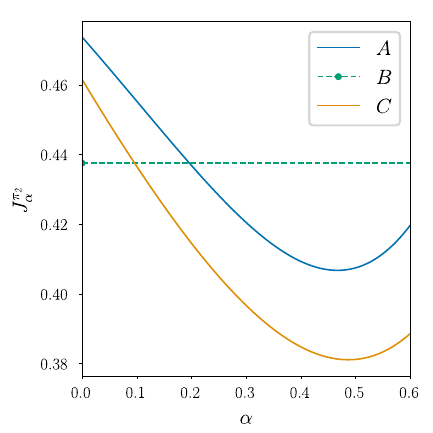}
		\caption{Non-Deter.~Kernel - Goal-R.~Objective}
         \label{fig:discont:d}
	\end{subfigure}
	\caption{Illustration of discontinuity of \eUDRL{}-generated quantities at the boundary of $(\Delta \mathcal{S})^{\mathcal{S}\times\mathcal{A}}$. The figures present values of policy and goal-reaching objective along two continuous one-parameter rays in $(\Delta \mathcal{S})^{\mathcal{S}\times\mathcal{A}}$, see items $A$ and $C$ in the legend. Horizontal axes show the value of the ray-parameter $\alpha$; the boundary is reached at $\alpha=0$, where the rays intersect. The exact value of the respective quantities at the boundary is represented by a horizontal line, see item $B$.}
 \label{fig:discont}
\end{figure}

We begin our discussion by observing that \eUDRL{}-generated quantities, like the policies $\pi_{n,\tkernel}$ and the respective goal-reaching objective $J_{\tkernel}^{\pi_{n,\tkernel}}$, in general do not need to be continuous (functions of the transition kernel) at all points $\tkernel \in (\Delta \mathcal{S})^{\mathcal{S}\times\mathcal{A}}$. Kernels located at the boundary\footnote{In this article we always interpret the terms boundary/interior in the elementary way, viewing $(\Delta \mathcal{S})^{\mathcal{S}\times\mathcal{A}}$ as continuously embedded into $\mathbb{R}^{m}$, where $m = |\mathcal{A}||\mathcal{S}|(|\mathcal{S}|-1)$. The term boundary/interior then refers to the elementary geometric understanding of the boundary/interior of a simplex.} of $(\Delta \mathcal{S})^{\mathcal{S}\times\mathcal{A}}$ are particularly problematic. Figure~\ref{fig:discont} illustrates this by showing the presence of continuous and discontinuous behavior of the policy and the goal-reaching objective near two specific kernels (where one is deterministic\footnote{Notice that deterministic kernels are located on the boundary of the simplex in general. This is because an arbitrary small perturbation that makes the zero-entries of a deterministic kernel negative will leave the simplex of probabilities.} and the other is non-deterministic) at the boundary of $(\Delta \mathcal{S})^{\mathcal{S}\times\mathcal{A}}$. All details of the construction of the MDPs underlying figure~\ref{fig:discont} can be found in the appendix~\ref{ap:examples}, examples \ref{ex:boundarypoint} and \ref{ex:detpoint}. Since \eUDRL{} converges to the optimum already for $n\geq1$ in case of a deterministic kernel, we use $n=2$ for our illustration. Specifically, figure~\ref{fig:discont} shows a comparison of the policy and the goal-reaching objective along two distinct, continuous rays of kernels in $(\Delta \mathcal{S})^{\mathcal{S}\times\mathcal{A}}$, where each ray is parametrized by a parameter $\alpha$. The rays intersect at the boundary when $\alpha=0$. Subfigure~\ref{fig:discont:b} suggests that the goal-reaching objective $J_{\tkernel}^{\pi_{2,\tkernel}}$ is continuous at the boundary (as $\alpha=0+$) if the kernel is deterministic, but subfigure~\ref{fig:discont:d} suggest that the goal-reaching objective can exhibit a non-removable discontinuity at the boundary in general. In view of the behavior in subfigure~\ref{fig:discont:d} the continuity of $J_{\tkernel}^{\pi_{2,\tkernel}}$ at a deterministic point (a product of $\Delta \mathcal{S}$ vertices, which also lies on the boundary) might come as a surprise and, as it turns out, is non-trivial to prove (see below). Notice, however, that \eUDRL{}-generated policies are discontinuous even in the case of a deterministic kernel (see subfigures~\ref{fig:discont:a},~\ref{fig:discont:c}). This can be seen as a consequence of the non-uniqueness of policies near optimality. This fact will play a major role in our proofs and cause a number of technical complications. One might wonder how $\pi_{n,\tkernel}$ and $J_{\tkernel}^{\pi_{n,\tkernel}}$ behave in the interior of $(\Delta \mathcal{S})^{\mathcal{S}\times\mathcal{A}}$, even though this plays a minor role for our discussion of convergence of \eUDRL{}. We leave the proof of continuity of \eUDRL{}-generated $\pi_{n,\tkernel}$ and $J_{\tkernel}^{\pi_{n,\tkernel}}$ strictly in the interior of $(\Delta \mathcal{S})^{\mathcal{S}\times\mathcal{A}}$~to appendix~\ref{ap:interiorcont}.
We also include a discussion describing the causes leading to discontinuities of \eUDRL{}-generated quantities at the end of appendix~\ref{ap:interiorcont}.

The remaining part of this section is structured as follows. Section~\ref{sse:contsegdist} introduces the discussion of continuity along given compatible families of CEs. In section~\ref{sse:theoremfinite}, we define a notion of relative limit and relative continuity and demonstrate the relative continuity of \eUDRL{}-generated policies and the continuity of related values at deterministic kernels for finite number of iterations. We extend the results to other segment
subspaces in section~\ref{sse:extendingfinite}. We conclude by showing the continuity of the goal-reaching objective at deterministic kernels for finite number of iterations and discussing the error to optimality in section~\ref{sse:Jfinite}.

Since we always deal with finite-dimensional vector spaces, all norms are equivalent and we will often omit a norm specification from our results. Similarly, we will use norms interchangeably in some of our continuity-proofs.  

\subsection{Continuity of Values and Segment Distribution Marginals}
\label{sse:contsegdist}

While \eUDRL{}-generated policies and goal-reaching objectives are not generally continuous functions of the transition kernel, this section shows that state-value functions, action-value functions and the segment distribution of marginals are indeed continuous (functions of the transition kernel). The difference is that here we interpret (see below) the state/action-values and segment-distribution marginals as functions of two variables $\tkernel, \pi$, whereas we interpreted them as a function solely of $\tkernel$ in the discussion of figure~\ref{fig:discont}. In other words, in our discussion of figure~\ref{fig:discont} the discontinuity appears because the value function is concatenated with a $\tkernel$-dependent policy.

\begin{lemma} (Continuity in compatible families)
\label{le:contoptbound}
Let $\{\mathcal{M}_{\tkernel} | \tkernel\in(\Delta \mathcal{S})^{\mathcal{S}\times\mathcal{A}} \}$ and $\{\bar{\mathcal{M}}_{\tkernel} |\tkernel\in (\Delta \mathcal{S})^{\mathcal{S}\times\mathcal{A}} \}$ be compatible families.
\begin{enumerate}
\item (Continuity of value)
Fix $\bar{s}\in \bar{S}_T$ and $a \in \mathcal{A}$.
Then the values $V^{\pi}_{\tkernel}(\bar{s})$ and action-values $Q^{\pi}_{\tkernel}(\bar{s},a)$ are continuous when viewed as a function of $(\tkernel,\pi)$. The optimal values $V^*_{\tkernel}(\bar{s})$ and action values $Q^*_{\tkernel}(\bar{s},a)$ are continuous as function of $\tkernel$.

\item (Upper bound on $\mathbb{P}_\tkernel$)
Fix $\bar{s}=(s,h,g)\in \bar{S}_T$ and $a \in \mathcal{A}$. Then for all $h' \geq h$ and
$g'\in \mathcal{G}$ it holds that
$$\prob_\tkernel(\rho(S_h)=g|S_0=s,H_0=h',G_0=g',A_0=a;\pi) \leq 
Q^*_\tkernel((s,h,g),a).$$

\item
(Continuity at value zero)
Let $\pi_0 > 0$ be a given policy, which is assumed constant in $\tkernel$, and let 
$\tkernel_0 \in (\Delta \mathcal{S})^{\mathcal{S}\times \mathcal{A}}$ be a given transition kernel. Let $\tkernel \mapsto \pi_{\tkernel}$
denote a policy that depends on a transition kernel. Then for all $\bar{s}^* \in \bar{\mathcal{S}}_T$ with $V_{\tkernel_0}^{\pi_0}(\bar{s}^*) = 0$ it holds that:
\begin{itemize}
    \item $\tkernel \mapsto V_{\tkernel}^{\pi_{\tkernel}}(\bar{s}^*)$ is continuous at $\tkernel_0$ and
    \item $V_{\tkernel_0}^{\pi_{\tkernel_0}}(\bar{s}^*) = 0$.
\end{itemize}
\item (Continuity of marginals)
Segment distribution marginals are continuous in the transition kernel and the policy. The normalization constant
of segment distribution $c(\tkernel,\pi)$ and $c(\tkernel,\pi)^{-1}$
are continuous.
\end{enumerate}
\end{lemma}
\begin{proof}
1. Values $V^{\pi}_{\tkernel}(\bar{s})$ and action-values $Q^{\pi}_{\tkernel}(\bar{s},a)$ and jointly continuous in $\tkernel$ and $\pi$ because the recursion relation for the value function only contains a finite number of multiplication and summation operations involving $\tkernel$ and $\pi$. The continuity of
$Q^{*}_{\tkernel}(\bar{s},a)$ follows from
$Q^*_{\tkernel}(\bar{s},a) =\max_{\pi}  Q^{\pi}_{\tkernel}(\bar{s},a)$ and
similarly for $V^{*}_{\tkernel}(\bar{s})$.

2. If $\pi^*$ is a policy that achieves the highest probability of executing the command \lq\lq{}reach goal $g$ in exactly $h$ steps\rq\rq{}, then
\begin{multline*}
\prob_\tkernel(\rho(S_h)=g|S_0=s,H_0=h',G_0=g',A_0=a;\pi) \\
\leq
\prob_\tkernel(\rho(S_h)=g|S_0=s,H_0=h,G_0=g,A_0=a;\pi^*) = Q^*_{\tkernel}((s,h,g),a).
\end{multline*}

3. There exists $\epsilon>0$  %
with $\pi_{\lambda} \leq 1 \leq \epsilon^{-1} \pi_0$. We claim that for all $\bar{s} = (s,h,g) \in \bar{\mathcal{S}}_T$ it holds that
$$
V_{\tkernel}^{\pi_{\tkernel}}(\bar{s})\leq V_{\tkernel}^{\pi_0}(\bar{s}) \epsilon^{-h}.
$$
This is shown by induction. For $h=1$ we get for all $\bar{s} = (s,1,g) \in \bar{\mathcal{S}}_T$ that
\begin{align*}
V_{\tkernel}^{\pi_{\tkernel}}(\bar{s})
&=
\sum_{a\in\mathcal{A}} Q_{\tkernel}((s,1,g),a)
\pi_{\tkernel}(a|s,1,g)
\\
&\leq
\sum_{a\in\mathcal{A}} Q_{\tkernel}((s,1,g),a) \epsilon^{-1} \pi_0(a|s,1,g)
=
\frac{1}{\epsilon} V^{\pi_0}_{\tkernel}(s,1,g).
\end{align*}
If the claim holds for $h \geq 1$ then, for all $\bar{s} = (s,h+1,g) \in \bar{\mathcal{S}}_T$, it also holds for $h+1$:
\begin{align*}
V_{\tkernel}^{\pi_{\tkernel}}(\bar{s})
&=
\sum_{a\in\mathcal{A}}
( \sum_{s'\in\mathcal{S}} V_{\tkernel}^{\pi_{\tkernel}}(s',h,g)
\tkernel(s'|s,a) )
\pi_{\tkernel}(a|s,h+1,g)
\\
&\leq
\sum_{a\in\mathcal{A}}
( \sum_{s'\in\mathcal{S}} \epsilon^{-h}V_{\tkernel}^{\pi_{0}}(s',h,g)
\tkernel(s'|s,a) )
\epsilon^{-1} \pi_{0}(a|s,h+1,g)
=
\epsilon^{-h-1}
V_{\tkernel}^{\pi_0}(s,h+1,g)
.
\end{align*}
Since $\tkernel \mapsto V_{\tkernel}^{\pi_0}$ is continuous by point 1. we conclude that
$$
0 \leq V_{\tkernel}^{\pi_{\tkernel}} (\bar{s}^*) \leq \epsilon^{-h}  V_{\tkernel}^{\pi_0} (\bar{s}^*) \rightarrow \epsilon^{-h} V_{\tkernel_0}^{\pi_0} (\bar{s}^*) = 0
\;\text{as}\;  \tkernel \rightarrow \tkernel_0.
$$

4.
Since $\prob_\tkernel(\mathcal{T}=\tau;\pi)$ is polynomial
in the components of $\tkernel$ and $\pi$, it is continuous in $(\tkernel,\pi)$. Since the overall number of trajectories is finite the normalizing term $c(\tkernel,\pi)$ is continuous in $(\tkernel,\pi)$ as a finite sum of marginals of  $\prob_\tkernel(\mathcal{T}=\tau;\pi)$.
Further, $c(\tkernel,\pi)>0$ due to the condition \eref{eq:mucond} on the initial distribution $\bar{\mu}$.
Therefore, the segment probabilities $\prob_\tkernel(\Sigma=\sigma;\pi)$
are continuous in $(\tkernel,\pi)$.
The continuity of the marginals of $\prob_\tkernel(\Sigma=\sigma;\pi)$ follows
since they are composed by finite sums of continuous functions $\prob_\tkernel(\Sigma=\sigma;\pi)$.
The continuity of $c^{-1}(\tkernel,\pi)$ follows from the
continuity and positivity of $c(\tkernel,\pi)$.
\end{proof}

\subsection{Relative Continuity of \eUDRL{} Policies and Continuity of Related Values}
\label{sse:theoremfinite}

Policies generated by \eUDRL{} are not generally continuous in deterministic kernels (see figure~\ref{fig:discont:c} and example~\ref{ex:detpoint} in the appendix for details). As explained above discontinuities in policies can lead to discontinuous behavior of derived quantities, like the goal-reaching objective. Another key aspect is the non-uniqueness of optimal policies, which additionally complicates the analysis of continuity. However, a weaker notion of continuity still holds for \eUDRL{}'s policies, which will be useful in the study of continuity of emergent quantities. We will utilize the notion of relative weak limit of~\cite{strupl2021reward}, except that in the finite setting of the article at hand weak convergence coincides with ordinary point-wise convergence and we drop the predicate \lq\lq{}weak\rq\rq{}. For completeness we recall the underlying notion of convergence of probability distributions on finite quotient spaces, which we call \emph{relative convergence} in this article. The standard notion of continuity then applies to maps of the form $f:Y \rightarrow \Delta X$ (e.g.~policies) with respect to the quotient topology.
\def\aprob{P}
\begin{definition}
\label{de:rellimit}
Let $X$ be a finite set, let $F \subset X$ be a subset, let $\Delta X$ denote the probability simplex over $X$. 
\begin{enumerate}
    \item
We say that a sequence of probabilities $(\aprob_k)_{k\geq0}$, $\aprob_k\in\Delta X$ converges relatively to $F$, $\aprob_k\xrightarrow{F} \aprob$, if and only if it converges over the quotient $X/F$, i.e.~
$$
(\forall x \in X\setminus F): \aprob_k(x) \rightarrow \aprob(x) \quad\text{and}\quad \sum_{x\in F}\aprob_k(x) \rightarrow \sum_{x\in F}\aprob(x).
$$
\item Let $(Y,d)$ be a metric space. We say that $f:Y \rightarrow \Delta X$ is continuous at $y'\in Y$ relative to $F$ if and only if for every sequence $(y_k)_{k\geq0}$, $y_k\in Y$ that converges to $y'$ the sequence
$(f(y_k))_{k\geq0}$ converges to $f(y')$ relatively to $F$. In this case we will write:
$$f(y)\xrightarrow[]{F}f(y'),\ \textnormal{as}\ y\rightarrow y'.$$
\end{enumerate}
\end{definition}
\begin{remark}
Recall that the above continuity of $f$ at $y'$ relative to $F$ can be written equivalently using limit of a function, i.e.  
$$
(\forall x \in X\setminus F): [f(y)](x) \rightarrow [f(y')](x) \quad\text{and}\quad \sum_{x\in F}[f(y)](x) \rightarrow \sum_{x\in F}[f(y')](x)
\quad\text{as}\quad y \rightarrow y'.
$$
Here the functions $y \mapsto  [f(y)](x)$, $x \in X\setminus F$ and $y \mapsto \sum_{x\in F}[f(y)](x)$.
Furthermore, the sequential approach used in point 2 of the definition is equivalent to the $\epsilon$-$\delta$ formalism, since we are working in metric spaces.
\end{remark}

In what follows we will demonstrate that the mapping $\tkernel \mapsto \pi_{n,\tkernel}$ has a
relative limit that at deterministic kernels equals to optimal policies.
Thus $\pi_{n,\tkernel}$ are relatively continuous at deterministic kernels.
Before we will state a simple lemma which will be needed during the proof.

\begin{lemma}
\label{le:relproduct}
Assume the same situation as in definition~\ref{de:rellimit}, where $f$ is continuous at $y'$ relative to $F$. Moreover, consider a second function $h:Y\rightarrow \mathbb{R}^{X}$ that is continuous at $y'\in Y$ and assume $h(y')$ is constant on $F$. Then the following  holds as $y\rightarrow y'$:
$$
\sum_{x \in X} [f(y)](x)\cdot[h(y)](x) \rightarrow \sum_{x \in X}  [f(y')](x)\cdot[h(y')](x).
$$
\end{lemma}
\begin{proof} We employ the defining properties of convergence relative to $F$.
Expanding the left hand side yields
$$
\sum_{x \in X} [f(y)](x)[h(y)](x)
=
\sum_{x \in F} [f(y)](x)[h(y)](x)+
\sum_{x \in X\setminus F} [f(y)](x)[h(y)](x).
$$
Since convergence relative to $F$ behaves like point-wise convergence on $X\setminus F$ we have
$$
\sum_{x \in X\setminus F} [f(y)](x)\cdot[h(y)](x)
\xrightarrow[]{}
\sum_{x \in X\setminus F} [f(y')](x)\cdot[h(y')](x).
$$
On the other hand, let $(y_k)_{k\geq0}$, $y_k\in Y$ be a sequence that converges to $y'$. Since $h(y')=c$ is constant on $F$, choosing $k_0$ large enough, we have for all $k>k_0$ that $|{h(y_k)-c}|\leq\varepsilon$. It follows for any $k>k_0$ that
\begin{multline*}
\left|{\sum_{x \in F} [f(y_k)](x)[h(y_k)](x)-\sum_{x \in F} [f(y')](x)[h(y')](x)}\right|
\\
\leq c\left|{\sum_{x \in F}[f(y_k)](x)-\sum_{x \in F} [f(y')](x)}\right|+|X|\cdot\varepsilon,
\end{multline*}
where $|X|$ stands for the finite number of elements in $X$. Making use of $f(y)\xrightarrow[]{F}f(y')$ the right hand side converges to $\epsilon|X|$. Since we can choose $\epsilon$ arbitrary small the left-hand side of the inequality has to converge to 0. This completes the proof.
\end{proof}

The following theorem is the main result of this section. It demonstrates that \eUDRL{}-generated policies converge relatively to their sets of optimal actions. The result is stated for the \eUDRL{} recursion on $\Seg$ space~\eref{eq:seg}, but it can be modified mutatis mutandis for recursions on trailing or diagonal segments.

\begin{theorem}(Relative continuity of \eUDRL{} policies and continuity of value at deterministic kernels)
\label{le:detcont}
Let $\{\mathcal{M}_{\tkernel} : \tkernel \in (\Delta \mathcal{S})^{\mathcal{S}\times\mathcal{A}}\}$
and $\{\bar{\mathcal{M}}_{\tkernel} : \tkernel \in (\Delta \mathcal{S})^{\mathcal{S}\times\mathcal{A}}\}$ be compatible families of MDPs.  Let $\tkernel_0 \in (\Delta \mathcal{S})^{\mathcal{S}\times\mathcal{A}}$
be a deterministic kernel and let $(\pi_{n,\tkernel})_{n\geq0}$ with $\pi_0 >0$ be a 
\eUDRL{}-generated sequence of policies.
\begin{enumerate}
\item
For all $\bar{s} \in \theinterestingstates$ the policy $\pi_{n,\tkernel}(\cdot|\bar{s})$ is continuous relative to $\oactions(\bar{s})$ at $\tkernel_0$.
\item
For all $\bar{s} \in \theinterestingstates$ and for all $a\in \mathcal{A}$, the values $Q^{\pi_n}_\tkernel(\bar{s},a)$, $V^{\pi_n}_\tkernel(\bar{s})$ are continuous functions of $\tkernel$ at $\tkernel_0$.
\item There exists a sequence of neighborhoods $(U_{\delta_n}(\tkernel_0))_{n\geq0}$ with $U_{\delta_{n+1}}(\tkernel_0)\subset U_{\delta_n}(\tkernel_0)$ such that for all $n\geq 0$ it holds:
\begin{itemize}
\item $\pi_{n,\tkernel}(a|\bar{s})$ is separated from $0$ for all $(a,\bar{s},\tkernel) \in (\supp \num_{\tkernel_0,\pi_0} 
\cap (\mathcal{A} \times \supp \nu_{\tkernel_0,\pi_0} ))\times U_{\delta_n}(\tkernel_0)$, and
\item the probabilities
$\prob_\tkernel(\stag{H}=h,\stag{G}_0=g|\stag{S}_0=s,l(\Sigma)=h;\pi_n)$
are separated from $0$ for all $((s,h,g),\tkernel) \in \theinterestingstates \times U_{\delta_n}(\tkernel_0)$.
\end{itemize}
\end{enumerate}
\end{theorem}
\begin{proof} The proof proceeds by induction on the index $n$ of the \eUDRL{} iteration. 

\emph{Base case (n=0):} Since $\pi_0$ is positive, constant in $\tkernel$ and state and action spaces are finite it follows that $\pi_0$ is continuous, relatively continuous, and separated from $0$. Since $\pi_0$ is continuous in $\tkernel$, and values are continuous in $(\tkernel,\pi_0)$ by lemma~\ref{le:contoptbound} point 1, values are continuous in $\tkernel$ at $\tkernel_0$ (after composition with $\pi_0$). Choose $2 > \delta_0 > 0$ such that $\|\tkernel -\tkernel_0\|_{\max} < \frac{1}{2}$ for $\tkernel \in \overline{U_{\delta_0}(\tkernel_0)}$. From lemma~\ref{le:contoptbound} point 4.~it follows that
the segment distribution marginal $\prob_\tkernel(\stag{H}_0=h,\stag{G}_0=g,\stag{S}_0=s,l(\Sigma)=h;\pi_0)$ is continuous. Since $(s,h,g) \in \theinterestingstates$, it follows from lemma~\ref{le:suppstab} point 3.~that the marginal is positive. 
In summary, the marginal is a positive and continuous function on the compact set $\overline{U_{\delta_0}(\tkernel_0)}$, which implies that it has a positive minimum.

\emph{Induction:} Assume all three statements hold for $n$. We aim to prove them for $n+1$. By point 3.~for all $\lambda\in\overline{U_{\delta_n}(\tkernel_0)}$ we have that
lemma~\ref{le:suppstab} applies. Consequently choosing a fixed $(s,h,g) \in \theinterestingstates \subset \supp \den_{\tkernel,\pi_n} \cap \supp \nu_{\tkernel,\pi_n} $ we can use lemma~\ref{le:recrewrites} point 1 to get
\begin{align}
\pi_{n+1}(a|s,h,g) &= \frac{u(a)}{\sum_{a\in \mathcal{A}}u(a)}\ \textnormal{with}\nonumber\\
u(a)&:=\sum_{N \geq h'\geq h, g' \in \mathcal{G}}
\prob_\tkernel (\rho(\stag{S}_h)=g|\stag{A}_0=a, \stag{S}_0, \stag{H}_0=h',
\stag{G}_0=g', l(\Sigma)=h; \pi_n)\nonumber\\
&\qquad\qquad\qquad\cdot\pi_n(a|s,h',g')
\prob_\tkernel (\stag{H}_0=h',
\stag{G}_0=g'| \stag{S}_0, l(\Sigma)=h; \pi_n).\label{eq:udef}
\end{align}
To demonstrate point 1.~we begin by providing lower and upper estimates on $u$. Using lemma~\ref{le:contoptbound} point 2.~we have an upper bound
\begin{equation}
u(a) \leq 
\sum_{N \geq h' \geq  h, g' \in \mathcal{G}}\hspace*{-1em}
Q^{*}_\tkernel((s,h,g),a)
\cdot
1
\cdot
\prob_\tkernel (\stag{H}_0=h',
\stag{G}_0=g'| \stag{S}_0, l(\Sigma)=h; \pi_n)
\leq 
Q^{*}_\tkernel((s,h,g),a)
\label{eq:uub}
\end{equation}
and the lower bound
\begin{equation}
\begin{aligned}
u(a) &\geq
Q^{\pi_n}_\tkernel((s,h,g),a) \pi_n(a|s,h,g)
\prob_\tkernel (\stag{H}_0=h, \stag{G}_0=g| \stag{S}_0, l(\Sigma)=h; \pi_n)\\
&\geq \begin{cases}Q^{\pi_n}_\tkernel((s,h,g),a) \epsilon \;\text{for}\; a\in \oactions(s,h,g),\\
0 \; \text{otherwise},
\end{cases}
\end{aligned}
\label{eq:ulb}
\end{equation}
where the existence of $\epsilon>0$ follows by point 3.~of the theorem by the induction assumption.
Using these bounds we now show that $\pi_{n+1}(a|s,h,g) \xrightarrow{\oactions(s,h,g)} \pi^*_{\tkernel_0}$ as $\tkernel \rightarrow \tkernel_0$. For this it is sufficient to show that $\pi_{n+1}(a|s,h,g) \rightarrow 0$ for all $a \notin \oactions(s,h,g)$ because $\pi_{n+1}(\oactions(s,h,g)|s,h,g)\rightarrow1$ follows automatically. We show that $u(a) \rightarrow 0$, $a\notin\oactions(s,h,g)$ and that the denominator $\sum_{a\in \mathcal{A}}u(a)$ is separated from $0$. First, $u(a) \rightarrow 0$, $a \notin \oactions(s,h,g)$ since the upper bound $Q^{*}_\tkernel((s,h,g),a) \rightarrow Q^{*}_{\tkernel_0}((s,h,g),a) = 0$ for $a \notin \oactions(s,h,g)$ by the continuity of $Q^{*}_\tkernel((s,h,g),a)$. By assumption $Q^{\pi_n}_\tkernel((s,h,g),a)$ is continuous and we can choose $\delta_{n} \geq \delta_{n+1}>0$ such that
$Q^{\pi_n}_\tkernel((s,h,g),a) \geq \frac{1}{2}$ for $a,(s,h,g),\tkernel \in (\supp \num_{\tkernel_0,\pi_0} \cap (\mathcal{A} \times \supp \nu_{\tkernel_0,\pi_0})) \times U_{\delta_{n+1}}(\tkernel_0)$. Finally, using the lower bound we conclude that
$$
\sum_{a\in \mathcal{A}} u(a) \geq
\epsilon \sum_{a\in \oactions(s,h,g)} Q^{\pi_n}_\tkernel((s,h,g),a)
\geq \epsilon \frac{1}{2} > 0.
$$
Thus $\pi_{n+1}(a|s,h,g)$ is relatively continuous in $\tkernel_0$. Using both bounds on $u$ we show that $\pi_{n+1}(a|s,h,g)$ is also
separated form $0$ for all $(s,h,g),a,\tkernel \in (\supp \num_{\tkernel_0,\pi_0} \cap (\mathcal{A} \times \supp \nu_{\tkernel_0,\pi_0})) \times U_{\delta_{n+1}}(\tkernel_0)$,
$$
\pi_{n+1}(a|s,h,g)
= 
\frac{u(a)}{\sum_{a\in \mathcal{A}}u(a)}
\geq
\frac{Q^{\pi_n}_\tkernel((s,h,g),a) \epsilon}
{
\sum_{a\in \mathcal{A}} Q^{*}_\tkernel((s,h,g),a)
}
\geq
\frac{\frac{1}{2}\epsilon}
{
\sum_{a\in \mathcal{A}} 1
}
\geq
\frac{1}{2|\mathcal{A}|}\epsilon > 0
.
$$
To demonstrate the continuity of state/action-values we will employ an additional induction argument (for fixed $n$) over the remaining horizon
$N \geq h \geq 1$.

\emph{Sub-induction base case ($h=1$):} For $h=1$, $(s,1,g)\in \bar{S}_T$, $a\in \mathcal{A}$ the action-value function
$$
Q^{\pi_{n+1}}_{\tkernel}((s,1,g),a)
=
\sum_{s'\in \rho^{-1}(\{g\})} \tkernel(s'|s,a)
$$
is continuous in $\tkernel$. 

\emph{Sub-induction:} Fix the remaining horizon $N \geq h' \geq 1$.
Assume that for all $(s,h,g)\in \theinterestingstates$, $h=h'$, $a\in\mathcal{A}$ the action values $Q^{\pi_{n+1}}_{\tkernel}((s,h,g),a)$ are continuous in $\tkernel$ at $\tkernel_0$. It remains to prove continuity of $V^{\pi_{n+1}}_{\tkernel}(s,h,g)$
for all $(s,h,g)\in \theinterestingstates$, $h=h'$ and of $Q^{\pi_{n+1}}_{\tkernel}((s,h+1,g),a)$ for all $(s,h+1,g)\in \theinterestingstates$, $h=h'$, $a\in\mathcal{A}$ at $\tkernel_0$.
From lemma~\ref{le:relproduct} and the relative continuity of $\pi_{n+1}$ we obtain
\begin{align*}
\lim_{\tkernel \rightarrow \tkernel_0}
V^{\pi_{n+1}}_{\tkernel}(s,h,g)
&=
\lim_{\tkernel \rightarrow \tkernel_0}
\sum_{a\in \mathcal{A}}
Q^{\pi_{n+1}}_{\tkernel}((s,h,g),a)
\pi_{n+1}(a|s,h,g)
\\
%
&=
\sum_{a\in \oactions(s,h,g)}
\pi^{*}_{\tkernel_0}(a|s,h,g)
=1 = V^{*}_{\tkernel_0}(s,h,g).
\end{align*}
Thus the continuity $V^{\pi_{n+1}}_{\tkernel}(s,h,g)$ in $\tkernel_0$ follows.
Now for all $(s,h+1,g)\in \theinterestingstates$, $h=h'$, $a\in\mathcal{A}$
the limit of the action-value function can be computed as
$$
\lim_{\tkernel \rightarrow \tkernel_0}
Q^{\pi_{n+1,\tkernel}}_{\tkernel}((s,h+1,g),a)
=
\lim_{\tkernel \rightarrow \tkernel_0}
V^{\pi_{n+1,\tkernel}}_{\tkernel}(s',h,g) \tkernel(s'|s,a)
,
$$
where we used that $\tkernel_0$ is deterministic, chose $s'$ so that $\tkernel_0(s'|s,a) = 1$ and made use of the fact that $\lambda(s''|s,a)\rightarrow0$ as $\tkernel\rightarrow\tkernel_0$ for $s''\neq s'$. It remains to show the continuity of $V^{\pi_{n+1,\tkernel}}_{\tkernel}(s',h,g)$ and that
$V^{\pi_{n+1,\tkernel_0}}_{\tkernel_0}(s',h,g) = Q^{\pi_{n+1,\tkernel_0}}_{\tkernel_0}((s,h+1,g),a) = Q^{*}_{\tkernel_0}((s,h+1,g),a)$.
First assume $a \in \oactions(s,h+1,g)$ then form lemma~\ref{le:suppstab} point 4.~it holds that $(s',h,g) \in \theinterestingstates$. But we have just proved $V_{\tkernel}^{\pi_{n+1,\tkernel}}(s',h,g) \rightarrow V_{\tkernel_0}^{\pi_{n+1,\tkernel_0}}(s',h,g) = 1$, $\tkernel \rightarrow \tkernel_0$ for such states.
Finally assume $a \notin \oactions(s,h+1,g)$ then from lemma~\ref{le:suppstab} point 5.~we conclude that $
Q_{\lambda_0}^{\pi_{n+1,\lambda_0}}((s,h+1,g),a) = 0 = V_{\lambda_0}^{\pi_{n+1,\lambda_0}}((s',h,g))\cdot 1. 
$
Then lemma~\ref{le:contoptbound} point 3.~implies
$$
V_{\lambda}^{\pi_{n+1,\lambda}}((s',h,g)) \rightarrow V_{\lambda_0}^{\pi_{n+1,\lambda_0}}((s',h,g)) = 0 
\; \text{for} \; \tkernel \rightarrow \tkernel_0
.
$$
This completes the sub-induction.

It remains to show that the probabilities $\prob_{\tkernel} (\stag{H}_0=h,\stag{G}_0=g \mid \stag{S}_0=s,l(\Sigma)=h,\rho(\stag{S}_h)=g;\pi_{n+1})$ remain separated from $0$ on a neighborhood of $\tkernel_0$. Notice that the segment distribution marginal
$\prob_{\tkernel} (\stag{H}_0=h,\stag{G}_0=g,\stag{S}_0=s,l(\Sigma)=h,\rho(\stag{S}_h)=g;\pi_{n+1}) > 0$
is continuous in $(\tkernel,\pi_{n+1,\tkernel})$ (viewed as independent variables)
but $\pi_{n+1,\tkernel}$ is not generally continuous in $\tkernel$ at $\tkernel_0$
(but only relatively continuous). Hence, it would not be sufficient to argue that \lq\lq{}a positive continuous function on a compact must have a positive minimum\rq\rq{} and use lemma~\ref{le:suppstab} point 3.~and lemma~\ref{le:contoptbound} point 4. This motivates the following, more elaborate, approach.

From lemma~\ref{le:suppstab} points 2., 3.~we have that $(s,h,g) \in \theinterestingstates = \supp \den_{\tkernel_0,\pi_{n+1}} \cap \supp \nu_{\tkernel_0,\pi_{n+1}}$ and that
$$
\prob_{\tkernel_0}(\stag{S}_0=s,\stag{H}_0=h,\stag{G}_0=g,l(\Sigma)=h;\pi_{n+1}) > 0
.
$$
It follows that there exists a segment $\sigma$ that makes the above marginal of the 
segment distribution positive. In turn there must be a trajectory $\tau$ with positive
probability.
We end up with the following inequalities
$$
\prob_{\tkernel_0}(\stag{S}_0=s,\stag{H}_0=h,\stag{G}_0=g,l(\Sigma)=h;\pi_{n+1})
\geq
\prob_{\tkernel_0}(\Sigma=\sigma;\pi_{n+1}) \geq
\frac{\prob_{\tkernel_0}(\mathcal{T}=\tau;\pi_{n+1})}
{c_{\tkernel_0}}
.
$$
The normalizing factor of the segment distribution $c_{\tkernel}$
can be bounded above $c_{\tkernel} \leq N(N+1)/2$ (see equation~\eref{eq:cupperbound}). Since the transition probabilities in $\prob_{\tkernel_0}(\mathcal{T}=\tau;\pi_{n+1})$ are $1$ as $\tkernel_0$ is deterministic and $\|\tkernel-\tkernel_0\|_{\max} < \frac{1}{2}$ on $U_{\delta_{n+1}}(\tkernel_0)\subset U_{\delta_0}(\tkernel_0)$ we bound them on $U_{\delta_{n+1}}(\tkernel_0)$ by $\frac{1}{2}$
from below. Similarly, since $\pi_{n+1,\tkernel_0}=\pi^{*}$ chooses
actions from the optimal action set $\oactions(\cdot)$, where 
$\pi_{n} > 0$ is separated from $0$ on $U_{\delta_{n+1}}(\tkernel_0)$,
we conclude
\begin{align*}
\prob_{\tkernel}(\stag{S}_0=s,\stag{H}_0=h,\stag{G}_0=g,l(\Sigma)=h;\pi_{n+1})
&\geq
\frac{
\prob_{\tkernel}(\mathcal{T}=\tau;\pi_{n+1})
}{c_\tkernel}=
\\
\frac{
\bar{\mu}(s_0,h_0,g)\prod_{t=0}^{N-1} \tkernel(s_{t+1}|s_t,a_t) \pi_{n+1}(a_t|s_t,h_0-t,g)
}
{c_\tkernel}
&\geq
\underbrace{
\frac{2
\bar{\mu}(s_0,h_0,g)
(\frac{1}{2})^N
(\epsilon')^N
}
{N(N+1)}
}_{C(s,h,g) :=}
>0
,
\end{align*}
where $\epsilon'>0$ denotes the respective lower bound. %
This implies that the probabilities $\prob_{\tkernel}(\stag{H}_0=h,\stag{G}_0=g\mid\stag{S}_0=s,l(\Sigma)=h;\pi_{n+1})$ are well-defined and bounded from below by $C(s,h,g)$.
Since the set $\theinterestingstates$ is finite we can choose a common positive lower bound as the minimum of all positive lower bounds.
\end{proof}
As an illustration of theorem~\ref{le:detcont} and of the behavior of the continuity of \eUDRL{}-generated policies, figure~\ref{fig:oscillationsPI} shows plots of the quantity $\min_{\bar{s}\in\theinterestingstates}\pi_n(\oactions(\bar{s})|\bar{s})$ in case of a random walk on $\mathbb{Z}_3$. All details of the construction of the MDP underlying figure~\ref{fig:oscillationsPI} can be found in the appendix~\ref{ap:examples}, example~\ref{ex:oscillations}. The figure shows how $\min_{\bar{s}\in\theinterestingstates}\pi_n(\oactions(\bar{s})|\bar{s})$ depends on $n$ for varying distances from a deterministic kernel (as measured by $\delta=\|\tkernel-\tkernel_0\|_1$) and varying initial conditions. The initial conditions are given by the starting policy for \eUDRL{} and are reflected by the offset of the respective graph on the vertical axis (at $n=0$). Different initial conditions are depicted using different colors. The dependency $\min_{\bar{s}\in\theinterestingstates}\pi_n(\oactions(\bar{s})|\bar{s})$ on $\delta$ is illustrated by different graphs, where the increase of $\min_{\bar{s}\in\theinterestingstates}\pi_n(\oactions(\bar{s})|\bar{s})$ as $\delta$ decreases is consistent with the established continuity of policy $\pi_{n}$ in $\tkernel$. It is also visible that for small values of $\delta$ the \eUDRL{} iteration leads to an increase of $\min_{\bar{s}\in\theinterestingstates}\pi_n(\oactions(\bar{s})|\bar{s})$, whereas for large values of $\delta$ the \eUDRL{} iteration leads to a deterioration of performance. The dependency of \eUDRL{}'s overall performance on the initial conditions at fixed $\delta$ can be estimated from the graphs by comparing the asymptotic behavior when $n$ is increasing. 
\begin{figure}[!htb]
\centering
\includegraphics[width=0.75\linewidth]{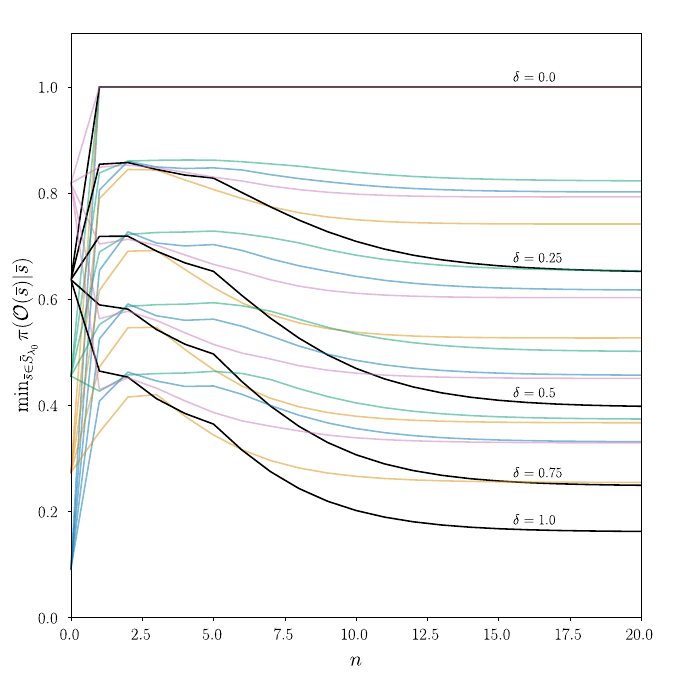}
\caption{Illustration of the behavior of $\min_{\bar{s}\in\theinterestingstates}\pi_n(\oactions(\bar{s})|\bar{s})$ when varying the distance to a deterministic kernel and varying the initial policy for a random walk on $\mathbb{Z}^3$. The legend shows the distance $\delta$ of the respective graph to the deterministic kernel. Each initial condition is depicted by a different color and plotted for various distances $\delta$. One particular initial condition is highlighted in black.}
\label{fig:oscillationsPI}
\end{figure}

\subsection{Extending the Continuity Results to Other Segment Sub-Spaces}
\label{sse:extendingfinite}

To cover algorithms like ODT, RWR (and some specific variants of \eUDRL{}) a modification of theorem~\ref{le:detcont} is needed that restricts the recursion (cf.~section~\ref{subsection:eUDRL} and following) to $\SegTrail$ or $\SegDiag$ subspaces.
\begin{theorem}
\label{le:detcontDiagTrail}
(Relative continuity of \eUDRL{} policies and continuity of values in deterministic kernels)
Theorem~\ref{le:detcont} remains valid under the renaming $\pi_n \rightarrow \pi_n^{\diag/\trail}$.
\end{theorem}
Notice that the renaming applies to policies only. It ignores the sets $\theinterestingstates$ and $\supp \num_{\tkernel_0,\pi_0} \cap (\mathcal{A} \times \supp \nu_{\tkernel_0,\pi_0})$, which is justified by the equations~\eqref{eq:saisthesame} and~\eqref{eq:isthesame}. There is also no need to investigate trailing/diagonal variants of the probability in theorem~\ref{le:detcont}, point 3 (by conditioning on the events $\stag{H}_0=l(\Sigma)$, $\stag{G}_0=\rho(\stag{S}_h)$). The probability to bound remains the same for all segment sub-spaces. This is a consequence of lemma~\ref{le:recrewrites}, particularly equation~\eqref{eq:udef}, which permits rewriting the \eUDRL{} recursion for all segment sub-spaces exclusively using terms that appear in the recursion rewrite for $\Seg$.
More broadly, lemma~\ref{le:recrewrites} simplifies proofs in $\Seg^{\diag/\trail}$ considerably since it allows us to restrict to quantities without the conditioning on $\stag{H}_0=l(\Sigma)$, $\stag{G}_0=\rho(\stag{S}_h)$ in all three versions of the proof.

\begin{proof} We follow the proof of theorem~\ref{le:detcont} and only highlight differences. First, we employ lemma~\ref{le:detoptDiagTrail} for $\Seg^{\diag/\trail}$ instead of lemma~\ref{le:detopt}, which we employed for $\Seg$. Second, as explained above, there is no need to bound diagonal/trailing variants of the probability in point 3. Third, the fact that we restricted the proof to $\bar{s} \in \theinterestingstates$ (in \eref{eq:udef}) implies that not only $\pi_{n+1,\tkernel}(\cdot|\bar{s})$ is defined  but also $\pi_{n+1,\tkernel}^{\trail}(\cdot|\bar{s})$ and
$\pi_{n+1,\tkernel}^{\mathrm{diag}}(\cdot|\bar{s})$ are defined, see \eref{eq:isthesame} and lemma~\ref{le:suppstabTrailDiag} point 2. For the recursion in $\SegTrail$ or $\SegDiag$ we have to modify $u$ (compare \eref{eq:udef} and \eref{eq:segtrailing}, \eref{eq:segdiag}) choosing
\begin{align*}
u^{\trail}(\pi_n^{\trail}) &:= 
\sum_{g' \in \mathcal{G}}
\prob_\tkernel (\rho(\stag{S}_h)=g|\stag{A}_0=a, \stag{S}_0, \stag{H}_0=h,
\stag{G}_0=g', l(\Sigma)=h; \pi_n^{\trail}  )
\\
&\qquad\qquad\qquad\qquad\cdot\pi_n^{\trail}(a|s,h,g')
\prob_\tkernel (\stag{H}_0=h,
\stag{G}_0=g'| \stag{S}_0, l(\Sigma)=h; \pi_n^{\trail}),
\\
u^{\mathrm{diag}}(\pi_n^{\diag}) &:= 
\prob_\tkernel (\rho(\stag{S}_h)=g|\stag{A}_0=a, \stag{S}_0, \stag{H}_0=h,
\stag{G}_0=g, l(\Sigma)=h; \pi_n^{\diag}  )
\\
&\qquad\qquad\qquad\qquad\cdot\pi_n^{\diag}(a|s,h,g)
\prob_\tkernel (\stag{H}_0=h,
\stag{G}_0=g| \stag{S}_0, l(\Sigma)=h; \pi_n^{\diag})
.
\end{align*}
Notice that given a policy $\pi_n$ it holds that
$$
u^{\mathrm{diag}}(\pi_n)\leq u^{\trail}(\pi_n)
\leq u(\pi_n).
$$
Since the introduced upper bound \eref{eq:uub} bounds $u(\pi_n)$ 
it also applies for $u^{\diag}(\pi_n)$ and $u^{\trail}(\pi_n)$.
Similarly, since 
all non-diagonal terms were lower bounded by zero,
the introduced lower bound \eref{eq:ulb} bounds $u^{\mathrm{diag}}(\pi_n)$ and
consequently it bounds also $u^{\trail}(\pi_n)$.
Thus the upper and lower estimates that are used in the proof of theorem~\ref{le:detcont} apply 
also for $u^{\trail}(\pi_n)$ and $u^{\diag}(\pi_n)$. Thus apart from the change of the definition of $u$ the proof remains unchanged.
\end{proof}
We will apply this method to derive universal bounds by lower bounding $u^{\diag}$
and upper bounding $u$ also in later proofs when studying the \eUDRL{}-type recursions.

\subsection{Continuity of the Goal-Reaching Objective for a Finite Number of Iterations}
\label{sse:Jfinite}
\begin{corollary}(Continuity of goal-reaching objective at deterministic points)
\label{le:detJcont}
Let $(\pi_{n,\tkernel})_{n\geq0}$ with $\pi_0 >0$ be a sequence of \eUDRL{}-generated policies.
Let $\tkernel_0 \in (\Delta \mathcal{S})^{\mathcal{S}\times\mathcal{A}}$
be a deterministic kernel.
Then the goal-reaching objective, viewed as a function of $\lambda$,
$J_{\tkernel}^{\pi_{n,\tkernel}} =
\sum_{\bar{s} \in \bar{\mathcal{S}}_T} \bar{\mu}(\bar{s}) V_{\tkernel}^{\pi_{n,\tkernel}}$
is continuous at $\lambda_0$ for all $n\geq 0$.
\end{corollary}
\begin{proof}
The goal-reaching objective can be written as follows
$$
J_{\tkernel}^{\pi_{n,\tkernel}}
=
\sum_{\bar{s} \in \supp \bar{\mu}} \bar{\mu}(\bar{s}) V_{\tkernel}^{\pi_{n,\tkernel}}
= 
\sum_{\bar{s} \in \supp \bar{\mu} \setminus \theinterestingstates } \bar{\mu}(\bar{s}) V_{\tkernel}^{\pi_{n,\tkernel}}
+
\sum_{\bar{s} \in \supp \bar{\mu} \cap \theinterestingstates}
\bar{\mu}(\bar{s}) V_{\tkernel}^{\pi_{n,\tkernel}},
$$%
where the continuity of the second summand follows from theorem~\ref{le:detcont}, point 2. For the first summand we prove that for all $\bar{s} \in \supp \bar{\mu} \setminus \theinterestingstates = \supp \bar{\mu} \setminus \supp \den_{\tkernel_0,\pi_0}$ it holds that
$V_{\tkernel_0}^{\pi_0}(\bar{s}) = 0$. Continuity then follows
from lemma~\ref{le:contoptbound}, point 3. For the sake of contradiction assume that $V_{\tkernel_0}^{\pi_0}(\bar{s}) > 0$ for a certain $\bar{s} = (s,h,g) \in \supp \bar{\mu} \setminus \supp \den_{\tkernel_0,\pi_0}$ (recall that $\supp \bar{\mu} \subset \bar{S}_T$).
This means that there exists a trajectory $\tau = \bar{s}_0, a_0, \ldots, \bar{s_h}$ with $\bar{s} = (s,h,g) = \bar{s}_0$, $\rho(s_h) = g$ and
$\prob_{\lambda_0} (A_0= a_0, \ldots, \rho(S_h) = g | \bar{S}_0 = \bar{s} ; \pi_0) > 0$. Since $\bar{\mu}(\bar{s}) > 0$ we conclude $\prob_{\lambda_0} (\mathcal{T}=\tau ; \pi_0) > 0$ but this contradicts $\bar{s} \notin \supp \den_{\tkernel_0,\pi_0}$.
\end{proof}
We have shown that the goal-reaching objective derived from \eUDRL{}-generated policies is continuous for any iteration $n \geq 0$. This implies that if an environment possesses a transition kernel that is close to a deterministic kernel, then the \eUDRL{} generated policy
$\pi_{n,\tkernel}$, $n\geq 1$ (with initial condition $\pi_0>0$) has
a goal-reaching objective $J_{\tkernel}^{\pi_{n,\tkernel}}$
close to the optimal goal-reaching objective $J_{\tkernel}^{\pi_{\tkernel}^*}$. We summarize this observation in a corollary.
\begin{corollary}\label{le:detJnearopt}(Near-optimal behavior near deterministic kernels)
Let $\tkernel_0$ be a deterministic transition kernel and let $(\pi_{n,\tkernel})_{n\geq0}$ with $\pi_0 >0$ be a \eUDRL{}-generated sequence of policies. Then for any fixed $n$ and all $\epsilon > 0$ there exists $\delta >0$ such that if $\tkernel \in U_{\delta}(\tkernel_0)$ then
$$
|J_{\tkernel}^{\pi_{n,\tkernel}} - J_{\tkernel}^{\pi_{\tkernel}^*}|
< \epsilon.
$$
\end{corollary}
\begin{proof}
Due to continuity of the optimal values in $\tkernel$ (see 
lemma~\ref{le:contoptbound}), we can choose $\delta_1>0$ such that
$$|J_{\tkernel}^{\pi_{\tkernel}^*} - J_{\tkernel_0}^{\pi_{\tkernel_0}^*}|
< \frac{\epsilon}{2}
.$$
By lemma~\ref{le:detopt} the \eUDRL{}-generated policies are optimal for deterministic kernels when $n \geq 1$. By the continuity of $J_{\tkernel}^{\pi_{n,\tkernel}}$ at
$\tkernel_0$ (see lemma~\ref{le:detJcont}), we can 
chose $\delta_2 >0 $ such that for all $\tkernel \in U_{\delta_2}(\tkernel_0)$ it holds that
$$
\frac{\epsilon}{2} > |J_{\tkernel}^{\pi_{n,\tkernel}} - J_{\tkernel_0}^{\pi_{n,\tkernel_0}}|
=
|J_{\tkernel}^{\pi_{n,\tkernel}} - J_{\tkernel_0}^{\pi_{\tkernel_0}^*}|
.
$$
By the triangle inequality and choosing 
$\delta := \min \{\delta_1,\delta_2 \}$ we obtain
$
|J_{\tkernel}^{\pi_{n,\tkernel}} - J_{\tkernel}^{\pi_{\tkernel}^*}|
<\frac{\epsilon}{2} + \frac{\epsilon}{2}=\epsilon$.
\end{proof}

This says that $J_{\tkernel}^{\pi_{n,\tkernel}}$ has a small error to optimality in a small neighborhood around a deterministic transition kernel for all finite $n \geq 1$. However, the lemma provides no information on how for a given level of error the size of the neighborhood $\delta$ (for given fixed $\epsilon$) depends on the iteration. As shown in~\citep{strupl2022upsidedown} by explicit counterexamples there exist non-deterministic environments, where \eUDRL{}-generated policies remain constant for $n\geq1$, such that there is no improvement at all. The same article shows that even if initialized with the optimal policy a priori, \eUDRL{} might produce non-optimal policies such that there is no monotonic improvement.

An illustration of the behavior of the goal-reaching objective for \eUDRL{}-generated policies in case of a random walk on $\mathbb{Z}_3$ is shown in the figure~\ref{fig:oscillations}. All details of the construction of the MDP underlying figure~\ref{fig:oscillations} can be found in the appendix~\ref{ap:examples}, example~\ref{ex:oscillations}. Figure~\ref{fig:oscillations:a} shows the dependency of $J_\tkernel^{\pi_{n,\tkernel}}$ on the iteration for varying distances from a deterministic kernel (as measured by $\delta=\|\tkernel-\tkernel_0\|_1$) and varying initial conditions. The initial conditions are given by the starting policy for \eUDRL{} and are reflected by the offset of the respective graph on the vertical axis (at $n=0$). Different initial conditions are depicted using different colors. The dependency $J_\tkernel^{\pi_{n,\tkernel}}$ on $\delta$ is illustrated by different graphs, where the increase of values $J_\tkernel^{\pi_{n,\tkernel}}$ as $\delta$ decreases is consistent with the established continuity of $J_\tkernel^{\pi_{n,\tkernel}}$ in $\tkernel$. It is also visible that for small values of $\delta$ the \eUDRL{} iteration leads to an increase of $J_\tkernel^{\pi_{n,\tkernel}}$, whereas for large values of $\delta$ the \eUDRL{} iteration leads to a deterioration of performance. The dependency of \eUDRL{}'s overall performance on the initial conditions at fixed $\delta$ can be estimated from the graphs by comparing the asymptotic behavior when $n$ is increasing. Figure~\ref{fig:oscillations:b} shows a closer resolution of one of the graphs of figure~\ref{fig:oscillations:a} for $\delta=0.25$. It is visible that two distinct regimes are present in the performance of \eUDRL{} in this example, with a sharp increase for small values of $n$ and a slow subsequent deterioration.
\begin{figure}[!htb]
	\centering
	\begin{subfigure}{0.49\linewidth}
		\includegraphics[width=\linewidth]{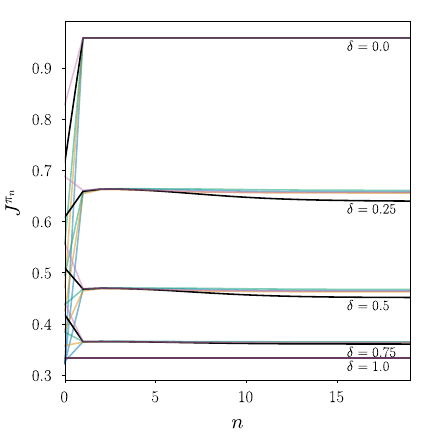}
		\caption{Goal-reaching objective}
  \label{fig:oscillations:a}
	\end{subfigure}
	\begin{subfigure}{0.49\linewidth}
		\includegraphics[width=\linewidth]{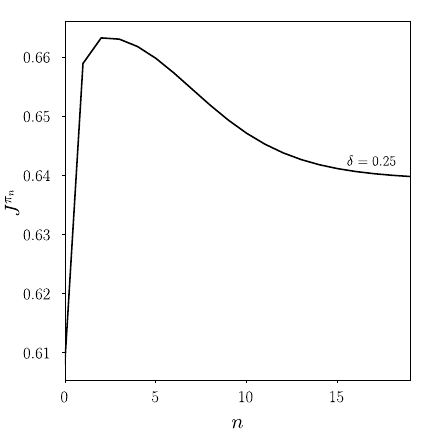}
		\caption{Goal-reaching objective - closer zoom}
   \label{fig:oscillations:b}
	\end{subfigure}
	\caption{Illustration of the continuity of the goal-reaching objective for varying distance to a deterministic kernel and varying initial policy for a random walk on $\mathbb{Z}^3$. The legend shows the distance $\delta$ of the respective graph to the deterministic kernel. Each initial condition is depicted by a different color and plotted for various distances $\delta$. One particular initial condition is highlighted in black in (a). One particular example is highlighted in closer zoom in (b).}
 \label{fig:oscillations}
\end{figure}
Furthermore, considering the asymptotic behavior present in figure~\ref{fig:oscillations} for different values of $\delta$ it appears that the set of accumulation points of the sequences of values of the goal-reaching objective also exhibits a form of continuity in $\tkernel$. The proofs of continuity properties of the sets of accumulation points of \eUDRL{}-generated policies and the associated goal-reaching objectives is the content of the coming section. 
\section{The Continuity of \eUDRL{} at Deterministic Transition Kernels as the Number of Iterations Approaches Infinity}
\label{se:continfty}
The main objective of this section is to investigate the relative continuity of \eUDRL{}-generated policies, the associated values and goal-reaching objective as the number of iterations approaches infinity. The rigorous asymptotic analysis of \eUDRL{}-type recursion schemes is open in the literature. A major hurdle is that it is unknown if, in general,  \eUDRL{} converges to a limiting policy. Since all policies form a compact set we will focus on the set of accumulation points of \eUDRL{}-generated policies instead. For this we will introduce a notion of \lq\lq{}relative continuity of sets of accumulation points of policies\rq\rq{}. Although, we do not provide a complete analysis of the continuity of asymptotic polices and values, we contribute significantly to the progress of the theory of 
\lq\lq{}RL as SL\rq\rq{} algorithms providing the first rigorous account of the topic.
%
The study of continuity of \eUDRL{} addresses the algorithm's asymptotic stability, i.e.~the question whether small changes within the environment entail a small change in the overall learning performance. Observe that the largest values of the goal-reaching objective in figure~\ref{fig:oscillations} seem to be reached roughly at the first to third iteration. Increasing the number of iterations further results in an effective deterioration of performance. This raises concerns on whether
the \eUDRL{} recursion scheme makes sense at all. Thinking along this line one might ask for guarantees on the behavior of the accumulation points of \eUDRL{}-generated policies and derived quantities. Since optimal values are always continuous (see lemma~\ref{le:contoptbound} point 1.) and \eUDRL{}-values are optimal for deterministic kernels (both in case of a finite number of iterations and also in the limit) a discontinuity of \eUDRL{}-values at $\tkernel_0$ (a deterministic kernel) would imply a discontinuous error between \eUDRL{}-values and optimal values in the limit (in the sense that $\|\tkernel-\tkernel_0\|_1\leq\delta$ but the values have a positive separation). Thus the set of accumulation points would show a discontinuous behavior. Since this deficiency does not occur at a finite number of iterations (see Theorem~\ref{le:detcont}), this would imply that \eUDRL{} is asymptotically unstable, making the effort to achieve any limiting behavior via \eUDRL{} a questionable pursuit.

As we are working in a metric space (rather than an abstract topological space), establishing a proof of continuity is tantamount to the derivation of error bounds on \eUDRL{}-generated quantities. Given a bound $\delta$ on the distance $\|\tkernel-\tkernel_0\|_1$ one could determine an upper bound $\epsilon$ on the error of, e.g., the goal-reaching objective. Building on this observation, the estimates derived during our exploration of continuity can help determine whether it is worthwhile to continue with the \eUDRL{} iteration to achieve a specific level of error. This observation extends mutatis mutandis to GCSL, ODT and other \lq\lq{}RL as SL\rq\rq{} approaches.

We have no estimates on the change of the visitation distribution
on the set of optimal policies (in contrast to, e.g., the action value function which is constant) to our disposal, which constitutes a major technical hurdle. We address this by restricting our discussion to the following special cases: \textit{(1)} the case when the support of the initial distribution includes
$\theinterestingstates$, which will be discussed in section~\ref{sse:specmu}, and \textit{(2)} the case when an optimal policy is unique on $\theinterestingstates$, which will be discussed in section~\ref{sse:specM1}.

\subsection{Asymptotic Continuity for Special Initial Distributions}
\label{sse:specmu}

The main result of this section is to demonstrate that the set of accumulation points of the sequence of \eUDRL{}-generated policies is relatively continuous at a deterministic kernel $\tkernel_0$ if $\theinterestingstates \subset
\supp \bar{\mu}$ (see theorem~\ref{le:limdetcont}). Similar to the preceding sections, the core idea that enables our asymptotic continuity analysis is to derive bounds for the \eUDRL{} policy recursion. This is achieved by bounding various individual terms occurring in the recursion, like $Q$-values (see lemmas \ref{le:opthbound} and \ref{le:contQfac}) and visitation probabilities (lemma~\ref{le:alpha}). However, the analysis must be carried out more carefully here as the individual estimates for the numerator and denominator of the policy recursion turn out to be insufficient. This leads us to develop a joint analysis of numerator and denominator during the course of the proof of the main result of this section, theorem~\ref{le:limdetcont}. Unsurprisingly, bounding the policy recursion via bounds to the respective individual terms can be achieved through monotonicity bounds for the underlying rational function (see remark~\ref{re:monotonicity}), which plays a key role in our discussion and will be used repeatedly.
This finally leads to an estimation of the \eUDRL{} recursion through a dynamical system (which we introduce in lemma~\ref{le:f}). This dynamical system possesses a unique fixed point that bounds the accumulation points of \eUDRL{}'s policy recursion and enjoys desirable properties as $\tkernel \rightarrow \tkernel_0$.

We derive some preliminary lemmata in section~\ref{ssse:specmulemmas}. The main theorem is contained in section~\ref{ssse:specmutheorem}. In section~\ref{ssse:extendingfinite} we show how it can be extended to the $\Seg^{\diag/\trail}$ segment spaces. As a corollary, we study the continuity of accumulation point sets also for other quantities like values and the goal-reaching objective in section~\ref{ssse:specmuaccbound}. This yields upper bounds on the errors of these quantities as a bonus. Somewhat surprisingly, these error bounds turn out to be independent of the initial condition $\pi_0 > 0$. Furthermore, we establish a preliminary result about the convergence rate of the policy error. We conclude with a few examples that illustrate the obtained bounds.

\subsubsection{Preliminary lemmata}
\label{ssse:specmulemmas}
\paragraph{Bounding optimal Values:}
Lemma~\ref{le:contoptbound} point 3.~asserts roughly that \eUDRL{}-generated state values are continuous at the points with value $0$ (see the lemma for the precise statement). We begin by showing a stronger version of the lemma, which makes the continuity explicit by bounding the discrepancy in values given a discrepancy $\delta$ in the transition kernels.
\begin{lemma} (Continuity at value zero)
\label{le:opthbound}
Let $\{\mathcal{M}_{\tkernel} | \tkernel\in(\Delta \mathcal{S})^{\mathcal{S}\times\mathcal{A}} \}$ and $\{\bar{\mathcal{M}}_{\tkernel} |\tkernel\in (\Delta \mathcal{S})^{\mathcal{S}\times\mathcal{A}} \}$ be compatible families. Let $\tkernel_0$ be a deterministic kernel. Choose $\delta\in(0,2)$ and assume $\tkernel\in U_{\delta}(\tkernel_0) = \{\tkernel | \max_{s,a\in \mathcal{S}\times\mathcal{A}} \| \tkernel(\cdot|s,a) -\tkernel_0(\cdot|s,a)\|_1 < \delta \}$. Then for any policy $\pi$, state $\bar{s}=(s,h,g) \in \bar{S}_T$ and action $a \in \mathcal{A}$ the following holds:
\begin{enumerate}
\item
Suppose $V_{\tkernel_0}^{*}(\bar{s}) = 0$
then $V_{\tkernel}^{\pi}(\bar{s}) \leq \frac{\delta h}{2} \leq \frac{\delta N}{2}.$
\item
Suppose $Q_{\tkernel_0}^{*}(\bar{s},a) = 0$
then $Q_{\tkernel}^{\pi}(\bar{s},a) \leq \frac{\delta h}{2} \leq \frac{\delta N}{2}.$
\end{enumerate}
\end{lemma}
\begin{proof}
We begin by showing that point 1.~implies point 2. Let $s^*$ be the unique state for which $\tkernel_0(s^*| s,a) =1$, let $\bar{s}^* = (s^*,h-1,g)$.
Notice that $V_{\tkernel_0}^{*}(\bar{s}^*) = 0$ (for $h\geq 2$) since the opposite would contradict the assumption $Q_{\tkernel_0}^{*}(\bar{s},a) = 0$.
\begin{align*}
Q_{\tkernel}^{\pi}(\bar{s},a) &= 
\prob_{\tkernel} (\rho(S_h)=g | \bar{S}_0=\bar{s},A_0 = a; \pi)
=
\sum_{s'\in \mathcal{S}}
\prob_{\tkernel} (\rho(S_h)=g, \bar{S}_1=\bar{s}' | \bar{S}_0=\bar{s},A_0 = a; \pi)
\\
&=
\sum_{s'\in \mathcal{S}}
\prob_{\tkernel} (\rho(S_h)=g| \bar{S}_1=\bar{s}'; \pi)  \prob_{\tkernel} ( \bar{S}_1=\bar{s}' | \bar{S}_0=\bar{s},A_0 = a; \pi)
\\
&=
\prob_{\tkernel} (\rho(S_h)=g| \bar{S}_1=\bar{s}^*; \pi)  \tkernel(s^*|s,a)
+
\sum_{s'\in \mathcal{S}, s'\neq s^* }
\prob_{\tkernel} (\rho(S_h)=g| \bar{S}_1=\bar{s}'; \pi)  \tkernel(s'|s,a),
\end{align*}
where $\bar{s}'=(s',h-1,g)$. If $h=1$ the first summand is 0 and the second summand can be bounded by $\sum_{s'\in \mathcal{S}, s'\neq s^* } \tkernel(s'|s,a)
\leq \frac{\delta}{2}$. For $h\geq 2$ we get
$$
Q_{\tkernel}^{\pi}(\bar{s},a)
\leq
V_{\tkernel}^{\pi}(\bar{s}^*)
+
\sum_{s'\in \mathcal{S}, s'\neq s^* } \tkernel(s'|s,a)
\leq
\frac{\delta (h-1)}{2}
+
\frac{\delta}{2}.
$$
We show point 1. Recall that $
V_{\tkernel}^{\pi}(\bar{s}) =
\prob_{\tkernel} (\rho(S_h)=g | \bar{S}_0= \bar{s}; \pi)$. Reaching $g$ from $s$ in $h$ steps has a $0$ probability under $\tkernel_0$ because $V_{\tkernel_0}^*(\bar{s})=0$. Thus we decompose the probability by restricting to the event $$\Lambda_k=\left\{\text{The first transition with $0$ probability according to $\tkernel_0$ occurred at step $k$}\right\},$$ i.e.
\begin{align*}
\prob_{\tkernel} (\rho(S_h)=g | \bar{S}_0= \bar{s}; \pi)=
\sum_{k=1}^h \prob_{\tkernel} (\rho(S_h)=g, \Lambda_k | \bar{S}_0= \bar{s}; \pi).
\end{align*}
We bound the individual summands by
\begin{align*}
&\prob_{\tkernel} (\rho(S_h)=g, \Lambda_k | \bar{S}_0= \bar{s}; \pi)\\
&=
\sum_{\substack{\bar{s}'=(s',h-(k-1),g)\in \bar{S}_T\\
\bar{s}''=(s'',h-k,g)\in \bar{S}_T}}
\hspace*{-12mm}\prob_{\tkernel} (\bar{S}_{k-1}=\bar{s}' | \bar{S}_0= \bar{s}; \pi)
\prob_{\tkernel} (\bar{S}_{k}=\bar{s}'', \Lambda_k| \bar{S}_{k-1}=\bar{s}'; \pi)
\prob_{\tkernel} (\rho(S_h)=g | \bar{S}_{k}=\bar{s}''; \pi)
\\
&\leq
\sum_{\bar{s}'=(s',h-(k-1),g)\in \bar{S}_T}
\sum_{\bar{s}''=(s'',h-k,g)\in \bar{S}_T}
\prob_{\tkernel} (\bar{S}_{k-1}=\bar{s}' | \bar{S}_0= \bar{s}; \pi)
\prob_{\tkernel} (\bar{S}_{k}=\bar{s}'', \Lambda_k| \bar{S}_{k-1}=\bar{s}'; \pi)
\\
&=
\sum_{\bar{s}'=(s',h-(k-1),g)\in \bar{S}_T}
\prob_{\tkernel} (\bar{S}_{k-1}=\bar{s}' | \bar{S}_0= \bar{s}; \pi)
\sum_{a\in \mathcal{A}}
\sum_{\bar{s}''=(s'',h-k,g)\in \bar{S}_T}
\tkernel(s'',\Lambda_k|s',a')\pi(a'|\bar{s}')
\\
&\leq
\sum_{\bar{s}'=(s',h-(k-1),g)\in \bar{S}_T}
\prob_{\tkernel} (\bar{S}_{k-1}=\bar{s}' | \bar{S}_0= \bar{s}; \pi)
\sum_{a\in \mathcal{A}}
\frac{\delta}{2}\pi(a'|\bar{s}')
\leq
\frac{\delta}{2}.
\end{align*}

\end{proof}

\paragraph{Continuity of action-values in quotient topology:}
The continuity of $Q_{\tkernel}^{\pi}$ as a function of $(\pi,\tkernel)$ was proven in lemma~\ref{le:contoptbound} point 1. The following lemma establishes a weaker variant that only requires the relative continuity of policies (theorem~\ref{le:detcont}). It establishes the continuity of $Q_{\tkernel}^{\pi}$ in a factor topology (where the factorization occurs at the policy input). It will also be useful for the proof of the main theorem below to show this form of continuity for a fixed remaining horizon.
An explicit estimate, which will be useful when constructing error bounds later, is provided in the second part of the lemma.

\begin{lemma}
\label{le:contQfac}
(Continuity of action-values in quotient topology) 
Let $\{\mathcal{M}_{\tkernel} | \tkernel\in(\Delta \mathcal{S})^{\mathcal{S}\times\mathcal{A}} \}$ and $\{\bar{\mathcal{M}}_{\tkernel} |\tkernel\in (\Delta \mathcal{S})^{\mathcal{S}\times\mathcal{A}} \}$ be compatible families. Let $\tkernel_0$ be a deterministic kernel.
For all $\epsilon >0$ there exists $\delta >0$ such that if $\tkernel\in U_{\delta}(\tkernel_0)$ and if the policy $\pi$ satisfies $2(1-\pi(\oactions(\bar{s}')|\bar{s}')) < \delta$ for
all $\bar{s'}\in \theinterestingstates$ then $|Q^{\pi}_{\tkernel}(\bar{s},\cdot)- Q^*_{\tkernel_0}(\bar{s},\cdot)|\leq\epsilon$ for all $\bar{s}\in \theinterestingstates$. The statement can be made explicit by making use of the following recursive estimate: For all $\bar{s}\in\theinterestingstates$ with $h\geq2$ and all $a\in\mathcal{O}(\bar{s})$ it holds that
\begin{align*}
|Q_{\tkernel}^{\pi}(\bar{s},a)-Q_{\tkernel_0}^{*}(\bar{s},a) |
&\leq
\|\tkernel(\cdot|s,a)-\tkernel_0(\cdot|s,a)\|_1\\
&+\max_{\bar{s}'=(s',h-1,g) \in\theinterestingstates}
2(1-\pi(\oactions(\bar{s}')|\bar{s}'))
\\
&+
\max_{\substack{\bar{s}'=(s',h-1,g) \in\theinterestingstates,\\a' \in \oactions(\bar{s}')}} %
|Q_{\tkernel}^{\pi}(\bar{s}',a')-Q_{\tkernel_0}^{*}(\bar{s}', a')|.
\end{align*}
\end{lemma}
In fact, we will prove the following more detailed statement. For all remaining horizons $h$, $1 \leq h \leq N$ and $\epsilon >0$ there exists $\delta >0$ such that if $\tkernel\in U_{\delta}(\tkernel_0)$ and if the policy $\pi$ satisfies $2(1-\pi(\oactions(\bar{s}')|\bar{s}')) < \delta$ for
all $\bar{s'}=(s',h',g')\in \theinterestingstates$ with $h'<h$ then $|Q^{\pi}_{\tkernel}(\bar{s},\cdot)- Q^*_{\tkernel_0}(\bar{s},\cdot)|\leq\epsilon$ for all $\bar{s}\in \theinterestingstates$ with remaining horizon $h$.

\begin{proof}
\textit{Base case ($h=1$):} In this case
$
Q^{\pi}_{\tkernel}((s,1,g),a) = \sum_{s'\in \rho^{-1}(\{g\})}
\tkernel(s'|s,a) 
$
is continuous in $\tkernel$ and independent of $\pi$.

\textit{Induction:} Now assume the statement holds for $h-1 \geq 1$ and let $\pi^*$ be an optimal policy for $\lambda_0$. Without loss of generality
$\pi^*(a|\bar{s}') \geq \pi(a|\bar{s}')$ for all $a \in \oactions(\bar{s}')$, 
$\bar{s}'\in \theinterestingstates$. For all $\bar{s} = (s,h,g) \in \theinterestingstates$ and $a\in \oactions(\bar{s})$ it holds that
\begin{align}
&|Q^{\pi}_{\tkernel}(\bar{s},a) - Q^*_{\tkernel_0}(\bar{s},a))|
=
\Bigg|
\sum_{s'\in\bar{\mathcal{S}}}
\tkernel(s'|s,a)
\sum_{a'\in\mathcal{A}}
\pi(a'|s',h-1,g)
Q_{\tkernel}^{\pi}((s',h-1,g),a')
\nonumber\\
&\quad-
\sum_{s'\in\bar{\mathcal{S}}}
\tkernel_0(s'|s,a)
\sum_{a'\in\mathcal{A}}
\pi^*(a'|s',h-1,g)
Q_{\tkernel_0}^*((s',h-1,g),a')
\Bigg|
\nonumber\\
&\leq
\sum_{s'\in\bar{\mathcal{S}}}
|\tkernel(s'|s,a)
-
\tkernel_0(s'|s,a)
|
\sum_{a'\in\mathcal{A}}
\pi(a'|s',h-1,g)
Q_{\tkernel}^{\pi}((s',h-1,g),a')
\nonumber\\
&\quad+
\sum_{s'\in\bar{\mathcal{S}}}
\tkernel_0(s'|s,a)
\sum_{a'\in\mathcal{A}}
|\pi(a'|s',h-1,g)
-
\pi^*(a'|s',h-1,g)|
Q_{\tkernel}^{\pi}((s',h-1,g),a')
\nonumber\\
&\quad+
\sum_{s'\in\bar{\mathcal{S}}}
\tkernel_0(s'|s,a)
\sum_{a'\in\mathcal{A}}
\pi^*(a'|s',h-1,g)
|
Q_{\tkernel}^{\pi}((s',h-1,g),a')
-
Q_{\tkernel_0}^*((s',h-1,g),a')
|
\nonumber\\
&\leq
\|\tkernel(\cdot|s,a)
-
\tkernel_0(\cdot|s,a)
\|_1
+
2(1-\pi(\oactions(\bar{s}'')|\bar{s}'')
+
\max_{a \in \oactions(\bar{s}'')}
|
Q_{\tkernel}^{\pi}(\bar{s}'',a')
-
Q_{\tkernel_0}^*(\bar{s}'',a')
|\\
\nonumber
&\leq
\delta
+
\delta
+
\|
Q_{\tkernel}^{\pi}(\bar{s}'',\cdot)-
Q_{\tkernel_0}^*(\bar{s}'',\cdot)
\|_{\max}\nonumber
\end{align}
where we made use of the fact that if $\tkernel_0(\cdot|s,a)$ is deterministic then there is exactly one $s''$, $\bar{s}'' =(s'',h-1,g)$ with $\tkernel_0(s''|s,a) = 1$ and $0$ otherwise. The assumption $\|\tkernel(s'|s,a)-\tkernel_0(s'|s,a)\|_1 \leq \delta$ is used in the last line. Further by lemma~\ref{le:suppstab} point 4. $\bar{s}''\in\theinterestingstates$ and it follows that
$
\sum_{a'\in\mathcal{A}}
|\pi(a'|\bar{s}'')-\pi^*(a'|\bar{s}'')| =
2
\sum_{a'\in \oactions(\bar{s}'')} \pi^*(a'|\bar{s}'')-\pi(a'|\bar{s}'')
=2(1-\pi(\oactions(\bar{s}'')|\bar{s}'') < \delta
$ since $Q_{\tkernel}^{\pi}((s',h-1,g),a') \leq 1$.
Further we used that $\pi^*(a'|\bar{s}'')$ is zero for $a'\notin \oactions(\bar{s}'')$. It follows from the induction assumption that
$|
Q_{\tkernel}^{\pi}((\bar{s}''),a')-
Q_{\tkernel_0}^*((\bar{s}''),a')
| \rightarrow 0$ for $\delta \rightarrow 0$ and the right hand side goes to $0$. To complete the induction we prove continuity also in the remaining case $a\notin \oactions(\bar{s})$. For all $\bar{s} = (s,h,g) \in \theinterestingstates$
and $a\notin \oactions(\bar{s})$ it holds
$$
|Q^{\pi}_{\tkernel}(\bar{s},a) - Q^*_{\tkernel_0}(\bar{s},a))|
= Q^{\pi}_{\tkernel}(\bar{s},a)
\leq
\frac{\delta N}{2} \rightarrow 0,\ \delta \rightarrow 0
,
$$
where we used $Q^*_{\tkernel_0}(\bar{s},a)) = 0$ by lemma~\ref{le:suppstab} point 5.\ and lemma~\ref{le:opthbound}. This completes the induction.
\end{proof}

\paragraph{Visitation probabilities are separated  from $0$:}
The assumption $\theinterestingstates \subset \supp \bar{\mu}$ allows one to bound state visitation probabilities in \eUDRL{} recursion rewrites (compare lemma~\ref{le:recrewrites}).
\begin{lemma}
\label{le:alpha}
{(Lower bound on visitation probabilities)} Let $\{\mathcal{M}_{\tkernel} | \tkernel\in(\Delta \mathcal{S})^{\mathcal{S}\times\mathcal{A}} \}$ and $\{\bar{\mathcal{M}}_{\tkernel} |\tkernel\in (\Delta \mathcal{S})^{\mathcal{S}\times\mathcal{A}} \}$ be compatible families. Let $\tkernel_0$ be a deterministic kernel. Assume that $\theinterestingstates \subset \supp \bar{\mu}$, where $\bar{\mu}$ denotes the
CE's initial distribution. Then for all $\bar{s} = (s,h,g) \in \theinterestingstates$, {all transition kernels $\tkernel$ and all policies $\pi$} it holds that
$$
\prob_{\tkernel}(H_0^{\Sigma}=h,G_0^{\Sigma}=g |
S_0^{\Sigma}=s,l(\Sigma)=h;\pi) > \alpha,
$$
where
$$
\alpha = \frac{2}{N(N+1)} \min_{\bar{s}\in \theinterestingstates} \bar{\mu}(\bar{s}) > 0.
$$
\end{lemma}

\begin{proof}
Write $v(s,h,g) := \prob_{\tkernel}(H_0^{\Sigma}=h,G_0^{\Sigma}=g,
 S_0^{\Sigma}=s,l(\Sigma)=h;\pi)$. By the definition of the segment distribution we obtain that
$$
\begin{aligned}
v(s,h,g)
&=
\sum_{\sigma:l(\sigma)=h, \bar{s}_0^{\sigma} = (s,h,g)}
c_{\tkernel,\pi}^{-1} \sum_{t<N-l(\sigma)}
\prob_{\tkernel} (\bar{S}_{t+l(\sigma)}=\bar{s}_{l(\sigma)}^{\sigma},
\ldots,\bar{S}_t = \bar{s}_0^{\sigma};\pi)
\\
&\geq
\sum_{\sigma:l(\sigma)=h, \bar{s}_0^{\sigma} = (s,h,g)}
c_{\tkernel,\pi}^{-1}
\prob_{\tkernel} (\bar{S}_{h}=\bar{s}_{h}^{\sigma},
\ldots,\bar{S}_0 = \bar{s}_0^{\sigma};\pi)
\\
&=
\sum_{\sigma:l(\sigma)=h, \bar{s}_0^{\sigma} = (s,h,g)}
c_{\tkernel,\pi}^{-1}
\prob_{\tkernel} (\bar{S}_{h}=\bar{s}_{h}^{\sigma},
\ldots,\bar{S}_1 = \bar{s}_1^{\sigma}, A_0=a_0^{\sigma}|\bar{S}_0 = \bar{s}_0^{\sigma};\pi)
\bar{\mu}(\bar{s}_0^{\sigma})
\\
&=
c_{\tkernel,\pi}^{-1}\bar{\mu}(\bar{s}_0^{\sigma}).
\end{aligned}
$$
In the last equation we made use of the fact that by summing over all possible segments with the starting condition $(s,h,g)$ we are actually summing
over all $h-1$ tuples of action and extended state pairs compatible with the CE framework, leaving out only tuples whose probability is $0$. The latter are tuples where the remaining horizon does not decrease by $1$ or where the goal is not constant. By equation~\eref{eq:cupperbound} we have $c_{\tkernel,\pi}^{-1} \geq \frac{2}{N(N+1)}$.
We choose $\alpha = \frac{2}{N(N+1)} \min_{\bar{s}\in \theinterestingstates} \bar{\mu}(\bar{s})>0$ and conclude for all $(s,h,g) \in \theinterestingstates$ that
\begin{align*}
\prob_{\tkernel}(H_0^{\Sigma}=h,G_0^{\Sigma}=g |
S_0^{\Sigma}=s,l(\Sigma)=h;\pi)
&=
\frac{v(s,h,g)}
{
\prob_{\tkernel}(S_0^{\Sigma}=s,l(\Sigma)=h;\pi)
}
\geq
\frac{v(s,h,g)}
{1}\\
&\geq
c_{\tkernel,\pi}^{-1}\bar{\mu}(s,h,g)
\geq
\frac{2}{N(N+1)} \min_{\bar{s}\in \theinterestingstates} \bar{\mu}(\bar{s})
=
\alpha.
\end{align*}
\end{proof}

\paragraph{Method for bounding the \eUDRL{} recursion:}

Our analysis of \eUDRL{}'s convergence draws on the properties of 
dynamical systems induced by certain rational functions.
Specifically, we proceed describing the \eUDRL{} recursion in terms of a multi-dimensional rational function $F:D\rightarrow\mathbb{R}$ on a domain $D=\prod_{k<K}[y_k,1]$, $y=(y_1,...,y_{K-1})\in[0,1]^K$. The recursive application of this function leads to a set of accumulation points that, in turn, is bounded by studying the fixed points of a simplified dynamical system that emerges through the iterative application of a scalar rational function. 

\begin{remark}\label{re:monotonicity} (Monotonicity) The function $g:[0,1]\rightarrow\mathbb{R}$ with
$g(x) := \frac{ax+b}{ax+b+c},
$
with $a,b\geq 0$, $c>0$ is non-decreasing.
\end{remark}

Bounding of $F$ proceeds in $K$ steps, where in each step $k<K$ the following estimates are used: \textit{1.} For all $z \in D$ the function $x_{k} \mapsto F(y_0,y_1,\ldots,y_{k-1},x_k,z_{k+1},\ldots,z_{K-1})$ has the a form specified in the remark. \textit{2.}~In conclusion for all $z \in D$ it holds that $ F(y_0,y_1,\ldots,y_{k-1},z_k,z_{k+1},\ldots,z_{K-1})\geq F(y_0,y_1,\ldots,y_{k-1},y_k,z_{k+1},\ldots,z_{K-1})$. Applying this reasoning along all coordinates leads to the estimate $F(z) \geq F(y)$ for all $z\in D$. $F(y)$, in turn, is fully described by a scalar rational function. The following lemma summarizes the features of a dynamical system that arises through the iterative application of this function.
Given a function $f:X\rightarrow X$ on a set $X$ write 
$f^{\circ n}:X\rightarrow X$ for the $n$-fold composition of $f$ with itself, i.e.~$f^{\circ n}(x) = (f\circ f \circ \ldots \circ f)(x)$.

\begin{lemma} (f-lemma)
\label{le:f}
Let $\gamma\in(0,1)$ and $f_{\gamma}:(0,1] \rightarrow (0,1]$ be defined as
$
f_{\gamma}(x) =
\frac{x}
{x + \gamma}
$. Then the following assertions hold.
\begin{enumerate}
\item $f_{\gamma}$ is strictly monotonically increasing.
\item $f_{\gamma}$ has a unique fixed point
$
x^*(\gamma) = 1-\gamma
$
and
$$
\begin{aligned}
f_{\gamma}(x) &> x \quad \text{for}\; x < x^*(\gamma),
\\
f_{\gamma}(x) &< x \quad \text{for}\; x > x^*(\gamma).
\end{aligned}
$$
\item For all $x \in (0,1]$ the iterated application of $f_\gamma$ converges point-wise, $f_{\gamma}^{\circ n} (x) \xrightarrow{n\rightarrow\infty} x^*(\gamma)$.

\item Assume a sequence $(y_n)_{n\geq0}$, $y_n\in [0,1]$, $y_0>0$ such that for all $n\geq 0$ we have
$y_{n+1} \geq f_{\gamma}(y_n)$.
Then $y_n \geq f_{\gamma}^{\circ n}(y_0)$ and $\liminf_n y_n \geq x^*(\gamma)$.
\end{enumerate}
\end{lemma}
Figure~\ref{fig:f} illustrates the dynamical system induced by iterative application of $f_{\gamma}$.
\begin{proof}
\textit{1.} The derivative of $f$ is positive.

\textit{2.} Solving the equation $f_{\gamma}(x) = x$ shows that $x^*=1-\gamma$ is a unique fixed point. The remaining statements follow by plugging $f$ into the respective inequalities.

\textit{3.} Write $I_1 = (0,x^*]$ and $I_2 = [x^*,1]$ and  $f_1 = f\restriction_{I_1}$, $f_2 = f\restriction_{I_2}$. $f$ is increasing so that $f(I_1) \leq f(x^*) = x^*$ and
thus $f(I_1) \subset I_1$, i.e. $f_1:I_1 \rightarrow I_1$.
Similarly, $f_2:I_2 \rightarrow I_2$. Fix $x\in I_1$ then the sequence $f^{\circ n}(x) = f_1^{\circ n}(x)$ is bounded from above by $x^*$. The sequence is also increasing by point 2. Thus the sequence converges to a limit in ${I}_{1}$.
Since $f_1$ is continuous on $I_1$, the limit has to be a fixed point. Since $x^*$ is the only fixed point we conclude $f_1^{\circ n}(x) \xrightarrow{n\rightarrow\infty} x^*$.
Similarly, for $x\in I_2$ we get $f^{\circ n}(x) = f_2^{\circ n}(x) \xrightarrow{n\rightarrow\infty} x^*$.

\textit{4.} We proceed by induction. For $n=0$ we have $y_0 \geq f^{\circ 0}(y_0)= y_0$. Assume that
$
y_n \geq f^{\circ n}(y_0)
$. It follows that $y_n>0$. Since $f$ is increasing $f(y_n) \geq f(f^{\circ n}(y_0))$. Making use of
$
y_{n+1} \geq f(y_n)
$ yields
$
y_{n+1} \geq f(y_n) \geq f(f^{\circ n}(y_0))
$, which concludes the induction. $\liminf_n y_n \geq x^*$ then follows directly from point 3.
\end{proof}
\begin{figure}
  \begin{subfigure}{0.48\linewidth}
		\includegraphics[width=\linewidth]{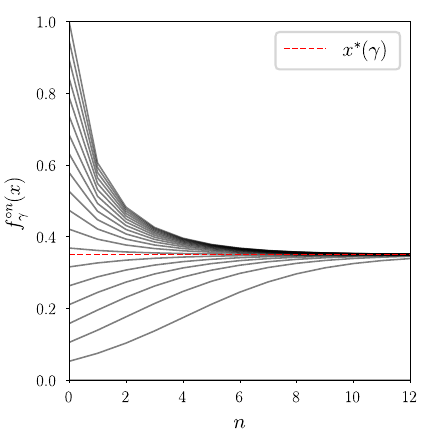}
		\caption{$\gamma=0.65$}
	\end{subfigure}
    \begin{subfigure}{0.48\linewidth}
		\includegraphics[width=\linewidth]{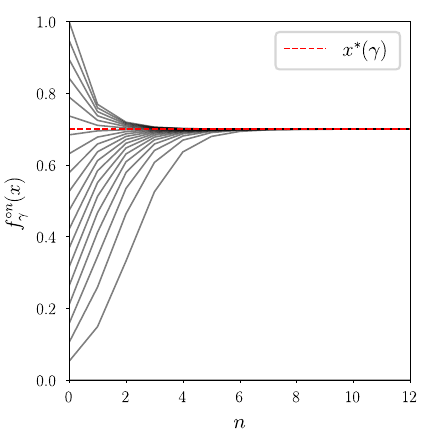}
		\caption{$\gamma=0.3$}
	\end{subfigure}
    \caption{Illustration of the dynamical system induced by iterative application of  $f_{\gamma}$. The figure shows the dependence of $f_{\gamma}^{\circ n}(x)$ given initial condition $x$ on iteration $n$.}
    \label{fig:f}
\end{figure}

\subsubsection{The main theorem}

We describe the continuity of \eUDRL{}-generated policies when the number of iterations approaches infinity. We have already discussed that the set of accumulation points is not empty. We proved in theorem~\ref{le:detopt} that for every finite iteration the \eUDRL{}-generated policy is optimal if the transition kernel $\tkernel_0$ is deterministic.  Since for all states $\bar{s}\in \theinterestingstates$ and actions $a \in \mathcal{A}\setminus\oactions(\bar{s})$ and $n \geq 1$ it holds that $\pi_n(a|\bar{s})=0$ it also follows that $\pi_n(a|\bar{s})\rightarrow 0$ as $n\rightarrow \infty$. As a consequence the set of accumulation points of the sequence of policies for a deterministic kernel is a subset of $\{ \pi_{\tkernel_0}^* \mid \pi_{\tkernel_0}^*\; \text{is an optimal policy for} \; \tkernel_0 \}$.
\label{ssse:specmutheorem}
\begin{theorem}
\label{le:limdetcont}
(Relative continuity of \eUDRL{} limit policies in deterministic kernels -- the case $\theinterestingstates \subset \supp \bar{\mu}$) Let $\{\mathcal{M}_{\tkernel} | \tkernel\in(\Delta \mathcal{S})^{\mathcal{S}\times\mathcal{A}} \}$ and $\{\bar{\mathcal{M}}_{\tkernel} |\tkernel\in (\Delta \mathcal{S})^{\mathcal{S}\times\mathcal{A}} \}$ be compatible families. Let $\tkernel_0\in(\Delta\mathcal{S})^{\mathcal{S}\times\mathcal{A}}$ be a deterministic kernel and let $(\pi_{n,\tkernel})_{n\geq0}$ be a sequence of \eUDRL{}-generated policies with initial condition $\pi_0 > 0$ and transition kernel $\tkernel$. Assume $\theinterestingstates \subset \supp \bar{\mu}$. Then for all $\pi_0>0$ the following holds:
\begin{enumerate}
\item Let $\mathcal{L}(\pi_0,\tkernel)$ denote the set of accumulation points of $(\pi_{n,\tkernel})_{n\geq0}$.
Then any function $u: (\pi_0,\tkernel)\mapsto u(\pi_0,\tkernel) \in \mathcal{L}(\pi_0,\tkernel)$ is relatively continuous in $\tkernel$
at $\tkernel_0$ on $\theinterestingstates$, i.e.~for all $\bar{s} \in \theinterestingstates$
$$[u(\pi_0,\tkernel)](\cdot|\bar{s}) \xrightarrow{\oactions(\bar{s})} \pi_{\tkernel_0}^*(\cdot|\bar{s})\quad\textnormal{as}\quad\tkernel\rightarrow \tkernel_0.$$
\item Let $\alpha>0$ be chosen as in lemma~\ref{le:alpha} and let  $\beta,\tilde{\beta} \in (0,1)$ be such that $\gamma = \frac{\tilde{\beta}}{(1-\beta)\alpha}<1$ and let
$$U_{\delta}(\tkernel_0) = \{ \tkernel | \max_{(s,a)\in \mathcal{S}\times\mathcal{A}}
\| \tkernel(\cdot|s,a) - \tkernel_0(\cdot|s,a)\|_1 < \delta \}.$$
There exists $\delta >0$ so that for all $\tkernel \in U_{\delta}(\tkernel_0)$ and all $\bar{s}\in \theinterestingstates$ it holds that

$$
\liminf_n \pi_{n,\tkernel}(\oactions(\bar{s})|\bar{s}) \geq x^*(\gamma)\quad\textnormal{and}\quad x^*(\gamma) \rightarrow 1\quad\textnormal{as}\quad (\beta,\tilde{\beta}) \rightarrow 0,
$$
where $x^*(\gamma)=1-\gamma$ (cf. lemma~\ref{le:f}). As a consequence for all $\bar{s}\in \theinterestingstates$ it holds that
$$
\liminf_n \pi_{n,\tkernel}(\oactions(\bar{s})|\bar{s}) \rightarrow 1\quad\textnormal{as}\quad\tkernel \rightarrow \tkernel_0.
$$
\end{enumerate}
\end{theorem}
At the core of the proof stands the analysis of dynamical systems and convergence induced by the iterative application of the rational function $f_\gamma(x)=\frac{x}{x+\gamma}$, see lemma~\ref{le:f}. Indeed we will show that for all horizons $N \geq h \geq 1$, all initial conditions $\pi_0>0$
and all $\beta,\tilde{\beta}\in (0,1)$ satisfying $\frac{\tilde\beta}{(1-\beta)\alpha}=\gamma<1$ 
there exists $n_0$ and $\delta >0$ such that for all $\tkernel \in U_{\delta}(\tkernel_0)$ and all $n>n_1\geq n_0$ and all $\bar{s}\in\theinterestingstates$ with remaining horizon $h$ it holds that
\begin{align}
\pi_{n,\tkernel}(\oactions(\bar{s})|\bar{s}) \geq
f_{\gamma}^{\circ (n-n_1)} (\pi_{n_1,\tkernel}(\oactions(\bar{s})|\bar{s}) )
.\label{eq:forMainTheorem}
\end{align}
To apply lemma~\ref{le:f} we need to ensure the positivity of the argument of $f_\gamma$. By theorem~\ref{le:detcont}, point 3.~we have that $\pi_{n,\tkernel}(\oactions(\bar{s})|\bar{s}) >0$ for all $n\geq 0$ for all $\tkernel \in U_{\delta_n}(\tkernel_0)$ for suitable $\delta_n >0$. However, we will require a slightly stronger statement that also ensures that $\pi_{n,\tkernel}(\oactions(\bar{s})|\bar{s}) >0$ 
for all $n\geq 0, \tkernel \in U_{2}(\tkernel_0), \bar{s} \in \theinterestingstates$ that we show in the course of the main theorem, see~equation~\eref{eq:pisep}.
Then by lemma~\ref{le:f} the right hand side converges to $x^*(\gamma)=1-\gamma$, which implies point 2.~of the theorem, see the proof for details. Furthermore point 1.~is an immediate consequence of point 2. We include point 1.~to showcase the relation to the notion of \lq\lq{}continuity of the set of accumulation points relative to the set of optimal actions\rq\rq{} that we mentioned earlier. Point 1.~says that any function $u$ that maps $(\pi_0,\tkernel)$ to an accumulation point is relatively continuous at $\tkernel_0$. Since the set of values of all such functions $u$ at  $(\pi_0,\tkernel)$ is exactly $\mathcal{L}(\pi_0,\tkernel)$ one can interpret 1.~as the relative
continuity of the set $\mathcal{L}(\pi_0,\tkernel)$.

\begin{proof}
It is obvious that point 2.~implies point 1.~since $\pi_{\tkernel_0}^*(\oactions(\bar{s})|\bar{s}) = 1$ on $\theinterestingstates$. Second, we prove that  equation~\eqref{eq:forMainTheorem} implies point 2.
Fix $\pi_0 >0$ and choose $\epsilon >0$. Since $x^*(\gamma) \rightarrow 1$ as
$\gamma \rightarrow 0$ (cf.\ lemma~\ref{le:f}) and $\gamma \rightarrow 0$ as $(\beta,\tilde{\beta}) \rightarrow 0$ we can pick $\beta,\tilde{\beta} >0$ so that
$x^*(\gamma) > 1-\epsilon$. 
Since equation~\eqref{eq:forMainTheorem} holds for every $h$ there are $\delta(h) > 0$ and $n_0(h)$ so that
for all $n > n_1 \geq n_0(h)$, $\tkernel \in U_{\delta(h)}(\tkernel_0)$
and all $\bar{s} \in \theinterestingstates$ the equation remains valid. Let us fix $n_1= \max_h n_0(h)$ and $\delta = \min_h \delta(h)$ and $\eta = \min_{\bar{s}\in \theinterestingstates} \pi_{n_1,\tkernel}(\oactions(\bar{s})|\bar{s})$. By lemma~\ref{le:f} $f_{\gamma}^{\circ (n-n_1)}$ is increasing and it follows that
$
\pi_{n,\tkernel}(\oactions(\bar{s})|\bar{s}) \geq
f_{\gamma}^{\circ (n-n_1)} (\eta).
$
Since $f_{\gamma}^{\circ n} (\eta) \rightarrow x^*(\gamma)$ as
$n\rightarrow \infty$
we obtain 
$
\liminf_n \pi_{n,\tkernel}(\oactions(\bar{s})|\bar{s}) \geq x^*(\gamma) > 1-\epsilon
$. As $\epsilon >0$ is arbitrary we can always find $\delta>0$
so that the above holds. We conclude that as $\tkernel \rightarrow \tkernel_0$
$$
\liminf_n \pi_{n,\tkernel}(\oactions(\bar{s})|\bar{s})
\rightarrow 1.
$$

It remains to prove equation~\eqref{eq:forMainTheorem}. Fixing $2>\delta>0$ lemma~\ref{le:suppstab} point 2.~asserts that $\theinterestingstates \subset \supp \den_{\tkernel,\pi_n} \cap \supp \nu_{\tkernel,\pi_n}$
for all $n\geq 0 $, $\pi_0 > 0$, $\tkernel \in U_{\delta}(\tkernel_0)$.
The policy $\pi_{n+1,\tkernel}(\cdot|\bar{s})$ is then well-defined by
\begin{equation}
\pi_{n+1,\tkernel}(\cdot|\bar{s}) = \frac{\num_{\tkernel,\pi_n}(\cdot|\bar{s})}{\den_{\tkernel,\pi_n}(\bar{s})}
\label{eq:pirec}
\end{equation}
for all $\bar{s}\in \theinterestingstates$. Using the notation
$$
\begin{aligned}
Q_{\tkernel}^{\pi_n,g,h}((s,h',g'),a)
&:= 
\prob_{\tkernel}(\rho(S^{\Sigma}_h)=g|A^{\Sigma}_0=a,S^{\Sigma}_0=s,H^{\Sigma}_0=h',G^{\Sigma}_0=g';\pi_n),
\\
v_{\tkernel,\pi_n}(h',g'|s,h) 
&:=
\prob_{\tkernel}(\rho(H_0^{\Sigma})=h',G_0^{\Sigma}=g'|S_0^{\Sigma}=s,l({\Sigma})=h;\pi_n)
\end{aligned}
$$
and writing $\bar{s} = (s,h,g) \in \theinterestingstates$ lemma~\ref{le:recrewrites} point 1.~allows one to rewrite the policy as
\begin{multline}
\pi_{n+1,\tkernel}(\oactions(\bar{s})|\bar{s}) = \frac{
\sum%
_{a\in \oactions(\bar{s}), h'\geq h, g' \in \mathcal{G}}
Q_{\tkernel}^{\pi_n,g,h}((s,h',g'),a)
\pi_{n,\tkernel}(a|s,h',g') v_{\tkernel,\pi_n}(h',g'|s,h)
}{
\sum%
_{a\in \mathcal{A}, h'\geq h, g' \in \mathcal{G}}
Q_{\tkernel}^{\pi_n,g,h}((s,h',g'),a)
\pi_{n,\tkernel}(a|s,h',g') v_{\tkernel,\pi_n}(h',g'|s,h)
}
.\label{eq:piM}
\end{multline}
We show some estimates that will be used later in the proof.

\emph{Bounding state visitation:}
By lemma~\ref{le:alpha} and using $\bar{s}\in\theinterestingstates \subset \supp \bar{\mu}$ there exists $\alpha > 0$ so that
\begin{equation}
v_{\tkernel,\pi_n}(h,g|s,h) > \alpha\nonumber
.
\end{equation}

\emph{Separation of $\pi_{n,\tkernel}$ from 0 on optimal actions:}
We will show that for each $n \geq 0$ there exist $\eta>0$
such that for all 
$\bar{s} \in \theinterestingstates$, $a\in \oactions(\bar{s})$, $\tkernel \in U_{\delta}(\tkernel_0)$ it holds that
\begin{equation}
\pi_{n,\tkernel}(a\mid \bar{s}) > \eta > 0.
\label{eq:pisep}
\end{equation}
The proof proceeds by induction. The base case $n=0$ follows since $\pi_0 >0$ does not depend on $\tkernel$.
For the induction step assume $\pi_{n,\tkernel}(a\mid \bar{s}) > \eta' > 0$ for all $(\bar{s},a,\tkernel)$ above. Using \eref{eq:piM} we get
$
\pi_{n,\tkernel}(a\mid \bar{s})
\geq
Q_{\tkernel}^{\pi_n} (\bar{s},a)
\pi_n (a| \bar{s})
v_{\tkernel,\pi_n}(h,g|s,h).
$
Because $\bar{s} = (s,h,g) \in \theinterestingstates$
and $a \in \oactions(\bar{s})$
we can fix a sequence
$\bar{s}_0,a_0,\ldots,\bar{s}_h$ with
$\bar{s}_0 = \bar{s}$, $a_0 = a$, $\rho(s_h) = g$
evidencing that the probability $Q_{\tkernel_0}^*(\bar{s},a) $ is positive. We deduce that, for all $i=0,\ldots,h-1$, all $\bar{s}_{i+1}$ belongs to $\theinterestingstates$ and all actions $a_i$ are optimal and $\tkernel_0(s_{i+1}\mid s_i, a_i) = 1$. Thus we can estimate $\tkernel(s_{i+1}\mid s_i, a_i) > 1-\frac{\delta}{2}$ (since $\tkernel_0(s_{i+1}\mid s_i, a_i) - \tkernel(s_{i+1}\mid s_i, a_i) < \frac{\delta}{2}$) and $\pi_{n,\tkernel}(a_i|\bar{s}) > \eta'$. This gives us the desired lower bound 
$$
\pi_{n+1,\tkernel}(a\mid \bar{s}) \geq 
\left(\prod_{i=0}^{h-1} \tkernel(s_{i+1} \mid s_i, a_i) \pi_{n,\tkernel}(a_i \mid s_i)\right) \cdot \alpha \geq (1-\frac{\delta}{2})^h (\eta')^h \alpha > 0.
$$
This completes proof of \eref{eq:pisep} by induction.

\emph{Bounding $Q_{\tkernel}^{\pi_n,g,h}$:} By lemma~\ref{le:contoptbound} and lemma~\ref{le:detopt} we have that for any policy $\pi$ for all $\bar{s} = (s,h,g)$
\begin{align}
Q_{\tkernel}^{\pi,g,h}((s,h',g'),a) \leq
Q_{\tkernel}^*((s,h,g),a)\rightarrow Q_{\tkernel_0}^*((s,h,g),a),\  
\tkernel \rightarrow \tkernel_0
.\label{eq:CQbound}
\end{align}

The proof proceeds by induction on the remaining horizon $h$. For the rest of the proof $\pi_0$ is assumed fixed.
\textit{Base case ($h=1$):}
We estimate $\pi_{n+1,\tkernel}(\oactions(\bar{s})|\bar{s})$ by showing lower estimates on the recursion~\eref{eq:piM}. By lemma~\ref{le:contQfac} we have that if $\bar{s} = (s,1,g)\in \theinterestingstates$ and $a\in \oactions(\bar{s})$ then
\begin{equation}
|Q_{\tkernel}^{\pi_n}(\bar{s},a) - Q_{\tkernel_0}^*(\bar{s},a)|
=
|Q_{\tkernel}^{\pi_n}(\bar{s},a) - 1| \rightarrow 0,\ \tkernel \rightarrow \tkernel_0.
\label{eq:QMbound1}
\end{equation}
Since $h=1$ $Q_{\tkernel}^{\pi_n}(\bar{s},a)$ does not depend on $\pi_n$. Fix $\beta,\tilde{\beta} \in (0,1)$ so that $1 > \gamma = \frac{\tilde{\beta}}{(1-\beta)\alpha}$. By~\eref{eq:CQbound} and~\eref{eq:QMbound1} we can fix $2>\delta>0$
so that for all $\tkernel \in U_{\delta}(\tkernel_0)$ and 
$\bar{s} = (s,1,g) \in \theinterestingstates$ it holds that
\begin{equation}
\begin{aligned}
\tilde{\beta} &> Q_{\tkernel}^*(\bar{s},a), \;\text{for}\; a \notin \oactions(\bar{s}),
\\
1-\beta &< Q_{\tkernel}^{\pi_n}(\bar{s},a), \;\text{for}\; a \in \oactions(\bar{s})
.\label{eq:Qbounds1}
\end{aligned}
\end{equation}

We proceed by bounding $\pi_{n+1,\tkernel}(\oactions(\bar{s})|\bar{s})$ in \eref{eq:piM} from below. For the denominator we find using \eref{eq:Qbounds1} and \eref{eq:CQbound} that
\begin{multline*}
\sum%
_{a\in \mathcal{A}, h'\geq 1, g' \in \mathcal{G}}
Q_{\tkernel}^{\pi_n,g,1}((s,h',g'),a)
\pi_{n,\tkernel}(a|s,h',g') v_{\tkernel,\pi_n}(h',g'|s,1)
\\
\leq
\sum%
_{a\in \oactions(\bar{s}), h'\geq 1, g' \in \mathcal{G}}
Q_{\tkernel}^{\pi_n,g,1}((s,h',g'),a)
\pi_{n,\tkernel}(a|s,h',g') v_{\tkernel,\pi_n}(h',g'|s,1)
+
\tilde{\beta}.
\end{multline*}
Exploiting the monotonicity remark~\ref{re:monotonicity} (M) to use the bounds in~\eref{eq:Qbounds1} and $v_{\tkernel,\pi_n}(h,g|s,h)>\alpha$ yields that $(\forall n\geq 0, \forall \tkernel \in U_{\delta}(\tkernel_0), \forall \bar{s} = (s,1,g) \in \theinterestingstates)$
\begin{align}
\pi_{n+1,\tkernel}(\oactions(\bar{s})|\bar{s})
&\geq^{(M)}
\frac{
\sum\limits_{a \in \oactions(\bar{s})}
(1-\beta)
\pi_{n,\tkernel}(a|s,1,g) v_{\tkernel,\pi_n}(1,g|s,1)
}{
\sum\limits_{a \in \oactions(\bar{s})}
(1-\beta)
\pi_{n,\tkernel}(a|s,1,g) v_{\tkernel,\pi_n}(1,g|s,1)
+
\tilde{\beta}
}\nonumber
\\
&\geq^{(M)}
\frac{
\sum\limits_{a \in \oactions(\bar{s})}
(1-\beta)
\pi_{n,\tkernel}(a|s,1,g) \alpha
}{
\sum\limits_{a \in \oactions(\bar{s})}
(1-\beta)
\pi_{n,\tkernel}(a|s,1,g) \alpha
+
\tilde{\beta}
}\nonumber
\\
&=
\frac{
(1-\beta) \alpha
\pi_{n,\tkernel}(\oactions(\bar{s})|s,1,g)
}{
(1-\beta) \alpha
\pi_{n,\tkernel}(\oactions(\bar{s})|s,1,g)
+
\tilde{\beta}
}
=
f_{\gamma}(\pi_{n,\tkernel}(\oactions(\bar{s})|s,1,g)).\label{eq:pimf1}
\end{align}

\emph{Induction:}
Fix $h>1$ and assume that~\eqref{eq:forMainTheorem} holds for all $h'<h$. Fix $\beta,\tilde\beta \in (0,1)$ so that $1 > \gamma = \frac{\tilde\beta}{(1-\beta)\alpha}$.
By lemma~\ref{le:contQfac} there exists $\delta'$ such that if the following two conditions are met
\begin{align*}
\tkernel &\in U_{\delta'}(\tkernel_0),\\
(\forall \bar{s}'=(s',h',g')\in \theinterestingstates, h'<h)&:\; 1-\pi(\oactions(\bar{s}')|\bar{s}') < \frac{\delta'}{2}
\end{align*}
then it holds that
$$
(\forall \bar{s} = (s,h,g) \in \theinterestingstates,
\forall a \in \oactions(\bar{s}), \tkernel \in U_{\delta'}(\tkernel_0)):\;
|1-Q_{\tkernel}^{\pi}(\bar{s},a)| < \beta
.
$$
The first condition is met by the choice $0 < \delta < \delta' < 2$. To meet the second condition we will use the induction assumption choosing $\beta',\tilde\beta' > 0$ such that
$\gamma' = \frac{\tilde\beta^{\prime}}{(1-\beta')\alpha} < 1$ and $x^*(\gamma')>1-\delta'/2$.
Applying the induction assumption requires a restriction on $\delta$. From the induction assumption and equation~\eref{eq:pisep} it follows that there exist $n_0'$ and $\eta>0$ such that for all $n>n_0'$, $\tkernel \in U_{\delta}(\tkernel_0)$ holds
$\pi_{n,\tkernel}(\oactions(\bar{s}')|\bar{s}') > f^{\circ n-n_0'}_{\gamma'}(\eta)$.
As $f^{\circ n-n_0'}_{\gamma'}(\eta) \rightarrow x^*(\gamma')$ there exists $n_0$ such that 
$\pi_{n,\tkernel}(\oactions(\bar{s}')|\bar{s}') > 1-\delta'/2$ for $n>n_0$,
$\tkernel \in U_{\delta}(\tkernel_0)$. Making use of~\eref{eq:CQbound} we can assume that $\delta>0$ is chosen such that
$$
( \forall \tkernel \in U_{\delta}(\tkernel_0),
\forall \bar{s} \in \theinterestingstates, a \notin \oactions(\bar{s}) ):
Q_{\tkernel}^*(\bar{s},a) < \tilde\beta
,
$$

We proceed by bounding $\pi_{n+1,\tkernel}(\oactions(\bar{s})|\bar{s})$ in \eref{eq:piM} from below. Assuming that $n \geq n_0$ we can utilize the obtained estimates as we already did for $h=1$. As a consequence
(for all $\bar{s}\in\theinterestingstates$ with remaining horizon $h$ and $\tkernel \in U_{\delta}(\tkernel_0)$) it holds
\begin{align*}    
\pi_{n+1,\tkernel}(\oactions(\bar{s})|\bar{s})
&\geq
\frac{
\sum%
_{a\in \oactions(\bar{s}),h'\geq h, g' \in \mathcal{G}}
Q_{\tkernel}^{\pi_n,h,g}((s,h',g'),a)
\pi_{n,\tkernel}(a|s,h',g') v_{\tkernel,\pi_n}(h',g'|s,h)
}{
\sum%
_{a\in \mathcal{A},h'\geq h, g' \in \mathcal{G}}
Q_{\tkernel}^{\pi_n,h,g}((s,h',g'),a)
\pi_{n,\tkernel}(a|s,h',g') v_{\tkernel,\pi_n}(h',g'|s,h)
}
\\
&\geq
\frac{
\sum\limits_{a \in \oactions(\bar{s})}
(1-\beta)
\pi_{n,\tkernel}(a|s,h,g) \alpha
}{
\sum\limits_{a \in \oactions(\bar{s})}
(1-\beta)
\pi_{n,\tkernel}(a|s,h,g) \alpha
+
\tilde{\beta}}
=
\frac{
(1-\beta) \alpha
\pi_{n,\tkernel}(\oactions(\bar{s})|\bar{s})
}{
(1-\beta) \alpha
\pi_{n,\tkernel}(\oactions(\bar{s})|\bar{s})
+
\tilde{\beta}
}\\
&=
f_{\gamma}(\pi_{n,\tkernel}(\oactions(\bar{s})|\bar{s}))
,
\end{align*}
which competes the proof.
\end{proof}

An illustration of the accumulation points of $\min_{\bar{s}\in\theinterestingstates}\pi_n(\oactions(\bar{s})|s)$ and the lower estimate provided by theorem~\ref{le:limdetcont}, point 2.~is given in Figure~\ref{fig:bandit} in the case of a $2$-armed bandit model. All details of the construction of the MDP underlying figure~\ref{fig:bandit} can be found in the appendix B, example~\ref{ex:bandit}. Figure~\ref{fig:bandit:a} shows how the lower bound derived in the theorem and the accumulation points of \eUDRL{}-generated policies behave as functions of the distance to a deterministic kernel. Figure~\ref{fig:bandit:b} shows how \eUDRL{} approaches those accumulation points as a function of both the iteration $n$ and also the initial policy (which is visible at $n=0$). Remark that while the theorem does specify the behavior of $\gamma$ as $\lambda\rightarrow\lambda_0$, it does not provide quantitative information on the specific dependency of $\gamma$ on $\delta$. The derivation of such estimates is the content of section~\ref{ssse:specmuaccbound} below. The graph for $x^*(\gamma)$ in Figure~\ref{fig:bandit:a} is computed using the estimates of corollary~\ref{le:limitbounds} (which provides the quantitative details for theorem~\ref{le:limdetcont}). Regrettably, the bounds derived from the corollary tend to become quite loose as the number of states and the maximum horizon increase. This limitation is an artifact of the bounding of visitation terms in the \eUDRL{} recursion, which was handled somewhat crudely in lemma~\ref{le:alpha}. We explore this issue in greater detail in section~\ref{ssse:specmuaccbound}. However, the primary aim of section~\ref{sse:specmu} was to establish the (relative) continuity of accumulation points of \eUDRL{}-generated policies and the associated goal-reaching objective. The bounds are rather an unexpected byproduct of this work. We emphasize that despite this limitation the proof of continuity and the bounds themselves have not been previously investigated or mentioned in the literature.

\begin{figure}
  \begin{subfigure}{0.48\linewidth}
		\includegraphics[width=\linewidth]{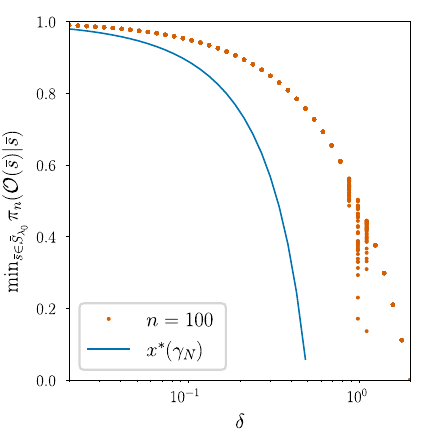}
		\caption{Dependency on distance to determin.~kernel.}
        \label{fig:bandit:a}
	\end{subfigure}
    \begin{subfigure}{0.48\linewidth}
		\includegraphics[width=\linewidth]{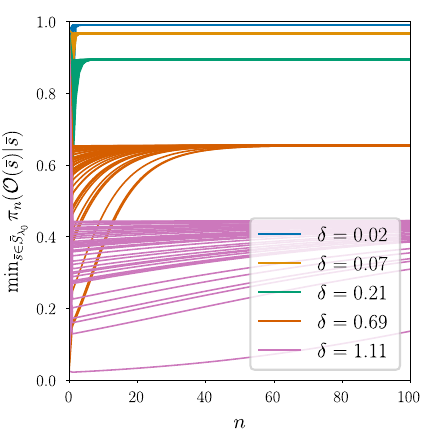}
		\caption{Dependency on iteration.}
        \label{fig:bandit:b}
	\end{subfigure}
    \caption{Illustration of estimates of theorem~\ref{le:limdetcont}, point 2. Plots show the behavior of $\min_{\bar{s}\in\theinterestingstates}\pi_n(\oactions(\bar{s})|s)$ for a $2$-armed bandit. Plot (a) shows the dependency on the distance $\delta$ to a deterministic kernel, where orange points depict accumulation points (approximated by values at $n=100$) and the estimate is depicted in blue (cf.\ corollary~\ref{le:limitbounds}). To illustrate the approach of \eUDRL{} towards the accumulation points plot (b) shows the dependency on the iteration $n$ for varying distances to a deterministic kernel (highlighted by different colors) and varying initial policy.}
    \label{fig:bandit}
\end{figure}

\subsubsection{Extending the continuity results to other segment sub-spaces}
\label{ssse:extendingfinite}

As before we need to modify theorem~\ref{le:limdetcont} to cover algorithms like ODT, RWR (and some specific variants of \eUDRL{}) restricting the recursion to $\SegTrail$ or $\SegDiag$ subspaces.
\begin{theorem}
\label{le:limdetcontDiagTrail}
(Relative continuity of accumulation points  of \eUDRL{}-generated policies at deterministic kernels -- the case $\theinterestingstates \subset \supp \bar{\mu}$)
Theorem~\ref{le:limdetcont} remains valid under the renaming $\pi_n \rightarrow \pi_n^{\diag/\trail}$.
\end{theorem}

The proof of theorem~\ref{le:limdetcontDiagTrail} follows along similar lines as that of theorem~\ref{le:limdetcont}. The differences are rather small and resemble those described in section~\ref{sse:extendingfinite}. First, to assert that \eUDRL{} is optimal at deterministic points one applies lemma~\ref{le:detoptDiagTrail} instead of lemma~\ref{le:detopt}.
Second, to prove that that $\pi_{n+1}^{\diag/\trail}$ is well-defined through equation~\eref{eq:pirec} one uses equation \eref{eq:isthesame} and lemma~\ref{le:suppstabTrailDiag}, point 2.~instead of lemma~\ref{le:suppstab}. Third, the application of point 1.~of lemma~\ref{le:recrewrites} in equation \eqref{eq:piM} has to be replaced with points 2.~or 3.~of the lemma respectively. To bound the denominator of~equation~\ref{eq:piM} one proceeds along the same line that leads to equation~\eref{eq:pimf1}, yielding the exactly same result as in~\eref{eq:pimf1}. The same reasoning applies to the induction step.
Finally, using lemma~\ref{le:recrewrites}, which allows one to rewrite the~\eUDRL{} recursion using only $\Seg$-space quantities, we can show that there are no further differences even for $\pi_{n+1}^{\diag/\trail}$.
The same approach applies also to the forthcoming corollary of theorem~\ref{le:limdetcont}, which follows theorem~\ref{le:limdetcont} but provides explicit estimates for \eUDRL{}-generated quantities.
\subsubsection{Estimating the location of accumulation points}
\label{ssse:specmuaccbound}

As noted before an important motivation for this work has been to determine whether or not continuing the \eUDRL{} iteration from a certain state is expected to improve the policy. This can be achieved through explicit estimates for \eUDRL{}-generated quantities and their distance to optimality in terms of the distance of transition kernels. The following corollary can be proved along the same lines as theorems~\ref{le:limdetcont} and~\ref{le:limdetcontDiagTrail} taking care of quantitative estimates explicitly. Recall that a convergent sequence $(y_n)_{n\geq0}$ with limit $y$ is said to be $q$-linearly convergent with rate $\gamma$ if and only if 
$$\lim_{n\rightarrow\infty} \Bigg|\frac{y_{n+1}-y}{y_{n}-y}\Bigg| = \gamma \in (0,1)
.$$

\begin{corollary}
\label{le:limitbounds}
(Estimating the location of accumulation points -- the case $\theinterestingstates \subset \supp \bar{\mu}$)
Under the conditions of theorem~\ref{le:limdetcont} assume that $1 > \delta > 0$ and set, as before, $\alpha = \frac{2}{N(N-1)}\min_{\bar{s}\in \theinterestingstates} \bar{\mu}(\bar{s})$.
Define the quantities $\tilde\beta=\frac{N\delta}{2}$, $x^*(\gamma)=1-\gamma$ (cf.~lemma~\ref{le:f}) and for a horizon $h$, $1\leq h\leq N$,
\begin{align*}
\beta_h &= \begin{cases}\max\{\delta,\tilde{\beta}\},\ \textnormal{if}\ h=1,\\
\delta + \kappa_{h-1} + \beta_{h-1},\ \textnormal{if} \ h \geq 2,
\end{cases}
\\
\gamma_h &= \frac{\tilde{\beta}}{(1-\beta_h)\alpha},
\\
\kappa_h &= 2( 1-x^*(\gamma_h) ).
\end{align*}
Further assume $\beta_h,\gamma_h \in (0,1)$ and notice that $\beta_h,\gamma_h,\kappa_h$ are increasing in $h$ and that $\beta_h,\gamma_h,\kappa_h\rightarrow 0$ as $\delta \rightarrow 0$. Then the following assertions hold for all $\tkernel \in U_{\delta}(\tkernel_0)$ 
and for all $\pi_0 >0$:
\begin{enumerate}
    \item 
$\displaystyle
\limsup_n \max_{\bar{s} \in \theinterestingstates }
2 ( 1- \pi_n(\oactions(\bar{s})|\bar{s})) \leq \kappa_{N}
$,
\item
$\displaystyle
\limsup_n \max_{(\bar{s},a) \in \theinterestingstates\times\mathcal{A} }
|Q_{\tkernel}^{\pi_n}(\bar{s},a) - Q_{\tkernel_0}^{*}(\bar{s},a)|
\leq \beta_N 
$,
\item
$\displaystyle
\limsup_n \max_{\bar{s} \in \theinterestingstates}
|V_{\tkernel}^{\pi_n}(\bar{s})- V_{\tkernel_0}^{*}(\bar{s}) |
\leq
\beta_N + \kappa_N 
$,
\item
$\displaystyle
\limsup_n 
|J_{\tkernel}^{\pi_n}- J_{\tkernel_0}^{*} |
\leq
\frac{N\delta}{2} + \beta_N + \kappa_N 
$,\ \textnormal{and}
\item for all $\epsilon>0$
there exist a sequence of $\tkernel_0$-optimal policies $(\pi_{\tkernel_0,n}^{*})_{n\geq 0}$ and $n_0$ so that for all $n\geq n_0$ and for all $\bar{s} \in \theinterestingstates$ it holds that
$$
\|
\pi_{n,\tkernel}(\cdot|\bar{s})
- \pi_{\tkernel_0,n}^{*} (\cdot|\bar{s}) 
\|_1
 \leq 2(1-f_{\gamma'}^{\circ (n-n_0)}(x_0)),
$$
where $x_0= \min_{\bar{s} \in \theinterestingstates} \pi_{n_0,\tkernel}(\oactions(\bar{s})|\bar{s}) > 0$,  $\beta'= \beta_N+\epsilon$, where $\epsilon$ is chosen such that $\gamma' = \frac{\tilde{\beta}}{(1-\beta')\alpha}\in(0,1)$.
\end{enumerate}
The terms on the right hand side of 1.-4.~converge to $0$ as $\delta \rightarrow 0$. In 5.~the sequence $(y_n)_{n\geq0}$ defined by 
$
y_n = 2(1-f_{\gamma'}^{\circ (n-n_0)}(x_0))
$
converges $q$-linearly to the limit $y=2(1-x_{\gamma'}^{*})$ at a convergence rate of $\gamma'$ which tends to 0 as $\delta \rightarrow 0$.
\end{corollary}

The bounds in the corollary are solely dependent on \(\alpha\) and \(\delta\), where \(\alpha\) is determined by the initial distribution and \(\delta\) represents the distance to a deterministic kernel. Notably, there is no dependence on the initial policy \(\pi_0 > 0\). This feature allows the estimates to be applied even when no information about the initial policy is available. However, this comes at the cost of producing relatively loose estimates. While the initial policy does not affect the asymptotic estimates, it can affect convergence time, compare figure~\ref{fig:bandit:b} and figure~\ref{fig:f} and consider the effect of $x_0$ in the point 5. All estimates converge to \(0\) as \(\delta \to 0\), which implies the continuity of the respective quantities at deterministic kernels (relative continuity in the case of the policy). This can be interpreted as a form of continuity of the set of accumulation points (see the remark below). It will also sometimes be more convenient to work with $\pi_{n,\tkernel}(\oactions(\bar{s})|\bar{s})$ rather than $2(1-\pi_{n,\tkernel}(\oactions(\bar{s})|\bar{s})$. We summarize this in a remark.
\begin{remark}
The estimate in point 1.~of corollary~\ref{le:limitbounds} is equivalent to the following estimate {${\liminf_n \min_{\bar{s} \in \theinterestingstates}
\pi_n(\oactions(\bar{s})|\bar{s}) \geq x^*(\gamma_{N})},$} cf.\ figure \ref{fig:bandit}, \ref{fig:gridworld}. The estimate in point 5.\ of the corollary is equivalent to $\pi_{n,\tkernel}(\oactions(\bar{s})|\bar{s}) \geq f_{\gamma'}^{\circ (n-n_0)}(x_0)$.
\end{remark}
\begin{remark}\label{re:accreformulation} The estimates in the corollary can be stated also in terms of accumulation points of \eUDRL{}-generated quantities. For example point 1.~is equivalent to (for all $\tkernel \in U_{\delta}(\tkernel_0)$ and $\pi_0 >0$)
$$
\sup_{\pi \;\text{accumulation point of}\; (\pi_{n,\tkernel})}
\max_{\bar{s}\in \theinterestingstates}
2(1 - \pi(\oactions(\bar{s})|\bar{s})) \leq \kappa_N.$$
\end{remark}

\begin{proof}
The choice of $1 >\delta >0$ and the form of $U_{\delta}(\tkernel_0)$ is compliant with lemma~\ref{le:suppstab}.
By this we mean that for all $\delta$ (with $U_{\delta}(\tkernel_0)$ defined as above) the points 1., 2., and 3.\ of the lemma~\ref{le:suppstab} hold.
Similarly, the choice of $\alpha$ is compliant with lemma~\ref{le:alpha}.
The restriction on $\gamma_h$: $1>\gamma_h >0$ comes from
lemma~\ref{le:f}, which ensures that we can use them safely with $f_{\gamma},x^*(\gamma)$
as defined in this lemma.

The result for the limits of $\beta_h,\kappa_h,\gamma_h,\tilde{\beta}$ follows from their definitions
and the fact that $x^*(\gamma)$ is a decreasing function in $\gamma$ with a  limit 1 for $\gamma\rightarrow 0$. Further, that $\beta_h,\kappa_h$ are increasing in $h$ follows again from their definition and the previously stated
properties of the function $x^*$.

During the proof we will also utilize that 
for all $\tkernel \in U_{\delta}(\tkernel_0)$,$\bar{s} \in \theinterestingstates$,$a\notin \oactions(\bar{s})$ holds
$$
Q_{\tkernel}^{*}(\bar{s},a) \leq \tilde{\beta} \leq \beta_h,
$$
where $N\geq h\geq 1$. 
The first inequality follows from lemma~\ref{le:opthbound}. The inequality $\tilde{\beta} < \beta_1$ follows from the definition of $\beta_1$. The rest follows
from the previously stated increasing property of $\beta_h$. The result about $Q_{\tkernel}^{*}(\bar{s},a)$ is useful for ensuring that
all action value function-related quantities $Q_{\tkernel}^{\pi_n,h,g}((s,h',g'),a)$ for non-optimal actions $a$ are bounded by $\beta_h$. 

\emph{the claim:}
First we aim to prove the following claim:
$$
\begin{aligned}
(\forall  N \geq h \geq 1, \forall \pi_0 >0 , \forall \tkernel \in U_{\delta}(\tkernel_0)) :\:
&\limsup_n
\max_{\substack{\bar{s} = (s,h',g) \in\theinterestingstates,h'=h,\\a \in \mathcal{A}}}
|Q_{\tkernel}^{\pi_n}(\bar{s},a) - Q_{\tkernel_0}^{*}(\bar{s},a)| \leq \beta_h,
\\
&\limsup_n
\max_{\bar{s} = (s,h',g) \in\theinterestingstates,h'=h}
2 (1-\pi_n(\oactions(\bar{s})|\bar{s}))
\leq \kappa_h.
\end{aligned}
$$
We will proceed by induction by the remaining horizon $h$.

\emph{the claim, base case:}
Assume $h=1$ and $(s,1,g)\in \theinterestingstates$. It holds
\begin{multline*}
|Q_{\tkernel}^{\pi_n}((s,1,g),a) - Q_{\tkernel_0}^{*}((s,1,g),a)|
=
\left|\sum_{s'\in \rho^{-1}(\{g\})} \tkernel(s'|s,a) -  \sum_{s'\in \rho^{-1}(\{g\})} \tkernel_0(s'|s,a)\right| 
\\
\leq
\sum_{s'\in \rho^{-1}(\{g\})} |\tkernel(s'|s,a) - \tkernel_0(s'|s,a)|
\leq
\| \tkernel(\cdot|s,a) - \tkernel_0(\cdot|s,a)\|_1,
\end{multline*}
which gives us, for all $\pi_0 >0$, $\tkernel \in U_{\delta}(\tkernel_0)$ and $n\geq 0$,
\begin{align*}
\max_{\substack{\bar{s}=(s,1,g)\in\theinterestingstates,\\a\in \mathcal{A}}}
|Q_{\tkernel}^{\pi_n}((s,1,g),a) - Q_{\tkernel_0}^{*}((s,1,g),a)|
&\leq
\max_{\substack{\bar{s}=(s,1,g)\in\theinterestingstates,\\a\in \mathcal{A}}}
\| \tkernel(\cdot|s,a) - \tkernel_0(\cdot|s,a)\|_1
\\
&\leq \delta \leq \beta_1,
\end{align*}
from which the first inequality in the claim follows.
For the second inequality, we will proceed like in \eref{eq:pimf1}
bounding $v_{\tkernel,\pi_n}(h,g|s,h)$ terms by $\alpha$ and action value function-related terms using $\beta_1,\tilde{\beta}$. Notice that for $h=1$ the $\limsup_n$ does not really matter
(see the last equation: $\beta_1$ bounds the error for all $n$).
This gives us (for all $n\geq 0$, $\pi_0>0$, $\tkernel\in U_{\delta}(\tkernel_0)$, $\bar{s} = (s,1,g) \in \theinterestingstates$)
$$
\pi_{n+1,\tkernel}(\oactions(\bar{s})|\bar{s}) \geq f_{\gamma_1}(\pi_{n,\tkernel}(\oactions(\bar{s})|\bar{s})).
$$
Using point 4.\ of lemma~\ref{le:f} then leads to 
$$
\liminf_n \pi_{n,\tkernel}(\oactions(\bar{s})|\bar{s}) \geq x^*(\gamma_1).
$$
After rearranging, we get (for all $\pi_0>0$, $\tkernel\in U_{\delta}(\tkernel_0)$)
$$
\limsup_n \max_{\bar{s}=(s,1,g)\in\theinterestingstates} 2(1-\pi_{n,\tkernel}(\oactions(\bar{s})|\bar{s})) \leq 2(1-x^*(\gamma_1)) = \kappa_1.
$$

\emph{the claim, the induction step $h\geq 2$:} Now assume that the claim holds for $h-1$. We aim to prove it for $h$.
For the first inequality of the claim we will use the the recursive estimate from 
the lemma~\ref{le:contQfac} (after substituting $\pi$ with $\pi_n$ and extending maximization in the last term from $a' \in \oactions(\bar{s}')$ to $a' \in \mathcal{A}$):
\begin{multline*}
(\forall \bar{s} = (s,h,g) \in \theinterestingstates, h\geq 2, \forall a\in \oactions(\bar{s}),\forall n\geq 0 ) :
\\
|Q_{\tkernel}^{\pi_n}(\bar{s},a)-Q_{\tkernel_0}^{*}(\bar{s},a) |
\leq
\|\tkernel(\cdot|s,a)-\tkernel_0(\cdot|s,a)\|_1
+
\max_{\bar{s}'=(s',h-1,g) \in\theinterestingstates}
2(1-\pi_n(\oactions(\bar{s}')|\bar{s}')).
\\
+
\max_{\bar{s}'=(s',h-1,g) \in\theinterestingstates, a' \in \mathcal{A}}
|Q_{\tkernel}^{\pi_n}(\bar{s}',a')-Q_{\tkernel_0}^{*}(\bar{s}',a')|.
\end{multline*}
Now we bound $\|\tkernel(\cdot|s,a)-\tkernel_0(\cdot|s,a)\|_1$ by $\delta$ and 
apply $\limsup_n(\cdot)$ on both sides of the inequality leading to
\begin{multline*}
\limsup_n |Q_{\tkernel}^{\pi_n}(\bar{s},a)-Q_{\tkernel_0}^{*}(\bar{s},a) |
\leq
\delta
+
\limsup_n
\max_{\bar{s}'=(s',h-1,g) \in\theinterestingstates}
2(1-\pi_n(\oactions(\bar{s}')|\bar{s}'))
\\
+
\limsup_n
\max_{\bar{s}'=(s',h-1,g) \in\theinterestingstates, a' \in \mathcal{A}}
|Q_{\tkernel}^{\pi_n}(\bar{s}',a')-Q_{\tkernel_0}^{*}(\bar{s}',a')|
\leq \delta + \kappa_{h-1} + \beta_{h-1} = \beta_h, 
\end{multline*}
where we applied the induction assumption.
Further for all $\bar{s} = (s,h,g) \in \theinterestingstates$ and 
$a\notin \oactions(\bar{s})$ holds
$|Q_{\tkernel}^{\pi_n}(\bar{s},a)-Q_{\tkernel_0}^{*}(\bar{s},a) | = Q_{\tkernel}^{\pi_n}(\bar{s},a) \leq Q_{\tkernel}^{*}(\bar{s},a) \leq \tilde{\beta} \leq \beta_1 \leq \beta_h$. After applying $\limsup_n$ we get
$$
\limsup_n |Q_{\tkernel}^{\pi_n}(\bar{s},a)-Q_{\tkernel_0}^{*}(\bar{s},a) | \leq \beta_h.
$$
Putting both results together and maximizing over $\bar{s},a$ on the left-hand side
(it is possible since the right-hand side is independent of $\bar{s},a$ and $\bar{\mathcal{S}}$ and $\mathcal{A}$ are finite) we obtain the first inequality
of the claim.

Now we aim to prove the second inequality of the claim.
We proceed similarly as in the case $h=1$, the difference is that now
$\limsup_n$ in the previously proved inequality with $\beta_{h}$ is important. To be able to bound the errors of the action value function-related terms for the horizon $h$ we have to
enlarge $\beta_{h}$; define $\beta' = \beta_{h}+\epsilon$, with
$\epsilon >0$ being a fixed arbitrary value so that $1> \gamma':= \frac{\tilde{\beta}}{(1-\beta')\alpha}$; and fix the $\pi_n$
sequence such that $\pi_0 >0$, $\tkernel \in U_{\delta}(\tkernel_0)$. Then
we can find $n_0$ such that for all $n \geq n_0$ it holds that (note that $\theinterestingstates$ is finite)
$$
\max_{\bar{s}=(s,h',g)\in\theinterestingstates, h'=h, a\in \mathcal{A}}
|Q_{\tkernel}^{\pi_n}(\bar{s},a)-Q_{\tkernel_0}^{*}(\bar{s},a)|
\leq
\beta'.
$$
Then we can continue as before and get for all $n\geq n_0$, and all $\bar{s}=(s,h',g)\in\theinterestingstates, h'=h$
$$
\pi_{n+1,\tkernel}(\oactions(\bar{s})|\bar{s}) = f_{\gamma'}(\pi_{n,\tkernel}(\oactions(\bar{s})|\bar{s})).
$$
Using point 4.\ of lemma~\ref{le:f} then leads again to (for all $\bar{s}=(s,h',g)\in\theinterestingstates, h'=h$)
$$
\liminf_n \pi_{n,\tkernel}(\oactions(\bar{s})|\bar{s}) \geq x^*(\gamma').
$$
Since we could choose $\epsilon >0$ arbitrarily (and $x^*(\gamma)$ is continuous in $\gamma$ and $\gamma=\frac{\tilde{\beta}}{(1-\beta)\alpha}$ is continuous in $\beta$) we get (for all $\bar{s}=(s,h',g)\in\theinterestingstates, h'=h$)
$$
\liminf_n \pi_{n,\tkernel}(\oactions(\bar{s})|\bar{s}) \geq x^*(\gamma_h),
$$
and since the right side is independent of the fixed sequence we get
$$
(\forall \pi_0 >0, \forall \tkernel \in U_{\delta}(\tkernel_0),
\forall \bar{s}=(s,h',g)\in\theinterestingstates, h'=h): \;
\liminf_n \pi_{n,\tkernel}(\oactions(\bar{s})|\bar{s}) \geq x^*(\gamma_h).
$$
From this follows
$$
(\forall \pi_0 >0, \forall \tkernel \in U_{\delta}(\tkernel_0) ): \;
\limsup_n \max_{\bar{s}=(s,h',g)\in\theinterestingstates, h'=h} 2(1-\pi_{n,\tkernel}
(\oactions(\bar{s})|\bar{s}) \leq 2(1-x^*(\gamma_h)) = \kappa_h
.
$$
This completes the proof of the claim from which also follows both points 1.\ and 2.

3.
For all $\tkernel\in U_{\delta}(\tkernel_0)$, $\pi_0>0$, $n\geq 0$, $\bar{s}\in \theinterestingstates$ it holds that
$$
\begin{aligned}
|V_{\tkernel}^{\pi_n}(\bar{s})- V_{\tkernel_0}^{*}(\bar{s}) |
&\leq
|\sum_{a\in \mathcal{A}} Q_{\tkernel}^{\pi_n}(\bar{s},a)\pi_{n,\tkernel}(a|\bar{s})
- \sum_{a\in \mathcal{A}} Q_{\tkernel_0}^{*}(\bar{s},a)\pi_{n,\tkernel}(a|\bar{s})|
\\
&\quad +
|\sum_{a\in \mathcal{A}} Q_{\tkernel_0}^{*}(\bar{s},a)\pi_{n,\tkernel}(a|\bar{s})
-
\sum_{a\in \mathcal{A}} Q_{\tkernel_0}^{*}(\bar{s},a)\pi_{\tkernel_0}^{*}(\pi_{n,\tkernel})(a|\bar{s})
|
\\
&\leq
\max_{\bar{s},a\in\theinterestingstates\times\mathcal{A}}
|Q_{\tkernel}^{\pi_n}(\bar{s},a)
- Q_{\tkernel_0}^{*}(\bar{s},a)|
+
\max_{\bar{s}\in\theinterestingstates}
2(1-\pi_{n,\tkernel}(\oactions(\bar{s})|\bar{s})).
\end{aligned}
$$
Since the right side of the inequality is not dependent on 
$\bar{s}$ we can also maximize over $\bar{s}$ on the left side and the apply $\limsup_n(\cdot)$, which completes the proof of this point.

4.
For all $\tkernel\in U_{\delta}(\tkernel_0)$, $\pi_0>0$, $n\geq 0$, $\bar{s}\in \theinterestingstates$ it holds that
\begin{align*}
|J_{\tkernel}^{\pi_n}- J_{\tkernel_0}^{*} |
&\leq
\sum_{\bar{s}\in\supp\bar{\mu}} \bar{\mu}(\bar{s})
|V_{\tkernel}^{\pi_n}(\bar{s})- V_{\tkernel_0}^{*}(\bar{s}) |
\\
&=
\sum_{\bar{s}\in\supp\bar{\mu}\setminus\theinterestingstates} 
\bar{\mu}(\bar{s})
V_{\tkernel}^{\pi_n}(\bar{s})
+
\sum_{\bar{s}\in\theinterestingstates} \bar{\mu}(\bar{s})
|V_{\tkernel}^{\pi_n}(\bar{s})- V_{\tkernel_0}^{*}(\bar{s}) |.
\end{align*}
Because of $\supp \bar{\mu} \subset \supp \nu_{\tkernel_0,\pi_0}$
it holds that $\supp \bar{\mu} \setminus \theinterestingstates = \supp \bar{\mu} \setminus \supp \den_{\tkernel_0,\pi_0}$.
Since $V_{\tkernel_0}^{*}(\bar{s}) = 0$ on $\supp \bar{\mu} \setminus  \supp \den_{\tkernel_0,\pi_0}$ (similar to corollary~\ref{le:detJcont}), we can use lemma~\ref{le:opthbound} to bound $V_{\tkernel}^{\pi_n}(\bar{s})$ by $\frac{N\delta}{2}$. The result follows by applying  $\limsup_n(\cdot)$ and then using point 3.\ above.

5.
We choose the sequence $(\pi_{\tkernel_0,n}^{*})_{n\geq 0}$ so that $\pi_{\tkernel_0,n}^*(a\mid\bar{s}) \geq \pi_{n,\tkernel}(a\mid\bar{s})$ for all $\bar{s}\in\theinterestingstates$ and all $a\in\oactions(\bar{s})$, which implies 
$\|
\pi_{n,\tkernel}(\cdot|\bar{s})
- \pi_{\tkernel_0,n}^{*} (\cdot|\bar{s}) 
\|_1
= 2(1-\pi_{n_0,\tkernel}(\oactions(\bar{s})|\bar{s})) 
$.
The first statement follows the same logic as the second statement of the claim
and point 4.\ of lemma~\ref{le:f}.
The only difference is that we are now bounding for all the remaining horizons.
Notice that the $\beta'$ we choose works for all the horizons because
$\beta' > \beta_N \geq \beta_h$, $h\leq N$.
The correctness of the choice of $x_0$ follows from the fact that $f_{\gamma'}$ is
increasing.
The statement about $y_n \rightarrow y$ follows also from point 4.\ of lemma~\ref{le:f}.
Finally, for the q-linear convergence we have (let us denote $x_n := f_{\gamma'}^{\circ(n-n_0)}(x_0)$)
\begin{align*}
\lim_{n\rightarrow\infty}
\frac{y_{n+1}-y_L}{y_{n}-y_L}
&=
\lim_{n\rightarrow\infty}
\frac{2(1-f_{\gamma'}(x_n))-2(1-x^{*}(\gamma'))}{2(1-x_n)-2(1-x^{*}(\gamma'))}
=
\lim_{n\rightarrow\infty}
\frac{ f_{\gamma'}(x_n)-x^{*}(\gamma')}{x_n-x^{*}(\gamma')}
\\
&=^{Heine}
\lim_{x\rightarrow x^{*}(\gamma')}
\frac{ f_{\gamma'}(x)-x^{*}(\gamma')}{x-x^{*}(\gamma')}
\\
&=^{L'Hospital}
\lim_{x\rightarrow x^{*}(\gamma')}
f_{\gamma'}'(x)
=
\lim_{x\rightarrow x^{*}(\gamma')}
( \frac{x}{x+ \gamma'} )'
=
\frac{
\gamma'
}{
(x^{*}(\gamma') + \gamma')^2
}
\\
&=
\gamma' < 1
,
\end{align*}
where we used Heine theorem, L'Hospital rule and definition of $x^*(\gamma')$ from lemma~\ref{le:f}.
\end{proof}

Given a deviation of $\delta$ from a deterministic kernel, corollary~\ref{le:limitbounds} can be used to bound the deviation in the accumulation points of, e.g., the goal-reaching objective (point 4.), or the policy (points 1.~and 5.). An illustration of the accumulation points of $\min_{\bar{s}\in\theinterestingstates}\pi_n(\oactions(\bar{s})|s)$ and the estimates in corollary~\ref{le:limitbounds} is provided in figure~\ref{fig:bandit} in the case of a $2$-armed bandit model (see section~\ref{ssse:specmutheorem} and the description there) and in figure~\ref{fig:gridworld} in case of a simple grid world model. All details of the construction of the MDP underlying figure~\ref{fig:gridworld} can be found in the appendix B, example~\ref{ex:gridworld}. Figure~\ref{fig:gridworld:a} shows how the bound in the corollary, point 5.~(blue line) and the accumulation points of \eUDRL{}-generated policies (orange marks) behave as functions of the distance $\delta$ to a deterministic kernel.  Figure~\ref{fig:gridworld:b} shows how \eUDRL{} approaches those accumulation points as a function of the iteration $n$. Varying initial policies are visible at $n=0$ and varying levels of $\delta$ are highlighted in different colors. Figure~\ref{fig:gridworld:c} contains the same data as Figure~\ref{fig:gridworld:b} but it is organized to reveal the dependency on $\delta$. Notice that initial conditions uniformly fill the plot's area. Figure~\ref{fig:gridworld:d} shows a map of the grid world. Similar to the example presented in figure~\ref{fig:bandit}, it is apparent that the bound in figure~\ref{fig:gridworld:a} (blue line) originates from continuity: it becomes tight as $\delta$ approaches $0$ and becomes loose as $\delta$ increases. Regrettably, the bound deteriorates quickly as $\delta$ increases above $\sim10^{-5}$, which limits the applicability in a practical setting. The origin of this behavior might be seen in lemma~\ref{le:alpha}, which uses a uniform bound $\alpha$ to estimate the visitation probabilities ignoring the specific dynamics of the environment. It remains a question for forthcoming research to develop methods that provide practical estimates that hold on intervals of $\delta$, $N$, number of states, etc.~of practical interest.

\begin{figure}
  \begin{subfigure}{0.48\linewidth}
		\includegraphics[width=\linewidth]{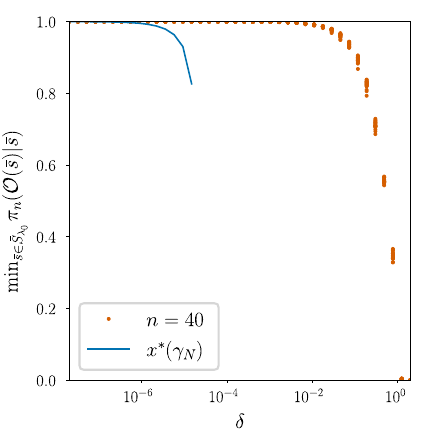}
		\caption{Dependency on distance to determin.~kernel.}
        \label{fig:gridworld:a}
	\end{subfigure}
     \begin{subfigure}{0.48\linewidth}
		\includegraphics[width=\linewidth]{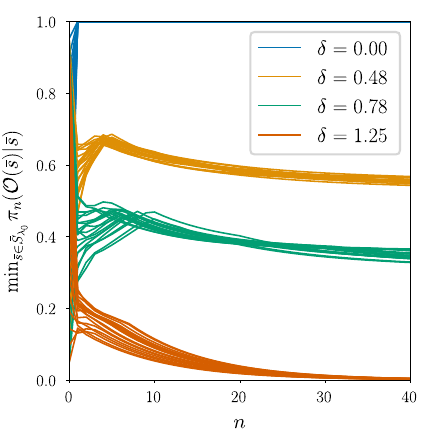}
		\caption{Dependency on iteration.}
        \label{fig:gridworld:b}
	\end{subfigure}\\
     \begin{subfigure}{0.48\linewidth}
		\includegraphics[width=\linewidth]{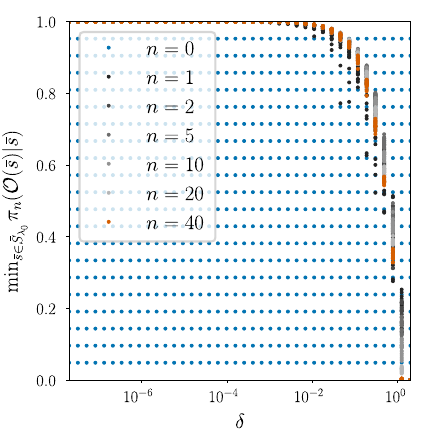}
		\caption{Dependency on distance to determin.~kernel.}
        \label{fig:gridworld:c}
	\end{subfigure}
     \begin{subfigure}{0.48\linewidth}
		\includegraphics[width=\linewidth]{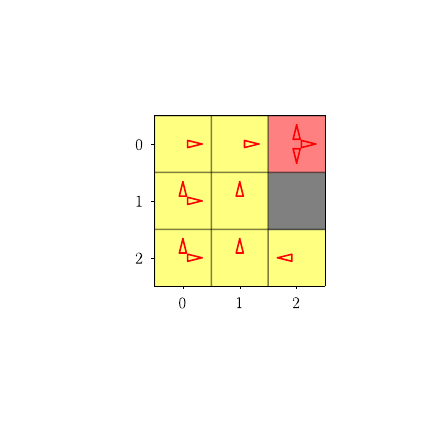}
		\caption{Map of the grid world}
        \label{fig:gridworld:d}
	\end{subfigure}
    \caption{Illustration of estimates of corollary~\ref{le:limitbounds} on the behavior of $\min_{\bar{s}\in\theinterestingstates}\pi_n(\oactions(\bar{s})|\bar{s})$ for the simple grid world example shown in plot (d). Plot (a) shows the dependency of accumulation points (depicted in orange, approximated by values at $n=100$) and the estimate in the corollary (blue). To illustrate the approach of \eUDRL{} towards the accumulation points, plot (b) shows the dependency on the iteration $n$ for varying distances to a deterministic kernel highlighted by different colors and varying initial policy. Plot (c) contains the same information as (b) but it is organized to reveal the dependency on  $\delta$, where varying numbers of iteration are highlighted by different colors. Plot (d) shows the map of the grid world. A wall is depicted in gray, the set $\theinterestingstates$ in yellow, {and the goal in red}. Arrows depict optimal actions associated with the specific state and goal.}
    \label{fig:gridworld}
\end{figure}

\subsection{Asymptotic Continuity Assuming Uniqueness of the Optimal Policy}
\label{sse:specM1}
The main result of this section is to demonstrate the continuity of accumulation points of the sequence of \eUDRL{}-generated policies under the assumption that the optimal policy is unique and deterministic. This assumption can be concisely expressed as follows: for a deterministic kernel $\tkernel_0$ the set of optimal actions has exactly one element $|\oactions(\bar{s})|=1$ on $\theinterestingstates$ (compare definitions~\ref{def:optimalAction} and~\ref{de:theinterestingstates} point 3). Recall also that restricting to $\theinterestingstates$ does not result in a loss of generality as the values (and potential discontinuity) of a policy outside $\theinterestingstates$ have no impact on the continuity of goal-reaching objective. The assumption $|\oactions(\bar{s})|=1$ entails that the factorization with respect to the set of optimal actions becomes trivial and the continuity relative to optimal actions reduces to ordinary continuity. In this situation theorem~\ref{le:detcont} implies that \eUDRL{}-generated policies are continuous at $\tkernel_0$. Working with continuous policies, as opposed to relatively continuous ones, leads to a substantial simplification compared to the discussion in the section~\ref{sse:specmu}. The uniqueness condition on the optimal policy is commonly encountered in publications on the analysis of convergence and stability of RL algorithms. But in the context of \eUDRL{} it is not straightforward to see how this condition could be employed without loss of generality. Similar to the preceding sections, the core idea that enables our asymptotic continuity analysis is to derive bounds for the \eUDRL{} policy recursion through monotonicity bounds for repeated application of rational functions (see remark~\ref{re:monotonicity}). This leads to the study of a dynamical system (see lemma~\ref{le:hlemma}), whose fixed-point properties reveal information about the accumulation points of \eUDRL{}'s policy recursion as $\tkernel\rightarrow\tkernel_0$.

The structure and content of this section roughly align with section~\ref{sse:specmu}. The discussion is often shorter due to the assumption of a unique optimal policy but our analysis is based on the study of a different rational function, which requires an additional lemma about the dynamical system properties. We derive the preliminary lemma in section~\ref{ssse:specM1lemmas}. The main theorem is contained in section~\ref{ssse:specmutheorem}. The result is generalized to other segment subspaces in section~\ref{ssse:specM1extending}. As a corollary of the main result, we provide explicit estimates on the location of accumulation point sets in section~\ref{ssse:specM1accbound}. This yields upper bounds on the errors of policies, values and goal-reaching objective. We conclude the discussion with examples.

\subsubsection{Preliminary lemma}
\label{ssse:specM1lemmas}

Our analysis of \eUDRL{}'s convergence in case of a unique optimal policy draws on the properties of a dynamical system induced by a certain rational function. The emergence of the dynamical system and its relation to the \eUDRL{} policy recursion are described in detail in section~\ref{ssse:specmulemmas}. Here we only analyze the relevant dynamical system in case of a unique optimal policy.

\begin{lemma}(h-lemma)
\label{le:hlemma}
Let $N\geq 1$ be a natural number, let $b_0 = \frac{1}{2N} (\frac{2N-1}{2N})^{2N-1}$ and let $b\in(0,b_0)$. Let ${h}_b:[0,1] \rightarrow [0,1]$ be defined as
$$
{h}_b(x) = \frac{x^{2N}}{x^{2N} + b}.
$$
Then the following assertions hold.
\begin{enumerate}
\item
${h}_b$ is increasing.
\item
${h}_b$ has exactly three fixed points $0$, $x_l(b)$ and $x_u(b)$, where $0 < x_l(b) < x_u(b) < 1$. It holds that
$$
\begin{aligned}
{h}_b(x) &> x, \quad \textnormal{for} \ x_l(b) < x < x_u(b),\\
{h}_b(x) &< x, \quad \textnormal{for} \ 0 < x < x_l(b), \ \textnormal{or}\ x_u(b) < x \leq 1
.
\end{aligned}
$$
\item $h_b( (x_l(b),1] ) \subset  (x_l(b),1]$
and $h_b^{\circ n}(x) \rightarrow x_u(b)$ for all $x \in (x_l(b), 1 ]$.
\item Let $(y_n)_{n\geq0}$, $y_n \in [0,1]$ be a sequence with $y_0 > x_l(b)$ and $y_{n+1} \geq h_b(y_n)$ for all $n\geq 0$. Then $y_n \geq h_b^{\circ n}(y_0)$ and $\liminf_{n} y_n \geq x_u(b)$.
\item 
There exist unique, continuous, strictly monotonic extensions
$\bar{x}_l:[0,b_0] \rightarrow [0,\frac{2N-1}{2N}]$, $\bar{x}_u:[0,b_0]\rightarrow [\frac{2N-1}{2N},1]$, where $\bar{x}_l$ is increasing and $\bar{x}_u$ is decreasing, so that $\bar{x}_l(b) = x_l(b)$ and $\bar{x}_u(b) = x_u(b)$ on $(0,b_0)$.
It follows that $x_l(b) \rightarrow 0$, $x_u(b) \rightarrow 1$ as
$b \rightarrow 0$ and $x_l(b) \rightarrow \frac{2N-1}{2N}$, $x_u(b) \rightarrow \frac{2N-1}{2N}$ as $b \rightarrow b_0$. Since $\bar{x}_l$ is just an extension of $x_l$ we will denote it by the same symbol $x_l$. Similarly, we denote $\bar{x}_u$ by $x_u$.
\item 
For $0 < b < b_0$ the values $x_u(b)$ and $x_l(b)$
can be computed by iterative application of the strict contractions $g_l:[0,\frac{2N-1}{2N}]\rightarrow [0,\frac{2N-1}{2N}]$, $g_l(x) = (\frac{b}{1-x})^{\frac{1}{2N-1}}$ and $g_u:[\frac{2N-1}{2N},1]\rightarrow [\frac{2N-1}{2N},1]$, $g_u(x) = 1-\frac{b}{x^{2N-1}}$.
\end{enumerate}
\end{lemma}

\begin{proof}
\emph{1.} The derivative of ${h}$ is non-negative. It is positive if $x>0$.

\emph{2.} $0$ is a solution of ${h}(x) = x$ and therefore
it is a fixed point. If $x>0$ then
${h}(x) > x \iff u(x) := x^{2N-1}(1-x) > b$ and 
${h}(x) = x \iff u(x) = b$. Define $u(0):=0$, such that $u(x)=0$ implies $x\in\{0,1\}$. Computing the derivative one can verify that
$u$ has two stationary points $x=0$ and $x=\frac{2N-1}{2N}$ and is increasing on $[0,\frac{2N-1}{2N}]$ and decreasing on $[\frac{2N-1}{2N},1]$. $u$ has a local maximum at $x=\frac{2N-1}{2N}$ with $u(\frac{2N-1}{2N}) =  \frac{1}{2N}(\frac{2N-1}{2N})^{2N-1}$. Making use of the assumptions $0 < b < \frac{1}{2N}(\frac{2N-1}{2N})^{2N-1}$ the intermediate value theorem implies that $u(x) =  b$ has exactly two solutions on $[0,1]$. The first solution, $x_l$, lies in $(0,\frac{2N-1}{2N})$. Similarly, the second solution, $x_u$, lies in $(\frac{2N-1}{2N},1)$. The rest of the statement follows since $u(x) > b$ on $(x_l,x_u)$ and $u(x) < b$ on $[0,x_l) \cup (x_u,1]$.

\textit{3.}
Denote $h_1=h\restriction_{I_1}$, $h_2 =h\restriction_{I_2}$, where 
$I_1=(x_l,x_u]$, $I_2= [x_u,1]$.
Since $h$ is increasing $x_l = h_1(x_l) < h_1(I_1) \leq h_1(x_u) = x_u$ and thus 
$h_1(I_1) \subset I_1$, i.e., $h_1:I_1\rightarrow I_1$. Similarly
$h_2:I_2\rightarrow I_2$. We show that if $x \in (x_l,1]$ then $h^{\circ n}(x) \rightarrow x_u$. If $x \in I_1$ then $h^{\circ n}(x) = h_1^{\circ n}(x)$ is an increasing and bounded sequence that has a limit. Since $h$ is continuous this limit is a fixed point. Since
there are two fixed points in $\bar{I}_1$ and
$x_l$ cannot be a limit for any $x \in I_1$ (since $h_1^{\circ n}(x) \geq x > x_l$) it follows $h_1^{\circ n}(x) \rightarrow x_u$. Similarly in the remaining case 
$h_2^{\circ n}(x) \rightarrow x_u$.

\textit{4.} The proof follows the same line as in point 4.~of lemma~\ref{le:f}.

\textit{5.}
Consider the increasing, continuous function $u_1=u\restriction_{[0,\frac{2N-1}{2N}]}$ and set
$\bar{x}_l = u_1^{-1}$, which is also increasing. For all $b\in (0,b_0)$ $u_1^{-1}(b)\in[0,\frac{2N-1}{2N}]$ is the unique solution of the equation $u_1(x)=b$ in this interval, so that $\bar{x}_l(b)=x_l(b)$. Since $u_1(0)=0$ and $u_1(\frac{2N-1}{2N})=b_0$ it follows
$\bar{x}_l(0) = 0$ and $\bar{x}_l(b_0) = \frac{2N-1}{2N}$.
Since $u_1:[0,\frac{2N-1}{2N}] \rightarrow [0,b_0]$ is continuous,
defined on interval and increasing it is an open map. It follows that $u_1$ is a homeomorphism and $\bar{x}_l$ is continuous. The discussion of $\bar{x}_u$ follows along the same line by consider the decreasing, continuous function $u_2=u\restriction_{[\frac{2N-1}{2N},1]}$ and setting
$\bar{x}_u= u_2^{-1}$.

\textit{6.}
Now assume that $b\in (0,b_0)$. We aim to prove that $g_l$ is a strict contraction on a complete metric space $I_l = [0,\frac{2N-1}{2N}]$ with the fixed point $x_l(b)$.
In order to prove $g_l(I_l) \subset I_l$
assume $x\in I_l$. Since $x \leq \frac{2N-1}{2N} \iff 1-x \geq \frac{1}{2N}$ it follows
$
g_l(x) = (\frac{b}{1-x})^{\frac{1}{2N-1}}
< (\frac{b_0}{\frac{1}{2N}})^{\frac{1}{2N-1}}
= \frac{2N-1}{2N}
$. Further, $g_l(x) = (\frac{b}{1-x})^{\frac{1}{2N-1}}
\geq 0$. Combining those points yields $g_l(x) \in I_l$. Completeness follows since $I_l$ is a compact metric space.
Note that for all $x\in I_l$ it holds that
$\frac{d g_l}{dx}(x) = \frac{b^{\frac{1}{2N-1}}}{2N-1}(\frac{1}{1-x})^{1+\frac{1}{2N-1}}
<
\frac{b_0^{\frac{1}{2N-1}}}{2N-1}(2N)^{1+\frac{1}{2N-1}} = 1
$.
Since $0\leq \frac{d g_l}{dx} < 1$ on $I_l$ and $I_l$ is compact and $\frac{d g_l}{dx}$ is continuous there exists $L < 1$ so that 
$|\frac{d g_l}{dx}| < L$.
By Lagrange's mean value theorem it follows that for all $x,x' \in I_l$ it holds that $|g_l(x)-g_l(x')| < L|x-x'|$. This proves that $g_l$ is a
strict contraction on the complete metric space
$I_l$.
From Banach's fixed point theorem it follows that $g_l$ has exactly
one fixed point and the iteration $g_l^{\circ n}(x_0)$ converges to this fixed point
for $n\rightarrow \infty$ for all $x_0 \in I_l$.
It remains to show that $x_l(b)$ is the fixed point of $g_l$. Since $x_l(b)$ is the solution
of $u_1(x) = b$ we have 
$
b = u_1(x_l(b)) = x_l^{2N-1}(b)(1-x_l(b)).
$
Since $b > 0$ this is equivalent to
$x_l(b) = (\frac{b}{1-x_l(b)})^{\frac{1}{2N-1}}=g_l(x_l(b))$.

The claim about $g_u$ can be proved in similar way. 
Assume that $b\in (0,b_0)$ and $x\in I_u = [\frac{2N-1}{2N},1]$.
Since 
$
g_u(x)
=
1 - \frac{b}{x^{2N-1}}
> 1 - \frac{b_0}{(\frac{2N-1}{2N})^{2N-1}}
=
1 - \frac{1}{2N}
=
\frac{2N-1}{2N}
$
and
$
g_u(x)
=
1 - \frac{b}{x^{2N-1}}
< 1 - \frac{0}{1^{2N-1}}
=
1
$
we get $g_u(x)\in I_u$.
Further for all $x\in I_u$ it holds that $0
\leq \frac{d g_u}{dx}(x)
=
\frac{b(2N-1)}{x^{2N}}
<
\frac{b_0(2N-1)}{(\frac{2N-1}{2N})^{2N}}
=
1
$
for all $x\in I_u$. We conclude (similarly as in case of $g_l$)
that $g_u$ is a strict contraction on a complete metric space.
Since
$
b = u_2(u_2^{-1}(b)) = u_2(x_u(b)) = x_u^{2N}(b) (1-x_u(b))
$
we get
$
x_u(b) = 1-\frac{b}{x_u^{2N}(b)} = g_u(x_u(b))
,
$
i.e.,
$x_u(b)$ is fixed point of $g_u$.
\end{proof}

\begin{figure}
    \centering
    \includegraphics[width=0.75\textwidth]{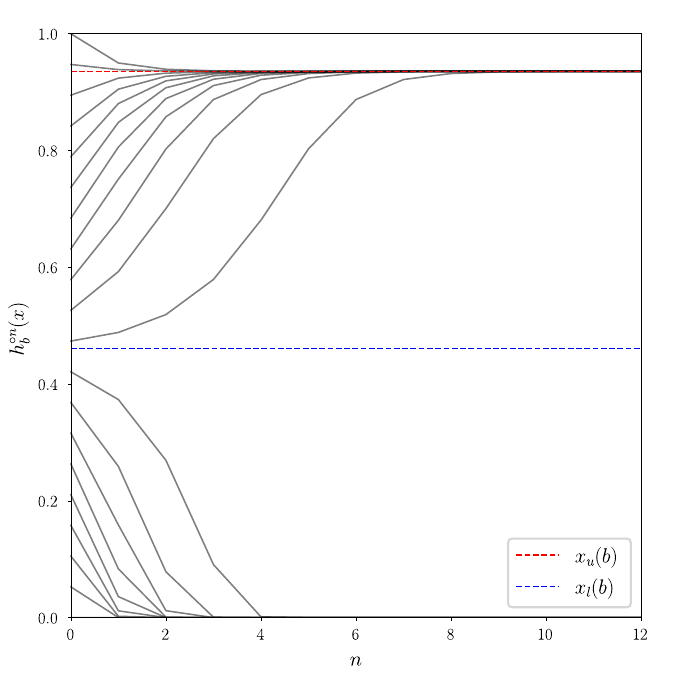}
    \caption{Illustration of the dynamical system induced by iterative application of  $h_b$. The figure shows the dependence of $h_{b}^{\circ n}(x)$ given initial condition $x$ on iteration $n$. The parameter $b$ is chosen as $b=\frac{b_0}{2}$ for $N=2$.
}
    \label{fig:h}
\end{figure}
Figure~\ref{fig:h} illustrates the dynamical system induced by iterative application of $h_{b}$.

\subsubsection{The main theorem}
\label{ssse:specM1theorem}

The following result is a counterpart of theorem~\ref{le:limdetcont} point 2. Instead of assuming $\theinterestingstates\subset \supp\bar\mu$ (as in theorem~\ref{le:limdetcont}) the following theorem applies in a situation where $|\oactions{(\bar{s})}|=1$ on $\theinterestingstates$. Together with lemma~\ref{le:detopt} the theorem demonstrates the continuity of the sets of accumulation points of UDRL generated policies. Since the reasoning is 
the same as in point 1.~of theorem~\ref{le:limdetcont} we skip the analogue of theorem~\ref{le:limdetcont} point 1.~and focus solely on point 2.

\begin{theorem}\label{le:limdetcontM1} (Continuity of \eUDRL{} limit policies at determinist.~points -- the case $|\oactions(\bar{s})|=1$ on $\theinterestingstates$) Let $\{\mathcal{M}_{\tkernel} | \tkernel\in(\Delta \mathcal{S})^{\mathcal{S}\times\mathcal{A}} \}$ and $\{\bar{\mathcal{M}}_{\tkernel} |\tkernel\in (\Delta \mathcal{S})^{\mathcal{S}\times\mathcal{A}} \}$ be compatible families. Let $\tkernel_0\in(\Delta\mathcal{S})^{\mathcal{S}\times\mathcal{A}}$ be a deterministic kernel and let $(\pi_{n,\tkernel})_{n\geq0}$ be a sequence of \eUDRL{}-generated policies with initial condition $\pi_0 > 0$ and transition kernel $\tkernel$. Assume that $|\oactions(\bar{s})|=1$ on $\theinterestingstates$. Then for all $\pi_0>0$ there exists $\delta\in(0,2)$ so that for all $\tkernel \in U_{\delta}(\tkernel_0)$ with
$$U_{\delta}(\tkernel_0) = \{ \tkernel | \max_{(s,a)\in \mathcal{S}\times\mathcal{A}}
\| \tkernel(\cdot|s,a) - \tkernel_0(\cdot|s,a)\|_1 < \delta \}$$
and all $\bar{s} \in \theinterestingstates$ it holds that
$$
\liminf_n \pi_{n,\tkernel}(\oactions(\bar{s})|\bar{s}) \geq x_u(b)\quad\text{and}\quad x_u(b) \rightarrow 1\quad\text{as}\quad\delta \rightarrow 0,
$$
where $$b = \frac{\delta N^2 (N+1)}{4 (1-\frac{\delta}{2})^{2N} \min_{\bar{s}'\in \supp \bar{\mu}} \bar{\mu}(\bar{s}')}$$
and $x_u$ is defined in lemma~\ref{le:hlemma}. It follows that
$$
\liminf_n \pi_{n,\tkernel}(\oactions(\bar{s})|\bar{s})  \rightarrow 1\quad\text{as}\quad\tkernel \rightarrow \tkernel_0.
$$
\end{theorem}
In fact we will prove the slightly stronger assertion that there exists $\delta\in(0,2)$ such that for all $\tkernel \in U_{\delta}(\tkernel_0)$ and $\bar{s} \in \theinterestingstates$ it holds that
\begin{align}
\min_{\bar{s} \in \theinterestingstates} \pi_{n+1,\tkernel}(\oactions(\bar{s})|\bar{s})
\geq h_b(\min_{\bar{s} \in \theinterestingstates} \pi_{n,\tkernel}(\oactions(\bar{s})|\bar{s}))
\label{eq:maintheoremM1}
\end{align}
where the function $h_b$ and quantities $x_u(b)$, $x_l(b)$ are defined in lemma~\ref{le:hlemma}. 
\begin{proof}
We utilize the notation introduced in the proof of theorem~\ref{le:limdetcont}. Assume $\pi_0 >0$ is arbitrary and fixed. Fix arbitrary $\delta\in(0,2)$ and assume $\tkernel \in U_{\delta}(\tkernel_0)$.
We show some estimates that will be used later in the proof.

\textit{Bounding state visitation:} %
 Assume $\bar{s}=(s,h,g) \in \theinterestingstates$.
By lemma~\ref{le:suppstab} point 3.~and lemma~\ref{le:detopt} (optimality at deterministic points) we have that for $n\geq 1$ it holds that
\begin{align*}
&\prob_{\tkernel_0} (S_0^{\Sigma} = s, l(\Sigma)=h, \rho(S_h^{\Sigma})=g, H_0^{\Sigma}=h, 
G_0^{\Sigma}=g;\pi_{\tkernel_0}^*)\\
&=\prob_{\tkernel_0} (S_0^{\Sigma} = s, l(\Sigma)=h, \rho(S_h^{\Sigma})=g, H_0^{\Sigma}=h, G_0^{\Sigma}=g;\pi_n ) >0.
\end{align*}
It follows that there exists $t$ such that $\prob_{\tkernel_0} (\bar{S}_t = \bar{s};\pi_{\tkernel_0}^*) > 0$. Thus there exists a prefix
$\bar{s}_0, a_0, \ldots , \bar{s}_t = \bar{s}$ with $\prob_{\tkernel_0} (\bar{S}_t = \bar{s},\ldots, a_0, \bar{s}_0 ;\pi_{\tkernel_0}^*)>0$ and it follows that $\bar{s}_i\in\theinterestingstates$ for $i=0,\ldots,t-1$. Assuming  $\tkernel_0$ and $\pi_{\tkernel_0}^*$ are deterministic   $\tkernel_0(s_{i+1}|s_i, a_i) = 1$ and $\pi_{\tkernel_0}^*(a_i|\bar{s}_i) = 1$ for the optimal action $a_i$. Noting that $\tkernel(s_{i+1}|s_i,a_i)
\geq 1 - \frac{\delta}{2}$
and
$\pi(a_i|\bar{s}_i) = \pi(\oactions(\bar{s}_i)|\bar{s}_i) \geq \min_{\bar{s}'\in\theinterestingstates} \pi(\oactions(\bar{s}')|\bar{s}')$ 
leads to the estimate
\begin{align*}
\prob_{\tkernel} (\bar{S}_t = \bar{s};\pi)
&\geq 
\prob_{\tkernel} (\bar{S}_t = \bar{s},\ldots, a_0, \bar{s}_0 ;\pi)
\\
&=\bar{\mu}(\bar{s}_0)
\pi(a_0|\bar{s}_0)
\tkernel(s_1|s_0,a_0)\ldots \tkernel(s|s_{t-1},a_{t-1})
\\
&\geq (1-\frac{\delta}{2})^N
\min_{\bar{s}'\in \supp \bar{\mu}} \bar{\mu}(\bar{s}')
\left(\min_{\bar{s}'\in\theinterestingstates} \pi(\oactions(\bar{s}')|\bar{s}')\right)^{N}.
\end{align*}
For the state visitation probability $v_{\tkernel,\pi}(h,g|s,h)$ this yields that
\begin{align}
v_{\tkernel,\pi}(h,g|s,h) &\geq  \prob_{\tkernel}(H_0^{\Sigma}=h, G_0^{\Sigma} = g, S_0^{\Sigma}=s, l(\Sigma) = h ; \pi)
\nonumber\\
&=
\sum_{\substack{\sigma:l(\sigma)=h,\\\bar{s}_0^{\sigma} = \bar{s}}} c^{-1} \sum_{t'\leq N-l(\sigma)}
\prob_{\tkernel} (\bar{S}_{t'}= \bar{s}_0^{\sigma}, \ldots, \bar{S}_{t+l(\sigma)}=\bar{s}_{l(\sigma)}^{\sigma} ; \pi)\nonumber
\\
&\geq
\sum_{\substack{\sigma:l(\sigma)=h,\\\bar{s}_0^{\sigma} = \bar{s}}} c^{-1} 
\prob_{\tkernel} (\bar{S}_{t}= \bar{s}_0^{\sigma}, \ldots, \bar{S}_{t+l(\sigma)}=\bar{s}_{l(\sigma)}^{\sigma} ; \pi)\nonumber
\\
&=
c^{-1}
\prob_{\tkernel} (\bar{S}_{t}= \bar{s}; \pi)
\sum_{\substack{\sigma:l(\sigma)=h,\\\bar{s}_0^{\sigma} = \bar{s}}} 
\prob_{\tkernel} (\bar{A}_{t}=a_0^{\sigma},\ldots,\bar{S}_{t+l(\sigma)}=\bar{s}_{l(\sigma)}^{\sigma} | \bar{S}_{t}=\bar{s}_0^{\sigma}; \pi)\nonumber
\\
&\geq
c^{-1} \prob_{\tkernel} (\bar{S}_{t}= \bar{s}; \pi)\nonumber
\\
&\geq
\frac{2}{N(N+1) }
\min_{\bar{s}'\in \supp \bar{\mu}} \bar{\mu}(\bar{s}')
(\min_{\bar{s}'\in\theinterestingstates} \pi(\oactions(\bar{s}')|\bar{s}'))^{N}
(1-\frac{\delta}{2})^N,
\label{eq:M1visitbound}
\end{align}
where we used the fact that $\sum_{\sigma:l(\sigma)=h,\bar{s}_0^{\sigma} = \bar{s}} 
\prob_{\tkernel} (\bar{A}_{t}=a_0^{\sigma},\ldots,\bar{S}_{t+l(\sigma)}=\bar{s}_{l(\sigma)}^{\sigma} | \bar{S}_{t}=\bar{s}_0^{\sigma}; \pi) = 1$.

\emph{Bounding $Q_{\tkernel}^{\pi_n,g,h}$:} If $\bar{s} \in \theinterestingstates$ and $a\in\mathcal{A}\setminus \oactions(\bar{s})$ then by lemma~\ref{le:opthbound}
\begin{align}
Q_{\tkernel}^{\pi_n,g,h}((s,h',g'),a) \leq Q_{\tkernel}^*(\bar{s}, a) \leq \frac{\delta N}{2}.
\label{eq:M1QoutMbound}
\end{align}

\emph{Bounding $Q_{\tkernel}^{\pi}$ for optimal actions:} 
Let $\bar{s} \in \theinterestingstates$ and $a\in \oactions(\bar{s})$. If $h=1$ then $\bar{s} = (s,1,g)$.
As $\tkernel_0$ is deterministic there exists  $s''$  so that  $\tkernel_0(s''|s,a) = 1$. Notice that since $\bar{s} \in \theinterestingstates$ and $a$ is optimal
it holds that $\rho(s'') = g$. Consequently $Q_{\tkernel}^{\pi}((s,1,g),a)
\geq
\tkernel(s''|s,a) \geq 1-\frac{\delta}{2}.
$ For $h=2$ $\bar{s} = (s,2,g)$ there exist $s''$ and $a\in\oactions(\bar{s})$ so that $(s'',1,g) \in \theinterestingstates$ and $\tkernel _0(s''|s,a) = 1$ by lemma~\ref{le:suppstab} point 4. It follows for $a \in \oactions(\bar{s})$ that
$$
\begin{aligned}
Q_{\tkernel}^{\pi}((s,2,g),a) 
&=
\sum_{s'\in \bar{S}_T}
\tkernel(s'|s,a) \sum_{a'\in \mathcal{A}} \pi(a'|(s',1,g)) Q_{\tkernel}^{\pi}(s',1,g,a')
\\
&\geq
\tkernel(s''|s,a) \sum_{a'\in \oactions(\bar{s})}  \pi(a'|(s'',1,g)) Q_{\tkernel}^{\pi}(s'',1,g,a')
\\
&\geq
(1-\frac{\delta}{2})^2  \pi(\oactions((s'',1,g))|(s'',1,g))
\geq
(1-\frac{\delta}{2})^2 \min_{\bar{s}'\in\theinterestingstates} \pi(\oactions(\bar{s}')|\bar{s}')
.
\end{aligned}
$$
Reasoning inductively we obtain for horizon $h$
\begin{align}
Q_{\tkernel}^{\pi}((s,h,g),a)
\geq
(1-\frac{\delta}{2})^N \left(\min_{\bar{s}'\in\theinterestingstates} \pi(\oactions(\bar{s}')|\bar{s}')\right)^{N-1}
.
\label{eq:M1QonMbound}
\end{align}

\emph{Recursive bounding of policies on critical states:}
We begin by estimating $\pi(\oactions(\bar{s})|\bar{s})$, $\bar{s} =(s,h,g) \in \theinterestingstates$ similarly to the proof of theorem~\ref{le:limdetcont}. We use lemma~\ref{le:suppstab} point 2. and lemma~\ref{le:recrewrites} point 1. to get
$(\forall n \geq 0, \forall \tkernel \in U_{\delta}(\tkernel_0))$ that
\begin{align*}
\pi_{n+1,\tkernel}(\oactions(\bar{s})|\bar{s}) 
&=
\frac{
\sum\limits_{a\in \oactions(\bar{s}), h'\geq h, g \in \mathcal{G}}
Q_{\tkernel}^{\pi_n,g,h}((s,h',g'),a)
\pi_{n,\tkernel}(a|s,h',g') v_{\tkernel,\pi_n}(h',g'|s,h)
}{
\sum\limits_{a\in \mathcal{A}, h'\geq h, g \in \mathcal{G}}
Q_{\tkernel}^{\pi_n,g,h}((s,h',g'),a)
\pi_{n,\tkernel}(a|s,h',g') v_{\tkernel,\pi_n}(h',g'|s,h)
}
\\
&\geq
\frac{
\sum\limits_{a\in \oactions(\bar{s}), h'\geq h, g \in \mathcal{G}}
Q_{\tkernel}^{\pi_n,g,h}((s,h',g'),a)
\pi_{n,\tkernel}(a|s,h',g') v_{\tkernel,\pi_n}(h',g'|s,h)
}{
\sum\limits_{a\in \oactions(\bar{s}), h'\geq h, g \in \mathcal{G}}
Q_{\tkernel}^{\pi_n,g,h}((s,h',g'),a)
\pi_{n,\tkernel}(a|s,h',g') v_{\tkernel,\pi_n}(h',g'|s,h)
+\frac{\delta N}{2}
},
\end{align*}
where we made use of the estimate~\eref{eq:M1QoutMbound} and lemma~\ref{le:contoptbound} point 2. For $(h',g') \neq (h,g)$ we bound 
$Q_{\tkernel}^{\pi_n,g,h}((s,h',g'),a)$ from below by zero (as they are probabilities)
and use the monotonicity remark~\ref{re:monotonicity}
to obtain
\begin{align*}
\pi_{n+1,\tkernel}(\oactions(\bar{s})|\bar{s}) \geq
\frac{
\sum_{a\in \oactions(\bar{s})}
Q_{\tkernel}^{\pi_n}((s,h,g),a)
\pi_{n,\tkernel}(a|s,h,g)
v_{\tkernel,\pi_n}(h,g|s,h)
}{
\sum_{a\in \oactions(\bar{s})}
Q_{\tkernel}^{\pi_n}((s,h,g),a)
\pi_{n,\tkernel}(a|s,h,g)
v_{\tkernel,\pi_n}(h,g|s,h)
+\frac{\delta N}{2}
}.
\end{align*}
Employing remark~\ref{re:monotonicity} a second time and the estimates \eref{eq:M1QonMbound} and \eref{eq:M1visitbound}
gives then
\begin{align*}
&\pi_{n+1,\tkernel}(\oactions(\bar{s})|\bar{s})\\
&\geq
\frac{
\frac{2}{N(N+1)}
(\min_{\bar{s}\in \supp \bar{\mu}} \bar{\mu}(\bar{s}) )
(1-\frac{\delta}{2})^{2N} (\min_{\bar{s}'\in \theinterestingstates} \pi_{n,\tkernel}(\oactions(\bar{s}')|\bar{s}'))^{2N-1}  \pi_{n,\tkernel}(\oactions(\bar{s})|\bar{s})
}{
\frac{2}{N(N+1)}
(\min_{\bar{s}\in \supp \bar{\mu}} \bar{\mu}(\bar{s}))
(1-\frac{\delta}{2})^{2N} (\min_{\bar{s}'\in \theinterestingstates} \pi_{n,\tkernel}(\oactions(\bar{s}')|\bar{s}'))^{2N-1}  \pi_{n,\tkernel}(\oactions(\bar{s})|\bar{s})
+\frac{\delta N}{2}
}.
\end{align*}
Finally, bounding  $\pi_{n,\tkernel}(\oactions(\bar{s})|\bar{s})$ from below by its minimum over
$\theinterestingstates$ we obtain $(\forall n \geq 0, \forall \tkernel\in U_{\delta}(\tkernel_0), \forall \bar{s}=(s,h,g) \in \bar{S}(\tkernel_0))$:
\begin{align*}
\pi_{n+1,\tkernel}(\oactions(\bar{s})|\bar{s})
&\geq
\frac{
\frac{2}{N(N+1)}
(\min_{\bar{s}\in \supp \bar{\mu}} \bar{\mu}(\bar{s}))
(1-\frac{\delta}{2})^{2N} (\min_{\bar{s}'\in \theinterestingstates} \pi_{n,\tkernel}(\oactions(\bar{s}')|\bar{s}'))^{2N} 
}{
\frac{2}{N(N+1)}
(\min_{\bar{s}\in \supp \bar{\mu}} \bar{\mu}(\bar{s}))
(1-\frac{\delta}{2})^{2N} (\min_{\bar{s}'\in \theinterestingstates} \pi_{n,\tkernel}(\oactions(\bar{s}')|\bar{s}'))^{2N} 
+\frac{\delta N}{2}
}
\\
&=
\frac{
(\min_{\bar{s}'\in \theinterestingstates} \pi_{n,\tkernel}(\oactions(\bar{s}')|\bar{s}'))^{2N} 
}{
(\min_{\bar{s}'\in \theinterestingstates} \pi_{n,\tkernel}(\oactions(\bar{s}')|\bar{s}'))^{2N} 
+b
}
= h_b(\min_{\bar{s}'\in \theinterestingstates} \pi_{n,\tkernel}(\oactions(\bar{s}')|\bar{s}'))
,
\end{align*}
where
$$
b= \frac{\frac{\delta N}{2}}{\frac{2}{N(N+1)}
(\min_{\bar{s}\in \supp \bar{\mu}} \bar{\mu}(\bar{s}))
(1-\frac{\delta}{2})^{2N}}.
$$
Since the right hand side of the inequality does no longer depend on $\bar{s}$ one can replace the left hand side by the minimum over $\bar{s}$, which proves the inequality~\eqref{eq:maintheoremM1}. Since $b\rightarrow 0$ as $\delta \rightarrow 0$
we can restrict $\delta\in(0,2)$ to ensure that $b < \frac{1}{2N}(\frac{2N-1}{2N})^{2N-1}$.

Without loss of generality we can assume that $\delta\in(0,2)$ is chosen so that
$\pi_0(\oactions(\bar{s})|\bar{s}) > x_l(b)$ for all $\bar{s} \in \theinterestingstates$. This is possible since $x_l(b) \rightarrow 0$ for $\delta \rightarrow 0$ (since $b\rightarrow 0$ as $\delta\rightarrow 0$) and $\pi_0(\oactions(\bar{s})|\bar{s}) > 0$ for all
$\bar{s} \in \theinterestingstates$. Having established $\min_{\bar{s}\in\theinterestingstates} \pi_0(\oactions(\bar{s})|\bar{s}) > x_l(b)$
we can employ lemma~\ref{le:hlemma} point 4. to inequality~\eqref{eq:maintheoremM1} to find for all $\tkernel \in U_{\delta}(\tkernel_0)$ and $\bar{s}\in\theinterestingstates$ that
$$
\liminf_n
\min_{\bar{s}\in\theinterestingstates} \pi_{n,\tkernel}(\oactions(\bar{s})|\bar{s}) \geq x_u(b)
$$
where $x_u(b) \rightarrow 1$, $b \rightarrow 0$. Since $b \rightarrow 0$, $\delta \rightarrow 0$ we also have that $x_u(b) \rightarrow 1$.
\end{proof}

\subsubsection{Extending the continuity results to other segment sub-spaces}
\label{ssse:specM1extending}

As before we need to modify theorem~\ref{le:limdetcontM1} to cover algorithms like ODT, RWR (and some specific variants of \eUDRL{}) restricting the recursion to $\SegTrail$ or $\SegDiag$ subspaces.
\begin{theorem}
\label{le:limdetcontM1DiagTrail}
(Relative continuity of accumulation points  of \eUDRL{}-generated policies at deterministic kernels -- the case $|\oactions(\bar{s})|=1$ on $\theinterestingstates$)
Theorem~\ref{le:limdetcontM1} remains valid under the renaming $\pi_n \rightarrow \pi_n^{\diag/\trail}$.
\end{theorem}

The proof of theorem~\ref{le:limdetcontM1DiagTrail}
follows along similar lines as that of theorem~\ref{le:limdetcontM1}. The differences are rather small and resemble those described in section~\ref{ssse:extendingfinite}.

\subsubsection{Estimating the location of accumulation points}
\label{ssse:specM1accbound}

We conclude with a corollary that provides explicit bounds on the accumulation points of \eUDRL{}-generated policies, values and goal-reaching objectives. See section~\ref{ssse:specmuaccbound} for a motivation.

\begin{corollary}\label{le:limitboundsM1}
(Estimating the location of accumulation points -- the case $|\oactions(\bar{s})| = 1$ on $\theinterestingstates$) Under the conditions of theorem~\ref{le:limdetcontM1} let $\delta_0\in(0,2)$ be the unique solution of the equation (cf.~theorem~\ref{le:limdetcontM1} and lemma~\ref{le:hlemma})
\begin{align*}
b(\delta)&=b_0,\quad\textnormal{where}\\
b(\delta)&=\frac{\delta N^{2} (N+1)}{4(1-\frac{\delta}{2})^{2N}\min_{\bar{s} \in \supp \bar{\mu}} \bar{\mu}(\bar{s}) },\\
b_0&=\frac{1}{2N}\left(\frac{2N-1}{2N}\right)^{2N-1}.
\end{align*}

Then there exist
two continuous, strictly monotonic functions
$x_l:[0,\delta_0] \rightarrow [0,\frac{2N-1}{2N}]$,
$x_u:[0,\delta_0] \rightarrow [\frac{2N-1}{2N},1]$,
where $x_l$ is increasing and $x_u$ decreasing with $x_l(0) = 0$, $x_u(0) = 1$ with the following property: For any fixed  $\delta\in(0,\delta_0)$ suppose that the policy $\pi_0$ satisfies $\pi_0>0$ and $\pi_0(\oactions(\bar{s})|\bar{s}) > x_l(\delta)$ for all $\bar{s} \in \theinterestingstates$ then for all $\tkernel \in U_{\delta}(\tkernel_0)$ and all $\bar{s} \in \theinterestingstates$ the following conclusions hold: 
\begin{enumerate}
\item
$\liminf_n \pi_{n,\tkernel}(\oactions(\bar{s})|\bar{s}) \geq x_u(\delta) \rightarrow 1, \delta \rightarrow 0$
or equivalently (in form of error to an optimal policy)
$\limsup_n (1-\pi_{n,\tkernel}(\oactions(\bar{s})|\bar{s})) \leq 1-x_u(\delta)$
\item
$\limsup_n   |V_{\tkernel}^{\pi_n}(\bar{s}) -V_{\tkernel_0}^*(\bar{s}) | \leq
1-(1-\frac{\delta}{2})^N x_u^N(\delta)$
\item
$
\limsup_n |J_{\tkernel}^{\pi_n} -J_{\tkernel_0}^{*}|
\leq
\frac{N\delta}{2} + (1-(1-\frac{\delta}{2})^N x_u^{N}(\delta) )
$
\item ($q$-linear convergence)
There exists a sequence of $\tkernel_0$-optimal policies $(\pi_{\tkernel_0,n}^{*})_{n\geq 0}$ such that for all $n$ it holds
$$
\|\pi_{n,\tkernel}(\cdot|\bar{s}) - \pi_{\tkernel_0,n}^{*}(\cdot|\bar{s})\|_1
\leq
2(1-h_{b}^{\circ n}(x_0))
,
$$
where
$x_0 = \min_{\bar{s} \in \theinterestingstates}
\pi_0 (\oactions(\bar{s})|\bar{s})$.
\end{enumerate}

The terms on the right hand side of 1.-3.~converge to $0$ as $\delta \rightarrow 0$. In 4.~the sequence $(y_n)_{n\geq0}$ defined by 
$
y_n = 2(1-h_{b}^{\circ n}(x_0))
$
converges $q$-linearly to the limit $y=2(1-x_u(\delta))$ at a convergence rate of $\frac{2Nb x_u^{2N-1}}{(x_u^{2N} + b)^2}
$ which tends to 0 as $\delta \rightarrow 0$. 
\end{corollary}
Recall that the functions $x_l$, $x_u$ were defined in lemma~\ref{le:hlemma} as functions of $b$, which, in turn, is a function of $\delta$. We slightly abuse notation in corollary~\ref{le:limitboundsM1} writing  $x_l(\delta)$ for $x_l(b(\delta))$ and $x_u(\delta)$ for $x_l(b(\delta))$.
\begin{proof}
First we have to show the uniqueness of $\delta_0$, which is 
defined as a solution to the equation
$$
b(\delta) = b_0, \quad \delta \in [0,2).
$$
For convenience we restrict the domain of $b(\delta)$
to be $[0,2)$, so that the restricted function is continuous. Since the first derivative of $b$ is positive
$$
\frac{d b}{d \delta}
= 
\frac{
N^{2} (N+1)}
{4 \min_{\bar{s} \in \supp \bar{\mu}} \bar{\mu}(\bar{s})}
\frac{
(1-\frac{\delta}{2})+N\delta}
{(1-\frac{\delta}{2})^{2N+1}}
> 0 \iff 
1 + \delta(N - \frac{1}{2})
> 0,
$$
there can be at most one solution to the equation.
Since $b(0) = 0$, $b(\frac{2}{N+1}) > b_0$ and making use of the mentioned continuity, there exists at least one solution
to the equation. This proves the existence of a unique solution $\delta_0$.

Here, the symbols $x_l,x_u$ refer to compositions $x_l := \bar{x}_l\circ b$, $x_u := \bar{x}_u\circ b$, where $\bar{x}_l,\bar{x}_u$ were defined in point 5.\ of lemma~\ref{le:hlemma}.
Since $b$ is strictly increasing, continuous, and defined on the
interval 
all claimed properties of the
compound functions follow from~\ref{le:hlemma} point 5.

1.
This point is just the theorem~\ref{le:limdetcontM1}. Similarly to the proof of the theorem~\ref{le:limdetcontM1} the condition $\pi_0(\oactions(\bar{s})|\bar{s}) > x_l(\delta)$ comes from
lemma~\ref{le:hlemma} point 4.

2.
Assume $\bar{s} \in \theinterestingstates$.
From \eref{eq:M1QonMbound} it follows
$$
V_{\tkernel}^{\pi_n}(\bar{s})
\geq
\sum_{a \in \oactions(\bar{s})}
Q_{\tkernel}^{\pi_n}
(\bar{s},a)
\pi_{n,\tkernel} (a| \bar{s})
\geq
(1-\frac{\delta}{2})^{N} ( \min_{\bar{s}'\in \theinterestingstates}
\pi(\oactions(\bar{s}')|\bar{s}') )^{N}
$$
Now we apply point 1.
$$
\liminf_{n}
V_{\tkernel}^{\pi_n}(\bar{s})
\geq
(1-\frac{\delta}{2})^{N} ( \min_{\bar{s}'\in \theinterestingstates} \liminf_{n}
\pi(\oactions(\bar{s}')|\bar{s}') )^{N}
\geq
(1-\frac{\delta}{2})^{N}
x_u^{N}
$$
Finally, we conclude that
$$
\limsup_{n}
|
V_{\tkernel_0}^*(\bar{s}) - V_{\tkernel}^{\pi_n}(\bar{s})
|
=
1-\liminf_{n} V_{\tkernel}^{\pi_n}(\bar{s})
\leq 
1-(1-\frac{\delta}{2})^{N}
x_u^{N}
.
$$

3.
The proof follows similarly as for corollary~\ref{le:limitbounds} point 4. except we utilize the previous 
point
$$
\begin{aligned}
|J_{\tkernel_0}^* - J_{\tkernel}^{\pi_n}|
&\leq
\sum_{\bar{s} \in \supp \bar{\mu} \setminus \theinterestingstates } 
\bar{\mu}(\bar{s}) |V_{\tkernel}^{\pi_n}(\bar{s}) |
+
\sum_{\bar{s} \in \supp \bar{\mu} \cap \theinterestingstates } 
\bar{\mu}(\bar{s}) |V_{\tkernel_0}^*(\bar{s}) - V_{\tkernel}^{\pi_n}(\bar{s}) |
\\
&\leq
\frac{N\delta}{2}
+
\sum_{\bar{s} \in \supp \bar{\mu} \cap \theinterestingstates } 
\bar{\mu}(\bar{s}) |V_{\tkernel_0}^*(\bar{s}) - V_{\tkernel}^{\pi_n}(\bar{s})
|,
\\
\limsup_n
|J_{\tkernel_0}^* - J_{\tkernel}^{\pi_n}|
&\leq
\frac{N\delta}{2}
+
1-(1-\frac{\delta}{2})^N x_u^N (\delta)
.
\end{aligned}
$$

4.
The sequence $(\pi_{\tkernel_0,n}^{*})_{n\geq 0}$ is chosen in the same way as in the corollary~\ref{le:limitbounds}.
The first inequality is just the theorem~\ref{le:limdetcontM1}
and lemma~\ref{le:hlemma}.
The the statement $y_n \rightarrow y$, $n \rightarrow \infty$ is lemma~\ref{le:hlemma} point 3.

Consider a sequence $x_n = h_b^{\circ n} (x_0)$, $n\geq 0$.
The last statement about the convergence rate follows by using Heine theorem and L'Hospital rule
\begin{align*}
\lim_{n\rightarrow \infty}
\frac{y_{n+1}-y}{y_n-y}
&=
\lim_{n\rightarrow \infty}
\frac{x_u-h_b(x_n)}{x_u-x_n}
\\
&=^{Heine}
\lim_{x\rightarrow x_u}
\frac{x_u-h_b(x)}{x_u-x}
\\
&=^{L'Hospital}
\lim_{x\rightarrow x_u}
h_b'(x)
=
\lim_{x\rightarrow x_u}
\frac{2Nb x^{2N-1}}{(x^{2N}+b)^2}
=
\frac{2Nb x_u^{2N-1}}{(x_u^{2N}+b)^2}
.
\end{align*}
Since $x_u \rightarrow 1$, $\delta \rightarrow 0$
and $b \rightarrow 0$, $\delta \rightarrow 0$ 
the rate $\frac{2Nb x_u^{2N-1}}{(x_u^{2N}+b)^2} \rightarrow 0$
for  $\delta \rightarrow 0$.
\end{proof}

Similar to corollary~\ref{le:limitbounds}, we can employ corollary~\ref{le:limitboundsM1} to bound the deviation in the accumulation points of \eUDRL{}-generated quantities given a deviation of $\delta$ from a deterministic kernel. In the case of corollary~\ref{le:limitbounds} we already presented the estimates for the accumulation points of $\min_{\bar{s}\in\theinterestingstates}\pi_n(\oactions(\bar{s})|\bar{s})$ for the simple bandit example~\ref{ex:bandit} in figure~\ref{fig:bandit} (see also section~\ref{ssse:specmutheorem} and the explanation there). In fact both conditions $\theinterestingstates \subset \supp \bar{\mu}$ and $(\forall \bar{s}\in \theinterestingstates: \: |\oactions(\bar{s})|=1)$ are simultaneously met by the bandit example, which allows us to compare the estimates from the corollaries~\ref{le:limitbounds} and~\ref{le:limitboundsM1}. A comparison is provided in figure~\ref{fig:banditxu}, which shows how the estimates from the corollaries (blue $x^*$, orange $x_u$, green $x_l$ lines) and the accumulation points of \eUDRL{}-generated policies (orange marks) behave as functions of the distance $\delta$ to a deterministic kernel. Recall that according to corollary~\ref{le:limitboundsM1} $x_l(\delta)$ is an unstable fixed point of the underlying dynamical system. If the initial condition on the policy is such that it is larger than this fixed point (cf.~corollary) then the accumulation points of the \eUDRL{} recursion are bounded from below by the fixed point $x_u(\delta)$. In other words the $x_l(\delta)$ graph separates the plane in two regions, where the lower bound by $x_u(\delta)$ is guaranteed in the upper region. 
\begin{figure}
   \centering
    \includegraphics[width=0.75\linewidth]{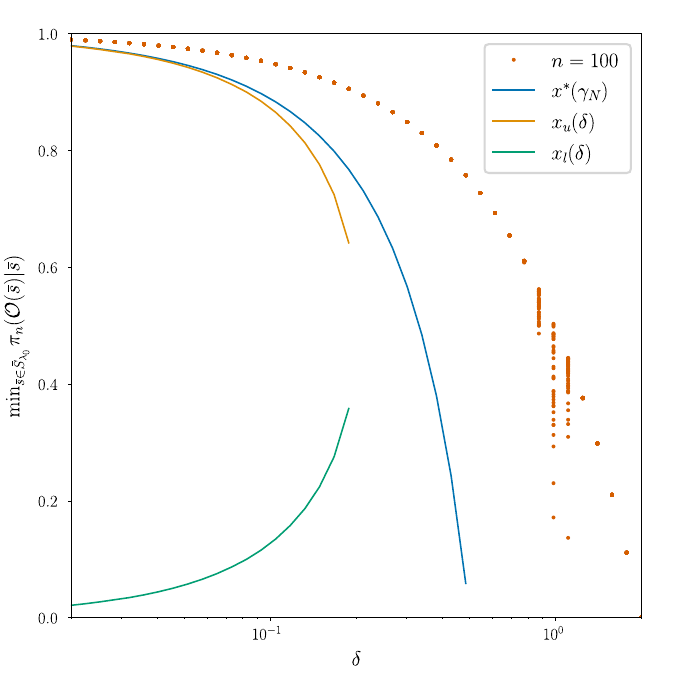}
    \caption{Comparison of estimates of corollaries~\ref{le:limitbounds},~\ref{le:limitboundsM1} on the behavior of $\min_{\bar{s}\in\theinterestingstates}\pi_n(\oactions(\bar{s})|\bar{s})$ in case of a simple bandit example. The plot shows the $\delta$-dependency of accumulation points (depicted in orange, approximated by values at $n=100$) and the estimates from corollary~\ref{le:limitbounds} (blue) and corollary~\ref{le:limitboundsM1} (orange, green).}
    \label{fig:banditxu}
\end{figure}

As a second example we illustrate how the estimates of corollary~\ref{le:limitboundsM1} apply to an ODT architecture trained to identify the optimal policy in a simple grid world (that has a unique optimal policy on $\theinterestingstates$), see figure~\ref{fig:ODTgridworld}. As outlined in section~\ref{se:background} ODT can be interpreted as a special instance of algorithms that fit within CE framework, where \eUDRL{}'s recursion operates on the space of trailing segments $\SegTrail$ rather than $\Seg$. All details of the construction of the MDP underlying figure~\ref{fig:ODTgridworld} and the details of the ODT algorithm are provided in appendix B, example~\ref{ex:ODTgridworld}. Figure~\ref{fig:ODTgridworld:a} shows how the bound in the corollary (orange and green lines) and the accumulation points of ODT-generated policies (orange marks) behave as functions of the distance $\delta$ to a deterministic kernel.  Figure~\ref{fig:ODTgridworld:b} shows how ODT approaches those accumulation points as a function of the iteration $n$. Varying initial policies are visible at $n=0$ and different levels of $\delta$ are highlighted in different colors. Figure~\ref{fig:ODTgridworld:c} contains the same data as Figure~\ref{fig:ODTgridworld:b} but it is organized to reveal the dependency on $\delta$. Figure~\ref{fig:ODTgridworld:d} shows a map of the grid world.
Given the structure of the grid world it is visible that the uniqueness condition on the optimal policy is restrictive in the sense that domains that satisfying the condition $(\forall \bar{s}\in \theinterestingstates: \: |\oactions(\bar{s})|=1)$ seem to occur rarely.

\begin{figure}
  \begin{subfigure}{0.475\linewidth}
		\includegraphics[width=\linewidth]{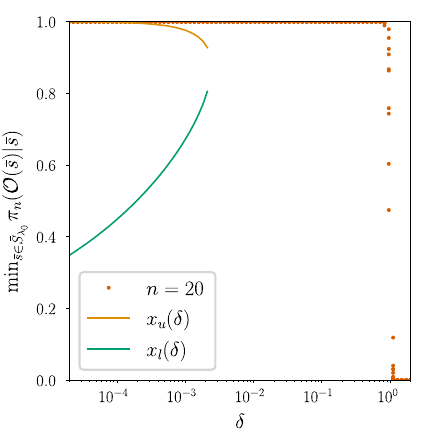}
		\caption{Dependency on distance to determin.~kernel.}
        \label{fig:ODTgridworld:a}
	\end{subfigure}
     \begin{subfigure}{0.475\linewidth}
		\includegraphics[width=\linewidth]{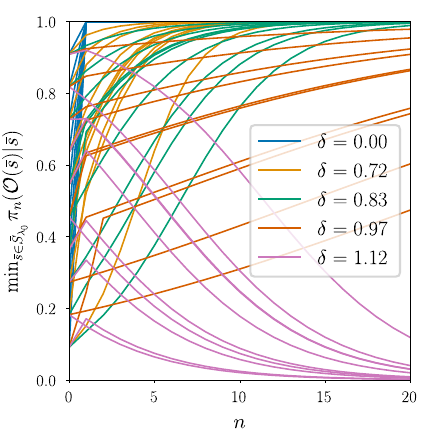}
		\caption{Dependency on iteration.}
        \label{fig:ODTgridworld:b}
	\end{subfigure}\\
     \begin{subfigure}{0.475\linewidth}
		\includegraphics[width=\linewidth]{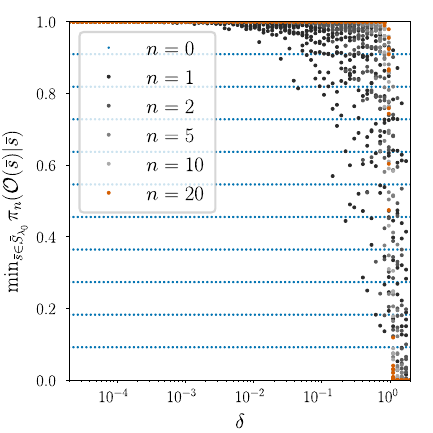}
		\caption{Dependency on distance to determin.~kernel.}
        \label{fig:ODTgridworld:c}
	\end{subfigure}
     \begin{subfigure}{0.475\linewidth}
		\includegraphics[width=\linewidth]{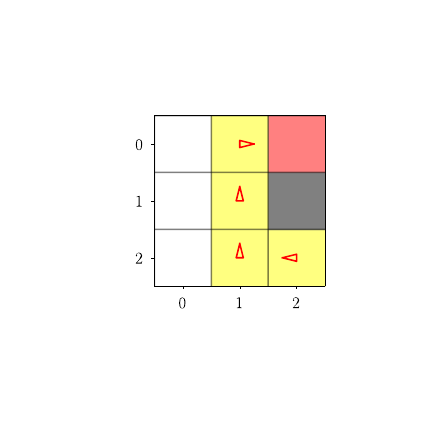}
		\caption{Map of grid world.}
        \label{fig:ODTgridworld:d}
	\end{subfigure}
    \caption{Illustration of the estimates of corollary~\ref{le:limitboundsM1} on the behavior of $\min_{\bar{s}\in\theinterestingstates}\pi_n(\oactions(\bar{s})|\bar{s})$ for ODT training in a simple grid world shown in plot (d). Plot (a) shows the dependency of accumulation points (depicted in orange, approximated by values at $n=100$) and the estimate in the corollary given (orange and green). To illustrate how ODT approaches towards the accumulation points, plot (b) shows the dependency on the iteration $n$ for varying distances to a deterministic kernel highlighted by different colors and varying initial policy. Plot (c) contains the same information as (b) but it is organized to reveal the dependency on $\delta$, where varying numbers of iteration are highlighted by different colors. Plot (d) shows the map of the grid world. A wall is depicted in gray, the set $\theinterestingstates$ in yellow, {and the goal in red}. Arrows depict optimal actions associated with the specific state and goal.}
    \label{fig:ODTgridworld}
\end{figure}

\section{Regularized Recursion}
\label{se:regrec}

This section examines regularization in \eUDRL{}, which is motivated by the following two facts. First, the standard ODT architecture utilizes entropy regularization, whereas we have modeled ODT without regularization so far. Second, on the mathematical side, regularization often leads to significant simplification when discussing continuity. For instance it allows for fully general continuity proofs in situations where it guarantees uniform visitation of all states $\bar s\in \theinterestingstates$ over all iterations of \eUDRL{} (equivalent to a lower bound on the state visitation distribution). As a case study we study a simple form of regularization, which, akin to $\epsilon$-greedy exploration, involves a convex combination of an algorithm's current policy and a uniform policy. More precisely, choosing $\epsilon\in(0,1)$ we extend the \eUDRL{} recursion of formula~\eqref{eq:recursionUsingNumeratorDenominator} to an $\epsilon$-regularized recursion ($\epsilon$-\eUDRL{}) by
\begin{align*}
\pi_{n+1,\epsilon}(a|s,h,g) &= (1-\epsilon)\frac{\num_{\tkernel,\pi_{n,\epsilon}}(a,s,h,g)}{\den_{\tkernel,\pi_{n,\epsilon}}(s,h,g)}+\epsilon\frac{1}{|\mathcal{A}|}\quad \text{for}\:  (s,h,g) \in \supp \den_{\tkernel,\pi_{n,\epsilon}},
\end{align*}
where $\pi_{n+1,\epsilon}(a|s,h,g)=1/{|\mathcal{A}|}$ outside the support $\supp \den_{\tkernel,\pi_{n,\epsilon}}$. The sequence of policies generated by this recursion automatically satisfies $\pi_{n,\epsilon}>0$, which ensures that the relevant state visitation probabilities are strictly positive. As compared to the discussion of continuity in the preceding sections this simplification avoids the complicated maintenance of support terms in the course of proofs (as policy-supports are always maximal). It should be said that the price to pay for this simplification is that $\epsilon$-\eUDRL{} does not converge to optimal policies for $\epsilon>0$. 

Below, we will present versions of the results previously encountered, now with regularization applied. Many of the results presented in sections~\ref{se:suppstab},~\ref{se:optdet} and~\ref{se:continfty} remain valid up to small modification. For brevity and to avoid repetition we decided to leave the precise statement of the $\epsilon$-\eUDRL{} versions of the involved lemmas to appendix~\ref{ap:regrec} focusing only on the main finding in this section. From a conceptual perspective lemma~\ref{le:suppstab} (\emph{Stability of $\theinterestingstates$}), lemma~\ref{le:detopt} (\emph{Optimality of eUDRL policies for deterministic transition kernels}), lemma~\ref{le:contQfac} (\emph{Continuity of action-values in quotient topology}), lemma~\ref{le:alpha} (\emph{Lower bound on visitation probabilities}), lemma~\ref{le:f} (\emph{$f$-lemma}) played key roles in the derivations of our main continuity results in sections~\ref{se:contfinite} and~\ref{se:continfty}. We recapitulate their content highlighting the specific modifications required for $\epsilon$-\eUDRL{}. Lemma~\ref{le:esuppstab} ($\epsilon$-\eUDRL{} version of lemma~\ref{le:suppstab}) addresses the behavior of supports of $\epsilon$-\eUDRL{}-generated quantities along compatible families of MDPs. Lemma~\ref{le:edetopt} ($\epsilon$-\eUDRL{} version of lemma~\ref{le:detopt}) describes the behavior of $\epsilon$-\eUDRL{} at deterministic transition kernels. Writing $\tkernel_0$ for a deterministic kernel, the set
$$\Pi_{\tkernel_0,\epsilon}^* = \left\{(1-\epsilon) \pi_{\tkernel_0}^* + \frac{\epsilon}{|\mathcal{A}|} \mid \pi_{\tkernel_0}^*\; \text{is an optimal policy for}\; \tkernel_0 \;\text{on}\; \theinterestingstates \right\}$$
will take the role of the set of optimal policies. Lemma~\ref{le:edetopt} asserts that the accumulation points of $\epsilon$-\eUDRL{} are contained in $\Pi_{\tkernel_0,\epsilon}^*$. To demonstrate the relative continuity of $\epsilon$-\eUDRL{} policies at an infinite number of iterations, thus, the set $\Pi_{\tkernel_0,\epsilon}^*$ should constitute the relative limit. The lemma also establishes estimates on value functions corresponding to policies in $\Pi_{\tkernel_0,\epsilon}^*$ that are important for the proof of the main theorem. Unlike the optimal action-value function $Q_{\tkernel_0}^*$ for \eUDRL{}, in the context of $\epsilon$-\eUDRL{} there are many possible action-value functions $Q_{\tkernel_0}^{\pi_{\epsilon}^*}$ depending on the choice of $\pi_{\tkernel_0,\epsilon}^* \in \Pi_{\tkernel_0,\epsilon}^*$. $Q_{\tkernel_0}^{\pi_{\epsilon}^*}$ are also not constant on $\oactions(\bar{s})$ in general, and, consequently, do not represent an equivalence class under the quotient map relative to $\oactions(\bar{s})$. Lemma~\ref{le:econtQfac} ($\epsilon$-\eUDRL{} version of lemma~\ref{le:contQfac}) discusses a variation designed to address this point. A lower bound $\alpha(\delta,\epsilon)$ on the state visitation probabilities is provided in lemma~\ref{le:ealpha} ($\epsilon$-\eUDRL{} version of the lemma~\ref{le:alpha}). As before the main continuity results follow by studying a dynamical system to bound the
$\epsilon$-\eUDRL{} policy recursion. The dynamical system properties emerging from an iterative application of a map $z_{\gamma,\epsilon,M}$ ($\epsilon$-\eUDRL{} version of $f_{\gamma}$) are discussed in lemma~\ref{le:z} ($\epsilon$-\eUDRL{} version of lemma~\ref{le:f}). As it turns out this dynamical system has a unique fixed point $x^*(\gamma,\epsilon,M)$. The theorem shows the relative continuity of the set of accumulation points of $\epsilon$-\eUDRL{}-generated policies in $\tkernel$ at $\tkernel_0$ as $n$ tends to infinity. Along the lines outlined in section~\ref{ssse:specmutheorem} we conclude that the set of accumulation points of the sequence of $\epsilon$-\eUDRL{}-generated policies at $\tkernel=\tkernel_0$ is nonempty and contained in $\Pi_{\tkernel_0,\epsilon}^*$. 
\begin{restatable}[]{theorem}{elimdetcont}
\label{le:elimdetcont}
(Relative continuity of $\epsilon$-\eUDRL{} limit policies in deterministic kernels) Let $\{\mathcal{M}_{\tkernel} : \tkernel \in (\Delta \mathcal{S})^{\mathcal{S}\times\mathcal{A}}\}$
and $\{\bar{\mathcal{M}}_{\tkernel} : \tkernel \in (\Delta \mathcal{S})^{\mathcal{S}\times\mathcal{A}}\}$ be compatible families.
Let $\tkernel_0$ be a deterministic kernel
and let $(\pi_{n,\tkernel,\epsilon})_{n\geq 0}$ be a sequence of $\epsilon$-\eUDRL{}-generated policies with initial condition $\pi_0 > 0$, transition kernel $\tkernel$ and regularization parameter $\epsilon\in(0,1)$.
Let $\epsilon_0\in(0,1)$ be a fixed regularization parameter. Then for all $\pi_0>0$ the following statements hold:

\begin{enumerate}
\item
Let $\mathcal{L}(\pi_0,\tkernel,\epsilon)$ denote the set of accumulation points  of $(\pi_{n,\tkernel,\epsilon})_{n\geq 0}$.
Then any function $u:(\pi_0,\tkernel,\epsilon)\rightarrow u(\pi_0,\tkernel,\epsilon)\in\mathcal{L}(\pi_0,\tkernel,\epsilon)$ is relatively continuous in $\tkernel,\epsilon$
at point $\tkernel_0,\epsilon_0$ on $\theinterestingstates$, i.e. for all $\bar{s} \in \theinterestingstates$ 
$$[u(\pi_0,\tkernel,\epsilon_0)](\cdot|\bar{s}) \xrightarrow{\oactions(\bar{s})} \pi_{\tkernel_0,\epsilon_0}^*(\cdot|\bar{s}) \quad
\text{as}\quad \tkernel\rightarrow \tkernel_0.$$

\item
Let $\alpha(\delta,\epsilon)$ is chosen as in lemma~\ref{le:ealpha} and let 
$\beta,\tilde{\beta},\epsilon \in (0,1)$ be such that $1>\gamma+\epsilon$, where $\gamma = \frac{\tilde{\beta}}{((1-\epsilon)^N-\beta)\alpha(\delta,\epsilon)}$ and let 
$$
U_\delta(\tkernel_0) = \{\tkernel\mid\max_{(s,a)\in\mathcal{S}\times\mathcal{A}}
\|
\tkernel(\cdot\mid s,a)-\tkernel_0(\cdot\mid s,a)
\|_1 < \delta
\}.
$$
There exists $\delta >0$ so that for all $\tkernel \in U_{\delta}(\tkernel_0)$ and all $\bar{s}\in \theinterestingstates$ it holds
$$
\liminf_n \pi_{n,\tkernel,\epsilon}(\oactions(\bar{s})|\bar{s}) \geq x^*(\gamma,\epsilon,|\oactions(\bar{s})|)
$$
and
$$
x^*(\gamma,\epsilon,|\oactions(\bar{s})|) \rightarrow 1-\epsilon_0\left(1-\frac{|\oactions(\bar{s})|}{|\mathcal{A}|}\right)
\quad\text{as}\quad
(\beta,\tilde{\beta},\alpha,\epsilon) \rightarrow \left(0,0,\alpha(0,\epsilon_0),\epsilon_0\right),
$$
where $x^*(\gamma,\epsilon,|\oactions(\bar{s})|)$ is given in lemma~\ref{le:z}. Consequently for all $\bar{s}\in \theinterestingstates$ it holds that
$$
\liminf_n \pi_{n,\tkernel,\epsilon}(\oactions(\bar{s})|\bar{s}) \rightarrow 1-\epsilon_0\left(1-\frac{|\oactions(\bar{s})|}{|\mathcal{A}|}\right)
\quad\text{as}\quad (\tkernel,\epsilon) \rightarrow (\tkernel_0,\epsilon_0).
$$
\end{enumerate}
\end{restatable}
The theorem is proved in appendix~\ref{ap:regrec}. Along the lines of section~\ref{ssse:extendingfinite}, we can extend the result also to the segment spaces $\Seg^{\diag}$ and $\Seg^{\trail}$.
\begin{restatable}[]{theorem}{elimdetcontDiagTrail}
\label{le:elimdetcontDiagTrail} (Relative continuity of accumulation points of $\epsilon$-\eUDRL{}-generated policies at deterministic kernels) Theorem~\ref{le:elimdetcont} remains valid under the renaming $\pi_{n,\epsilon} \rightarrow \pi_{n,\epsilon}^{\diag/\trail}$.
\end{restatable}
Details are provided in appendix~\ref{ap:regrec}. The estimates given in theorem~\ref{le:elimdetcont} can also be made entirely explicit along the lines of section~\ref{ssse:specmuaccbound}, providing estimates for $\epsilon$-\eUDRL{}-generated quantities as a functions of the distance $\delta$ from the deterministic kernel $\tkernel_0$.
\begin{restatable}[]{corollary}{elimitbounds}
\label{le:elimitbounds}
(Estimating the location of accumulation points -- $\epsilon$-\eUDRL{})
Under the conditions of theorem~\ref{le:elimdetcont} assume that $\delta\in(0,1)$ and set
$$
\alpha = \frac{2}{N(N-1)} \left(\min_{\bar{s}\in \theinterestingstates} \bar{\mu}(\bar{s})\right) 
\left(\frac{\epsilon}{|\mathcal{A}|}\right)^N \left(1-\frac{\delta}{2}\right)^N.
$$
Define the quantities $\tilde{\beta} = \frac{N\delta}{2}$, 
$x^*(\gamma,\epsilon,M) = \frac{\hat{x}^*+\sqrt{(\hat{x}^*)^2+\frac{4\gamma\epsilon M}{|\mathcal{A}|}}}{2}$ where $\hat{x}^* = 1-\epsilon(1-\frac{M}{\mathcal{A}})-\gamma$ (cf. lemma~\ref{le:z}) and for horizon $h$, $1\leq h\leq N$,
\begin{align*}
\beta_h &= \begin{cases}\max\{\delta,\tilde{\beta}\},\ \textnormal{if}\ h=1,\\
\delta + \kappa_{h-1} + \beta_{h-1},\ \textnormal{if} \ h \geq 2,
\end{cases}
\\
\gamma_h &= \frac{\tilde{\beta}}{((1-\epsilon)^N-\beta_h)\alpha},
\\
\kappa_h &= \max_{\bar{s}=(s,h',g)\in \theinterestingstates,h'=h} 2\left( 1-\epsilon\left(1-\frac{|\oactions(\bar{s})|}{|\mathcal{A}|}\right)-x^*(\gamma_h,\epsilon,\oactions(\bar{s}))\right).
\end{align*}
Further assume $\beta_h,\gamma_h \in (0,1)$, $1> \gamma_h + \epsilon$ and notice that $\beta_h,\kappa_h,\gamma_h$, are increasing in $h$ and that
 $\beta_h,\kappa_h,\gamma_h \rightarrow 0$ as $\delta \rightarrow 0$.
Then the following assertions hold for all $\tkernel \in U_{\delta}(\tkernel_0)$ 
and for all $\pi_0 >0$:
\begin{enumerate}
\item
$\displaystyle
\limsup_n \max_{\bar{s} \in \theinterestingstates }
2 \left( 1-\epsilon\left(1-\frac{|\oactions(\bar{s})|}{|\mathcal{A}|}\right) - \pi_{n,\epsilon}(\oactions(\bar{s})|\bar{s})\right) \leq \kappa_{N},
$
\item
$\displaystyle
(\exists (\pi_{n,\epsilon}^*), \pi_{n,\epsilon}^*\in \Pi_{\tkernel_0,\epsilon}^* ):\;
\limsup_n \max_{\bar{s},a \in \theinterestingstates\times\mathcal{A} }
|Q_{\tkernel}^{\pi_{n,\epsilon}}(\bar{s},a) - Q_{\tkernel_0}^{\pi_{n,\epsilon}^*}(\bar{s},a)|
\leq \beta_N,
$
\item
$\displaystyle
(\exists (\pi_{n,\epsilon}^*), \pi_{n,\epsilon}^*\in \Pi_{\tkernel_0,\epsilon}^* ):\;
\limsup_n \max_{\bar{s} \in \theinterestingstates}
|V_{\tkernel}^{\pi_{n,\epsilon}}(\bar{s})- V_{\tkernel_0}^{\pi_{n,\epsilon}^*}(\bar{s}) |
\leq
\beta_N + \kappa_N,
$
\item
$\displaystyle
(\exists (\pi_{n,\epsilon}^*), \pi_{n,\epsilon}^*\in \Pi_{\tkernel_0,\epsilon}^* ):\;
\limsup_n 
|J_{\tkernel}^{\pi_{n,\epsilon}}- J_{\tkernel_0}^{\pi_{n,\epsilon}^*} |
\leq
\frac{N\delta}{2} + \beta_N + \kappa_N,
$

\item
for all $\epsilon'>0$ there
exists $n_0$ and $(\pi_{n,\epsilon}^*)_{n\geq 0}$, $\pi_{n,\epsilon}^*\in \Pi_{\tkernel_0,\epsilon}^*$ so that for all $n\geq n_0$ and for all $\bar{s} \in 
\theinterestingstates$ it holds that
$$
\|
\pi_{n,\epsilon}(\cdot|\bar{s})
- \pi_{n,\epsilon}^{*} (\cdot|\bar{s}) 
\|_1
\leq
2\left(1-\epsilon\left(1-\frac{|\oactions(\bar{s})|}{|\mathcal{A}|}\right)-z_{\gamma',\epsilon,|\oactions(\bar{s})|}^{\circ (n-n_0)}(x_0)\right)
$$
where $x_0 = \frac{\epsilon|\oactions(\bar{s})|}{|\mathcal{A}|}$, $\beta' > \beta_N+\epsilon'$, where $\epsilon'$ is chosen such that $1 > \gamma' + \epsilon'$ with $ \gamma' = \frac{\tilde{\beta}}{((1-\epsilon)^N-\beta')\alpha} > 0$.
\end{enumerate}

The terms on the right hand side of 1.-4. converge to 0 as $\delta \rightarrow 0$.
In 5. the sequence $(y_n)_{n\geq 0}$ defined by 
$
y_n = 2(1-\epsilon(1-\frac{|\oactions(\bar{s})|}{|\mathcal{A}|})-z_{\gamma',\epsilon,|\oactions(\bar{s})|}^{\circ (n-n_0)}(x_0))
$
converges q-linearly to the limit 
$
y = 2(1-\epsilon(1-\frac{|\oactions(\bar{s})|}{|\mathcal{A}|})-x^{*}(\gamma',\epsilon,|\oactions(\bar{s})|))
$
at convergence rate
$
\frac{(1-\epsilon)\gamma'}{(x^{*}(\gamma',\epsilon,|\oactions(\bar{s})|)+\gamma')^2}
$
which tends to 0 as $\delta \rightarrow 0$.
\end{restatable}
The result is proved in appendix~\ref{ap:regrec}. Similarly to remark~\ref{re:accreformulation}, points 1.-4. of the corollary can be formulated in terms of accumulation points. An illustration of the accumulation points of $\min_{\bar{s}\in \theinterestingstates} \pi_{n,\tkernel,\epsilon} (\oactions(\bar{s})|\bar{s})$ and the lower estimate provided by corollary~\ref{le:elimitbounds} is given in figure~\ref{fig:ebandit} in the case of a 2-armed bandit
model. All details of the construction of the MDP underlying figure~\ref{fig:ebandit} can be found in the appendix B, example~\ref{ex:bandit}. Figure~\ref{fig:ebandit:a} shows the $\delta$-dependency of accumulation points of $\epsilon$-\eUDRL{} together with the estimates on accumulation points $\min_{\bar{s}\in\theinterestingstates} x^*(\gamma,\epsilon,|\oactions(\bar{s})|)$ for several values of regularization parameter. For reference, we also plot the estimates for accumulation points for \eUDRL{} in figure~\ref{fig:ebandit:b} although they are not directly comparable to those of $\epsilon$-\eUDRL{}. Figure~\ref{fig:eODTgridworld} illustrates the derived estimates in case of a regularized ODT recursion applied to a grid world. This example is characterized by a non-deterministic optimal policy, cf.~\ref{fig:eODTgridworld:b}, such that the $x_u$-based bound for ODT without regularization cannot be used. All details of the construction of the MDP underlying figure~\ref{fig:eODTgridworld} can be found in the appendix B, example~\ref{ex:eODTgridworld}.

\begin{figure}
  \begin{subfigure}{0.48\linewidth}
		\includegraphics[width=\linewidth]{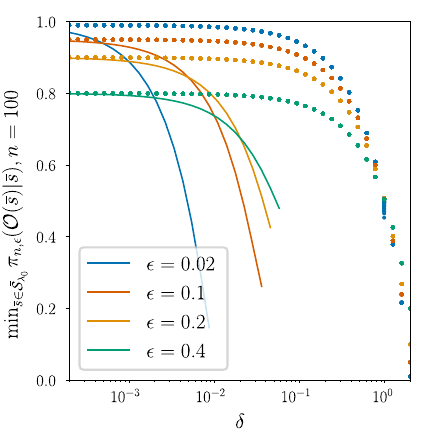}
		\caption{Accumulation point estimates - $\epsilon$-\eUDRL{}.}
        \label{fig:ebandit:a}
	\end{subfigure}
     \begin{subfigure}{0.48\linewidth}
		\includegraphics[width=\linewidth]{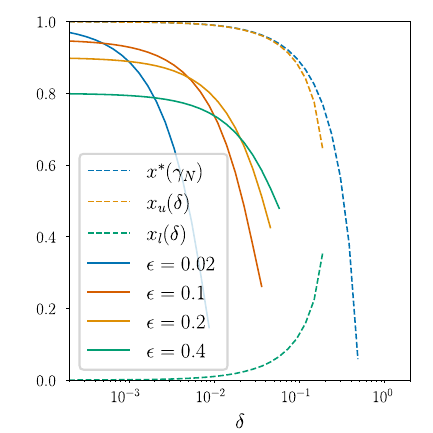}
		\caption{Acc. point estimates - $\epsilon$-\eUDRL{} \& \eUDRL{}.}
        \label{fig:ebandit:b}
	\end{subfigure}
    \caption{Illustration of estimates of corollary~\ref{le:elimitbounds} for $\epsilon$-\eUDRL{}. Plots show the behavior of $\min_{\bar{s}\in\theinterestingstates}\pi_n(\oactions(\bar{s})|\bar{s})$ and estimates for a $2$-armed bandit for several values of the regularization parameter. Plot (a) showcases the dependency on the distance $\delta$ to a deterministic kernel, where dotted graphs depict accumulation points (approximated by values at $n=100$) and the estimates $\min_{M} x^*(\gamma,\epsilon,M)$ provided in the corollary are depicted by solid graphs. For comparison, plot (b) shows the estimates provided in the corollary~\ref{le:elimitbounds} together with the estimates $x^*(\gamma_N)$ and $x_u(\delta)$ from corollaries \ref{le:limitbounds} and \ref{le:limitboundsM1}.}
    \label{fig:ebandit}
\end{figure}

\begin{figure}
  \begin{subfigure}{0.48\linewidth}
		\includegraphics[width=\linewidth]{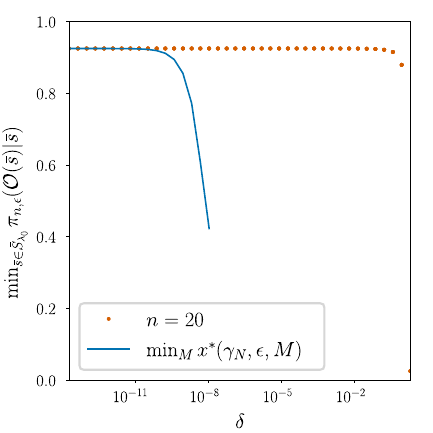}
		\caption{Accumulation point estimates - reg.~ODT}
        \label{fig:eODTgridworld:a}
	\end{subfigure}
     \begin{subfigure}{0.48\linewidth}
		\includegraphics[width=\linewidth]{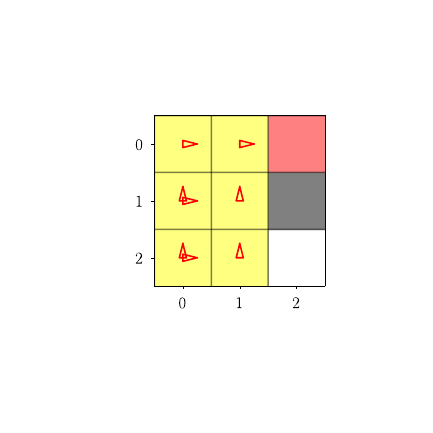}
		\caption{Map of grid world.}
        \label{fig:eODTgridworld:b}
	\end{subfigure}
    \caption{Illustration of estimates of corollary~\ref{le:elimitbounds} for a regularized ODT recursion. Plots show the behavior of accumulation points of  $\min_{\bar{s}\in\theinterestingstates}\pi_n(\oactions(\bar{s})|\bar{s})$ and respective estimates for a grid world example. Plot (a) shows accumulation points of $\min_{\bar{s}\in \theinterestingstates} \pi_{n,\tkernel,\epsilon} (\oactions(\bar{s})|\bar{s})$ of the ODT recursion with uniform regularization, $\epsilon=0.1$, together with the lower bound $\min_{M} x^*(\gamma,\epsilon,M)$. Plot (b) shows the map of the grid world. A wall is depicted in gray, the set $\theinterestingstates$ in yellow, and the goal in red. Arrows depict optimal actions associated with the specific state and goal.}
    \label{fig:eODTgridworld}
\end{figure}
We conclude this section with a brief summary. Adding regularization allowed us to carry out the discussion of continuity in full generality. This covers domains like, e.g., example~\ref{ex:eODTgridworld}, which does not fall into the special cases of section~\ref{se:continfty}. On the other hand the more regularization is applied,
the further the asymptotic policy must deviate from an optimal policy. Notice that policies from $\Pi_{\tkernel_0,\epsilon}^*$, $\bar{s} \in \theinterestingstates$, are only assigning a mass of $1-\epsilon\left(1-\frac{|\oactions(\bar{s})|}{|\mathcal{A}|}\right)$ to $\oactions(\bar{s})$ while an optimal policy is assigning a mass of $1$ to $\oactions(\bar{s})$, cf.~the discussion after the lemma~\ref{le:econtQfac}. On the other hand a smaller value of $\epsilon$ entails a smaller lower 
the bound on the state visitation terms $\alpha(\delta,\epsilon)$. This, in turn, affects $\gamma$
in that a smaller distance to a deterministic kernel $\delta$ is needed to recover asymptotic behavior for $\delta \rightarrow 0$, cf.~corollary~\ref{le:elimitbounds} and lemma~\ref{le:z} for the behavior of respective quantities. The trade-off in the choice of $\epsilon$ is apparent the comparison of $x^*(\gamma,\epsilon,M)$ bounds for several values of $\epsilon$ in figure~\ref{fig:ebandit}.

\section{Related Work}
\label{se:relwork}

In this section, we start by discussing key early works that study the theme of using iterated supervised learning to solve reinforcement learning problems in section~\ref{sec:rw:rwr}.
We then look at works associated with the more specific goal-conditioned reinforcement learning in section~\ref{sec:rw:gcrl}.
After, we tackle the most relevant works to this article---those directly relating to upside-down reinforcement learning---in section~\ref{sec:rw:udrl}.
We end this section by discussing work associated with the transformer architecture used by decision transformers and online decision transformers in section~\ref{sec:rw:transformers}.

\subsection{Reward-Weighted Regression}
\label{sec:rw:rwr}

The idea of using \emph{iterated supervised learning to solve reinforcement learning problems} was previously investigated by \citet{peters2007reinforcement} in the form of Reward-Weighted Regression (RWR),
which itself built on the Expectation-Maximization framework for RL by \citet{dayan1997using}.
However, the setting explored in their works is very limited; the work of \citet{dayan1997using} considered only a bandit scenario
(only one state and a finite number of actions) and
\citet{peters2007reinforcement} restricted their work to immediate reward
problems.
The extension of RWR to consider a full episodic setting (considering not only immediate rewards but rather full returns) was done by \citet{wierstra2008episodic}
and \citet{kober2011policy}.
Efficient off-policy schemes in the context of RWR were later discussed by \citet{hachiya2009efficient,hachiya2011reward}.
Finally, the use of deep neural network approximators was introduced to RWR by \citet{peng2019advantage}.

Theoretical studies of RWR are more sparse, with the monotonicity of RWR proved by~\citet{dayan1997using} and~\citet{peters2007reinforcement} in their respective settings. Both of the aforementioned monotonicity proofs build on the Expectation-Maximization paradigm.
More than a decade after the work on monotonicity, RWR's convergence to a global optimum was
proved by \citet{strupl2021reward} for compact state and action spaces (assuming an infinite number of samples and no function approximation).

RWR remains particularly relevant for this work because \eUDRL{} on $\Seg^{\diag}$
coincides with RWR (as discussed in section~\ref{se:recrewrites}).
This link motivated our approach to the continuity proofs
of \eUDRL{} generated quantities at deterministic kernels.

\subsection{Goal-conditioned Reinforcement Learning Without Relabeling}
\label{sec:rw:gcrl}

To our knowledge, the first work that could be considered goal-conditional RL was published
by 
\citet{schmidhuber1991learning,schmidhuber1990learning} in the context of learning selective attention.
In this work, there are extra goal-defining input patterns that encode various tasks so that the RL machine knows which task to execute next.
The rewards are granted at the end of each task.
The need to propagate learning signals through a non-differentiable
environment is elevated by a separate environment model network.
Essentially, it is model-based RL with sparse rewards, which is equivalent to SL except samples are collected online.
However, note that, in contrast to UDRL, this work and also all other works in this paragraph are missing segment/trajectory relabeling.
The work by \citet{schmidhuber1991learning}
also contributes to a body of literature on attention which provides grounds for today transformer architectures.
The options framework of \citet{sutton1999between} could also be considered as goal-conditioned RL (an option determines a policy).
Providing goals to a policy network provides possibility of generalization to unseen goals when representing the policy by deep network (similarly for value functions).
Building on the work of \citet{sutton2011horde}, the work of \citet{schaul2015universal} proposes convenient approximators of goal-conditioned value functions.
For an example of a more recent work on goal-conditioned RL, we can cite the paper by \cite{faccio2023goal} which introduced goal-conditioned policy generators for deep policy networks.

\subsection{UDRL}
\label{sec:rw:udrl}

UDRL \citep{schmidhuber2019reinforcement} extends the \emph{iterated SL} idea by
learning command-conditioned policies using
trajectory segments relabeled with the achieved goals/returns.
This was motivated by the desire to obtain
generalization across goals/returns (when representing the UDRL policy by
a deep neural network), and by the desire for
more efficient data usage (there are more segments in a given trajectory than there are states).
Trajectory relabeling appeared already before in Hindsight Experience Replay (HER) by \citet{andrychowicz2017hindsight} and we describe how our work is related to this work in more detail below. In the first instance, \citet{schmidhuber2019reinforcement} developed the specific algorithm referred to here as \eUDRL{} for deterministic environments. However, \citet{schmidhuber2019reinforcement} has also speculated on a potential variant of the algorithm for application in stochastic environments.

There are many important stochastic
environments which are close to deterministic in the sense that they only carry minor non-determinism. With this motivation (a return-reaching variant of) \eUDRL{} was
demonstrated by \cite{srivastava2019training} to solve successfully many non-deterministic environments (e.g. MuJoCo tasks) featuring transition kernels exhibiting only small non-determinism.
This was followed by \cite{ghosh2021learning}
in context of GCSL (essentially \eUDRL{} restricted to fixed horizon and state reaching task) with further impressive benchmark performance. Concurrently there appeared also work of \cite{kumar2019reward}
which is again essentially \eUDRL{}.

The idea of combining \eUDRL{} with transformer architectures \citep{schmidhuber1992learningFWP,vaswani2017attention,schlag2021linear} as policy approximator was successfully demonstrated in the Decision Transformer (DT) architecture proposed by \citet{chen2021decision}.
\citet{chen2021decision}, however, only looked at the offline RL setting (equivalent to a single \eUDRL{} iteration).
Shortly after the DT architecture was introduced, a generalized version was proposed by \citet{furuta2021generalized}, which led to a non-trivial performance improvement.

\citet{strupl2022upsidedown} and \citet{paster2022cant}
described the problems and causes of \eUDRL{} convergence issues 
in environments with stochastic transition dynamics.
While the work of \cite{strupl2022upsidedown} rewrote the recursion in \eUDRL{} to identify the convergence issues,
the work of \cite{paster2022cant} directly proposed a solution to this issue in the context of DTs. Several subsequent works were proposed to fix the issues with \eUDRL{} stochastic environments, such as the work by \citet{yang2022dichotomy} in context of DT and 
the work of \citet{faccio2023goal} in context of \eUDRL{}.

Some contributions are of particular relevance for the article at hand. We describe these works in more detail below. 

\paragraph{\citet{brandfonbrener2023does}}
discusses \eUDRL{} (and related algorithms)
in context of offline RL, i.e.,
it investigates just one iteration of \eUDRL{} aiming for
sample bounds. In contrast, we are developing continuity results and bounds for both any finite number of iterations and for the asymptotic case where there is an infinite sample size.
As a consequence, there is an overlap between this work and that of \citet{brandfonbrener2023does} in some simple results (namely, the continuity of expected return (goal-reaching objective in our case) for the first iteration as shown by theorem 1 of \citet{brandfonbrener2023does}). However, note that the continuity for the
first iteration is immediate, as the initial policies do not depend on $\tkernel$. The main difficulty in the article at hand arises as policies are functions of $\tkernel$. This leads to a
discontinuity of policies in $\tkernel$
at $\tkernel_0$ (a deterministic kernel) in $\theinterestingstates$ for the second and later iterations.\footnote{We refer the reader to the discussion at the end of appendix \ref{ap:interiorcont} which justifies this statement.} To overcome this difficulty, we utilize a weaker notion of continuity.
Here, we used the notion of relative continuity (see the induction step
in our proof of theorem~\ref{le:detcont}).
Developing asymptotic results leads to yet another level of complexity (see sections \ref{se:continfty}
and \ref{se:regrec}).

\paragraph{\citet{ghosh2021learning}}
introduces the following lower bound on the goal-reaching objective (see theorem 3.1 of \citet{ghosh2021learning}):
$$
J^{\pi} \geq J_{\mathrm{GCSL}}(\pi) - 4T(T-1)\alpha^2 + C,
$$
where $J^{\pi}$ denotes the goal-reaching objective
for a policy $\pi$, $J_{\mathrm{GCSL}}(\pi)$
denotes a GCSL objective for the policy $\pi$
(the same as \eUDRL{} objective \eref{eq:objective} for trailing segments), $T$ denotes a fixed horizon, $\alpha' := \max_{\bar{s}} D_{\mathrm{TV}}(\pi(\cdot|\bar{s})|\pi_{\mathrm{old}}(\cdot|\bar{s}))$
denotes total variation distance of the policy $\pi$ from
policy $\pi_{\mathrm{old}}$ which was used to collect trajectories, and $C$ is a constant on $\pi$. There are (at least) two problems when
one tries to adopt this result in our work to, for example, assess the continuity of the goal-reaching objective (or the $\delta$-dependent error bounds for the goal-reaching objective).

The first of the aforementioned problems is that to use the bound in a meaningful way, one has to minimize $\alpha'$. This could be done by assuming that a sequence of GCSL (\eUDRL{}) policies has a limit (then $\alpha' \rightarrow 0$).
This is, however, difficult to prove.
To maintain a certain level of rigor, we carefully avoid the limit assumption and instead investigate all the possible accumulation points.
From this perspective, this result is not useful as, for a finite number of iteration, $\alpha'$ could be large.
Nor is it useful for the asymptotic case as we could
not afford the limit assumption.

The second of the aforementioned problems is that the dependence on $\delta$ (the distance of the transition kernel $\tkernel$ to a deterministic kernel
$\tkernel_0$)
is not considered, as, in order to use the bound
(e.g. for assessing continuity in $\tkernel$ at $\tkernel_0$),
one has to determine the dependence on $\delta$
for all three terms on the right side.
This seems to be more challenging than determining this 
dependence for $J^{\pi}$ itself (as we do
here) and would most likely still be much less precise (see the point above).

\paragraph{\citet{kumar2019reward}} introduced an algorithm for learning return-conditioned policies that is essentially \eUDRL{} on $\Seg_{\trail}$ with fixed horizon.
However, while the authors do include some theoretical discussion on the properties of \eUDRL{}, this discussion does not touch on the distance of an MDP transition kernel to a deterministic kernel (the role played by deterministic kernels is not investigated at all). 
Therefore, with regards to the continuity of the transition kernel we are particularly interested in here, this work has little to offer.

\paragraph{\cite{leibovich2023learning}} introduced the Iterated Inversion for Learning Control algorithm~\citep[Algorithm 2]{leibovich2023learning} (IT-IN). This algorithm can be seen as a modification of GCSL, which is studied extensively in~\cite{ghosh2021learning} and the article at hand. The following are important differences:
\begin{itemize}
\item
IT-IN assumes only deterministic environments and deterministic policies that are fitted using mean square loss. To promote exploration noise is added to the policy.
\item
IT-IN adapts the initial goal distribution in every iteration (similarly to eUDRL implementations in context of return-reaching tasks, see e.g. \cite{srivastava2019training}). 
\item
IT-IN handles intent (or desired behavior) by conditioning the policy on an embedding of desired state trajectories (instead of goal states as in GCSL). IT-IN aims to replicate entire state trajectories.
\end{itemize}
The experimental results focus on the evaluation of convergence and performance comparison to RL and Inverse RL baselines. The theoretical analysis demonstrates the convergence of~\cite[Agorithm 1]{leibovich2023learning} using Newton's method under the following premises:
\begin{itemize}
\item
As a consequence of the restriction of the investigation to deterministic environments there exists a map $\mathcal{F}$ between action and state trajectories. The map $\mathcal{F}$ is assumed to be bijective, $\mathcal{F}$ and its inverse are both assumed to be continuously differentiable. There are several additional assumptions on $\mathcal{F}$ including boundedness of Jacobian, Jacobian inverse and Jacobian difference.
\item
The inverse of $\mathcal{F}$ (a policy proxy) is approximated by fitting a linear map on a finite number of samples. No exploration noise is added to policy outputs. The initial goal distribution is assumed fixed over iterations (no adaptation).
\end{itemize}
The theory developed is largely orthogonal to ours. While we extend the known positive behavior of \eUDRL{} in deterministic environments to non-deterministic ones by exploring continuity in the transition kernel (assuming no function approximation and an unlimited sample size), \citet{leibovich2023learning} focuses on proving IT-IN convergence for a specific subclass of deterministic environments, assuming finite sample size and a particular function approximation class.

\paragraph{\cite{andrychowicz2017hindsight}} introduced Hindsight Experience Replay (HER), the popular extension of experience replay that has be applied to many off-policy RL algorithms, such as DQN and DDPG.
HER is closely related to \eUDRL{} in some aspects.
Like \eUDRL{}, HER deals with goals, a goal map, and goal dependent policies (note that a dependence on the time step/remaining horizon could be accounted for by including it into the state representation).
HER is formulated for a fixed horizon with discounting (no discounting case is included).
The principle of the HER extension relies in populating the replay buffer not only with the trajectory that the agent actually encountered, but also with its \lq\lq{}relabeled\rq\rq{} version, where the original goals are replaced with the actually achieved final state and the corresponding rewards.
HER then uses an off-policy algorithm to learn value functions and a policy using data in the replay buffer.

The main distinguishing characteristics of  \eUDRL{} with respect to HER are that
(1) HER uses not just the relabeled trajectories but also the original trajectories, i.e., with the original intended goal;
and, (2) HER does not fit the next policy directly to some action conditional (at least when DQN or DDPG is used as the offline algorithm as in the original HER paper), e.g., DQN aims to rather learn the critic and derive the policy as an epsilon-greedy
policy using the critic.
These are the main problems when one wants to connect HER and \eUDRL{} recursion.
The second problem could be elevated by using a convenient RL algorithm (e.g. RWR) after which, because the substantial part of the data is not relabeled, we are left with some crossover between RWR and \eUDRL{} recursions.
This observation could already hint some properties of an  HER extension of RWR.
Note that some of the problems HER shares with \eUDRL{} were already reported, analyzed, and/or fixed in the literature. We refer the interested reader to the works of \citet{lanka2018archer} and \citet{schramm2023usher}.

\paragraph{\citet{paster2022cant}} describes the causes behind DT divergence in stochastic environments and proposes a fix.
To resolve the issues with convergence in stochastic domains, instead of conditioning on return, \citet{paster2022cant} propose to condition on a statistics
$I(\tau)$ (where $\tau$ stands for a trajectory) which is independent of the stochasticity of the environment. In their work, $I$ is represented by NN and learned using an adversarial scheme. Trajectories are then clustered according to $I$ values with $I$ values subsequently mapped to the average return of their corresponding cluster.
The result is then that by conditioning the policy on $I$ values, the corresponding 
average return can be reached consistently, i.e., in expectation.
This is in contrast to when conditioning on a statistics which is generally dependent on the environment's stochasticity, such as the trajectory return (or some general abstract goals like in \eUDRL{}).
For a detailed description of what it means to be \lq\lq{}independence of the environment stochasticity\rq\rq{}
please refer to the original article by \citet{paster2022cant}.

\subsection{Transformers}
\label{sec:rw:transformers}

Transformers with unnormalized linearized self attention (also known as Unnormalized Linear Transformers, or ULTRAs), appeared at least as early as the 1990s in works by \cite{schmidhuber1992learningFWP,schmidhuber1991learningFWP} under name fast-weight programmers~\citep{schlag2021linear}.
The goal of FWPs was to obtain a more storage-efficient alternative to recurrent networks by processing each input with both a slow and a fast network, where the output of the slow network provided changes to the weights of the fast network.
The attention terminology was introduced shortly after by \citet{schmidhuber1993reducing}.
The modern transformer architecture was introduced by \citet{vaswani2017attention} and scaled quadratically in input size.
More recently, variants of the \citet{vaswani2017attention} architecture have reverted to using linearized attention (e.g., see the work by \citet{katharopoulos2020transformers} and \citet{choromanski2020rethinking}), as this results in a linear scaling.

\section{Conclusion}%
\label{se:conclusion}
Our contribution lies in being the first rigorous treatise on the convergence and stability of some key\lq\lq reinforcement learning through supervised learning\rq\rq{} approaches. Specifically, our study focuses on both MDPs of the CE type (on the side of their mathematical formulation) and on \eUDRL{} (on the side of the algorithms investigated). Taken together, these frameworks offer sufficient generality to accommodate common \lq\lq{}RL through SL\rq\rq{} training schemes, including GCSL, ODT and RWR. In particular, we have shown how GCSL, ODT, and RWR can be expressed in the framework of CEs and interpreted as derivatives of \eUDRL{}. This is achieved by restricting the technical analysis to specific segment spaces $\SegTrail$ and $\SegDiag$. In essence, the training iteration of GCSL, ODT and RWR can be understood as instances of the \eUDRL{} training iteration on $\Seg$, $\SegTrail$ or $\SegDiag$, for a conveniently chosen CE. To accommodate such training schemes, we expanded our analysis of convergence and stability to the segment spaces $\SegTrail$ and $\SegDiag$, in addition to proving results for $\Seg$. Our findings can roughly be categorized into three clusters: (1) stability analysis of training schemes at a finite number of iterations, (2) investigation of the asymptotic properties of training schemes, and (3) results that complement the discussion of algorithms targeting a broader and more complete mathematical picture.

\subsection{Investigation of Stability at a Finite Number of Iterations}

In this work, we have demonstrated the instability of \eUDRL{} at the boundary of the space of transition kernels. Specifically, we have provided examples of environments that are characterized by non-deterministic kernels located on the boundary, where even a tiny perturbation of the kernel can result in a significant, arbitrary shift in the value or in the goal-reaching objective generated by the training schemes. Examples \ref{ex:boundarypoint} and \ref{ex:detpoint}, representing non-deterministic or deterministic transition kernels, respectively, illustrate the presence of non-removable discontinuities in \eUDRL{}-generated policies at the boundary of the transition kernel space for iterations $n \geq 2$. In contrast, within the interior of the kernel space, both policies and values remain continuous. However, at deterministic kernels, while values are continuous, policies might have non-removable discontinuities. To address this, we introduced the concept of relative continuity, which corresponds to a form of continuity with respect to a quotient topology. Using this concept, we showed that \eUDRL{} policies are relatively continuous at deterministic kernels for any finite number of iterations. This entails the continuity of the goal-reaching objective and the stability of \eUDRL{}-type training schemes at deterministic kernels at any finite number of iterations.%
Together with the convergence of \eUDRL{} to optimality at deterministic kernels, this implies near-optimality of the goal-reaching objective at transition kernels that are close to deterministic.

\subsection{Convergence and Investigation of Stability at an Infinite Number of Iterations}

In this work, we investigated the asymptotic behavior of \eUDRL{}-generated quantities as the number of iterations tends to infinity. We established the relative continuity of the sets of accumulation points of \eUDRL{} policies and the continuity of the associated goal-reaching objectives at deterministic kernels, specifically for two important special cases. These cases are characterized by additional assumptions that facilitate the analysis of continuity.

\begin{enumerate}
\item The support of the CE's initial distribution contains the set of \lq\lq{}critical states\rq\rq{} $\theinterestingstates \subset \supp \bar{\mu}$. In simple terms, $\theinterestingstates$ consists of the states that are essential to prove the continuity of the goal-reaching objective at a given deterministic kernel $\tkernel_0$. States that cannot be reached by the CE, in particular, are excluded from $\theinterestingstates$.
\item The optimal policy is unique on $\theinterestingstates$, which means that for all states $(\forall \bar{s}\in \theinterestingstates): |\oactions(\bar{s})|=1$. This condition alleviates the complexity of studying relative continuity, significantly simplifying the analysis of stability and convergence.
\end{enumerate}
Thus, both conditions entail near-optimal behavior of the goal-reaching objective at near-deterministic kernels in the asymptotic limit, as is also the case at any finite number of iterations. Making the established estimates fully explicit, we derived estimates that quantify the error from optimality based on the distance of the transition kernel from a deterministic kernel. We also derived a bound on the \eUDRL{} policy error and assessed its $q$-linear convergence rate. Although we believe that the outlined conditions encompass a wide range of practical scenarios, a fully general discussion of the relative continuity of accumulation point sets for \eUDRL{}-generated policies at deterministic kernels remains an open problem. Motivated by the widespread use of regularization in \lq\lq{}RL through SL\rq\rq{} training schemes (e.g., in ODT), we investigated $\epsilon$-greedy regularization of the policy iteration of \eUDRL{}. We established the relative continuity of the set of accumulation points of $\epsilon$-eUDRL-generated policies in full generality. As before, this entails the continuity of the accumulation point sets of the corresponding goal-reaching objective. We also showed that the goal-reaching objective exhibits near-optimal behavior at near-deterministic kernels in the asymptotic limit for $\epsilon$-\eUDRL{}. Akin to the discussion without regularization, we provide estimates that quantify the error from optimality given the distance of the transition kernel from a deterministic kernel for $\epsilon$-\eUDRL{}. Notice, however, that the asymptotic analysis is conducted in full generality without relying on the special assumptions used in previous cases.

\subsection{Further Results}
We established the continuity of \eUDRL{}-generated policies and goal-reaching objectives at any finite number of iterations for transition kernels located within the interior of the space of all kernels. Unlike deterministic kernels, this does not immediately entail convergence to near-optimality. We include this result to complete the mathematical picture and to stipulate further work. We also included a proof of the optimality of \eUDRL{} at deterministic kernels: While this result may be considered as implicitly understood within the literature, we decided to write it up for completeness. The mathematical investigation is supplemented by a range of worked-out examples that might prove useful for the continued development of this research area. Specifically, we illustrated our results through examples involving a $2$-armed bandit, a random walk on $\mathbb{Z}_3$ and a $3\times 3$ grid world domain. The code which was used for computing the examples and generating the associated figures is available at \url{https://github.com/struplm/eUDRL-GCSL-ODT-Convergence-public}

\acks{%
This work was supported by the European Research Council (ERC, Advanced Grant Number 742870), the Swiss National Supercomputing Centre (CSCS, Project s1090), and by the Swiss National Science Foundation (Grant Number  200021\_192356, Project NEUSYM). We also thank both the NVIDIA Corporation for donating a DGX-1 as part of the Pioneers of AI Research Award and IBM for donating a Minsky machine.}

\bibliography{main}

\newpage

\appendix

\section{The Segment Distribution and its Factorization}
\label{ap:SegmentDist}
Since we assume a deterministic reward function, rewards are completely determined by state-action-state transitions and can be omitted from segments or trajectories. Once a CE has reached an absorbing state in $\bar{\mathcal{S}}_A$, the policy becomes irrelevant as all actions lead to the same absorbing state. For this reason, it is sufficient to restrict the discussion to all policies $\pi$ that deterministically select a unique action $a_A$ in all absorbing states $\bar{\mathcal{S}}_A$, i.e.,
$\forall \bar{s} \in \bar{\mathcal{S}}_A:\:\pi(a_A|\bar{s}) = 1$. In this section, we assume that all trajectories are generated by using one fixed policy $\pi$.
The length $l(\cdot)$ of a trajectory is defined as the number of transitions
until an absorbing state is entered for the first time.
Regarding trajectories, we can consider just prefixes of length $N$ because the
CE MDP bounds the remaining horizon component by $N$. Further, we can
restrict ourselves to the subspace $\Traj \subset (\bar{\mathcal{S}}\times\mathcal{A})^N\times\bar{\mathcal{S}}$ allowed
by remaining horizon/goal dynamics of CE, i.e., for a trajectory
$\tau = (\bar{s}_0,a_0,\ldots,\bar{s}_N) \in \Traj$, $\bar{s}_t = (s_t,h_t,g_t)$, $0\leq t\leq N$,
the remaining horizon $h_t$ decreases by 1
from its initial value $h_0=l(\tau)$ till 0 when entering to
an absorbing state. The goal $g_t=g_0$ remains unchanged.
Thus $\tau$ is fully determined by the initial horizon $h_0=l(\tau)$, the initial goal $g_0$, the
states $s_0,\ldots,s_{l(\tau)}$, and actions $a_0,\ldots,a_{l(\tau)-1}$, i.e.,
$\tau = ((s_0,h_0,g_0),a_0,(s_1,h_0-1,g_0),a_1,\ldots,(s_{l(\tau)},0,g_0),a_A,\ldots,(s_{l(\tau)},0,g_0))$.
The probability of $\tau$ is given as:
$$
\prob( \mathcal{T} = \tau;\pi)
=
\left( \prod_{t=1}^{l(\tau)}
\tkernel(s_t | a_{t-1}, s_{t-1})
\right)
\times
\left( \prod_{t=0}^{l(\tau)-1}
\pi(a_t|\bar{s}_t)
\right)
\times
\bar{\mu}(\bar{s}_0),
$$
where $\mathcal{T}:\Omega\rightarrow\Traj$, $\mathcal{T} = ((S_0,H_0,G_0),A_0,\ldots,(S_N,H_N,G_N))$ is the trajectory random variable map.

\paragraph{Segment Distribution:}
Segments are assumed to be continuous chunks of trajectories in $\Traj$,
so they respect the CE horizon/goal dynamics.
We assume them to be always contained within the length of a trajectory.
Notice that \eUDRL{} is learning actions just in a state which happens to be the first state of a segment.
Since learning actions in absorbing states is not meaningful, we will assume the first state of a segment to be transient (i.e., $\in \bar{S}_T$)\footnote{One could likewise additionally assume the original MDP component of this state to be transient for the similar reason.}.
Thus, segments $\sigma$ are fully determined by the
segment length, denoted by $l(\sigma)$ (the number of transitions), by the
remaining horizon and goal at the segment beginning
$h_0^{\sigma}$,$g_0^{\sigma}$, and by $l(\sigma)+1$ states and $l(\sigma)$
actions. Without loss of generality we will identify $\sigma$ with such a tuple
$\sigma=(l(\sigma),s_0^{\sigma},h_0^{\sigma},g_0^{\sigma},a_0^{\sigma},s_1^{\sigma},a_1^{\sigma},\ldots,s_{l(\sigma)}^{\sigma})$.
The space of all such tuples will be denoted $\Seg$.
Formally, we will assume that there exists a random variable map $\Sigma: \Omega \rightarrow \Seg$, $\Sigma = (l(\Sigma),\stag{S}_0,\stag{H}_0,\stag{G}_0,\stag{A}_0,\stag{S}_1,\stag{A}_1,\ldots,\stag{S}_{l(\Sigma)})$
with distribution $d_{\Sigma}^{\pi}$ given
below.

We construct the segment distribution $d_{\Sigma}^{\pi}$ in a similar way as the state visitation distribution---summing across appropriate trajectory distribution marginals.
This causes the result to be un-normalized (among other restrictions, e.g., on the first state of the segment), therefore we have to include a normalization constant $c$. ($\forall \sigma \in \Seg$):
$$
\begin{aligned}
&(d_{\Sigma}^{\pi_n}(\sigma) :=)\:
\prob(\Sigma=\sigma;\pi)
=
\\
&= c^{-1}\sum_{t \leq N -l(\sigma)}
\prob(S_t=s_0^{\sigma},H_t=h_0^{\sigma},G_t=g_0^{\sigma},A_t=a_0^{\sigma}, \ldots, S_{t+l(\sigma)}=s_{l(\sigma)}^{\sigma} ;\pi)
\\
&=
c^{-1}
\sum_{t \leq N -l(\sigma)}
\left( \prod_{i=1}^{l(\sigma)}
\tkernel(s_{i}^{\sigma} | a_{i-1}^{\sigma}, s_{i-1}^{\sigma})
\right)
\!\cdot\!
\left( \prod_{i=0}^{l(\sigma)-1}
\pi(a_i^{\sigma}|\bar{s}_i^{\sigma})
\right)
\cdot
\prob(S_t=s_0^{\sigma},H_t=h_0^{\sigma},G_t = g_0^{\sigma};\pi)
\\
&=
c^{-1}
\left( \prod_{i=1}^{l(\sigma)}
\tkernel(s_{i}^{\sigma} | a_{i-1}^{\sigma}, s_{i-1}^{\sigma})
\right)
\!\cdot\!
\left(\prod_{i=0}^{l(\sigma)-1}
\pi(a_i^{\sigma}|\bar{s}_i^{\sigma})
\right)
\cdot
\sum_{t \leq N -l(\sigma)}
\prob(S_t=s_0^{\sigma},H_t=h_0^{\sigma},G_t = g_0^{\sigma};\pi)
,
\end{aligned}
$$
where
$$
c := \sum_{\sigma\in \Seg}
\sum_{t \leq N -l(\sigma)}
\prob(S_t=s_0^{\sigma},H_t=h_0^{\sigma},G_t=g_0^{\sigma},A_t=a_0^{\sigma}, \ldots, S_{t+l(\sigma)}=s_{l(\sigma)}^{\sigma} ;\pi) > 0
.
$$
Since we constrained $N\geq 1$, $\supp \bar{\mu} \subset \bar{S}_T$ (which prevents 
degenerate cases), there must be a non-zero probability assigned to some segments
leading to the full summation in the definition of $c$ being positive. 
This means that $c^{-1}$ and the segment distribution are defined correctly.
Given the factorized form of the segment distribution we can conclude the following (through computing marginals and calculating conditional probability ratios):
\begin{equation*}
\prob(l(\Sigma)=k,\stag{S}_0=s_0,\stag{H}_0=h_0,\stag{G}_0=g;\pi)
=
c^{-1}
\sum_{t \leq N -k}
\prob(S_t=s_0,H_t=h_0,G_t = g;\pi)
\end{equation*}
\begin{multline*}
\prob(\stag{A}_0=a_0,\stag{S}_1=s_1,\stag{A}_1=a_1,\ldots,\stag{S}_k=s_k
|\stag{S}_0=s_0,\stag{H}_0=h_0,\stag{G}_0=g_0,l(\Sigma)=k
;\pi)
\\
=
\left( \prod_{i=1}^{k}
\tkernel(s_{i} | a_{i-1}, s_{i-1})
\right)
\cdot
\prod_{i=0}^{k-1}
\pi(a_i|\bar{s}_i)
\end{multline*}
and
$(\forall 0 \leq i \leq k)$:
\begin{multline}
\prob(\stag{A}_i=a_i|\stag{S}_i=s_i,\ldots,\stag{S}_0=s_0,\stag{H}_0=h_0,\stag{G}_0=g,l(\Sigma)=k;\pi)
\nonumber
\\
\begin{aligned}
&=\pi(a_i|s_i,h_0-i,g)
\\
&=\prob(\stag{A}_i=a_i|\stag{S}_i=s_i,\stag{H}_0=h_0,\stag{G}_0=g;\pi), 
\end{aligned}
\tag{\ref{eq:segactions}}
\end{multline}
\begin{multline}
\prob(\stag{S}_i=s_i|\stag{S}_{i-1}=s_{i-1},\stag{A}_{i-1}=a_{i-1},\ldots,\stag{S}_0=s_0,\stag{H}_0=h_0,\stag{G}_0=g_0,l(\Sigma)=k;\pi)
\\
\begin{aligned}
&=
\tkernel(s_i|s_{i-1},a_{i-1})
\\
&=
\prob(\stag{S}_i=s_i|\stag{S}_{i-1}=s_{i-1},\stag{A}_{i-1}=a_{i-1})
\end{aligned}
\tag{\ref{eq:segtransitions}}
\end{multline}
After defining $\stag{H}_i := \stag{H}_0-i,\stag{G}_i := \stag{G}_0$, we can write
$\prob(\stag{A}_i=a_i|\stag{S}_i=s_i,\stag{H}_i=h_i,\stag{G}_i=g_i;\pi)= \pi(a_i|s_i,h_i,g_i)
$.

\paragraph{Bounding $c$:}
For further investigation it is useful to bound the normalizing constant
$c$. We can write
\begin{equation*}
\begin{aligned}
0 < c &= \sum_{\sigma \in \Seg}
\sum_{t \leq N -l(\sigma)}
\prob(S_t=s_0^{\sigma},H_t=h_0^{\sigma},G_t=g_0^{\sigma},A_t=a_0^{\sigma}, \ldots, S_{t+l(\sigma)}=s_{l(\sigma)}^{\sigma} ;\pi)
\\
&=
\sum_{k=1}^{N}
\sum_{t \leq N - k}
\sum_{\sigma : l(\sigma)=k}
\prob(S_t=s_0^{\sigma},H_t=h_0^{\sigma},G_t=g_0^{\sigma},A_t=a_0^{\sigma}, \ldots, S_{t+l(\sigma)}=s_{l(\sigma)}^{\sigma} ;\pi)
\\
&\leq
\sum_{k=1}^{N}
\sum_{t \leq N - k}
1
=
\frac{N(N+1)}{2}
.
\end{aligned}
\tag{\ref{eq:cupperbound}}
\end{equation*}

\section{Motivating Examples}
\label{ap:examples}

We begin this section by demonstrating discontinuities of \eUDRL{} generated quantities in the transition kernel given specific example of compatible families of MDPs.
In order to demonstrate a discontinuity at a specific point $\tkernel_0 \in (\Delta \mathcal{S})^{\mathcal{S}\times\mathcal{A}}$
we take limits with respect to two distinct rays meeting at $\tkernel_0$ and showing that they disagree. Apart of providing plots
depicting how the value of investigated quantity depends on the ray parameter we also provide exact computation of the limits.
During the computation we will use the form of \eUDRL{} recursion
introduced by \citet{strupl2022upsidedown}
($\forall a\in \mathcal{A},(s,h,g)\in \supp \den_{\tkernel,\pi_n}$):
\begin{equation}
\begin{aligned}
\pi_{n+1}(a|s,h,g)
&=
\prob_{\tkernel}(\stag{A}_0=a|\stag{S}_0=s,l(\Sigma)=h,\rho(\stag{S}_{l(\Sigma)}) = g; \pi_n)
\\
&=
\frac{
Q_{A}^{\pi_n,g}(s,h,a)
\pi_{A,n} (a|s,h)
}
{
\sum_{a\in\mathcal{A}} Q_{A}^{\pi_n,g}(s,h,a) \pi_{A,n} (a|s,h)
}
,
\end{aligned}
\label{eq:recursion}
\end{equation}
where $Q_{A}$ denotes an ``average" $Q$-value and $\pi_{A,n}$ denotes
an ``average" policy
\begin{equation}
\begin{aligned}
Q_{A}^{\pi_n,g}(s,h,a) &= 
\prob_{\tkernel}(\rho (\stag{S}_{l(\Sigma)})=g | \stag{A}_0=a, \stag{S}_0=s,l(\Sigma)=h; \pi_n)
\\
\pi_{A,n} (a|s,h) &=
\sum_{h'\geq h,g'\in\mathcal{G}}
\pi_n (a|h',g',s)
\prob_{\tkernel}( \stag{H}_0=h', \stag{G}_0=g' |\stag{S}_0=s,l(\Sigma)=h; \pi_n)
.\label{eq:apolicy}
\end{aligned}
\end{equation}
The derivation of the above \eUDRL{} recursion rewrite uses the same techniques as the proof of lemma \ref{le:recrewrites}.%

At points $\tkernel$ lying on a boundary of $(\Delta \mathcal{S})^{\mathcal{S}\times\mathcal{A}}$ the continuity can break. Therefore, we cannot prove continuity on the whole
boundary.
This is illustrated in the following counter-example.

\begin{example} (non-removable discontinuity of goal reaching objective at a boundary point)
\label{ex:boundarypoint}
Consider an MDP $\mathcal{M} = (\mathcal{S},\mathcal{A},\tkernel,\mu,r)$
with three states $\mathcal{S} = \{0,1,2\}$ and three actions
$\mathcal{A} = \{0,1,2\}$. The state 0 is the only initial state.
We consider its CE $\bar{M}$ with maximum horizon $N=1$, $\mathcal{G}:=\mathcal{S}$, $\rho := \mathrm{id}_{\mathcal{S}}$. The CE's initial distribution fixes the initial remaining horizon at $H_0=1$ and initial goals
are distributed as $\prob(G_0=0) = \prob(G_0=2) = \frac{1}{2}$ (so $\prob(G_0=1) = 0$).
It follows that we are left with only two transient states $(0,1,0)$, $(0,1,2) \in \mathcal{S}\times\mathbb{N}_0\times\mathcal{G}$ differing just by the goal component which are actually visited. Thus it suffices
to provide policies and values only for these states. 
Since we have $N=1$
and the only initial state $S_0 = 0$ it suffice to give the transition kernel only
from this state.
We split the computation into three parts and label them by letters A,B and C for later reference.

\vspace{1em}
\noindent\textbf{A} First we define the parametric transition kernel (with parameter $\alpha \in [0,1]$)
$$
\begin{tabular}{c|ccc}
$\tkernel_{\alpha}(g|a)$ & $g=0$ & $g=1$ & $g=2$ \\
\hline
$a=0$  & $1-\alpha$          & $\frac{\alpha}{4}$ & $\frac{3\alpha}{4}$ \\
$a=1$  & $\frac{3\alpha}{4}$ & $1-\alpha$         & $\frac{\alpha}{4}$  \\
$a=2$  & $\frac{1}{2}$       & $\frac{1}{2}$      & $0$ \\
\end{tabular}
$$
Which we identify with the matrix
$$
\underset{a\in \mathcal{A}, g \in \mathcal{G}}{[\tkernel_{\alpha}(g|a)]} =
\left(
\begin{array}{ccc}
1-\alpha          & \frac{\alpha}{4} & \frac{3\alpha}{4} \\
\frac{3\alpha}{4} & 1-\alpha         & \frac{\alpha}{4}  \\
\frac{1}{2}       & \frac{1}{2}      & 0 \\
\end{array}
\right)
\xrightarrow{\alpha\rightarrow 0+}
\left(
\begin{array}{ccc}
1 & 0 & 0 \\
0 & 1 & 0  \\
\frac{1}{2} & \frac{1}{2} & 0 \\
\end{array}
\right)
=
\underset{a\in \mathcal{A}, g \in \mathcal{G}}{[\tkernel_{0+}(g|a)]}.
$$
We will be
interested how values (policies) behave at and around the point $\alpha = 0$.
For computing the \eUDRL{} policies we will use the recursion
\eqref{eq:recursion}.
We will note argument names inside kernels and value functions
(in case they were substituted by actual values) so it is not necessary
to always look for how we fixed individual argument positions.\footnote{
This is only to improve readability of the example by helping the reader to keep track of substitutions. We use different color (gray)
for these extra notes.
}
Using \eref{eq:apolicy} we get
$Q_{A}^{\pi_n,g} (\subnote{s=}0,\subnote{h=}1,a) = \tkernel(g|a,\subnote{s=}0)$.
Since the condition $\supp \bar{\mu} \subset \bar{\mathcal{S}}_T$ and fixed horizon 1 the segments coincide with trajectories. Therefore, the distribution of the first extended state of a segment is the same
as the initial distribution $\bar{\mu}$:
$\prob(\stag{H}_0=1,\stag{G}_0=g'| \stag{S}_0=0,l(\Sigma)=1 ; \pi_n )
= \prob(G_0=g')$.
Thus we get the following relation for the ``average" policy:
$\pi_{A,n}(a|\subnote{s=}0,\subnote{h=}1) = \sum_{g' \in \mathcal{G}} \pi_n(a|\subnote{s=}0,\subnote{h'=}1,g') \prob(G_0=g')$.
The recursion \eqref{eq:recursion} can then be rewitten
\begin{align*}
\pi_{n+1}(a|\subnote{s=}0,\subnote{h=}1,g) &\propto
Q_{A}^{\pi_n,g} (\subnote{s=}0,\subnote{h=}1,a)
\pi_{A,n}(a|\subnote{s=}0,\subnote{h=}1)
\\
&=
\tkernel(g|a,\subnote{s=}0)
\pi_{A,n}(a|\subnote{s=}0,\subnote{h=}1)
.
\end{align*}
Assuming we start with uniform initial condition $\pi_0$ we obtain also
uniform $\pi_{A,0} = \frac{1}{3}$ leaving 
$\pi_{1}(a|\subnote{s=}0,\subnote{h=}1,g) \propto \tkernel(g|a,\subnote{s=}0)
\frac{1}{3} \propto \tkernel(g|a,\subnote{s=}0)$.
Thus we get
$$
\underset{a\in \mathcal{A}, g \in \mathcal{G}}{[\pi_{1,\alpha}(a|g)]} =
\left(
\begin{array}{ccc}
\frac{1-\alpha}{1-\alpha + \frac{3\alpha}{4} +\frac{1}{2}} &
\frac{\frac{\alpha}{4}}{\frac{\alpha}{4} + 1-\alpha + \frac{1}{2}} &
\frac{3}{4}
\\
\frac{\frac{3\alpha}{4}}{1-\alpha + \frac{3\alpha}{4} +\frac{1}{2}} &
\frac{1-\alpha}{\frac{\alpha}{4} + 1-\alpha + \frac{1}{2}} &
\frac{1}{4}
\\
\frac{\frac{1}{2}}{1-\alpha + \frac{3\alpha}{4} +\frac{1}{2}} &
\frac{\frac{1}{2}}{\frac{\alpha}{4} + 1-\alpha + \frac{1}{2}} &
0
\end{array}
\right)
\xrightarrow{\alpha \rightarrow 0+}
\left(
\begin{array}{ccc}
\frac{2}{3} &
0 &
\frac{3}{4}
\\
0 &
\frac{2}{3} &
\frac{1}{4}
\\
\frac{1}{3} &
\frac{1}{3} &
0
\end{array}
\right)
=
\underset{a\in \mathcal{A}, g \in \mathcal{G}}{[\pi_{1,0+}(a|g)]}
.
$$
This gives us
$$
\underset{a\in \mathcal{A}}{[\pi_{A,1,0+}]}
=
\underset{a\in \mathcal{A}, g \in \mathcal{G}}{[\pi_{1,0+}(a|g)]}
\underset{g \in \mathcal{G}}{[\prob(G_0=g)]}
=
\underset{a\in \mathcal{A}, g \in \mathcal{G}}{[\pi_{1,0+}(a|g)]}
\left(\begin{array}{c}\frac{1}{2}\\0\\\frac{1}{2}\end{array}\right)
= 
\left(\begin{array}{c}\frac{17}{24}\\\frac{1}{8}\\\frac{1}{6}\end{array}\right)
$$
We deduce from
$\pi_{2}(a|\subnote{s=}0,\subnote{h=}1,g) \propto \tkernel(g|a,\subnote{s=}0)
\pi_{A,1}(a)$ that 
$$
\lim_{\alpha \rightarrow 0+}\pi_{2}(a|\subnote{s=}0,\subnote{h=}1,g)
= \frac{ \lim_{\alpha \rightarrow 0+} \tkernel(g|a,\subnote{s=}0) \cdot  \lim_{\alpha \rightarrow 0+} \pi_{A,1}(a) }{ \sum_{a\in \mathcal{A}} \lim_{\alpha \rightarrow 0+} \tkernel(g|a,\subnote{s=}0) \cdot  \lim_{\alpha \rightarrow 0+} \pi_{A,1}(a)}
$$
provided all limits exists and division makes sense. Fortunately we will be interested
just in the column for $g=0$ and $g=1$
$$
\underset{a\in \mathcal{A}}{[\pi_{2,0+}(a|\subnote{g=}0)]} = \left(
\begin{array}{c}
\frac{\frac{17}{24}1}{\frac{17}{24}1 + \frac{1}{8}0 + \frac{1}{6}\frac{1}{2} }
\\
\frac{\frac{1}{8}0}{\frac{17}{24}1 + \frac{1}{8}0 + \frac{1}{6}\frac{1}{2} }
\\
\frac{\frac{1}{6}\frac{1}{2}}{\frac{17}{24}1 + \frac{1}{8}0 + \frac{1}{6}\frac{1}{2}}
\end{array}
\right)
=
\left(
\begin{array}{c}
\frac{\frac{17}{24}}{\frac{19}{24}}
\\
0 
\\
\frac{\frac{1}{12}}{\frac{19}{24}}
\end{array}
\right)
=
\left(
\begin{array}{c}
\frac{17}{19}
\\
0 
\\
\frac{2}{19}
\end{array}
\right)
,
\quad
\underset{a\in \mathcal{A}}{[\pi_{2,0+}(a|\subnote{g=}1)]} = \left(
\begin{array}{c}
0
\\
\frac{3}{5}
\\
\frac{2}{5}
\end{array}
\right).
$$
Finally using 
$$
V_{0+}^{\pi_2}(\subnote{s=}0,\subnote{h=}1,g) = \sum_{a\in \mathcal{A}}\tkernel_{0+}(g|\subnote{s=}0,a) \pi_{2,0+}(a|g)$$
and
$$J_{\tkernel,\pi_{2},0+} = \sum_{g\in \mathcal{G}} V_{0+}^{\pi_2}(\subnote{s=}0,\subnote{h=}1,g) \bar{\mu}(\subnote{s=}0,\subnote{h=}1,g) =
\sum_{g\in \mathcal{G}}
V_{0+}^{\pi_2}(\subnote{s=}0,\subnote{h=}1,g) \prob(G_0=g)
$$
the values
$V_{0+}^{\pi_2}(\subnote{s=}0,\subnote{h=}1,g=0) = \frac{18}{19}$, $V_{0+}^{\pi_2}(\subnote{s=}0,\subnote{h=}1,g=2) = 0$ and the goal reaching objective
$J_{\tkernel,\pi_{2},0+}= \frac{9}{19}$ can be computed.

\vspace{1em}
\noindent{}\textbf{B}
Now we aim to re-compute everything for $\alpha = 0$. Since $\tkernel$ is continuous in $\tkernel$ we get
$\tkernel_0 = \tkernel_{0+}$. Further since $\tkernel_0(g=2|\cdot) = 0$ we
get $\pi_{1,0}(\cdot|g=2) = \frac{1}{\mathcal{A}} = \frac{1}{3}$. Thus
$$
\underset{a\in \mathcal{A}, g \in \mathcal{G}}{[\pi_{1,0}(a|g)]} =
\left(
\begin{array}{ccc}
\frac{2}{3} &
0 &
\frac{1}{3}
\\
0 &
\frac{2}{3} &
\frac{1}{3}
\\
\frac{1}{3} &
\frac{1}{3} &
\frac{1}{3}
\end{array}
\right)
,\quad
\underset{a\in \mathcal{A}}{[\pi_{A,1,0}]}
=
\underset{a\in \mathcal{A}, g \in \mathcal{G}}{[\pi_{1,0}(a|g)]}
\left(\begin{array}{c}\frac{1}{2}\\0\\\frac{1}{2}\end{array}\right)
= 
\left(\begin{array}{c}\frac{1}{2}\\\frac{1}{6}\\\frac{1}{3}\end{array}\right)
$$

$$
\underset{a\in \mathcal{A}}{[\pi_{2,0}(a|\subnote{g=}0)]}
=
\underset{a\in \mathcal{A}} {[\frac{\tkernel_0(a|\subnote{g=}0)) 
\pi_{A,1,0}(a)}{\sum_{a\in \mathcal{A}} \tkernel_0(a|\subnote{g=}0))\pi_{A,1,0}(a)}]}
=
\left(
\begin{array}{c}
\frac{\frac{1}{2}1}{\frac{1}{2}1 + \frac{1}{6}0 + \frac{1}{3}\frac{1}{2}}
\\
\frac{\frac{1}{6}0}{\frac{1}{2}1 + \frac{1}{6}0 + \frac{1}{3}\frac{1}{2}}
\\
\frac{\frac{1}{3}\frac{1}{2}}{\frac{1}{2}1 + \frac{1}{6}0 + \frac{1}{3}\frac{1}{2}}
\end{array}
\right)
=
\left(
\begin{array}{c}
\frac{\frac{1}{2}}{\frac{4}{6}}
\\
0 
\\
\frac{\frac{1}{6}}{\frac{4}{6}}
\end{array}
\right)
=
\left(
\begin{array}{c}
\frac{3}{4}
\\
0 
\\
\frac{1}{4}
\end{array}
\right),
$$
similarly
$$
\underset{a\in \mathcal{A}}{[\pi_{2,0}(a|\subnote{g=}1)]}
=
\left(
\begin{array}{c}
0
\\
\frac{1}{2}
\\
\frac{1}{2}
\end{array}
\right)
.
$$
Finally we get
$$
V_{0}^{\pi_2}(\subnote{s=}0,\subnote{h=}1,g=0) = \frac{7}{8},\quad V_{0}^{\pi_2}(\subnote{s=}0,\subnote{h=}1,g=2) = 0,\quad
J_{\tkernel,\pi_{2},0}= \frac{7}{16}
,
$$
which means that we have a discontinuity at $\alpha = 0$: a boundary point.

\vspace{1em}
\noindent{}\textbf{C}
Moreover, this discontinuity could not be removed as can be seen
by computing the limit with respect to another ray, e.g.,
$$
\underset{a\in \mathcal{A}, g \in \mathcal{G}}{[\tkernel_{\alpha}'(g|a)]} =
\left(
\begin{array}{ccc}
1-\alpha          & \frac{3\alpha}{4} & \frac{\alpha}{4} \\
\frac{\alpha}{4} & 1-\alpha         & \frac{3\alpha}{4}  \\
\frac{1}{2}       & \frac{1}{2}      & 0 \\
\end{array}
\right)
.
$$
Notice that $\tkernel_{0+}' = \tkernel_{0}' = \tkernel_{0+} = \tkernel_{0}$. 
Following the same procedure as in part A we obtain the following results:
$$
\begin{gathered}
\underset{a\in \mathcal{A}, g \in \mathcal{G}}{[\pi_{1,0+}'(a|g)]} =
\left(
\begin{array}{ccc}
\frac{3}{2} & 0            & \frac{1}{4} \\
0           & \frac{2}{3}  & \frac{3}{4}  \\
\frac{1}{3} & \frac{1}{3}  & 0 \\
\end{array}
\right)
,\quad
\underset{a\in \mathcal{A}}{[\pi_{A,1,0+}'(a)]} =
\left(
\begin{array}{c}
\frac{11}{24} \\
\frac{3}{8} \\
\frac{1}{6}  \\
\end{array}
\right)
,
\\
\underset{a\in \mathcal{A}}{[\pi_{2,0+}'(a|g=0)]} =
\left(
\begin{array}{c}
\frac{11}{13} \\
0 \\
\frac{2}{13}  \\
\end{array}
\right)
,\quad
\underset{a\in \mathcal{A}}{[\pi_{2,0+}'(a|g=1)]} =
\left(
\begin{array}{c}
0 \\
\frac{9}{11} \\
\frac{2}{11}  \\
\end{array}
\right),
\\
V_{0+}^{'\pi_2}(g=0) = \frac{12}{13},\quad
V_{0+}^{'\pi_2}(g=2) = 0,\quad
J_{0+}^{'\pi_2} = \frac{6}{13}.
\end{gathered}
$$
\end{example}

\textbf{Discussing alternative definitions of \eUDRL{} recursion outside of $\supp \den_{\tkernel,\pi_n}$:} 
Since the computation of the limits in \textit{\textbf{A}} and \textit{\textbf{C}} does not depend
on the way we defined the \eUDRL{} recursion outside of $\supp \den_{\tkernel,\pi_n}$ in \eref{eq:recursionUsingNumeratorDenominator},
the presented non-removable discontinuities of $\pi_2$ and $J^{\pi_2}$
at $\tkernel_0$ cannot vanish by employing any alternative definition
of the recursion outside of $\supp \den_{\tkernel,\pi_n}$. This means there is no suitable alternative definition of \eUDRL{} recursion (outside of $\supp \den_{\tkernel,\pi_n}$) which would remove the present discontinuity, e.g. in the goal-reaching objective. Therefore, the concrete
definition of the recursion in \eref{eq:recursionUsingNumeratorDenominator} outside of $\supp \den_{\tkernel,\pi_n}$ does not matter much from this point of view (the discontinuities will always be there). The same reasoning applies also for discontinuities presented in the next example.

The next example illustrates non-removable discontinuity of
\eUDRL{} generated policies (for $n \geq 2$)
at deterministic points. However,
while we will see the limits with respect to different rays
disagree, they are all optimal policies. This could motivate us
to prove relative continuity (some weaker notion of continuity)
for these policies, as was successfully achieved in theorem \ref{le:detcont}.
\begin{example}
\label{ex:detpoint}
(non-removable discontinuity of policy at deterministic point)
Let the MDP $\mathcal{M}$ be the same as in example \ref{ex:boundarypoint} except that we vary the parametric
transition kernels to illustrate the behavior around (and at) the specific deterministic transition kernel. Further, we change the distribution of the initial goals to the uniform distribution, i.e., $\prob(G_0 = 0) = \prob(G_0 = 1) = \prob(G_0 = 2) = \frac{1}{3}$.
Since the computation proceeds as in the example \ref{ex:boundarypoint} we just introduce the new kernels and summarize
the results.

\vspace{1em}
\noindent{}\textbf{A}
$$
\underset{a\in \mathcal{A}, g \in \mathcal{G}}{[\tkernel_{\alpha}(g|a)]} =
\left(
\begin{array}{ccc}
1-\alpha  & \alpha   & 0 \\
0         & 1-\alpha & \alpha  \\
\alpha    & 1-\alpha & 0 \\
\end{array}
\right)
\xrightarrow{\alpha\rightarrow 0+}
\left(
\begin{array}{ccc}
1  & 0  & 0 \\
0  & 1  & 0 \\
0  & 1  & 0 \\
\end{array}
\right)
=
\underset{a\in \mathcal{A}, g \in \mathcal{G}}{[\tkernel_{0+}(g|a)]}
(=
\underset{a\in \mathcal{A}, g \in \mathcal{G}}{[\tkernel_{0}(g|a)]}
)
.
$$
Notice continuity at $\alpha = 0$ ($\tkernel_{0+} = \tkernel_{0}$).
The results follow:
$$
\begin{gathered}
\underset{a\in \mathcal{A}, g \in \mathcal{G}}{[\pi_{1,0+}(a|g)]} =
\left(
\begin{array}{ccc}
1 & 0            & 0 \\
0           & \frac{1}{2}  & 1  \\
0 & \frac{1}{2}  & 0 \\
\end{array}
\right)
,\quad
\underset{a\in \mathcal{A}}{[\pi_{A,1,0+}(a)]} =
\left(
\begin{array}{c}
\frac{1}{3} \\
\frac{1}{2} \\
\frac{1}{6}  \\
\end{array}
\right)
,
\\
\underset{a\in \mathcal{A}}{[\pi_{2,0+}(a|g=0)]} =
\left(
\begin{array}{c}
1 \\
0 \\
0  \\
\end{array}
\right)
,\quad
\underset{a\in \mathcal{A}}{[\pi_{2,0+}(a|g=1)]} =
\left(
\begin{array}{c}
0 \\
\frac{3}{4} \\
\frac{1}{4}  \\
\end{array}
\right)
,\\
V_{0+}^{\pi_2}(g=0) = 1,\quad
V_{0+}^{\pi_2}(g=1) = 1,\quad
V_{0+}^{\pi_2}(g=2) = 0,\quad
J_{0+}^{\pi_2} = \frac{2}{3}.
\end{gathered}
$$

\vspace{1em}
\noindent{}\textbf{B}
$$
\begin{gathered}
\underset{a\in \mathcal{A}, g \in \mathcal{G}}{[\pi_{1,0}(a|g)]} =
\left(
\begin{array}{ccc}
1 & 0            & \frac{1}{3} \\
0 & \frac{1}{2}  & \frac{1}{3} \\
0 & \frac{1}{2}  & \frac{1}{3} \\
\end{array}
\right)
,\quad
\underset{a\in \mathcal{A}}{[\pi_{A,1,0}(a)]} =
\left(
\begin{array}{c}
\frac{4}{9} \\
\frac{5}{18} \\
\frac{5}{18}  \\
\end{array}
\right)
,\quad
\underset{a\in \mathcal{A},g\in \mathcal{G}}{[\pi_{2,0}(a|g)]} =
\left(
\begin{array}{ccc}
1 & 0 & \frac{1}{3} \\
0 & \frac{1}{2} & \frac{1}{3} \\
0 & \frac{1}{2} & \frac{1}{3} \\
\end{array}
\right)
,\\
V_{0}^{\pi_2}(g=0) = 1,\quad
V_{0}^{\pi_2}(g=1) = 1,\quad
V_{0}^{\pi_2}(g=2) = 0,\quad
J_{0}^{\pi_2} = \frac{2}{3}.
\end{gathered}
$$

\vspace{1em}
\noindent{}\textbf{C}
$$
\underset{a\in \mathcal{A}, g \in \mathcal{G}}{[\tkernel_{1,\alpha}'(a|g)]} :=
\left(
\begin{array}{ccc}
1-\alpha & 0            & \alpha \\
\alpha   & 1-\alpha  & 0  \\
0        & 1-\alpha  & \alpha \\
\end{array}
\right).
$$
Notice that $\tkernel_{1,0+}' = \tkernel_{1,0}' = \tkernel_{1,0} = \tkernel_{1,0+}$.
$$
\begin{gathered}
\underset{a\in \mathcal{A}, g \in \mathcal{G}}{[\pi_{1,0+}'(a|g)]} =
\left(
\begin{array}{ccc}
1 & 0            & \frac{1}{2} \\
0 & \frac{1}{2}  & 0  \\
0 & \frac{1}{2}  & \frac{1}{2} \\
\end{array}
\right)
,\quad
\underset{a\in \mathcal{A}}{[\pi_{A,1,0+}'(a)]} =
\left(
\begin{array}{c}
\frac{1}{2} \\
\frac{1}{6} \\
\frac{1}{3}  \\
\end{array}
\right)
,\quad
\underset{a\in \mathcal{A},g\in \mathcal{G}}{[\pi_{2,0+}'(a|g)]} =
\left(
\begin{array}{ccc}
1 & 0           & \frac{3}{5} \\
0 & \frac{1}{3} & 0 \\
0 & \frac{2}{3} & \frac{2}{5}  \\
\end{array}
\right)
,\\
V_{0+}^{'\pi_2}(g=0) = 1,\quad
V_{0+}^{'\pi_2}(g=1) = 1,\quad
V_{0+}^{'\pi_2}(g=2) = 0,\quad
J_{0+}^{'\pi_2} = \frac{2}{3}.
\end{gathered}
$$
\end{example}
The dependence of the goal reaching objective and the specific 
$\pi_2$ component on $\alpha$ is depicted in figure \ref{fig:discont}
(for both examples \ref{ex:boundarypoint} and \ref{ex:detpoint}).
While policies are clearly discontinuous at $\alpha = 0$ in both
examples, we see continuity of a goal-reaching objective (on A and C rays) at the deterministic kernel
(example \ref{ex:detpoint}). This hints at a possible 
continuity of goal reaching objectives and relative continuity
of policies at deterministic points (kernels), as was later 
proved in this paper. Unfortunately, we cannot hope for
the similar (general) result in a case of non-deterministic boundary points (kernels) which was clearly demonstrated in example \ref{ex:boundarypoint} and figures \ref{fig:discont:c} and \ref{fig:discont:d}. In addition to the examples above, we include a discussion describing the causes that lead to discontinuities of \eUDRL{}-generated quantities at the end of section~\ref{ap:interiorcont}. 

In the following example, we want to illustrate (relative) continuity of \eUDRL{} generated policies and goal reaching objectives for a finite number of iterations, which was proved in section \ref{se:contfinite}. 
\begin{example}
\label{ex:oscillations} (example featuring random walk on $\mathbb{Z}_3$)
In this example, we consider an MDP $\mathcal{M} = (\mathcal{S},\mathcal{A},\tkernel,\mu,r)$ with state space $\mathcal{S} = \mathbb{Z}_3$ and action space $\mathcal{A}=\{0,1\}$.
The transition kernel depends on the parameter $\alpha \in [0,1]$ that regulates the
stochasticity.
When action $0$ is selected, the MDP stays in the current state $s\in\mathcal{S}$ with probability $1-\alpha$ and transits to the state $s+1$ with probability $\alpha$.
When action $1$ is selected, the MDP stays in the current state $s \in \mathcal{S}$ with probability $\alpha$ and transits to the state $s+1$ with probability $1-\alpha$, i.e., the states we are transiting stay the same; just the probabilities are swapped. For both actions, the transition kernels are essentially
different random walks on $\mathcal{S} = \mathbb{Z}_3$.
We consider $\mathcal{M}$'s CE $\bar{\mathcal{M}}$ with remaining horizon $N=8$,
$\mathcal{G}:=\mathcal{S}$, $\rho :=\mathrm{id}_{\mathcal{S}}$.
The CE's initial distribution is a uniform distribution over all extended states.
\end{example}
The figures \ref{fig:oscillationsPI} and \ref{fig:oscillations} show the dependence
of $\min_{\bar{s}\in\theinterestingstates} \pi_{n,\tkernel}(\oactions(\bar{s})|\bar{s})$ and the goal reaching objective respectively on iteration $n$.
Continuity in the transition kernel is hinted
at by these figures as the plots for different $\delta := \|\tkernel-\tkernel_0\|_1$
approach the plot for $\delta = 0 $ (i.e., $\tkernel_0$ plot) smoothly as
$\delta \rightarrow 0$. Plots in this particular example are interesting
as they exhibit significant oscillations in the first few iterations.
A more detailed plot for the oscillating goal reaching objective (for one specific $\tkernel$ and one specific initial condition) is shown in \ref{fig:oscillations:b}. Here we see that the best results are obtained
for either iteration 2 or 3. Furthermore, the performance of the goal reaching objective deteriorates
when \eUDRL{} is continued.
The general non-monotonicity of a goal reaching objective could raise the question of
whether or not it makes sense to continue \eUDRL{} iterations in order to reach some
\lq\lq{}optimal\rq\rq{} behavior. These questions motivate our theoretical investigation
of the asymptotic properties of \eUDRL{} iteration in section \ref{se:continfty}.

In order to compare bounds derived for all special cases from section
\ref{se:continfty}, we include a simple bandit example below (which
complies with both of the conditions $(\forall \bar{s}\in \theinterestingstates): |\oactions(\bar{s})|=1$ and $\theinterestingstates \subset \supp \bar{\mu}$).
\begin{example}\label{ex:bandit} (a simple bandit)
We consider an MDP $\mathcal{M} = (\mathcal{S},\mathcal{A},\tkernel,\mu,r)$ with a state space $\mathcal{S} = \{0,1\}$ and an action space $\mathcal{A}=\{0,1\}$.
The transition kernel depends on parameter $\alpha \in [0,1]$ which regulates
stochasticity.
When action $0$ is selected, the MDP stays in the current state $s\in\mathcal{S}$ with probability $1-\alpha$ and transitions to the other state with probability $\alpha$.
When action $1$ is selected, the MDP stays in the current state $s \in \mathcal{S}$ with probability $\alpha$ and transitions to the other state with probability $1-\alpha$.
We consider $\mathcal{M}$'s CE $\bar{\mathcal{M}}$ with remaining horizon $N=1$,
$\mathcal{G}:=\mathcal{S}$, $\rho :=\mathrm{id}_{\mathcal{S}}$.
The CE's initial distribution is a uniform distribution over states with the original MDP state component $s=0$. Since the initial horizon can be just 1, there are just two initial states in the CE which differ in the goal component: $(0,1,0)$ and $(0,1,1)$.
\end{example}
For detailed plots concerning the bandit example, see figures \ref{fig:bandit}, \ref{fig:banditxu}, \ref{fig:ebandit} and \ref{fig:banditap}.
\begin{figure}
    \begin{subfigure}{0.48\linewidth}
		\includegraphics[width=\linewidth]{exportedsvg/banditpitraj_svg-tex.pdf}
		\caption{Dependency on iteration.}
        \label{fig:banditap:a}
	\end{subfigure}
     \begin{subfigure}{0.48\linewidth}
		\includegraphics[width=\linewidth]{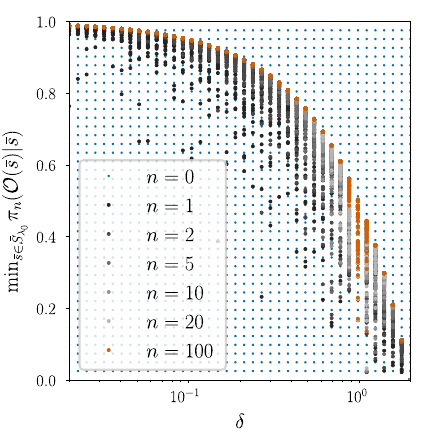}
		\caption{Dependency on distance to determin.~kernel.}
        \label{fig:banditap:b}
	\end{subfigure}\\    
    \caption{Illustration of behavior of $\min_{\bar{s}\in\theinterestingstates}\pi_n(\oactions(\bar{s})|s)$ for \eUDRL{} recursion in a example~\ref{ex:bandit}. 
    Plot (a) shows the dependency on the iteration $n$ for varying distances to a deterministic kernel highlighted by different colors and varying initial policy. Plot (b) contains the same information as (a) but it is organized to reveal the dependency on $\delta$, where varying numbers of iteration are highlighted by different colors.}
    \label{fig:banditap}
\end{figure}

Next, we include an example featuring a tiny grid world domain.
In this example, we wanted to demonstrate $x^*$ bound on \eUDRL{} policy
accumulation points introduced in section \ref{se:continfty}.
In order to satisfy the assumption $\theinterestingstates \subset \supp \bar{\mu}$ required by this bound, we chosen the uniform initial distribution for the associated CE.
\begin{example}\label{ex:gridworld} (a tiny grid world domain)
The domain is described by a 3 by 3 map (see figure \ref{fig:gridworld:d}).
The gray color depicts a wall. Other positions correspond to the MDP
states. Thus, we have a domain 
$\mathcal{M} = (\mathcal{S},\mathcal{A},\tkernel,\mu,r)$
with 8 states ($|\mathcal{S}| = 8$) and 4 actions 
($\mathcal{A}=\{\text{``right"},\text{\lq\lq{}left\rq\rq{}},\text{\lq\lq{}up\rq\rq{}},\text{\lq\lq{}down\rq\rq{}}\}$).
The transition kernel is parameterized by the parameter $\alpha \in [0,1]$,
regulating the distance from a deterministic kernel (for $\alpha =0$ the kernel
is deterministic and we denote it $\tkernel_0$).
In case of $\alpha = 0$, the kernel transits as indicated by the name of an action if possible. If it is not possible (one would move out of the grid or into a wall), the MDP stays in its current state.
In case of non-deterministic kernels ($\alpha >0$), the exact kernel
depends on \lq\lq{}available states\rq\rq{} which is union (over all actions) of all states the deterministic kernel can move in one step. If given an action $a$
the deterministic kernel would transit to $s'$ and \lq\lq{}available states\rq\rq{}$\setminus \{s'\} \neq \emptyset$, the kernel transits to 
$s'$ with probability $1-\alpha$ and remaining probability $\alpha$
is uniformly distributed between all states in \lq\lq{}available states\rq\rq{}$\setminus \{s'\}$. If \lq\lq{}available states\rq\rq{}$\setminus \{s'\} = \emptyset$, then 
the kernel transits to $s'$ with probability 1.
We consider $\mathcal{M}$'s CE $\bar{\mathcal{M}}$ with the remaining horizon $N=4$,
$\mathcal{G}:=\mathcal{S}$, $\rho :=\mathrm{id}_{\mathcal{S}}$.
The CE's initial distribution is the uniform distribution over all CE states.
\end{example}
The bunch of example accumulation points, together with the lower bound $x^*(\gamma_N)$ on the kernel distance $\delta:=\|\tkernel-\tkernel_0\|_1$,
is plotted in the figure \ref{fig:gridworld:a}. The map
of the grid world domain is shown in figure \ref{fig:gridworld:d}.

Further, we give an example which features the $x_u$ bound on the \eUDRL{}
policy accumulation points introduced in section \ref{se:continfty} together
with ODT recursion. In order to do this, we followed the approach
explained in the background section which assumes fixed horizon scenario with the reward at the end. This approach employs first a state space transformation (which consists of forming a $K$-tuples of states)
and then builds CE on top of the new state space.
The fixed horizon and the optimal policy uniqueness (just on $\theinterestingstates$) requirements are handled by a convenient choice
of the CE initial distribution and the goal map. Finally, we utilized the
fact that ODT recursion becomes just \eUDRL{} recursion (on the
described CE with transformed states) on the restricted segment space $\Seg^{\trail}$.
\begin{example}\label{ex:ODTgridworld} (a tiny grid world domain with ODT recursion)
The underlying MDP is the same as in the previous example \ref{ex:gridworld}
with just the initial distribution having restricted support (we will comment on it more precisely when detailing the associated CE).
Then we form a new MDP by considering $K$-tuples of the underlying 
MDP states with $K=3$ and extend the kernel accordingly.
Over this new MDP, we form the CE. The CE's maximal horizon was set to $N=4$
and the goal space was defined $\mathcal{G} :=\{0,1\}$ with the goal map
mapping $K$-tuples finishing with position $(0,2)$ to $1$ and others
to $0$. We pick just one initial state for this CE corresponding
to position $(2,2)$ with a remaining horizon $4$ and goal $1$ (see the grid world map on \ref{fig:ODTgridworld:d}).
This setting ensured the fixed horizon requirement (horizon was fixed to 4 by the choice of initial state) and the requirement for the uniqueness of the optimal policy on $\theinterestingstates$.
Now, note that the positions in the right column are not included in $\theinterestingstates$. This is because they are not in
$\supp \den_{\tkernel_0,\pi_0}$ ($\pi_0 >0$), i.e., there is no trajectory
which would connect the initial state in position $(2,2)$ and the goal 1
(the position $(0,2)$) of length $N=4$ going through these states.
Therefore, $\theinterestingstates$ spans over $K$-tuples finishing with
the remaining positions. Moreover, since there is only one trajectory
connecting $(2,2)$ with $(0,2)$ in 4 steps, there is only one remaining
horizon available for each position in $\theinterestingstates$.
This leads to the uniqueness of the optimal policy on $\theinterestingstates$.
So we can employ the $x_u$ bound for this domain.
\end{example}
The example's accumulation points of $\min_{\bar{s}\in \theinterestingstates} \pi_{n,\tkernel} (\oactions(\bar{s})|\bar{s})$ of ODT recursion together with the lower bound $x_u(\delta)$ are shown in figure \ref{fig:ODTgridworld:a}.

Finally, we give an example of a grid word which does not fall in the special cases investigated in section~\ref{se:continfty}. On this example we will showcase ODT recursion with regularization (using uniform distribution). This recursion (in the specific setting of the example) can be understood as a special case of $\epsilon$-\eUDRL{} which was investigated in section~\ref{se:regrec}.
\begin{example}\label{ex:eODTgridworld} (a tiny grid world domain with ODT recursion allowing for a non-deterministic optimal policy on $\theinterestingstates$)
This example is exactly the same as the previous example \ref{ex:ODTgridworld} except that we change the initial state
position to $(2,0)$. This change allows for non-determinism of
an optimal policy on $\theinterestingstates$ (see the map in the figure \ref{fig:eODTgridworld:b}). Thus, both
special conditions we have introduced for bounding the accumulation points
of \eUDRL{} (ODT without regularization) recursion fail.
However, we can still bound regularized \eUDRL{} (ODT) recursions.
\end{example}
The example accumulation points of $\min_{\bar{s}\in \theinterestingstates} \pi_{n,\tkernel} (\oactions(\bar{s})|\bar{s})$ of the ODT recursion together with the $\epsilon$-\eUDRL{} based lower bound $\min_M x^*(\gamma_N,\epsilon,\delta)$ are plotted in figure \ref{fig:eODTgridworld:a}.

\section{Continuity at Interior Kernels for a Finite Number of Iterations}
\label{ap:interiorcont}

In this section, we aim to prove continuity of \eUDRL{}-generated policies and related quantities at the interior points of the set of all transition kernels $(\Delta\mathcal{S})^{\mathcal{S}\times\mathcal{A}}$.
Note that an interior point $\tkernel \in ((\Delta\mathcal{S})^{\mathcal{S}\times\mathcal{A}})^{\circ} = ((\Delta\mathcal{S}^{\circ})^{\mathcal{S}\times\mathcal{A}})$ can be equivalently characterized as a transition kernel satisfying $\tkernel > 0$.
The following lemma, which could be understood as the interior point counterpart of the lemma~\ref{le:suppstab},
discusses the stability properties of the set $\supp \den_{\tkernel,\pi}$ subject to changes in the transition kernel $\tkernel>0$ and the policy $\pi$.
\begin{lemma}(support properties)
\label{le:suppprop}
Let $\{\mathcal{M}_{\tkernel} : \tkernel\in (\Delta\mathcal{S})^{\mathcal{S}\times\mathcal{A}} \}$ and $\{\bar{\mathcal{M}}_{\tkernel} : \tkernel\in (\Delta\mathcal{S})^{\mathcal{S}\times\mathcal{A}}\}$ be compatible families of MDPs.
Let $\tkernel, \tkernel' \in (\Delta\mathcal{S})^{\mathcal{S}\times\mathcal{A}}$ be transition kernels with $\tkernel, \tkernel' > 0$ and $\pi,\pi' \in (\Delta \mathcal{A})^{\mathcal{S}}$ be policies. Then it holds that
$$
\supp \den_{\tkernel,\pi} = \supp \den_{\tkernel',\pi'}.
$$

\end{lemma}

\begin{proof}
Assume $\bar{s} = (s,h,g) \in \supp \den_{\tkernel,\pi}$. This means
that there exist a trajectory $\tau$ and $0 \leq t \leq N$ with $\prob_{\tkernel,\pi}(\mathcal{T}=\tau) >0$,
$l(\tau) \geq t+h$, $s_t^{\tau} = s$, $\rho(s_{t+h}^{\tau})=g$.
Let us construct a new trajectory $\tau'$ from $\tau$ by altering
actions so that $a_{t'}^{\tau'} \in \supp \pi'(\cdot|\bar{s}_{t'}^{\tau})$
for all $0\leq t'\leq N$. Since $\tkernel'>0$, we also get
$\prob_{\tkernel',\pi'}(\mathcal{T}=\tau') >0$ causing
$\bar{s}  \in \supp \den_{\tkernel',\pi'}$.

\end{proof}
Lemma~\ref{le:suppprop} asserts
that the set $\supp \den_{\tkernel,\pi}$ is constant over the interior of the set of all transition kernels and all policies.
This is in striking difference with the behavior of $\supp \den_{\tkernel,\pi}$ on the neighborhood of a deterministic kernel (or generally a kernel located at the boundary of the set of all transition kernels), as discussed in lemma~\ref{le:suppstab} and example~\ref{ex:detpoint} (or example~\ref{ex:boundarypoint} for boundary points), where on every neighborhood the $\supp \den_{\tkernel,\pi}$ could change abruptly.
In the following lemma we will exploit this property to deliver a much simpler continuity proof than in the case of a deterministic kernel (cf. theorem~\ref{le:detcont}).

\begin{lemma}(Continuity of \eUDRL{} policies and values at interior points)
\label{le:intcont}
Let $\{\mathcal{M}_{\tkernel} : \tkernel\in (\Delta\mathcal{S})^{\mathcal{S}\times\mathcal{A}} \}$ and $\{\bar{\mathcal{M}}_{\tkernel} : \tkernel\in (\Delta\mathcal{S})^{\mathcal{S}\times\mathcal{A}}\}$ be compatible families of MDPs.
Let $(\pi_{n,\tkernel})$, $\pi_0 \in (\Delta \mathcal{A})^{\mathcal{S}}$
be \eUDRL{} generated sequence of policies.
Let $\tkernel_0>0$ be a transition kernel, i.e., $\tkernel_0$ is  an interior point $\tkernel_0 \in ((\Delta \mathcal{S})^{\mathcal{S}\times\mathcal{A}})^{\circ}$.
Then for all $n\geq 0$ it holds that:
\begin{enumerate}
\item
For all $\bar{s} = (s,h,g) \in \bar{\mathcal{S}}_T$ the policy
$\pi_{n+1,\tkernel}(\cdot|\bar{s})$ is continuous in  $\tkernel$ at $\tkernel_0$.
\item
For all $\bar{s} = (s,h,g) \in \bar{\mathcal{S}}_T$
the values $V_{\lambda}^{\pi_n}(\bar{s})$ and $Q_{\lambda}^{\pi_n}(\bar{s},\cdot)$
are continuous in $\tkernel$ at the point $\tkernel_0$.
In addition, the goal reaching objective $J_{\tkernel}^{\pi_n}$ is continuous in $\tkernel$ at the point $\tkernel_0$. (Here we view $V_{\lambda}^{\pi_n}$, $Q_{\lambda}^{\pi_n}$, and $J_{\tkernel}^{\pi_n}$ as functions of a single parameter $\tkernel$ resulting from composition with $\pi_{n,\tkernel}$.)
\end{enumerate}
\end{lemma}
\begin{proof}
Since the interior $((\Delta \mathcal{S})^{\mathcal{S}\times\mathcal{A}})^{\circ}$ is an open set, we can fix $\delta >0$
so that $U_{\delta}(\tkernel_0)$ is contained in the interior.

1.
The proof proceeds by induction on $n$.

\emph{Base case $n=0$:}
$\pi_0$ is continuous in $\tkernel$ as it is constant in $\tkernel$.

\emph{Induction:} Assume the statement holds for $n\geq0$, we aim to prove it for $n+1$.
Let us fix $\bar{s} \in \bar{\mathcal{S}}_T$.
First, assume $\bar{s} \notin \supp \den_{\tkernel_0,\pi_n}$ then from lemma~\ref{le:suppprop} it holds that
$\bar{s} \notin \supp \den_{\tkernel,\pi_n}$ for all 
$\tkernel\in U_{\delta}(\tkernel_0)$.
This in turn means that, as defined by \eref{eq:recursionUsingNumeratorDenominator}, $\pi_{n+1,\tkernel}(\cdot|\bar{s})=\frac{1}{|\mathcal{A}|}$ for all $\tkernel\in U_{\delta}(\tkernel_0)$, and therefore
$\pi_{n+1,\tkernel}(\cdot|\bar{s})$ is continuous in $\tkernel$ at point $\tkernel_0$.
Finally, assume $\bar{s}\in\supp \den_{\tkernel_0,\pi_n}$.
Then it holds that $\bar{s}\in\supp \den_{\tkernel,\pi_n}$ for all $\tkernel\in U_{\delta}(\tkernel_0)$.
This means that $\pi_{n+1,\tkernel}(\cdot|\bar{s})$ is defined as
$$
\pi_{n+1,\tkernel}(a|\bar{s}) = \frac{\num_{\tkernel,\pi_n}(\bar{s},a)}
{\den_{\tkernel,\pi_n}(\bar{s})}
$$
for all $\tkernel \in U_{\delta}(\tkernel_0)$.
From point 4 of lemma~\ref{le:contoptbound}, both $\num_{\tkernel,\pi_n}(\bar{s},\cdot)$ and $\den_{\tkernel,\pi_n}(\bar{s})$ are continuous in $(\tkernel,\pi_n)$ as segment distribution marginals. Further, from the induction assumption, $\pi_{n,\tkernel}(\cdot|\bar{s}')$ is continuous in $\tkernel$ at $\tkernel_0$ for all $\bar{s}' \in \bar{\mathcal{S}}_T$. Therefore both $\num_{\tkernel,\pi_n}(\bar{s},\cdot)$ and $\den_{\tkernel,\pi_n}(\bar{s})$,
when compounded with $\pi_{n,\tkernel}$, are continuous
in $\tkernel$ at $\tkernel_0$.
This together with $\den_{\tkernel_0,\pi_n}(\bar{s}) >0$ implies that $\pi_{n+1,\tkernel}(a|\bar{s})$ is continuous in $\tkernel$ at $\tkernel_0$.

2.
Assume $n>0$. First, we prove the statement about values. Let us fix $\bar{s}\in \bar{\mathcal{S}}_T$.
The values $V_{\tkernel}^{\pi_n}(\bar{s})$ and $Q_{\tkernel}^{\pi_n}(\bar{s},a)$ are both continuous
in $(\tkernel,\pi_n)$ from point 1 of lemma~\ref{le:contoptbound}. Since
$\pi_{n,\tkernel}$ are continuous in $\tkernel$ at $\tkernel_0$ from point 1.\, both $V_{\tkernel}^{\pi_n}(\bar{s})$ and $Q_{\tkernel}^{\pi_n}(\bar{s},a)$, when compounded with $\pi_{n,\tkernel}$,
are continuous in $\tkernel$ at $\tkernel_0$.
Finally, the continuity of the goal-reaching objective
$$
J_{\tkernel}^{\pi_n}
=
\sum_{\bar{s}\in \bar{S}_T}
V_{\tkernel}^{\pi_n}(\bar{s})
\bar{\mu}(\bar{s})
$$
in $\tkernel$ at $\tkernel_0$
follows from the continuity of $V_{\tkernel}^{\pi_n}$.

\end{proof}

\paragraph{When and why discontinuities arise at boundary points:}
Now, when we introduced the proof of continuity of \eUDRL{}-generated
quantities at the interior points (of $\Delta\mathcal{S}^{\mathcal{S}\times\mathcal{A}}$---the set of all transition kernels), we are at the right position
to investigate when this type of proof would break if we replace
an interior point with a boundary point.
In order to discuss general boundary points (not only deterministic ones) we will need the following two generalizations\footnote{Note that these can be omitted if one is interested solely in deterministic points.}. 
First, we need to generalize the notion of $\theinterestingstates$ for general transition kernels $\tkernel_0\in\Delta\mathcal{S}^{\mathcal{S}\times\mathcal{A}}$ (allowing also for non-deterministic $\tkernel_0$) which is trivial. Second, we will need to generalize  lemma~\ref{le:suppstab} to a general kernel $\tkernel_0$.

\begin{lemma}\label{le:suppstabgeneral}(Stability of $\theinterestingstates$) The conclusions of lemma~\ref{le:suppstab} remain valid under following changes: We replace ``a deterministic transition kernel $\tkernel_0$" with ``a transition kernel $\tkernel_0$". We replace quantification in points 1., 2.\ and 3.\ ``for all $n\geq0$ and all $\tkernel\in U_2(\tkernel_0)$" with ``There exists $\delta\in(0,2)$ such that for all $n\geq0$ and all $\tkernel\in U_{\delta}(\tkernel_0)$".     
\end{lemma}
\begin{proof}
The proof remains the same as in the original lemma~\ref{le:suppstab} except that we have to find a convenient $\delta>0$ for points 1., 2.\ and 3.,
i.e., we have to find $\delta>0$ so that the following variation of \eref{eq:kernelInclusion} holds
\begin{equation}
(\forall \tkernel \in U_{\delta}(\tkernel_0)) :
\supp \tkernel_0 \subset \supp \tkernel.
\label{eq:kernelinc}
\end{equation}
But this is just a consequence of the continuity of the identity map $\tkernel \mapsto \tkernel$.
\end{proof}

Now we will follow the proof of lemma~\ref{le:intcont} in the first two iterations of \eUDRL{} while assuming $\tkernel_0$ to be a boundary point and see when it fails.
First, we assume $\pi_0>0$ and we fix $\delta>0$ so that \eref{eq:kernelinc} holds. 
 Thus the conclusions of points 1., 2.\ and 3.\ of lemma~\ref{le:suppstabgeneral} hold.

\emph{Iteration $n=0$:}
Here, $\pi_0$ is assumed constant in $\tkernel$ and therefore continuous in $\tkernel$ at $\tkernel_0$. Similarly, from point 1.\ of lemma~\ref{le:contoptbound}, the values $V_{\tkernel}^{\pi_0}(\bar{s})$ and $Q_{\tkernel}^{\pi_0}(\bar{s},a)$ (after composing with $\pi_0$) are continuous in $\tkernel$ at $\tkernel_0$ for all $\bar{s}\in\bar{S}_T$. Finally, the goal-reaching objective $J_{\tkernel}^{\pi_0}$ (cf. \eref{eq:goalreachingobj}) is continuous in  $\tkernel$ at $\tkernel_0$ due to the described continuity of values $V_{\tkernel}^{\pi_0}(\bar{s})$. 

\emph{Iteration $n=1$:}
Let us fix $\bar{s}\in\bar{S}_T$.
First assume $\bar{s}\notin \supp\den_{\tkernel_0,\pi_0}$.
For any neighborhood $U_{\delta'}(\tkernel_0)$, $0<\delta'<\delta$ we can find an interior point $\tkernel' \in U_{\delta'}(\tkernel_0)$, i.e., we can find a sequence of interior points $(\tkernel_{k}')_{k\geq 0}$ converging to $\tkernel_0$. From lemma~\ref{le:suppprop}, only one of the following cases can happen: either $\bar{s}\in \supp\den_{\tkernel_{k}',\pi_0}$ for all $k$ or $\bar{s}\notin \supp\den_{\tkernel_{k}',\pi_0}$ for all $k$.
If the first case is true, $\pi_{1,\tkernel}(\cdot|\bar{s})$ can be discontinuous at $\tkernel_0$ due to discontinuity of $\supp \den_{\tkernel,\pi_0}$ at $\tkernel_0$, as evidenced by the sequence $(\tkernel_{k}')_{k\geq 0}$ (the definition of $\pi_{1,\tkernel}(\cdot|\bar{s})$, cf. \ref{eq:recursionUsingNumeratorDenominator}, changes abruptly from $\num/\den$ ratio to $1/|\mathcal{A}|$).
This behavior is illustrated in example \ref{ex:boundarypoint} (of non-deterministic boundary point $\tkernel_0$) in the appendix (see $\pi_{1,0+}$, $\pi_{1,0}$, $\pi_{1,0+}'$ in the third column ($g=2$)),
and also in example \ref{ex:detpoint} (of deterministic $\tkernel_0$) in the appendix (see again the third column).

In case $\bar{s}\in\supp\den_{\tkernel_0,\pi_0}$,
it holds that $\bar{s}\in\supp\den_{\tkernel_0,\pi_0} \subset \supp\den_{\tkernel,\pi_0}$ for all $\tkernel\in U_{\delta}(\tkernel_0)$ by \eref{eq:kernelinc}.
This means that $\pi_{1,\tkernel}(\cdot|\bar{s})$ is defined on the whole $U_{\delta}(\tkernel_0)$, using the ratio $\num_{\tkernel,\pi_0}(\bar{s},a)/\den_{\tkernel,\pi_0}(\bar{s})$. Following the reasoning in point 1.\ of lemma~\ref{le:intcont}, utilizing the continuity of $\pi_0$, point 4.\ of the lemma~\ref{le:contoptbound}, and $\den_{\tkernel_0,\pi_0}(\bar{s})>0$, we conclude that $\pi_{1,\tkernel}(\cdot|\bar{s})$ is continuous in $\tkernel$ at $\tkernel_0$.

\emph{Iteration $n=2$:}
Apart from the discontinuities of $\pi_{2,\tkernel}(\cdot|\bar{s})$ (in $\tkernel$ at $\tkernel_0$) which can arise due to $\bar{s}\notin\supp\den_{\tkernel_0,\pi_1}$, in the same way as was already described for $n=1$, there can be, additionally, also discontinuities for states in the set $\theinterestingstates$. In order to discuss these, let us fix $\bar{s}\in\theinterestingstates$. Since $\theinterestingstates\subset \supp\den_{\tkernel,\pi_1}\cap\supp\nu_{\tkernel,\pi_1}$ (by point 2.\ of lemma~\ref{le:suppstabgeneral}) the policy $\pi_{2,\tkernel}(\cdot|\bar{s})$ is defined by the ratio
$$
\pi_{2,\tkernel}(\cdot|\bar{s})
=
\frac{\num_{\tkernel,\pi_1}(\bar{s},\cdot)}{\den_{\tkernel,\pi_1}(\bar{s})}
$$
for all $\tkernel\in U_{\delta}(\tkernel_0)$.
However, since $\pi_{1,\tkernel}$ can be already discontinuous for some $\bar{s}'\in\bar{S}_T\setminus \supp\den_{\tkernel_0,\pi_0}$ (as we have described in paragraph for $n=1$), this discontinuity can propagate to $\num_{\tkernel,\pi_1}(\bar{s},\cdot)$, $\den_{\tkernel,\pi_1}(\bar{s})$ 
causing a discontinuity of $\pi_{2,\tkernel}(\cdot|\bar{s})$
defined via the above ratio.

This behavior is illustrated in example \ref{ex:boundarypoint} in the appendix (see $\pi_{2,0+}$, $\pi_{2,0}$, $\pi_{2,0+}'$ first column ($g=0$)),
and also in example \ref{ex:detpoint} in the appendix (see the second column and figure \ref{fig:discont:a}).

\emph{Summary:}
To sum up, policies can become discontinuous already at
iteration $n=1$, but this occurs outside of $\theinterestingstates$. At iteration $n=2$, we start also observing discontinuities on $\theinterestingstates$,
which is more problematic.
While at deterministic kernels the policies
on $\theinterestingstates$ are relatively continuous,
causing the goal-reaching objective to be continuous
(cf. theorem~\ref{le:detcont} and corollary~\ref{le:detJcont}). In
non-deterministic boundary points this is not the case
and policy discontinuities (for $n\geq 2$) propagates through values
to goal-reaching objective (cf. figure \ref{fig:discont:d}).

\section{Regularized Recursion --- Lemmas and Proofs}
\label{ap:regrec}

Here we introduce and prove the various lemmas and theorems used in section \ref{se:regrec}.

\subsection{Preliminary Lemmata}
The following lemma is the $\epsilon$-\eUDRL{} version of lemma \ref{le:suppstab}.

\begin{lemma}\label{le:esuppstab} ($\epsilon$-\eUDRL{} version of support stability)
Let $\{\mathcal{M}_{\tkernel} : \tkernel \in (\Delta \mathcal{S})^{\mathcal{S}\times\mathcal{A}}\}$
and $\{\bar{\mathcal{M}}_{\tkernel} : \tkernel \in (\Delta \mathcal{S})^{\mathcal{S}\times\mathcal{A}}\}$ be compatible families.
Let $(\pi_{n,\tkernel,\epsilon})_{n\geq 0}$ be a sequence of 
policies generated by the $\epsilon$-\eUDRL{} iteration given a transition kernel $\tkernel\in (\Delta S)^{\mathcal{S}\times\mathcal{A}}$ and  an initial condition $\pi_0$ (that does not depend on $\tkernel$) and a regularization parameter $1> \epsilon >0$.
Fix a deterministic transition kernel $\tkernel_0 \in \Delta \mathcal{S}^{\mathcal{S}\times\mathcal{A}}$.
Then for all initial conditions $\pi_0 > 0$ it
holds:
\begin{enumerate}
    \item For all $n\geq 0$ and all $\tkernel \in U_{2}(\tkernel_0)$ we have that $\supp \num_{\tkernel_0,\pi_0} \cap ( \mathcal{A} \times \supp \nu_{\tkernel_0,\pi_0} ) \subset \supp \num_{\tkernel,\pi_{n,\epsilon}} \cap ( \mathcal{A} \times \supp \nu_{\tkernel,\pi_{n,\epsilon}} )$,
    where the inclusion becomes equality for $\tkernel = \tkernel_0$.
    \item For all $n\geq 0$ and all $\tkernel \in U_{2}(\tkernel_0)$ we have that $\theinterestingstates \subset \supp \den_{\tkernel,\pi_{n,\epsilon}} \cap \supp \nu_{\tkernel,\pi_{n,\epsilon}}$,
    where the inclusion becomes an equality for $\tkernel = \tkernel_0$.
    \item For all $n\geq 0$ and all $\tkernel \in U_{2}(\tkernel_0)$ we have that $\prob_\tkernel (\stag{S}_0=s, l(\Sigma)=h, \rho(\stag{S}_h)=g, \stag{H}_0=h, \stag{G}_0=g; \pi_{n,\epsilon} ) > 0$ for all $(s,h,g)\in \theinterestingstates$.   
\end{enumerate}
The points 4. and 5. of the original lemma \ref{le:suppstab} hold also for $\epsilon$-\eUDRL{}.
\end{lemma}

\begin{proof}
The equation \eref{eq:kernelInclusion} is proved exactly the same as in the proof of the original lemma \ref{le:suppstab}.

1.
There is an easier alternative proof
for this point which we present now.
Since $\pi_{n,\epsilon} > 0$ for $n\geq 0$
(because of the regularization and the assumption that $\pi_0 > 0$) we
trivially have $\supp \pi_{n,\epsilon} = \supp \pi_0$.
Now fix $\tkernel \in U_{2}(\tkernel_0)$.
From $\supp \tkernel_0 \subset \supp \tkernel$ (cf. equation \eref{eq:kernelInclusion}) and just
stated fact about policy supports it follows:
\begin{equation}
\supp \nu_{\tkernel_0,\pi_0} \subset
\supp \nu_{\tkernel,\pi_{n,\epsilon}},
\quad
\supp \num_{\tkernel_0,\pi_0} \subset
\supp \num_{\tkernel,\pi_{n,\epsilon}}.
\end{equation}
Now we take the first inclusion and perform the cartesian product
with $\mathcal{A}$ and intersection with $\supp \num_{\tkernel_0,\pi_0}$. Further we take the second inclusion
and perform intersection with $\supp \nu_{\tkernel,\pi_{n,\epsilon}} \times \mathcal{A}$.
We obtain the following chain of inclusions:
\begin{align*}
( \supp \nu_{\tkernel_0,\pi_0}  \times \mathcal{A} )
\cap \supp \num_{\tkernel_0,\pi_0}
&\subset
(\supp \nu_{\tkernel,\pi_{n,\epsilon}}  \times \mathcal{A} )
\cap \supp \num_{\tkernel_0,\pi_0}
\\
&\subset
(\supp \nu_{\tkernel,\pi_{n,\epsilon}}  \times \mathcal{A} )
\cap \supp \num_{\tkernel,\pi_{n,\epsilon}},
\end{align*}
which leaves the result. The equality is proved exactly the same
as in the original proof.

2. The proof is unchanged.

3. The proof can be simplified but otherwise repeated with minimal changes as the new version of \eref{eq:taupositive} trivially holds (since $\pi_{n,\epsilon}>0$ for all $n \geq 0$).

The remaining proof of 4. and 5. is unchanged.
\end{proof}

The following lemma is the $\epsilon$-\eUDRL{} version of the lemma \ref{le:detopt}.
\begin{lemma}\label{le:edetopt}(optimality of $\epsilon$-\eUDRL{} policies for deterministic transition kernel)
Let $\tkernel_0$ be
a deterministic transition kernel and let $\mathcal{M}=(\mathcal{S},\mathcal{A},\tkernel_0,\mu,r)$ be a respective MDP with CE $\bar{\mathcal{M}}=(\bar{\mathcal{S}},\mathcal{A},\bar{\tkernel}_0,\bar{\mu},\bar{r},\rho)$. 
Assume $\pi_0 > 0$ and let $(\pi_{n,\epsilon})_{n\geq 0}$ be the policy sequence generated by the $\epsilon$-\eUDRL{} iteration given $\pi_0$, $\tkernel_0$ and regularisation parameter $\epsilon$. Then it holds:
\begin{enumerate}
\item For all $n \geq 0$ the policy $\pi_{n+1,\epsilon}$ has the form (on $\theinterestingstates$) 
$$
\pi_{n+1,\epsilon} = (1-\epsilon) \pi_{n+1}^* + \frac{\epsilon}{|\mathcal{A}|},
$$
where 
$$
\pi_{n+1}^* = \frac{\num_{\tkernel_0,\pi_{n,\epsilon}}(a,s,h,g)}{\den_{\tkernel_0,\pi_{n,\epsilon}}(s,h,g)}
$$
is an optimal policy on $\theinterestingstates$.
\item For all $n\geq 1$ and all $\bar{s}=(s,h,g)\in \theinterestingstates$ it holds
$$
V^{\pi_{n,\epsilon}}(\bar{s}) \geq (1-\epsilon)^h \geq (1-\epsilon)^N.
$$
\item For all $n\geq 1$ and all $\bar{s}=(s,h,g)\in \theinterestingstates$ it holds
\begin{align*}
Q^{\pi_{n,\epsilon}}(\bar{s},a) &\geq (1-\epsilon)^{h-1} \geq (1-\epsilon)^N \quad \text{for} \:  a \in \oactions(\bar{s}),
\\
&= 0 \quad \text{otherwise.} 
\end{align*}
\end{enumerate}
\end{lemma}

\begin{proof}
1.
The equations is just a rewrite of
the $\epsilon$-\eUDRL{} recursion.
The easiest way to prove
optimality of $\pi_{n}^*$ is to realize that
$\pi_{n}^*$ is defined similarly as $\pi_n$ (in the original lemma)
except that it uses $\epsilon$-version of \lq\lq{}$\num/\den$\rq\rq{}-ratio.
Since the proof of point 1. (in the original lemma) is
essentially just a statement about the supports and these
do not change (due to points 1. and 2. of lemma \ref{le:suppstab} and points 1. and 2. of
lemma \ref{le:esuppstab}), i.e.,

\begin{align*}
( \supp \nu_{\tkernel_0,\pi_n}  \times \mathcal{A} )
\cap \supp \num_{\tkernel_0,\pi_n}
&=
( \supp \nu_{\tkernel_0,\pi_0}  \times \mathcal{A} )
\cap \supp \num_{\tkernel_0,\pi_0}
\\
&=
( \supp \nu_{\tkernel_0,\pi_{n,\epsilon}}  \times \mathcal{A} )
\cap \supp \num_{\tkernel_0,\pi_{n,\epsilon}},
\\
\supp \den_{\tkernel_0,\pi_n}
\cap \supp \nu_{\tkernel_0,\pi_n}
&=
\theinterestingstates
=
\supp \den_{\tkernel_0,\pi_{n,\epsilon}}
\cap \supp \nu_{\tkernel_0,\pi_{n,\epsilon}}
\end{align*}
the $\pi_n^*$ has to be optimal too.

2.\ \& 3.
The proof follows by induction on the remaining horizon $h$.
For $h=1$ the $Q$-values are independent of a policy
and therefore are optimal, i.e., for all 
$\bar{s} = (s,1,g) \in \theinterestingstates$, it holds that
$Q^{\pi_{n,\epsilon}}((s,1,g),a) = 1$ for $a\in \oactions(s,1,g)$
and $Q^{\pi_{n,\epsilon}}((s,1,g),a) = 0$ otherwise.
Now assume that 2. and 3. holds for a fixed horizon $h$.
We obtain (for all $\bar{s} = (s,h,g) \in \theinterestingstates$)
\begin{align*}
V^{\pi_{n,\epsilon}}(\bar{s})
&=
\sum_{a\in \mathcal{A}} \pi_{n,\epsilon}(a|\bar{s})
Q^{\pi_{n,\epsilon}}(\bar{s},a)
\\
&\geq
\sum_{a\in \oactions(\bar{s})} \pi_{n,\epsilon}(a|\bar{s})
(1-\epsilon)^{h-1}
=
(1-\epsilon)^{h-1} \pi_{n,\epsilon}(\oactions(\bar{s})|\bar{s})
\\
&=
(1-\epsilon)^{h-1} ((1-\epsilon) \pi_n^*(\oactions(\bar{s})|\bar{s})
+\frac{\epsilon |\oactions(\bar{s})|}{|\mathcal{A}|})
=
(1-\epsilon)^{h-1} (1-\epsilon(1-\frac{|\oactions(\bar{s})|}{|\mathcal{A}|}) )
\\
&\geq
(1-\epsilon)^{h},
\end{align*}
where we used the induction assumption and the point 1.
Further for all $\bar{s} = (s,h+1,g) \in \theinterestingstates$ it holds
\begin{align*}
&Q^{\pi_{n,\epsilon}}(\bar{s},a)
\\
&=\begin{cases}
\sum_{s'\in \mathcal{S}} \tkernel_0(s'|s,a)
V^{\pi_{n,\epsilon}}(s',h,g)
=
\tkernel_0(s''|s,a)
V^{\pi_{n,\epsilon}}(s'',h,g) \geq (1-\epsilon)^{h} \quad \text{for}\; a \in \oactions(\bar{s}),
\\
0 \quad \text{otherwise,}
\end{cases}
\end{align*}
where for $a\in \oactions(\bar{s})$ we used that the kernel
is deterministic and therefore there exists a $s'' \in \mathcal{S}$
such that $\tkernel_0(s''|s,a)=1$ and, further, that from point 4.\ of lemma
\ref{le:esuppstab}, $(s'',h,g) \in \theinterestingstates$.
The statement for $a\notin \oactions(\bar{s})$ is just point 5.\ of lemma \ref{le:esuppstab}.
This completes the induction.
\end{proof}

The class of policies $\Pi_{\tkernel_0,\epsilon}^* := \{(1-\epsilon) \pi_{\tkernel_0}^* + \frac{\epsilon}{|\mathcal{A}|} \mid \pi_{\tkernel_0}^*\; \text{is an optimal policy for}\; \tkernel_0 \;\text{on}\; \theinterestingstates \}$
which was discussed in the previous lemma 
will be of great importance in the
following text because (given $\bar{s}\in \theinterestingstates$) it constitutes
the limit relative to $\oactions(\bar{s})$ we would like to achieve when proving
relative continuity of $\epsilon$-\eUDRL{} policies in $\tkernel$ at $\tkernel_0$.
Unlike the optimal action-value function
$Q_{\tkernel_0}^*$, which is constant on 
$\oactions(\bar{s})$ (and therefore factors through a quotient map relative to $\oactions(\bar{s})$), there are many $Q_{\tkernel_0}^{\pi_{\epsilon}^*}$ depending on $\pi_{\tkernel_0,\epsilon}^*\in \Pi_{\tkernel_0,\epsilon}^*$ which are generally non-constant on $\oactions(\bar{s})$.
Note that during the proof of points 2.\ and 3.\ of the above lemma 
we exclusively used policy properties following from $\pi_{n,\epsilon} \in \Pi_{\tkernel_0,\epsilon}^*$ for $n>0$. Therefore,
the points 2.\ and 3.\ describe
lower bounds on the values $Q_{\tkernel_0}^{\pi_{\epsilon}^*}$ and $V_{\tkernel_0}^{\pi_{\epsilon}^*}$, respectively, where $\pi_{\tkernel_0,\epsilon}^* \in \Pi_{\tkernel_0,\epsilon}^*$.

The following lemma is the $\epsilon$-\eUDRL{} version of the lemma \ref{le:alpha}.
\begin{lemma}
\label{le:ealpha} ($\epsilon$-\eUDRL{} version of the lower bound on visitation probabilities)
Let 
$\{\mathcal{M}_{\tkernel} : \tkernel \in (\Delta \mathcal{S})^{\mathcal{S}\times\mathcal{A}}\}$
and $\{\bar{\mathcal{M}}_{\tkernel} : \tkernel \in (\Delta \mathcal{S})^{\mathcal{S}\times\mathcal{A}}\}$ be compatible families.
Let $\tkernel_0$ be a deterministic kernel. Let $(\pi_{n,\tkernel,\epsilon})_{n\geq 0}$ denotes $\epsilon$-\eUDRL{} generated sequence given
the initial condition $\pi_0>0$, the transition kernel $\tkernel$, and the regularization parameter $\epsilon$.
Assume $\delta\in(0,2)$. 
Then for all $n>0$ all $\bar{s} = (s,h,g)\in\theinterestingstates$ and all $\tkernel\in U_{\delta}(\tkernel_0)$ it holds
$$
\prob_{\tkernel}(\stag{H}_0=h, \stag{G}_0=g | \stag{S}_0=s, l(\Sigma)=h; \pi_{n,\epsilon}) \geq \alpha(\delta,\epsilon),
$$
where
$$
\alpha(\delta,\epsilon) = 
\frac{2}{N(N+1)}
(\min_{\bar{s}'\in \supp \bar{\mu}}\bar{\mu}(\bar{s}'))
(\frac{\epsilon}{|\mathcal{A}|})^N
(1-\frac{\delta}{2})^N
> 0.
$$
Notice that
$
\alpha(\delta,\epsilon) \rightarrow \alpha(0,\epsilon_0)$ as $(\delta,\epsilon) \rightarrow (0,\epsilon_0)$,
where we assume an $\epsilon_0 \geq 0$, and where
$$
\alpha(0,\epsilon_0) = 
\frac{2}{N(N+1)}
(\min_{\bar{s}'\in \supp \bar{\mu}}\bar{\mu}(\bar{s}'))
\frac{1}{|\mathcal{A}|^N}\epsilon_0^N > 0
$$
for $\epsilon_0 > 0$.
\end{lemma}

\begin{proof}
We can bound the $\epsilon$-\eUDRL{} policy
$$
(\forall n >0, \bar{s} \in \theinterestingstates):
\pi_{n,\tkernel,\epsilon}(\cdot|\bar{s}) \geq \frac{\epsilon}{|\mathcal{A}|} =: \alpha' >0.
$$
Further, we continue similarly as in theorem \ref{le:limdetcontM1} \lq\lq{}Bounding the visitation term\rq\rq{} except we utilize the
above $\alpha'$ bound for lower bounding the policy terms.
In short, from $\bar{s}\in\theinterestingstates$ we have $\bar{s}\in \supp \nu_{\tkernel_0,\pi_0}$ meaning there exists a prefix
$\bar{s}_0,a_0,\ldots,\bar{s}_t = \bar{s}$ with positive probability
$\prob_{\tkernel_0}(\bar{S}_t=\bar{s},\ldots,A_0 = a_0,\bar{S}_0=\bar{s}_0;\pi_0)>0$.
Rewriting this probability as a product of $\tkernel_0$ terms and $\pi_0$ terms
then replacing the $\tkernel_0$ with $\tkernel$ and $\pi_0$
with $\pi_{n,\tkernel,\epsilon}$ doing exactly the same bounding of $\tkernel$
terms as in theorem \ref{le:limdetcontM1} and bounding $\pi_{n,\tkernel,\epsilon}$ using the $\alpha'$ bound above we obtain
$$
\prob_{\tkernel}(\bar{S}_t=\bar{s};\pi_{n,\tkernel,\epsilon})
\geq
\left(\min_{\bar{s}'\in \supp\bar{\mu}} \bar{\mu}(\bar{s}')\right)
(\alpha')^N
(1-\frac{\delta}{2})^N
$$
the rest of the proof is the same as in theorem \ref{le:limdetcontM1}
leading to the $\alpha$ bound above.
\end{proof}
Note that the bound could be easily extended to $n\geq 0$ by putting
$$
\alpha':= \min\{ \min_{\bar{s}'\in \theinterestingstates, a \in \mathcal{A}} \pi_0(a|\bar{s}'), \frac{\epsilon}{|\mathcal{A}|} \}.
$$
Although we will suffice with the above simpler version.

We will call a policy $\pi$ \emph{$\epsilon$-regular}
if and only if, for all $\bar{s} \in \theinterestingstates$ and $a \in \mathcal{A}$, it holds that $\pi(a|\bar{s}) > \frac{\epsilon}{|\mathcal{A}|}$.
The following remark describes equivalent statements about the distance of an $\epsilon$-regular policy to $\Pi_{\tkernel_0,\epsilon}^*$.

\begin{remark}
\label{re:epsreg}
(distance of $\epsilon$-regular policy to $\Pi_{\tkernel_0,\epsilon}^*$)
Assume $\delta > 0$ and $\pi_{\epsilon}$ is an
$\epsilon$-regular policy, then the following statements are equivalent:
\begin{description}
    \item[(a)] 
    $2(1-\epsilon(1-\frac{|\oactions(\bar{s})|}{|\mathcal{A}|})-\pi_{\epsilon}(\oactions(\bar{s}')|\bar{s}')) < \delta$
    \item[(b)] 
$
(\exists \pi_{\epsilon}^{*} \in \Pi_{\tkernel_0,\epsilon}^{*} ):
\quad
\|\pi_{\epsilon}^* - \pi_{\epsilon}\|_1 < \delta
\quad
\wedge
\quad
(
(\forall \bar{s}\in \theinterestingstates, \forall a \in \oactions(\bar{s})):
\pi_{\epsilon}^{*}(a|\bar{s}) \geq \pi_{\epsilon}(a|\bar{s})
)
$
    \item[(c)] 
$
(\exists \tilde{\pi}_{\epsilon}^{*} \in \Pi_{\tkernel_0,\epsilon}^{*} ):
\quad
\|\tilde{\pi}_{\epsilon}^* - \pi_{\epsilon}\|_1 < \delta
$
\end{description}
\end{remark}
Note that $\epsilon$-\eUDRL{} generated policies are $\epsilon$-regular, as well as policies from $\Pi_{\tkernel_0,\epsilon}^*$.
A policy from $\Pi_{\tkernel_0,\epsilon}^*$ can be characterized by
being $\epsilon$-regular and putting maximum mass on $\oactions(\bar{s})$
for all $\bar{s}\in \theinterestingstates$, where this 
maximum is $1-\epsilon(1-\frac{|\oactions(\bar{s})|}{|\mathcal{A}|})$.
\begin{proof} (of the remark \ref{re:epsreg})
We begin by the proof of 
(c) $\implies$ (b).
Assume (c).
Since $\pi_{\epsilon}$ might not be in $\Pi_{\tkernel_0,\epsilon}^*$,
we have $\pi_{\epsilon}(\oactions(\bar{s})|\bar{s}) \leq 1-\epsilon(1-\frac{|\oactions(\bar{s})|}{|\mathcal{A}|}) \;( =  \tilde{\pi}_{\epsilon}^*(\oactions(\bar{s})|\bar{s}))$.
We can easily construct $\pi_{\epsilon}^* \in \Pi_{\tkernel_0,\epsilon}^*$ satisfying the second condition in (b), e.g., by
$$
\pi_{\epsilon}^*(a|\bar{s}) =
\begin{cases}
\pi_{\epsilon}(a|\bar{s}) 
+
(
1-\epsilon(1-\frac{|\oactions(\bar{s})|}{|\mathcal{A}|})
-
\pi_{\epsilon}(\oactions(\bar{s})|\bar{s}) 
)/|\oactions(\bar{s})|
\quad \text{for}\; a \in \oactions(\bar{s})
\\
\frac{\epsilon}{|\mathcal{A}|}
\quad \text{otherwise}.
\end{cases}
$$
Since $\tilde{\pi}_{\epsilon}^*$ might not satisfy the
second condition in (b), we have $\delta > \|\tilde{\pi}_{\epsilon}^*-\pi_{\epsilon}\|_1 \geq \|\pi_{\epsilon}^*-\pi_{\epsilon}\|_1$. This concludes the proof of (c) $\implies$ (b).

The implication (b) $\implies$ (a) is trivial:
\begin{align*}
\delta > \|\pi_{\epsilon}^*-\pi_{\epsilon}\|_1
&= 
\pi_{\epsilon}^*(\oactions(\bar{s})|\bar{s})
- 
\pi_{\epsilon}(\oactions(\bar{s})|\bar{s})
+
\pi_{\epsilon}(\mathcal{A}\setminus \oactions(\bar{s})|\bar{s})
-
\pi_{\epsilon}^*(\mathcal{A}\setminus \oactions(\bar{s})|\bar{s}) 
\\
&=
2(\pi_{\epsilon}^*(\oactions(\bar{s})|\bar{s})
- 
\pi_{\epsilon}(\oactions(\bar{s})|\bar{s}))
\\
&=
2(
1-\epsilon(1-\frac{|\oactions(\bar{s})|}{|\mathcal{A}|})
- 
\pi_{\epsilon}(\oactions(\bar{s})|\bar{s}))
),
\end{align*}
where we used $\pi_{\epsilon}^*(\mathcal{A}|\bar{s}) = \pi_{\epsilon}(\mathcal{A}|\bar{s}) = 1$.

The implication (a) $\implies$ (c) is proved as follows.
Assuming (a) we can construct $\tilde{\pi}_{\epsilon}^* \in \Pi_{\tkernel_0,\epsilon}^*$
in the same way as we were constructing $\pi_{\epsilon}^*$
in the proof of (c) $\implies$ (b).
\end{proof}

In the following lemma, we will aim to show that when $\pi_{\epsilon}$ (we will consider just $\epsilon$-regularized policies) is close to $\Pi_{\tkernel_0,\epsilon}^*$ and 
$\tkernel$ is close to $\tkernel_0$, then $Q_{\tkernel}^{\pi_{\epsilon}}$ is close to the set $\{Q_{\tkernel_0}^{\pi_{\epsilon}^*} \mid \pi_{\tkernel_0,\epsilon}^* \in  \Pi_{\tkernel_0,\epsilon}^* \}$.
The following lemma is the $\epsilon$-\eUDRL{} version of the lemma \ref{le:contQfac}.
\begin{lemma}
\label{le:econtQfac}
($\epsilon$-\eUDRL{} version of the continuity of action-values in quotient topology)
Let $\{\mathcal{M}_{\tkernel} : \tkernel \in (\Delta \mathcal{S})^{\mathcal{S}\times\mathcal{A}}\}$
and $\{\bar{\mathcal{M}}_{\tkernel} : \tkernel \in (\Delta \mathcal{S})^{\mathcal{S}\times\mathcal{A}}\}$ be compatible families.
Let $\tkernel_0$ be a deterministic kernel.
For all $\epsilon' >0$ there exists $\delta >0$ such that if $\tkernel\in U_{\delta}(\tkernel_0)$ and the $\epsilon$-regular policy $\pi_\epsilon$ satisfies,
for 
all $\bar{s}=(s,h,g)\in \theinterestingstates$,
$$
2(1-\epsilon(1-\frac{|\oactions(\bar{s})|}{|\mathcal{A}|})-\pi_{\epsilon}(\oactions(\bar{s})|\bar{s})) < \delta,
$$
then
$|Q^{\pi_\epsilon}_{\tkernel}(\bar{s},\cdot) - Q^{\pi_{\epsilon}^*}_{\tkernel_0}(\bar{s},\cdot)| < \epsilon'$ for all $\bar{s}=(s,h,g)\in \theinterestingstates$ for some $\pi_{\tkernel_0,\epsilon}^* \in \Pi_{\tkernel_0,\epsilon}^*$.
The statement can be made explicit by a recursive estimate:
for all $\bar{s} = (s,h,g) \in \theinterestingstates$, all $h\geq 2$, and all $a\in \oactions(\bar{s})$ it holds that
\begin{align*}
|Q_{\tkernel}^{\pi_{\epsilon}}(\bar{s},a)-Q_{\tkernel_0}^{\pi_{\epsilon}^*}(\bar{s},a) |
&\leq
\|\tkernel(\cdot|s,a)-\tkernel_0(\cdot|s,a)\|_1
\\
&\quad+
\max_{\bar{s}'=(s',h-1,g) \in\theinterestingstates}
2(1-\epsilon(1-\frac{|\oactions(\bar{s}')|}{|\mathcal{A}|})-\pi_{\epsilon}(\oactions(\bar{s}')|\bar{s}'))
\\
&\quad+
\max_{\bar{s}'=(s',h-1,g) \in\theinterestingstates, a' \in \oactions(\bar{s}')}
|Q_{\tkernel}^{\pi_{\epsilon}}(\bar{s}',a')-Q_{\tkernel_0}^{\pi_{\epsilon}^*}(\bar{s}',a')|.
\end{align*}
\end{lemma}

\begin{proof}
The proof follows as in lemma \ref{le:contQfac}
by induction on the remaining horizon.
In the induction step, we simply show that $|Q_{\tkernel}^{\pi_{\epsilon}}(\bar{s},a)-Q_{\tkernel_0}^{\pi_{\epsilon}^*}(\bar{s},a) | \rightarrow 0$
for $\delta \rightarrow 0$ by bounding it by $\delta$
and using the induction assumption. The trick is to use
the implication (a) $\implies$ (b) of the remark \ref{re:epsreg} to select a
convenient $\pi_{\epsilon}^*$ for a given $\pi_{\epsilon}$.
\end{proof}

The following lemma is the $\epsilon$-\eUDRL{} version of lemma \ref{le:f}.
\begin{lemma}
\label{le:z}
Let $\gamma,\epsilon > 0$, $1 > \gamma +\epsilon$, $|\mathcal{A}| \geq M > 0$ and $z_{\gamma,\epsilon,M}:[0,1] \rightarrow [0,1]$ be defined as
$$
z_{\gamma,\epsilon,M}(x) =
(1-\epsilon)\frac{x}
{x + \gamma}
+ \epsilon \frac{M}{|\mathcal{A}|}.
$$
Then the following assertions hold
\begin{description}
\item{1.} $z_{\gamma,\epsilon,M}$ is increasing.
\item{2.} $z_{\gamma,\epsilon,M}$ has a unique fixed point $x^*(\gamma,\epsilon,M)$
$$
x^*(\gamma,\epsilon,M) = \frac{\hat{x}^* + \sqrt{(\hat{x}^*)^2+\frac{4\gamma\epsilon M}{|\mathcal{A}|}}}{2},
$$
where $\hat{x}^* = 1-\epsilon(1-\frac{M}{\mathcal{|A|}})-\gamma$, $\hat{x}^* \leq x^*(\gamma,\epsilon,M)$ and
$1> x^*(\gamma,\epsilon,M) >0$. Further,
$$
x^*(\gamma,\epsilon,M)
\rightarrow
x^*(0,\epsilon_0,M) =  1-\epsilon_0(1-\frac{M}{|\mathcal{A}|}) \quad\text{as}\quad
(\gamma,\epsilon)  \rightarrow (0,\epsilon_0),
$$
for any  $1>\epsilon_0 \geq 0$, and it holds that
$$
\begin{aligned}
z_{\gamma,\epsilon,M}(x) &> x \quad \text{for}\; x < x^*(\gamma,\epsilon,M),
\\
z_{\gamma,\epsilon,M}(x) &< x \quad \text{for}\; x > x^*(\gamma,\epsilon,M).
\end{aligned}
$$
\item{3.} For all $x\in [0,1]$ the iterated application of $z_{\gamma,\epsilon,M}$ converges pointwise, $z_{\gamma,\epsilon,M}^{\circ n} (x) \xrightarrow{n\rightarrow\infty} x^*(\gamma,\epsilon,M)$.
\item{4.} Assume a sequence $(y_n)$,$y_n \in [0,1]$ such that for all $n\geq 0$ we have 
$y_{n+1} \geq z_{\gamma,\epsilon,M}(y_n)$.
Then $y_n \geq z_{\gamma,\epsilon,M}^{\circ n}(y_0)$ and $\liminf_n y_n \geq x^*(\gamma,\epsilon,M)$.

\end{description}
\end{lemma}

\begin{proof}
1. The derivative of $z$ is positive.

2. We follow the same approach as in lemma \ref{le:f} except the resulting
equality/inequality is now quadratic (instead of linear). However there is just one
root in $[0,1]$ (the other one is always negative) which is also the
only fixed point $x^*(\gamma,\epsilon,M)$.
From $\frac{4\gamma\epsilon M}{|\mathcal{A}|} >0$ follows the inequality 
$\hat{x}^* \leq x^*$.
To prove the bounds $1>x^*(\gamma,\epsilon,M)>0$ is trivial. 

3.\ \&\ 4. The rest follows as in lemma \ref{le:f} except it is somewhat more simplified
as there are no complications with the point $x=0$.
\end{proof}

\begin{figure}
    \centering
    \includegraphics[width=0.75\textwidth]{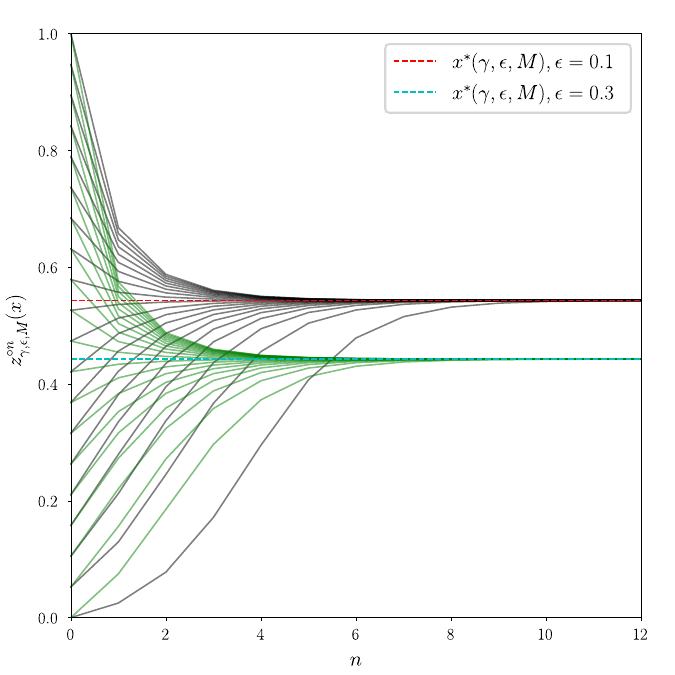}
    \caption{The map $z_{\gamma,\epsilon,M}$ as a dynamical system:
    the dependence of $z_{\gamma,\epsilon,M}^{\circ n}(x)$ (repeated application of $z_{\gamma,\epsilon,M}$ given initial condition $x$) on iteration $n$ for various
    initial conditions $x$. The parameters are set with 
    $|\mathcal{A}| = 4$, $M = 1$, and $\gamma = 0.4$. The regularization parameter is set to  $\epsilon = 0.1$ for black trajectories
    and to $\epsilon = 0.3$ for green trajectories, respectively.    
}
    \label{fig:z}
\end{figure}
The figure \ref{fig:z} illustrates the dynamical system induced by $z_{\gamma,\epsilon,M}$.
The lemma shows some of its properties (e.g., the fixed point). 

\subsection{The Main Theorem}
The following theorem is the $\epsilon$-\eUDRL{} version of theorem \ref{le:limdetcont}.
\elimdetcont*
At the core of the proof is the analysis of dynamical systems and convergence
induced by iterative application of the rational function $z_{\gamma,\epsilon,M}(x) = (1-\epsilon)\frac{x}{x+\gamma} +\epsilon\frac{M}{|\mathcal{A}|}$ (see lemma \ref{le:z}). Indeed, we will show that for all horizons $N \geq h \geq 1$ and all $\beta,\tilde{\beta},\epsilon \in (0,1)$ satisfying $\gamma+\epsilon<1$ (with $\gamma$ as in point 2. above), 
there exists $n_0$ and $\delta >0$ such that, for all initial conditions
$\pi_0>0$
all $\tkernel \in U_{\delta}(\tkernel_0)$, $n>n_1\geq n_0$,
and all $\bar{s}\in\theinterestingstates$ with remaining horizon $h$, it holds that
\begin{equation}
\pi_{n,\tkernel,\epsilon}(\oactions(\bar{s})|\bar{s}) \geq
z_{\gamma,\epsilon,|\oactions(\bar{s})|}^{\circ (n-n_1)} (\frac{\epsilon|\oactions(\bar{s})|}{|\mathcal{A}|})
.
\label{eq:eforMainTheorem}
\end{equation}
Then, by lemma \ref{le:z}, the right-hand side converges to $x^*(\gamma,\epsilon,|\oactions(\bar{s})|)$, which implies point~2.\ of the theorem (see the proof for details). Furthermore, point~1.\ is a consequence of point~2.\ The argument for this is the same as in section \ref{ssse:specmutheorem}.
\begin{proof}
The proof is analogous to the proof of theorem \ref{le:limdetcont} thus we will concentrate on
discussing the small differences (of course, with the lemmas introduced in this appendix instead of their original versions).

Since the lower bound $\pi_{n,\epsilon}(\oactions(\bar{s})) \geq \frac{\epsilon|\oactions(\bar{s})|}{|\mathcal{A}|}\geq \frac{\epsilon}{|\mathcal{A}|} > 0$ (for all $\bar{s}\in \theinterestingstates$) coming from $\epsilon$-regularity of $\epsilon$-\eUDRL{}-generated policies is used in \eref{eq:eforMainTheorem}, we do not need to prove any analogies to \eref{eq:pisep}.

In point 1., $(\beta,\tilde{\beta},\alpha(\delta,\epsilon),\epsilon) \rightarrow (0,0,\alpha(0,\epsilon_0),\epsilon_0)$
(where $\epsilon_0 >0$ and $\alpha(0,\epsilon_0) >0$ from lemma \ref{le:z})
implies $\gamma \rightarrow 0$ (continuity of $\gamma$ at point $(0,0,\alpha(0,\epsilon_0),\epsilon_0)$). Further,
$(\gamma,\epsilon) \rightarrow (0,\epsilon_0)$ causes
$x^* \rightarrow 1-\epsilon_0(1-\frac{|\oactions(\bar{s})|}{|\mathcal{A}|})$, according to lemma \ref{le:z}.

The proof proceeds by showing implications 2.$\implies$1. and \eref{eq:eforMainTheorem}$\implies$2. similarly as in the original proof.
Here we will focus on the proof of \eref{eq:eforMainTheorem}.
Fixing $2>\delta>0$ and $\epsilon >0$,
 lemma \ref{le:esuppstab} asserts that $\theinterestingstates \subset \supp \den_{\tkernel,\pi_{n,\epsilon}} \cap \supp \nu_{\tkernel,\pi_n}$
for all $n\geq 0$,$\pi_0 \geq 0$, $\tkernel \in U_{\delta}(\tkernel_0)$.
The policy $\pi_{n+1,\tkernel,\epsilon}$ is then well defined by
$$
\pi_{n+1,\tkernel,\epsilon}(\cdot|\bar{s}) =
\frac{\num_{\tkernel,\pi_{n,\epsilon}}(\cdot|\bar{s})}
{\den_{\tkernel,\pi_{n,\epsilon}}(\cdot|\bar{s})}
+ \frac{\epsilon}{|\mathcal{A}|}
\label{eq:epirec}
$$
for all $\bar{s} \in \theinterestingstates$.
Analogous to \eref{eq:piM}, we can rewrite the above equation similarly
as in theorem \ref{le:limdetcont} using $Q_{\tkernel}^{\pi_{n,\epsilon},g,h}$
and $v_{\tkernel,\pi_{n,\epsilon}}$.
The proof proceeds by induction on the remaining horizon $h$.

\emph{Base case $(h=1)$:}
Now we fix $\beta,\tilde{\beta} \in (0,1)$ arbitrary so that
$1> \gamma + \epsilon$, where $\gamma = \frac{\tilde{\beta}}{((1-\epsilon)^N-\beta)\alpha(\delta,\epsilon)}$.
Note that we will be restricting (decreasing) $\delta>0$
several times in the proof.
Since $\alpha(\delta,\epsilon)$ is deceasing in $\delta$, the decrease in $\delta$ will cause an increase in $\alpha$, which further causes a decrease in $\gamma$ leaving the constraint $1> \gamma + \epsilon$ in place.
According to 
point 1.\ of lemma \ref{le:contoptbound} (continuity of $Q^*_{\tkernel}$),
point 5.\ of lemma \ref{le:esuppstab},
 lemma \ref{le:econtQfac}, and lemma \ref{le:edetopt}, we can fix $2>\delta >0$
so that for all $\tkernel \in U_{\delta}(\tkernel_0)$,
all $\bar{s} = (s,1,g) \in \theinterestingstates$, and all $n\geq 1$ it holds that
$$
\begin{aligned}
\tilde{\beta}  &> Q_{\tkernel}^*(\bar{s},a)\quad \text{for}\; a\notin \oactions(\bar{s}),
\\
(1-\epsilon)^N - \beta &< Q_{\tkernel}^{\pi_{n,\epsilon}}(\bar{s},a) \quad \text{otherwise.}
\end{aligned}
$$
Further, we lower bound the recursion for $\pi_{n+1,\tkernel,\epsilon}(\oactions(\bar{s})|\bar{s})$ in exactly
the same way as in theorem \ref{le:limdetcont} using the above $\tilde{\beta}$ and $\beta$ bounds, and also the $\alpha$ bound on the 
state visitation terms leaving $(\forall n \geq 0, \forall \pi_0 >0, \forall \tkernel \in U_{\delta}(\tkernel_0),\forall \bar{s} = (s,1,g) \in \theinterestingstates )$
\begin{align}
\pi_{n+1,\tkernel,\epsilon}(\oactions(\bar{s})|\bar{s})
&\geq
\frac{((1-\epsilon)^N-\beta)\alpha \pi_{n,\tkernel,\epsilon}(\oactions(\bar{s})|\bar{s})}
{((1-\epsilon)^N-\beta)\alpha \pi_{n,\tkernel,\epsilon}(\oactions(\bar{s})|\bar{s}) + \tilde{\beta}}
+
\frac{\epsilon |\oactions(\bar{s})|}{|\mathcal{A}|}
\nonumber\\
&=
z_{\gamma,\epsilon,\oactions(\bar{s})} (\pi_{n,\tkernel,\epsilon}(\oactions(\bar{s})|\bar{s})).
\label{eq:epimf1}
\end{align}

\emph{Induction step:} Now we aim to prove that the statement holds for a fixed $h$ with $h>1$, while working under the assumption that it holds for a smaller $h$ (induction assumption).
Fix $\beta,\tilde{\beta} \in (0,1)$ arbitrary so that
$1 > \gamma + \epsilon$, where $\gamma = \frac{\tilde{\beta}}{((1-\epsilon)^N -\beta)\alpha(\delta,\epsilon)}$.
By lemma \ref{le:econtQfac} there exists $\delta'>0$ such that
if the following two conditions are met
$$
\begin{gathered}
\tkernel \in U_{\delta'}(\tkernel_0),
\\
(\forall \bar{s}' = (s',h',g') \in \theinterestingstates,h' < h ):
\quad
1 - \epsilon(1 - \frac{\oactions(\bar{s}')}{|\mathcal{A}|}) - \pi_{\epsilon}(\oactions(\bar{s}')|\bar{s}')
< \frac{\delta'}{2}
\end{gathered}
$$
then it holds that
$$
(\forall \bar{s} = (s,h,g) \in \theinterestingstates, \forall a \in \oactions(\bar{s}), \forall \tkernel\in U_{\delta'}(\tkernel_0) ):
\quad
Q_{\tkernel}^{\pi_{\epsilon}}(\bar{s},a) \geq
(1-\epsilon)^N - \beta.
$$
The first condition is met by the choice $0<\delta < \delta'<2$.
To meet the second condition, we will use the induction assumption
choosing $\beta',\tilde{\beta}' >0$ such that $\gamma'+\epsilon < 1$
(where $\gamma' = \frac{\tilde{\beta}'}{((1-\epsilon)^N-\beta')\alpha(\delta,\epsilon)}$) and $x^*(\gamma',\epsilon,|\oactions(\bar{s}')|) > 1-\epsilon(1-\frac{|\oactions(\bar{s}')|}{|\mathcal{A}|})-\delta'/2$
(this can be done since $x^*(\gamma',\epsilon,|\oactions(\bar{s}')|) \rightarrow 1 - \epsilon(1-\frac{|\oactions(\bar{s}')|}{|\mathcal{A}|})$ for $(\beta',\tilde{\beta}') \rightarrow 0$). 
Applying the induction assumption requires a restriction on $\delta$. From the induction assumption, it follows that there exists $n_0'$ such that,
for all $n>n_0', \tkernel \in U_{\delta}(\tkernel_0)$, it holds that
$\pi_{n,\tkernel,\epsilon}(\oactions(\bar{s}')| \bar{s}') > z_{\gamma',\epsilon,|\oactions(\bar{s}')|}^{n-n_0'}(\frac{\epsilon|\oactions(\bar{s}')|}{|\mathcal{A}|})$.
As $z_{\gamma',\epsilon,|\oactions(\bar{s}')|}^{n-n_0'}(\frac{\epsilon|\oactions(\bar{s}')|}{|\mathcal{A}|}) \rightarrow x^*(\gamma',\epsilon,|\oactions(\bar{s}')|)$ there exists $n_0$ such that
$\pi_{n,\tkernel,\epsilon}(\oactions(\bar{s}')|\bar{s}') > 1 -\epsilon(1-\frac{|\oactions(\bar{s}')|}{|\mathcal{A}|}) - \delta'/2$
for $n>n_0$, $\tkernel \in U_{\delta}(\tkernel_0)$.

In addition, by restricting $\delta > 0$ accordingly, we force the bound $Q_{\tkernel}^*(\bar{s},a) <\tilde{\beta}$ for all $\bar{s} \in \theinterestingstates$ and all $a \notin \oactions(\bar{s})$. Finally, we can apply
$\beta$, $\tilde{\beta}$ and state visitation bounds as we
did before (for $h=1$) to get the desired result.
\end{proof}

\subsection{Extending the Continuity Results to Other Segment Sub-Spaces}
As before we need to modify theorem \ref{le:elimdetcont} to cover algorithms like ODT with regularization restricting the recursion to $\Seg^{\diag}$ or $\Seg^{\trail}$. First, we will introduce $\Seg^{\diag/\trail}$ variants of lemmas \ref{le:esuppstab} and \ref{le:edetopt}.
Note that equations \eref{eq:isthesame}, \eref{eq:saisthesame} and \eref{eq:optathesame}
depends only on $\pi_0$ and are therefore applicable also in $\Seg^{\diag/\trail}$.
Following the same reasoning as in section \ref{sse:suppstabDiagTrail},
one can extend the results in lemma \ref{le:esuppstab} also to $\Seg^{\diag/\trail}$.
\begin{lemma}\label{le:esuppstabDiagTrail} ($\epsilon$-\eUDRL{} version of support stability for $\Seg^{\diag/\trail}$) The lemma \ref{le:esuppstab} remains valid under renaming $\pi_{n,\epsilon}\rightarrow \pi_{n,\epsilon}^{\diag/\trail}$.
\end{lemma}
Similarly we can derive the variant of lemma \ref{le:edetopt} for $\Seg^{\diag/\trail}$.
\begin{lemma}\label{le:edetoptDiagTrail} (Optimality of $\epsilon$-\eUDRL{} policies for deterministic kernels in $\Seg^{\diag/\trail}$) The lemma \ref{le:edetopt} remains valid under renaming $\pi_{n,\epsilon}\rightarrow \pi_{n,\epsilon}^{\diag/\trail}$.
\end{lemma}
The proof follows the same lines as in lemma \ref{le:edetopt} except in point 1.\ we use lemma \ref{le:esuppstabDiagTrail} instead of lemma \ref{le:esuppstab} and equations \eref{eq:isthesame} and \eref{eq:saisthesame}. Finally, we can state the $\Seg^{\diag/\trail}$ version of the theorem \ref{le:elimdetcont}.
\elimdetcontDiagTrail*
The proof follows similarly as the proof of theorem \ref{le:elimdetcont} except of small differences which resembles those described already in section \ref{ssse:extendingfinite}.
First, to assert that $\epsilon$-\eUDRL{}-generated policies are in $\Pi_{\tkernel_0,\epsilon}^*$ at deterministic points one applies lemma \ref{le:edetoptDiagTrail} instead of lemma \ref{le:edetopt}. Second, to prove that $\pi_{n+1,\epsilon}^{\diag/\trail}$ is
well defined through equation \eref{eq:epirec}, one uses equation \eref{eq:isthesame} and point 2.\ of lemma \ref{le:esuppstabDiagTrail} instead of lemma \ref{le:esuppstab}. Third, the application of point 1.\ of lemma \ref{le:recrewrites} has to be replaced with points 2.\ and 3.\ of the lemma respectively. To bound the $\epsilon$-\eUDRL{} recursion, one proceeds along the same lines yielding the exactly the same result as in \eref{eq:epimf1}. The same reasoning applies to the induction step. Finally, making use of lemma \ref{le:recrewrites}, which allows us to rewrite the $\epsilon$-\eUDRL{} recursion using only $\Seg$-space-related quantities, shows that there are no further differences even for $\pi_{n+1,\epsilon}^{\diag/\trail}$.
The same approach also applies to the forthcoming corollary of the theorem \ref{le:elimdetcont}.

\subsection{Estimating the Location of Accumulation Points}
The following corollary is the $\epsilon$-\eUDRL{} version of the corollary \ref{le:limitbounds}.
\elimitbounds*
\begin{proof}
The proof follows the same as the proof of corollary \ref{le:limitbounds} except for slight variations in the corollary statement which has to be addressed. However, the main arguments are the same and we will not repeat them here. Again we rather try to
concentrate on differences. Of course, one has to use lemmas introduced in this appendix instead of their original versions.
The definition of $\alpha$ comes from lemma \ref{le:ealpha}.
We decided to drop $\delta,\epsilon$ arguments for convenience because
in this corollary proof there are no any frequent readjustments of
$\delta$ in opposite to the previous theorem.

\emph{The claim} has now a sightly different form $(\forall  N \geq h \geq 1, \forall \pi_0 >0 , \forall \tkernel \in U_{\delta}(\tkernel_0)):$
\begin{align*}
&
(\exists (\pi_{n,\tkernel_0,\epsilon}^*), \pi_{n,\tkernel_0,\epsilon}^*\in \Pi_{\tkernel_0,\epsilon}^*):
\;
\limsup_n
\max_{\bar{s} = (s,h',g) \in\theinterestingstates,h'=h, a \in \mathcal{A}}
|Q_{\tkernel}^{\pi_{n,\epsilon}}(\bar{s},a) - Q_{\tkernel_0}^{\pi_{n,\epsilon}^*}(\bar{s},a)| \leq \beta_h,
\\
&\limsup_n
\max_{\bar{s} = (s,h',g) \in\theinterestingstates,h'=h}
2 (1-\epsilon(1-\frac{|\oactions(\bar{s})|}{|\mathcal{A}|})-\pi_{n,\epsilon}(\oactions(\bar{s})|\bar{s}))
\leq \nu_h
\end{align*}
Notice that even before we were using convenient $\pi_{\tkernel_0}^*$
to a policy $\pi$ in the original version of lemma \ref{le:econtQfac}. The difference is that now
we have to mention the chosen sequence $(\pi_{n,\tkernel_0,\epsilon}^*)$ explicitly in the claim
statement, because $Q_{\tkernel_0}^{\pi_{n,\epsilon}^*}$ depends on the particular choice of $\pi_{n,\epsilon}^* \in \Pi_{\tkernel_0,\epsilon}^*$ (unlike the optimal value $Q_{\tkernel_0}^*$).

Points 1.\ and 2.\ follows again directly from the claim.
points 3., 4., and 5.\ are proved in exactly the same way as in the original corollary. We comment only on point 5.\ here. 
We chose the sequence $(\pi_{n,\epsilon}^{*})_{n\geq 0}$ so that $\pi_{n,\epsilon}^*(a\mid\bar{s}) \geq \pi_{n,\epsilon}(a\mid\bar{s})$, for all $\bar{s}\in\theinterestingstates$ and all $a\in\oactions(\bar{s})$, which implies 
$\|
\pi_{n,\epsilon}(\cdot|\bar{s})
- \pi_{n,\epsilon}^{*} (\cdot|\bar{s}) 
\|_1
= 2(1-\epsilon(1-\frac{|\oactions(\bar{s})|}{|\mathcal{A}|})-\pi_{n,\epsilon}(\oactions(\bar{s})|\bar{s}))
$.
We conclude with the computation of
the convergence rate. 
After applying Heine and L'Hospital
theorems (in exactly the same way as in the original proof)
one is left with
\begin{align*}
\lim_{n\rightarrow \infty} \frac{y_{n+1}-y_L}{y_n-y_L}
&=
\lim_{x\rightarrow x^*(\gamma',\epsilon,|\oactions(\bar{s})|)}
z_{\gamma',\epsilon,|\oactions(\bar{s})|}'
\\
&=
\lim_{x\rightarrow x^*(\gamma',\epsilon,|\oactions(\bar{s})|)}
\frac{(1-\epsilon)\gamma'}{(x+\gamma')^2}
=
\frac{(1-\epsilon)\gamma'}{(x^*(\gamma',\epsilon,|\oactions(\bar{s})|)+\gamma')^2}
\\
&\leq
\frac{(1-\epsilon)\gamma'}{(\hat{x}^*(\gamma',\epsilon,|\oactions(\bar{s})|)+\gamma')^2}
=
\frac{(1-\epsilon)\gamma'}{(1-\epsilon(1-\frac{|\oactions(\bar{s})|}{|\mathcal{A}|}))^2}
\\
&\leq
\frac{\gamma'}{(1-\epsilon)}
=
\gamma' + \frac{\epsilon\gamma'}{(1-\epsilon)}
\leq
\gamma' + \frac{\epsilon\gamma'}{\gamma'}
=
\gamma'+\epsilon
<
1.
\end{align*}
\end{proof}

\end{document}